\documentclass[twoside,11pt]{article}

\usepackage{jmlr2e}
\usepackage{bm}
\usepackage{bbm}

\usepackage[english]{babel}

\usepackage{xparse}
\usepackage{cancel}
\usepackage{multirow}

\usepackage{appendix}
\usepackage{chngcntr}
\usepackage{etoolbox}
\usepackage{lipsum}
\usepackage{tikz}
\usepackage{float}

\usepackage{amsmath}
\usepackage{arydshln}

\usepackage{amsmath}
\usepackage{blkarray}

\usepackage[final]{pdfpages}

\usepackage[noframe]{showframe}
\usepackage{framed}
\usepackage{lipsum}
\definecolor{color0}  {RGB}{174,225,254} 
\definecolor{color1}  {RGB}{220,227,248} 
\definecolor{color2}  {RGB}{28,130,185} 
\definecolor{color3}  {RGB}{255,253,250} 
\newenvironment{svgraybox}{%
	\MakeFramed{\advance\hsize-\width \FrameRestore\FrameRestore}}%
{\endMakeFramed}
\definecolor{shadecolor}{gray}{0.75}

\usepackage{xcolor}

\usepackage{tcolorbox}
\usepackage{extarrows}

\usepackage{graphicx}  
\usepackage{float}  
\usepackage{subfigure}  
\usepackage[boxed]{algorithm2e}
\usepackage{algpseudocode}
\usepackage{amsmath}
\usepackage{graphics}
\usepackage{epsfig}

\usepackage{blkarray}

\usepackage[framemethod=tikz]{mdframed}
\newcommand{\mdframecolor}{gray!10}
\newcommand{\mdframehideline}{true}

\newcommand{\mdframecolorNote}{gray!10}
\newcommand{\mdframehidelineNote}{true}

\usepackage{listings}

\usepackage[colorlinks]{hyperref}

\usepackage{booktabs}

\usepackage{hyperref}
\hypersetup{
    colorlinks,
    citecolor=black,
    filecolor=black,
    linkcolor=black,
    urlcolor=black
}

\newcommand{\exampbar}{\hfill $\square$\par}

\newcommand{\gap}{\,\,\,\,\,\,\,\,\,}
\newcommand{\natu}{\mathbb{N}}
\newcommand{\real}{\mathbb{R}}
\newcommand{\leadto}{\qquad\underrightarrow{ \text{leads to} }\qquad}

\mathchardef\mhyphen="2D
\newcommand{\indicator}{\mathbbm{1}}

\newcommand{\niw}{\mathrm{NIW}}
\newcommand{\nix}{\mathrm{NIX}}
\newcommand{\nig}{\mathrm{NIG}}
\newcommand{\diag}{\mathrm{diag}}

\newcommand{\bxn}{\mathbf{x}_n}
\newcommand{\bxi}{\mathbf{x}_i}
\newcommand{\xknoi}{\mathcal{X}_{k,-i}}

\newcommand{\bznoi}{\mathbf{z}_{-i}}
\newcommand{\nknoi}{\textit{N}_{k,-i}}  

\newcommand{\bmo}{\bm{m}_0}
\newcommand{\bso}{\bm{S}_0}

\newcommand{\bxstar}{\bx^{\star}}
\newcommand{\xstar}{x^{\star}}
\newcommand{\Nstar}{{N^{\star}}}

\newcommand{\mathcalX}{\mathcal{X}}

\newcommand{\smu}{\mu}
\newcommand{\ssigma}{\sigma}
\newcommand{\discrete}{\mathrm{Discrete}}
\newcommand{\dirichlet}{\mathrm{Dirichlet}}
\newcommand{\multinomial}{\mathrm{Multinomial}}

\newcommand{\tr}{\mathrm{tr}}

\newcommand{\gammadist}{\mathrm{Ga}}
\newcommand{\inversegammadist}{\mathrm{IG}}
\newcommand{\betadist}{\mathrm{Beta}}
\newcommand{\wishartdist}{\mathrm{Wi}}
\newcommand{\inversewishart}{\mathrm{IW}}
\newcommand{\inversechidist}{\mathrm{\chi^{-2}}}
\newcommand{\bernoulli}{\mathrm{Bernoulli}}

\newcommand{\normal}{\mathcal{N}}

\newtheorem{examp2}{Example}[section]

\newcommand{\Exp}{\mathrm{E}}
\newcommand{\E}{\mathrm{E}}
\newcommand{\Var}{\mathrm{Var}}
\newcommand{\Cov}{\mathrm{Cov}}

\newcommand{\bzero}{\boldsymbol{0}}
\newcommand{\bpi}{\bm{\pi}}
\newcommand{\balpha}{\boldsymbol\alpha}
\newcommand{\bbeta}{\boldsymbol\beta}

\newcommand{\bgamma}{\boldsymbol\gamma}
\newcommand{\bepsilon}{\boldsymbol\epsilon}

\newcommand{\bmu}{\boldsymbol\mu}

\newcommand{\bSigma}{\boldsymbol\Sigma}
\newcommand{\bomega}{\boldsymbol\omega}
\newcommand{\bOmega}{\boldsymbol\Omega}

\newcommand{\bLambda}{\boldsymbol\Lambda}
\newcommand{\btheta}{\boldsymbol\theta}
\newcommand{\bTheta}{\boldsymbol\Theta}

\newcommand{\bA}{\bm{A}}
\newcommand{\bb}{\bm{b}}
\newcommand{\bB}{\bm{B}}

\newcommand{\bC}{\bm{C}}

\newcommand{\bD}{\bm{D}}
\newcommand{\be}{\bm{e}}

\newcommand{\bI}{\bm{I}}

\newcommand{\bK}{\bm{K}}

\newcommand{\bL}{\bm{L}}
\newcommand{\bmm}{\bm{m}}
\newcommand{\bM}{\bm{M}}

\newcommand{\bN}{\bm{N}}

\newcommand{\bR}{\bm{R}}

\newcommand{\bS}{\bm{S}}

\newcommand{\bU}{\bm{U}}
\newcommand{\bv}{\bm{v}}
\newcommand{\bV}{\bm{V}}

\newcommand{\bx}{\bm{x}}
\newcommand{\bX}{\bm{X}}
\newcommand{\by}{\bm{y}}
\newcommand{\bY}{\bm{Y}}
\newcommand{\bz}{\bm{z}}
\newcommand{\bZ}{\bm{Z}}

\NewDocumentCommand{\emphbf}{O{G}}{\emph{\textbf{#1}}}

\jmlrheading{1}{2021}{1-48}{4/00}{10/00}{Jun Lu}

\ShortHeadings{A survey on Bayesian inference for Gaussian mixture model}{Jun Lu}
\firstpageno{1}

\begin{document}

\title{A survey on Bayesian inference for Gaussian mixture model}

\author{
\begin{center}
\name Jun Lu \\ 
\email jun.lu.locky@gmail.com
\end{center}
}


\maketitle

\begin{abstract}
Clustering has become a core technology in machine learning, largely due to its application in the field of unsupervised learning, clustering, classification and density estimation. 
A frequentist approach exists to hand clustering based on mixture model which is known as the EM algorithm where the parameters of the mixture model are usually estimated into a maximum likelihood estimation framework.
Bayesian approach for finite and infinite Gaussian mixture model generates point estimates for all variables as well as associated uncertainty in the form of the whole estimates' posterior distribution. 

The sole aim of this survey is to give a self-contained introduction to concepts and mathematical tools in Bayesian inference for finite and infinite Gaussian mixture model in order to seamlessly introduce their applications in subsequent sections. However, we clearly realize our inability to cover all the useful and interesting results concerning this field and given the paucity of scope to present this discussion, e.g., the separated analysis of the generation of Dirichlet samples by stick-breaking and Polya's Urn approaches. We refer the reader to literature in the field of Dirichlet process mixture model for a much detailed introduction to the related fields. Some excellent examples include \citep{frigyik2010introduction, murphy2012machine, bda2014, hoff2009first}.

This survey is primarily a summary of purpose, significance of important background and techniques for Gaussian mixture model, e.g., Dirichlet prior, Chinese restaurant process, and most importantly the origin and complexity of the methods which shed light on their modern applications. The mathematical prerequisite is a first course in probability. Other than this modest background, the development is 
self-contained, with rigorous proofs provided throughout.

\end{abstract}

\begin{keywords}
Dirichlet distribution, Gaussian models, Finite Gaussian mixture model, Infinite Gaussian mixture model, Chinese restaurant process, Exchangeability, Hyperprior, Log-concavity, ARS, Pruning Gibbs sampling, Clustering metrics, . 
\end{keywords}

\newpage
\tableofcontents

\newpage
\part{Introduction}

\section{Introduction}
Model-based approaches relies on discrete mixture models. The simplest approach to model-based clustering relies on a finite mixture model framework, which assumes that the number of clusters in the general population is a fixed finite number that does not grow with the sample size. The model-based approach, assuming the data come from a mixture of distributions, has the advantage of permitting principled statistical inferences compared to other procedures based largely on heuristics, such as K-means. 
It is well known that inference on the number of clusters and cluster allocation can be very sensitive to both the choice of within-cluster parametric distribution and to violations of the finite mixture assumption. For example, if the true data-generating distribution does not correspond exactly to a finite mixture, then usual estimates of the number of clusters will diverge with increasing sample size. These problems are compounded for high-dimensional data, and often as the dimension of the data increases, more and more clusters are introduced. As the number of clusters increases, clusters become less and less interpretable and statistical efficiency decreases. The goal of this survey is to introduce the mathematical background of the model-based approaches and summarize the existing methods that are robust and scalable.
Our general view is that it does not make sense for one to assume that only finitely many clusters are represented in an infinitely large population; indeed, as samples are added we fully expect new types of individuals to be observed that are not yet represented, though the rate of observing these new types is expected to be quite slow if the sample size is already large. In addition, we would very much like to avoid a common artifact in current clustering methods in which the number of clusters tends to increase as the dimensionality of the data increases. We take a nonparametric Bayesian view to allow for uncertainty in the true data generating model.
Bayesian approach for finite and infinite Gaussian mixture model generates point estimates for all variables as well as associated uncertainty in the form of the whole estimates' posterior distribution. 
For decades Dirichlet process mixture (DPM) models have been extensively used for clustering, classification
and density estimation. In analyses of infinite mixture models, a common concern is over-fitting with redundant mixture components having small weight value. This is called the non-identifiability problem. Specifically, many researchers have noticed that the DPM posterior tends to overestimate the number of components empirically \citep{ji2010cdp, west1993hierarchical, miller2013simple}. This overestimation seems to occur because there
are typically a few superfluous ``extra" clusters or noise, and among researchers using DPMs for clustering, this is an annoyance that is sometimes dealt with by pruning such clusters in an ad hoc way - that is, by removing them before calculating statistics such as, the weight of each clusters or the number of clusters \citep{fox2007hierarchical, west1993hierarchical}.
Many generalizations and alternatives to DPM mixtures and the corresponding CRP, e.g., Pitman-Yor process \citep{perman1992size}, weighted CRP \citep{ishwaran2003generalized, lo2005weighted} in general do not solve the problem with too many clusters. Indeed many of the generalizations are designed to introduce new clusters at a power law instead of log rate.
In \citep{mccullagh2008many}, an upper bound of the number of components is fixed in advance to limit the number in modeling. However, these two methods based on simple upper bound or pruning small cluster by some thresholds can not be directly used for real world data, because when you choose a larger upper bound, DPM models can still result in small clusters, and choosing the best thresholds is usually difficult. In this survey, along with the basic background about Bayesian inference for the finite and infinite mixture models, we will also introduce how to shrink small clusters during sampling. 

In analyses of finite mixture models, a common concern is over-fitting in which redundant mixture components having similar locations are introduced. Over-fitting can have an negative impact on mixture models especially for clustering, since this leads to an unnecessarily complex model and thus sacrifice the accuracy of result to a large extent. \citep{rousseau2011asymptotic} studied and proved the asymptotic behavior of the posterior distribution in an over-fitted Bayesian mixture models. In \citep{rousseau2011asymptotic}, they proved that a carefully choice for the parameters in Dirichlet distribution prior will asymptotically empty out the redundant or extra components when the number of observations grows. However, several challenging practical issues arise. For example, for small to moderate sample sizes, the weight assigned to redundant components is often not negligible or the result may not be satisfactory. This can be attributed to non-identifiability problems
in which case distinguishing between components with similar locations can be difficult. This issue results in substantial uncertainty in clustering and estimation of the number of components in practice.

Clustering is also one of the most widely used applications in the analysis of gene data \citep{lian2010sparse}, for example, for cancer subtype discovery. We may have two different problems in this discovery, 1) Obviously not all the gene features possess discriminative value for different cancer subtypes; 2) also if fewer gene features are used, the procedure might fail to distinguish between some of the subtypes. Many researchers proposed to first reduce dimension by performing the principal component analysis (PCA) on the features and then fitting a Bayesian mixture model to the reduced features, e.g., \citep{bernardo2003bayesian}. However, difficulty and un-necessarity in interpreting the raw attributes arise, and the top principal components usually do not necessarily carry the most significant discriminative features for clustering, thus the procedure is rather suboptimal. Interesting readers can find more details about this topic in the references above. Again, the sole aim of this survey is to introduce the mathematical background for Gaussian mixture model via Bayesian inference.


\subsection{Notations}
In all cases, scalars will be denoted in a non-bold font possibly with subscripts (e.g., $\alpha$, $\alpha_i$). We will use bold face lower case letters possibly with subscripts to denote vectors (e.g., $\bmu$, $\bx$, $\bx_n$, $\bz$) and
bold face upper case letters possibly with subscripts to denote matrices (e.g., $\bSigma$, $\bL$). The $i^{th}$ element of a vector $\bz$ will be denoted by $z_i$ in non-bold font. And in all cases, vectors are formulated in a column rather than in a row.

The transpose of a matrix $\bX$ will be denoted by $\bX^T$ and its inverse will be denoted by $\bX^{-1}$. We will denote the $p \times p$ identity matrix by $\bI_p$. A vector or matrix of all zeros will be denoted by a bold face zero $\bm{0}$ whose size should be clear from context, or we denote $\bm{0}_p$ to be the vector of all zeros with $p$ entries.

In specific, we will use the notation denoted in Table~\ref{table:general_notation}, Table~\ref{table:niw_prior} and Table~\ref{table:mixture_model} for the text, or otherwise indicated especially in each section. 

\begin{table}[htbp]\caption{Table of general notation}
	\begin{center}
		\begin{tabular}{r c p{10cm} }
			\toprule
			
			$f(\bx) \propto g(\bx)$ & $\triangleq$ & $f$ is proportional to $g$, means there is a constant $c$ such that $f(\bx) = cg(\bx)$ for all $\bx$ \\
			$p(\bx|\btheta)$ & $\triangleq$ & generator / likelihood\\
			$p(\btheta)$ & $\triangleq$ & prior likelihood\\
			$p(\btheta| \bx)$ & $\triangleq$ & posterior likelihood\\
			$p(\bx)$ & $\triangleq$ & marginal likelihood\\
			$p(\bx_{N+1} | \bx_{1:N})$ & $\triangleq$ & posterior predictive distribution\\
			$\triangle_K$ & $\triangleq$ & ($K-1$)-dimensional probability simplex living in $\mathbb{R}^K$\\
			
			\multicolumn{3}{c}{}\\
			\multicolumn{3}{c}{\underline{Decision Variables}}\\
			\multicolumn{3}{c}{}\\
			$\delta_{\bx_0}(\bx)$ & $=$ & \(\left\{\begin{array}{rl}
				1,  & \text{if $\bx$ = $\bx_0$} \\
				0,  & \text{otherwise} \end{array} \right.\)\\
			\bottomrule
		\end{tabular}
	\end{center}
	\label{table:general_notation}
\end{table}

\begin{table}[htbp]\caption{Table of notation for normal-inverse-Wishart prior}
	\begin{center}
		\begin{tabular}{r c p{10cm} }
			\toprule
			$\bbeta = (\bmo, \kappa_0, \nu_0, \bS_0)$& $\triangleq$ &  Parameters for the normal-inverse-Wishart prior on mean vector $\bmu$ and covariance matrix $\bSigma$ of a multivariate Gaussian distribution. The interpretation for the individual parameters are given below.\\
			$\bmo$& $\triangleq$ & Prior mean for $\bmu$.\\
			$\kappa_0$& $\triangleq$ & How strongly we believe the above prior.\\
			$\bS_0$ & $\triangleq$ & Proportional to prior mean for $\bSigma$.\\
			$\nu_0$ & $\triangleq$ & How strongly we believe the above prior.\\
			\bottomrule
		\end{tabular}
	\end{center}
	\label{table:niw_prior}
\end{table}

\begin{table}[htbp]\caption{Table of notation for mixture model}
	\begin{center}
		\begin{tabular}{r c p{10cm} }
			\toprule
			$N$ & $\triangleq$ & Number of data vectors.\\
			$D$ & $\triangleq$ & Dimension of data vectors.\\
			$\bx_i \in \mathbb{R}^D$ & $\triangleq$ & The $i^{th}$ data vector.\\
			$\mathcal{X} = \bx_{1:N}= \{\bx_1, \bx_2, \ldots , \bx_N\}$  & $\triangleq$ & Set of data vectors.\\
			$\mathcal{X}_{-i}$    & $\triangleq$ & All data vectors apart from $\bx_i$.\\
			$\mathcal{X}_k$    & $\triangleq$ & Set of data vectors from mixture component $k$.\\
			$\mathcal{X}_{k,-i} $     & $\triangleq$ & Set of data vectors from mixture component $k$, without taking $\bx_i$ into account.\\
			$N_k$  & $\triangleq$ & Number of data vectors from mixture component $k$.\\
			$N_{k,-i} $   & $\triangleq$ & Number of data vectors from mixture component $k$, without taking $\bx_i$ into account.\\
			$K$ & $\triangleq$ & Number of components in a finite mixture model.\\
			$z_i \in {1, 2, \ldots, K}$& $\triangleq$ & Discrete latent state indicating which component the observation $\bx_i$ belongs to.\\
			$\bz = (z_1, z_2, \ldots , z_N )$& $\triangleq$ & Latent states for all observations $\bx_1, \bx_2, \ldots , \bx_N$.\\
			$\bz_{-i}$& $\triangleq$ & All latent states excluding $z_i$.\\
			$\bmu$& $\triangleq$ & Mean vector of a multivariate Gaussian density. A subscript is used to for a particular component in a mixture model, e.g., $\bmu_k$.\\
			$\bSigma$& $\triangleq$ & Covariance matrix of a multivariate Gaussian density. A subscript is used for a particular component in a mixture model, e.g. $\bSigma_k$.\\
			$\pi_k = p(z_i = k)$ & $\triangleq$ &  Prior probability that data vector $\bx_i$ will be assigned to mixture component $k$.\\     
			$\bpi = (\pi_1, \pi_2, \ldots , \pi_K)$& $\triangleq$ & Prior assignment probability for all $K$ components.\\
			$\balpha = (\alpha_1, \alpha_2, \ldots, \alpha_K)$& $\triangleq$ & Parameter for Dirichlet prior on the mixing weights $\bpi$.\\
			\bottomrule
		\end{tabular}
	\end{center}
	\label{table:mixture_model}
\end{table}


\newpage
\part{Monte Carlo methods for probabilistic inference}\label{chapter:mc_methods}

This survey focuses on Markov chain Monte Carlo methods for probabilistic inference, which draws conclusions from a probabilistic model.

This chapter surveys the mathematical details of probabilistic inference, focusing on those aspects that will provide the foundation for the rest of this survey. 

\section{The Bayesian approach}
In modern statistics, Bayesian approaches
have become increasingly more important and widely used. Thomas Bayes came up this idea but died before publishing it. Fortunately, his friend Richard Price carried on his work and published it in 1764. In this section, we describe the basic ideas about Bayesian approach and use the Beta-Bernoulli model and Bayesian linear model as an appetizer of the pros and prior information of Bayesian models. 

Let $\mathcalX (\bx_{1:N})= \{\bx_1, \bx_2, \ldots, \bx_N\}$ be the observations of $N$ data points, and suppose they are independent and identically distributed (i.i.d.), with the probability parameterized by $\btheta$. Note that the parameters $\btheta$ might include the hidden variables, for example the latent variables in a mixture model to indicate which cluster a data point belongs to. 

The idea of Bayesian approach is to assume a \textit{prior} probability distribution for $\btheta$ with hyperparameters $\balpha$ (i.e., $p(\btheta| \balpha)$) - that is, a distribution representing the plausibility of each possible value of $\btheta$ before the data is observed. Then, to make inferences about $\btheta$, one simply considers the conditional distribution of $\btheta$ given the observed data. This is referred to as the \textit{posterior} distribution, since it represents the plausibility of each possible value of $\btheta$ after seeing the data.
Mathematically, this is expressed via Bayes’ theorem,
\begin{equation}
p(\btheta | \mathcalX, \balpha) = \frac{p(\mathcalX | \btheta ) p(\btheta | \balpha)}{p(\mathcalX | \balpha)} = \frac{p(\mathcalX | \btheta ) p(\btheta | \balpha)}{\int_{\btheta}  p(\mathcalX, \btheta | \balpha) }  = \frac{p(\mathcalX | \btheta ) p(\btheta | \balpha)}{\int_{\btheta}  p(\mathcalX | \btheta ) p(\btheta | \balpha) }  \propto p(\mathcalX | \btheta ) p(\btheta | \balpha),
\label{equation:posterior_abstract_for_mcmc}
\end{equation}
where $\mathcalX$ is the observed data set. In other words, we say the posterior is proportional to the likelihood times the prior. 

More generally, the Bayesian approach - in a nutshell - is to assume a prior distribution on any unknowns ($\btheta$ in our case), and then just follow the rules of probability to answer any questions of interest. For example, when we find the parameter based on the maximum posterior probability of $\bbeta$, we turn to maximum a posteriori (MAP) estimator.

\section{Approximate inference}
For this survey, we focus on approximate probabilistic inference methods. 
In certain cases, it is computationally feasible to compute the posterior exactly. For example, exponential families with conjugate priors often enable analytical solutions.
Although exact inference methods exist and are precise and useful for certain classes of problems, exact inference methods in complicated models is usually intractable, because these methods typically depend on integrals, summations, or intermediate representations that grow large as the state space grows too large so as to make the computation inefficient. For example, we may use conjugate priors in a Gaussian mixture model. However, the model is hierarchical and is too complicated to compute the exact posterior. In these cases, approximate probabilistic inference methods are rather useful and necessary. 

Generally, \textit{variational methods} and \textit{Monte Carlo methods} \citep{bonawitz2008composable} are two main classes of approximate inference. We here give a brief comparison of the two methods. 
In variational inference methods, we first approximate the full model with a simpler model in which the inference questions are tractable. Then, the parameters of this simplified model are calculated by some methods (e.g. by optimization methods) to minimize a measure of the dissimilarity between the original model and the simplified version; this calculation usually performs deterministically because of the optimization methods used. Finally, certain queries can be calculated and executed in the simplified model. In other words, the main idea behind variational methods is to pick a family of distributions over the parameters with its own \textit{variational parameters} - $q(\btheta | \boldsymbol\nu)$ where $\boldsymbol\nu$ is the variational parameters. Then, find the setting of the parameters that makes $q$ close to the posterior of interest. As a detailed example, we can refer to \citep{ma2014bayesian}. The main advantage of variational methods is deterministic; however, the corresponding results are in the form of a lower bound of the desired quantity, and the tightness of this bound depends on the degree to which the simplified distribution can model the original posterior distribution. The variational inference is an important tool for Bayesian deep learning \citep{jordan1999introduction, graves2011practical, hoffman2013stochastic, ranganath2014black, mandt2014smoothed}.

On the contrary, in Monte Carlo methods we first draw a sequence of samples from the true target posterior distribution. Then certain inference questions are then answered by using this set of samples as an approximation of the target distribution itself. Monte Carlo methods are guaranteed to converge – if you want a more accurate answer, you just need to run the inference for longer; in the limit of running the Monte Carlo algorithm forever, the approximation results from the samples converge to the the target distribution (see Section~\ref{sec:monte_carlo_methods}). 

\section{Monte Carlo methods (MC)}\label{sec:monte_carlo_methods}
In Monte Carlo methods, we first draw $N$ samples $\btheta_1, \btheta_2, \ldots, \btheta_N$ from the posterior distribution $p(\btheta | \mathcalX, \balpha)$ in \eqref{equation:posterior_abstract_for_mcmc},
and then approximate the distribution of interest by
\begin{equation}
p(\btheta | -) \approx \overset{\sim}{p}(\btheta | -) = \frac{1}{N} \sum_{n=1}^N \delta_{\btheta_n}(\btheta),
\end{equation}
where $\delta_{\btheta_i}(\btheta)$ is the Dirac delta function\footnote{The Dirac delta function $\delta_{\bx_0}
(\bx)$ has the properties that it is non-zero and equals to 1 only at $\bx = \bx_0$.}. As the number of samples increases, the approximation
(almost surely) converges to the true target distribution, i.e., $\overset{\sim}{p}(\btheta) \overset{\overset{a.s.}{N\rightarrow \infty} }{\longrightarrow} p(\btheta)$.

This kind of sampling-based methods are extensively used in modern statistics, due to their ease of use
and the generality with which they can be applied. The fundamental problem solved by
these methods is the approximation of expectations such as
\begin{equation}
\E {h(\bTheta)} = \int_{\btheta} h(\btheta) p(\btheta) d \btheta,
\end{equation}
in the case of a continuous random variable $\bTheta$ with probability density function (p.d.f.) $p$. Or 
\begin{equation}
\E {h(\bTheta)} = \sum_{\btheta} h(\btheta) p(\btheta),
\end{equation}
in the case of a discrete random variable $\bTheta$ with probability mass function (p.m.f.) $p$. The general principle at work is that such expectations can be approximated by
\begin{equation}
\E {h(\bTheta)} \approx \sum_{n=1}^N h(\btheta_n),
\end{equation}

If it were generally easy to draw samples directly from $p(\btheta | \mathcalX, \balpha)$, the Monte Carlo
story would end here. Unfortunately, this is usually intractable. We can consider the posterior form $p(\btheta | \mathcalX, \balpha) = \frac{p(\mathcalX | \btheta ) p(\btheta | \balpha)}{p(\mathcalX | \balpha)}$, where in many problems $p(\mathcalX | \btheta ) p(\btheta | \balpha)$ can be computed easily, but $p(\mathcalX | \balpha)$ cannot due to integrals, summations etc. In this case Markov chain Monte Carlo is especially useful. 

\subsection{Markov chain Monte Carlo (MCMC)}


Markov chain Monte Carlo (MCMC) algorithms, also called samplers, are numerical approximation algorithms. Intuitively, it is a stochastic hill-climbing approach to inference, operating over the complete data set. This inference method is designed to spend most of the computational efforts to sample points from the high probability regions of true target posterior distribution $p(\btheta | \mathcalX, \balpha)$ \citep{andrieu2003introduction, bonawitz2008composable, hoff2009first, geyer2011introduction}. In this sampler, a Markov chain stochastic walk is taken through the state space $\bTheta$ such that the probability of being in a particular state $\btheta_t$ at any point in the walk is $p(\btheta_t | \mathcalX, \balpha)$. Therefore, samples from the true posterior distribution $p(\btheta |\mathcalX, \balpha)$ can be approximated by recording the samples (states) visited by the stochastic walk and some other post-processing methods such as thinning. The stochastic walk is a Markov chain, i.e. the choice of state at time $t + 1$ depends only on its previous state - the state at time $t$.  Formally, if $\btheta_t$ is the state of the chain at time $t$, then $p(\btheta_{t+1}|\btheta_1, \btheta_2,\ldots, \btheta_t) = p(\btheta_{t+1}| \btheta_t)$. Markov chains are history-free, we can get two main advantages from this history-free property:
\begin{itemize}
\item From history-free, the Markov chain Monte Carlo methods  can be run for an unlimited number of iterations without consuming additional memory space;
\item The history-free property also indicates that the MCMC stochastic walk can be completely characterized by $p(\btheta_{t+1}| \btheta_t)$, known as the \textit{transition kernel}.
\end{itemize}
We then focus on the discussion of the transition kernel. The transition kernel $\bK$ can also be formulated as a linear transform, thus if $p_t = p_t (\btheta)$ is a row vector which encodes the probability of the walk being in state $\btheta$ at time $t$, then $p_{t+1} = p_t \bK$. If the stochastic walk starts from state $\btheta_0$, then the distribution from this initial state is the delta distribution $p_0 = \delta_{\btheta_0} (\btheta)$ and the state distribution for the chain after step $t$ is $p_t = p_0\bK^t$. We can easily find that the key to Markov chain Monte Carlo is to choose kernel $\bK$ such that $\underset{t\rightarrow \infty}{\mathrm{lim}} p_t = p(\btheta | \mathcalX, \balpha)$, independent on the choice of $\btheta_0$. Kernels with this property are said to converge to an \textbf{equilibrium distribution} $p_{eq} = p(\btheta | \mathcalX)$. Convergence is guaranteed if both of the following criteria meet (see \citep{bonawitz2008composable}):
\begin{itemize}
\item $p_{eq}$ is an invariant (or stationary) distribution for $\bK$. A distribution $p_{inv}$ is an invariant distribution for $\bK$ if $p_{inv} = p_{inv} \bK$;
\item $\bK$ is \textit{ergodic}. A kernel is ergodic if it is \textit{irreducible} (any state can be reached from any other state) and \textit{aperiodic} (the stochastic walk never gets stuck in cycles).
\end{itemize}

There are a large number of MCMC algorithms, too many to review here. Popular families include Gibbs sampling, Metropolis-Hastings (MH), slice sampling, Hamiltonian Monte Carlo, Adaptive rejection sampling and many others. Though the name is misleading, Metropolis-within-Gibbs (MWG) was developed first by \cite{metropolis1953equation}, and MH was a generalization of MWG \citep{hastings1970monte}. All MCMC algorithms are known as special cases of the MH algorithm. Regardless of the algorithm, the goal in Bayesian inference is to maximize the unnormalized joint posterior distribution and collect samples of the target distributions, which are marginal posterior distributions, later to be used for inference queries.

The most generalizable MCMC algorithm is the Metropolis-Hastings (MH) generalization \citep{metropolis1953equation, hastings1970monte} of the MWG algorithm. The MH algorithm extended MWG to include asymmetric proposal distributions. 
In this method, it converts an arbitrary proposal kernel $q(\btheta_{\star} | \btheta_t  )$ into a transition kernel with the desired invariant distribution $p_{eq}(\btheta)$. In order to generate a sample from
a MH transition kernel, we first draw a proposal $\btheta_{\star} \sim q( \btheta_{\star} | \btheta_t )$, then evaluates the MH acceptance probability by 
\begin{equation}
P[A(\btheta_{\star}| \btheta_{t} )]  = \min \left(1, \frac{p(\btheta_{\star} | \balpha)q(\btheta_{t} | \btheta_{\star}  )}{p(\btheta_{t} | \balpha)q(\btheta_{\star} | \btheta_{t}  )} \right),
\end{equation}
with probability $P[A(\btheta_{\star} | \btheta_{t}  )] $ being the proposal is accepted and we set $\btheta_{t+1} =  \btheta_{\star}$; otherwise the
proposal is rejected and we set $\btheta_{t+1} =  \btheta_{t}$. That is 
\begin{equation}
\btheta_{t+1} =\left\{
                \begin{array}{ll}
                  \btheta_\star,  \text{ with probability } P[A(\btheta_{\star} | \btheta_{t}  )]; \\
                  \btheta_{t},  \text{ with probability } 1 - P[A( \btheta_{\star} | \btheta_{t}  )] .
                \end{array}
              \right.
\end{equation}
Intuitively, we may find that $\frac{p(\btheta_{\star} | \balpha)}{p(\btheta_{t} | \balpha)}$ term tends to accept moves that lead to higher probability parts of the state space, while also the $\frac{q( \btheta_{t} | \btheta_{\star} )}{q(\btheta_{\star} | \btheta_{t}  )}$ term tends to accept moves that are easy to undo. A random walk demo of MH is available online by Chi Feng \footnote{\url{http://www.junlulocky.com/mcmc-demo/}}. Because in MH, we only evaluate $p(\btheta)$ as part of the ratio $\frac{p(\btheta_{\star} | \balpha)}{p(\btheta_{t} | \balpha)}$, we do not need compute $p(\mathcalX | \balpha)$ as mentioned in Section~\ref{sec:monte_carlo_methods}. 

The key in MH is the proposal kernel $q(\btheta_{\star} | \btheta_t  )$. However, the transition kernel is not $q(\btheta_{\star} | \btheta_t  )$. Informally, the kernel $K(\btheta_{t+1} | \btheta_t  )$ in MH is 
$$
p (\btheta_{t+1} | \mathrm{accept}) P[\mathrm{accept}] + p (\btheta_{t+1}|\mathrm{reject}) P[\mathrm{reject}].
$$ 
\citet{tierney1998note} introduced that the precise transition kernel is
\begin{equation}
\begin{aligned}
K(\btheta_t \rightarrow \btheta_{t+1}) &= p(\btheta_{t+1} | \btheta_t) \\
						&= q(\btheta_{t+1} | \btheta_t  ) A(\btheta_{t+1} | \btheta_t  ) + \delta_{\btheta_t}(\btheta_{t+1}) \int_{\btheta_\star} q(\btheta_{\star} | \btheta_t  ) (1 - A(\btheta_{\star} | \btheta_{t}  )) .
\end{aligned}
\end{equation}

%
%

\subsection{MC V.S. MCMC}
As shown in previous sections, the purpose of Monte Carlo or Markov chain Monte Carlo approximation is to obtain a sequence of parameters values $\{\btheta^{(1)}, \ldots, \btheta^{(N)}\}$ such that 
\begin{equation}
\frac{1}{N} \sum_{n=1}^N h(\btheta^{(n)}) \approx \int_{\btheta} h(\btheta) p(\btheta) d\btheta, 
\end{equation}
for any functions $h$ of interest in case of continuous random variable. In other words, we want the empirical average of $\{h(\btheta^{(1)}), \ldots,$ $h(\btheta^{(N)})\}$ to approximate the expected value of $h(\btheta)$ under a target probability distribution $p(\btheta)$. In order for this to be a good approximation for a wide range of functions $h$, we need the empirical distribution of the simulated sequence $\{\btheta^{(1)}, \ldots, \btheta^{(N)}\}$ to look like the target distribution $p(\btheta)$. MC and MCMC are two ways of generating such a sequence. MC simulation, in which we generate independent samples from the target distribution, is in some sense the "true situation". Independent MC samples automatically create a sequence that is representative of $p(\btheta)$, which means the probability that $\btheta^{(n)} \in A$ for any set $A$ is
\begin{equation}
\int_{A} p(\btheta) d\btheta.
\end{equation}
where $n \in \{1, \ldots, N\}$. However, this is not true for MCMC samples, in which case all we are sure of is that
\begin{equation}
\lim_{n \rightarrow \infty} Pr(\theta^{(n)} \in A) = \int_A p(\btheta) d\btheta.
\end{equation}

\subsection{Gibbs sampler}\label{section:gibbs-sampler}

Gibbs sampling was introduced by Turchin \citep{turchin1971computation}, and later by brothers Geman and Geman \citep{geman1984stochastic} in the context of image restoration. The Geman brothers named the algorithm after the physicist J. W. Gibbs, some eight decades after his death, in reference to an analogy between the sampling algorithm and statistical physics.


Gibbs sampling is applicable when the joint distribution is not known explicitly or is difficult to sample from directly, but the conditional distribution of each variable is known and easy to sample from. A Gibbs sampler generates a draw from the distribution of each parameter or variable in turn, conditional on the current values of the other parameters or variables. Therefore, a Gibbs sampler is a componentwise algorithm. In our example, given some data $\mathcalX$ and a probability distribution $p(\btheta | \mathcalX, \balpha)$ parameterized by $\btheta = \{\theta_1, \theta_2, \ldots, \theta_p\}$. We can successively draw samples from the distribution by sampling from
\begin{equation}
\theta_i^{(t)} \sim p(\theta_i | \btheta_{-i}^{(t-1)}, \mathcalX, \balpha),
\end{equation}
where $\btheta_{-i}^{(t-1)}$ is all current values of $\btheta$ in the $(t-1)^{th}$ iteration except for $\theta_i$. If we sample long enough, these $\theta_i$ values will be random samples from the distribution $p$. 

In deriving a Gibbs sampler, it is often helpful to observe that
\begin{equation}
p(\theta_i \,|\, \btheta_{- i}, \mathcalX)
        = \frac{
            p(\theta_1, \theta_2, \ldots,\theta_p, \mathcalX)
        }{
            p(\btheta_{- i}, \mathcalX)
        } \propto p(\theta_1, \theta_2, \ldots,\theta_p, \mathcalX).
\end{equation}
That is, the conditional distribution is proportional to the joint distribution. We will get a lot of benefits from this simple observation by dropping constant terms from the joint distribution (relative to the parameters we are conditioned on).

Shortly, as a simplified example, given a joint probability distribution $p(\theta_1,\theta_2|\mathcalX)$, a Gibbs sampler would draw $p(\theta_1|\theta_2,\mathcalX)$ , then $p(\theta_2|\theta_1,\mathcalX)$ iteratively.
The procedure defines a sequence of realization of random variables $\theta_1$ and $\theta_2$ 
\begin{equation}
	(\theta_1^0, \theta_2^0), (\theta_1^1, \theta_2^1), (\theta_1^2, \theta_2^2), \cdots \nonumber
\end{equation}
which converges to the joint distribution $p(\beta_1, \beta_2)$. More details about Gibbs sampling can be found in \citep{turchin1971computation, geman1984stochastic, hoff2009first, gelman2013bayesian}.

\subsection{Adaptive rejection sampling (ARS)}
The purpose of this algorithm is to provide an relatively efficient way to sample from a distribution from the large class of log-concave densities \citep{gilks1992adaptive, wild1993algorithm}. We only overview the algorithm here, we can find more details in \citet{gilks1992adaptive} and \citet{wild1993algorithm}. And a Python implementation is available online. \footnote{\url{https://github.com/junlulocky/ARS-MCMC}}



\subsubsection{Rejection Sampling}

\begin{figure}[h!]
\centering
  \includegraphics[width=0.5\textwidth]{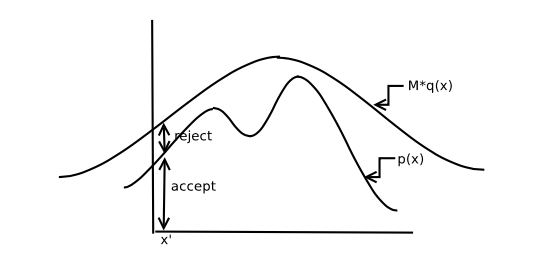}
  \caption{Rejection Sampling. Figure from Michael I. Jordan's lecture notes.}
  \label{fig:rejection_sampling}
\end{figure}

In rejection sampling, we want to sample from a target probability density function $p(x)$, given that we can sample from a probability density function $q(x)$ easily. The target density $p(x)$ is not known. But the idea is that, if $M \times q(x)$ forms an envelope over $p(x)$ for some $M > 1$ as shown in Figure~\ref{fig:rejection_sampling}, i.e.
\begin{equation}
\frac{p(x)}{q(x)} < M, \textit{ for all x.}
\end{equation}
Then if we sample some $x_i$ from $q(x)$, and if $y_i=u \times M\times q(x_i)$ lies below the region under $p(x)$ for some $u \sim \mathrm{Uniform}(0,1)$, then accept $x_i$, otherwise we reject $x_i$. 

Informally, what the method does is to sample $x_i$ from some distribution and then it decides whether to accept it or reject it.

\subsubsection{Adaptive Rejection Sampling}\label{section:ars-sampling}
\begin{figure}[h!]
\centering
  \includegraphics[width=0.5\textwidth]{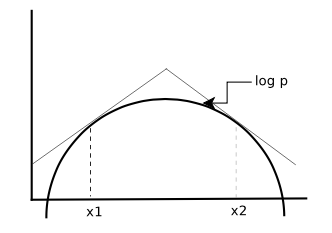}
  \caption{Adaptive Rejection Sampling. Figure from Michael I. Jordan's lecture notes.}
  \label{fig:adaptive_rejection_sampling}
\end{figure}

This method works only for log-concave densities. The basic idea is to form an upper envelope (the upper bound on $p(x)$) adaptively and use this to replace $M\times q(x)$ in rejection sampling.

As shown in Figure~\ref{fig:adaptive_rejection_sampling}, the log density $\log p(x)$ is considered. We then sample $x_i$ from the upper envelope, and either accepted or rejected as in rejection sampling. If it is rejected, a tangent is drawn passing through $x = x_i$ and $y = \log(p)$ and the tangent is used to reduce the upper envelope to reduce the number of rejected samples. The intersections of these tangent planes enable the formation of envelope adaptively. To sample from the upper envelope, we need to transform from log space by exponentiating and using properties of the exponential distribution.

\section{Bayesian appetizers}
In this section, we will take some examples to better understand the ideas behind Bayesian approaches where we will show the semi-conjugate priors with Gibbs sampler and full conjugate priors without approximate inference. Feel free to skip this section if the readers already have basic knowledge in Bayesian inference.
\subsection{An appetizer: Beta-Bernoulli model}\label{sec:beta-bernoulli}
We formally introduce a Beta-Bernoulli model to show how the Bayesian approach works.
The Bernoulli distribution models binary outcomes, i.e., outputting two possible values. The likelihood under this model is just the probability mass function of Bernoulli distribution:
\begin{equation}
	\bernoulli(x|\theta) = p(x|\theta) = \theta^x (1-\theta)^{1-x} \indicator(x\in \{0,1\}). \nonumber
\end{equation}
That is, 
$$
\bernoulli(x|\theta)=p(x|\theta)=\left\{
\begin{aligned}
	&1-\theta ,& \mathrm{\,\,if\,\,} x = 0;  \\
	&\theta , &\mathrm{\,\,if\,\,} x =1,
\end{aligned}
\right.
$$
where $\theta$ is the probability of outputting 1 and $1-\theta$ is the probability of outputting 0.
The mean of the Bernoulli distribution is $\theta$. 
Suppose $x_1, x_2, ..., x_n$ are drawn i.i.d. from $Bernoulli(x|\theta)$. Then, the likelihood under Bernoulli distribution is given by 
$$
\begin{aligned}
	\text{likelihood} = 	p(x_{1:n} |\theta) &= \theta^{\sum x_i} (1-\theta)^{n-\sum x_i},
\end{aligned}
$$
which is a distribution on $x_{1:n}$ and is called the \textbf{likelihood function} on $x_{1:n}$.

And we will see the prior under this model is the probability density function of Beta distribution:
\begin{equation}
	\mathrm{prior} = \betadist(\theta|a, b)=	p(\theta|a, b) =\frac{1}{B(a,b)} \theta^{a-1}(1-\theta)^{b-1} \indicator(0<\theta<1), \nonumber
\end{equation}
where $B(a,b)$ is the Euler's beta function and it can be seen as a normalization term.

We put a Beta prior on the parameter $\theta$ of Bernoulli distribution. The posterior is obtained by 
\begin{equation}
	\begin{aligned}
		\mathrm{posterior} = p(\theta|x_{1:n}) &\propto p(x_{1:n} |\theta) p(\theta|a,b) \\
		&=\theta^{\sum x_i} (1-\theta)^{n-\sum x_i} \times \frac{1}{B(a,b)} \theta^{a-1}(1-\theta)^{b-1}\indicator(0<\theta<1) \\
		&\propto \theta^{a+\sum x_i-1}(1-\theta)^{b+n-\sum x_i-1} \indicator(0<\theta<1) \\
		&\propto \betadist(\theta | a+\sum x_i, b+n-\sum x_i). \nonumber
	\end{aligned}
\end{equation}
We find that the posterior distribution shares the same form as the prior distribution. When this happens, we call the prior as \textbf{conjugate prior}. The conjugate prior has a nice form such that it is easy to work with for computing the posterior probability density function and its derivatives, and sampling from the posterior.

\begin{remark}[Prior Information in Beta-Bernoulli Model]
	A comparison of the prior and posterior formulation would find that the hyperparameter $a$ is the prior number of $1$'s in the output and $b$ is the prior number of 0's in the output. And $a+b$ is the prior information about the sample size.
\end{remark}

\begin{remark}[Bayesian Estimator]
	From this example by Beta-Bernoulli model, like maximum likelihood estimator and method of moment (MoM, i.e., using the moment information to get the model parameter.), Bayesian model is also a kind of point estimator. But Bayesian models output a probability of the parameter of interest $p(\theta |x_{1:n})$. 
	
	When we want to predict for new coming data, we do not give out the prediction by a direct model $p(x_{n+1} | \theta)$. But rather an integration:
	\begin{equation}
		p(x_{n+1} | x_{1:n}) = \int p(x_{n+1} | \theta) p(\theta | x_{1:n}) d\theta.\nonumber
	\end{equation}
	In another word, $x_{n+1}$ is dependent of $x_{1:n}$. $x_{1:n}$ provide information on $\theta$, which in turn provides information on $x_{n+1}$ (i.e., $x_{1:n} \rightarrow \theta \rightarrow x_{n+1}$).
\end{remark}

\begin{examp2}[Amount of Data Matters]\label{example:amountofdata}
	Suppose we have three observations for the success in Bernoulli distribution: 
	
	\item 1. 10 out of 10 are observed to be success (1's);
	\item 2. 48 out of 50 are observed to be success (1's);
	\item 3. 186 out of 200 are observed to be success (1's).
	
	So, what is the probability of success in the Bernoulli model? Normal answer to case 1, 2, 3 are 100\%, 96\% and 93\% respectively. But an observation of 10 inputs is rather a small amount of data and noise can make it less convincing. 
	
	Suppose we put a $Beta(1,1)$ prior on the Bernoulli distribution. The posterior probability of success for each case would be $\frac{11}{12}=91.6\%$, $\frac{49}{52}=94.2\%$ and $\frac{187}{202}=92.6\%$ respectively. Now we find the case 1 has less probability of success compared to case 2.
	
	A Bayesian view of the problem naturally incorporates the amount of data as well as its average. This special case shown here is also called the Laplace's rate of succession \citep{ollivier2015laplace}. Laplace's ``add-one" rule of succession modifies the observed frequencies in a sequence of successes and failures by adding one to the observed counts. This improves prediction by avoiding zero probabilities and corresponds to a uniform Bayesian prior on the parameter. 
	\hfill $\square$\par
\end{examp2}

\begin{mdframed}[hidealllines=\mdframehidelineNote,backgroundcolor=\mdframecolorNote,frametitle={Why Bayes?}]
	This example above shows that Bayesian models consider prior information on the parameters in the model making it particularly useful to regularize regression problems where data information is limited. And this is why the Bayesian approach gains worldwide attention for decades. 
	
	The prior information $p(\theta)$ and likelihood function $p(x|\theta)$ represent a rational person's belief, and then the Bayes' rule is an optimal method of updating this person's beliefs about $\theta$ given new information from the data \citep{fahrmeir2007regression, hoff2009first}.
	
	The prior information given by $p(\theta)$ might be wrong if it does not accurately represent our prior beliefs. However, this does not mean that the posterior $p(\theta | x)$ is not useful. A famous quote is ``all models are wrong, but some are useful" \citep{box1987empirical}. If the prior $p(\theta)$ approximates our beliefs, then the posterior $p(\theta | x)$ is also a good approximation to posterior beliefs.
\end{mdframed}

\subsection{An appetizer: Bayesian linear model with zero-mean prior}\label{sec:bayesian-zero-mean}
Assume $\boldsymbol{y} = \boldsymbol{X}\bbeta + \bepsilon$ where $\bepsilon \sim \normal(\bzero, \sigma^2 \boldsymbol{I})$ and $\sigma^2$ is fixed (a detailed analysis of this model can be found in \citep{rasmussen2003gaussian, hoff2009first, lu2021rigorous}), this additive Gaussian noise assumption gives rise to the likelihood. Let $\mathcalX (\bx_{1:n})= \{\bx_1, \bx_2, \ldots, \bx_n\}$ be the observations of $n$ data points,
\begin{equation}
	\mathrm{likelihood} = \by | \bX, \bbeta, \sigma^2 \sim \normal(\bX\bbeta, \sigma^2\bI). \nonumber
\end{equation}
Suppose we specify a Gaussian prior with zero-mean over the weight parameter 
\begin{equation}
	\mathrm{prior} = 	\bbeta \sim \normal(\bzero, \bSigma_0). \nonumber
\end{equation}
By the Bayes' theorem ``$\mathrm{posterior} \propto \mathrm{likelihood} \times \mathrm{prior} $", we get the posterior
\begin{equation}
	\begin{aligned}
		\mathrm{posterior}&= p(\bbeta|\by,\bX, \sigma^2) \\
		&\propto p(\by|\bX, \bbeta, \sigma^2) p(\bbeta | \bSigma_0) \\
		&=  \frac{1}{(2\pi \sigma^2)^{n/2}} \exp\left(-\frac{1}{2\sigma^2} (\by-\bX\bbeta)^\top(\by-\bX\bbeta)\right) \\
		&\,\,\,\,\,\,\, \times \frac{1}{(2\pi)^{n/2}|\bSigma_0|^{1/2}}\exp\left(-\frac{1}{2} \bbeta^\top\bSigma_0^{-1}\bbeta\right) \\
		&\propto \exp\left(-\frac{1}{2} (\bbeta - \bbeta_1)^\top \bSigma_1^{-1} (\bbeta - \bbeta_1)\right), \nonumber
	\end{aligned}
\end{equation}
where $\bSigma_1 = (\frac{1}{\sigma^2} \bX^\top\bX + \bSigma_0^{-1})^{-1}$ and $\bbeta_1 =  (\frac{1}{\sigma^2}\bX^\top\bX + \bSigma_0^{-1})^{-1}(\frac{1}{\sigma^2}\bX^\top\by)$. Therefore the posterior distribution is also a Gaussian distribution (same form as the prior distribution):
\begin{equation}
	\mathrm{posterior} = \bbeta|\by,\bX, \sigma^2 \sim \normal(\bbeta_1, \bSigma_1). \nonumber
\end{equation}
\textbf{A word on the notation}: note that we use $\{\bbeta_1,\bSigma_1\}$ to denote the posterior mean and posterior covariance in the zero-mean prior model. Similarly, the posterior mean and posterior covariance in semi-conjugate prior and full-conjugate prior models will be denoted as $\{\bbeta_2,\bSigma_2\}$ and $\{\bbeta_3,\bSigma_3\}$ respectively (see sections below).

\begin{mdframed}[hidealllines=\mdframehidelineNote,backgroundcolor=\mdframecolorNote,frametitle={Connection to OLS}]
	In this case, we do not need to assume $\boldsymbol{X}$ has full rank generally.
	Note further that if we assume $\bX$ has full rank, in the limit, when $\bSigma_0 \rightarrow \bzero$, $\bbeta_1 \rightarrow \hat{\bbeta} = (\bX^\top\bX)^{-1}\bX\by$, in which case, maximum a posteriori (MAP) estimator from Bayesian model goes back to ordinary least squares estimator. And the posterior is $\bbeta|\by,\bX, \sigma^2 \sim \normal (\hat{\bbeta}, \sigma^2(\boldsymbol{X}^\top \boldsymbol{X})^{-1})$, which shares similar form as the OLS estimator $\hat{\bbeta} \sim \normal(\bbeta, \sigma^2(\boldsymbol{X}^\top \boldsymbol{X})^{-1})$ under Gaussian disturbance (see \citep{lu2021rigorous}).
\end{mdframed}


\begin{svgraybox}
	\begin{remark}[Ridge Regression]
		In least squares approximation, we use $\bX\bbeta$ to approximate $\by$. Two issues arise: the model can potentially overfit and $\bX$ may not have full rank. In ridge regression, we regularize large value of $\bbeta$ and thus favor simpler models. Instead of minimizing $||\by-\bX\bbeta||^2$, we minimize $||\by-\bX\bbeta||^2+\lambda||\bbeta||^2$, where $\lambda$ is a hyper-parameter that can be tuned:
		\begin{equation}
			\mathop{\arg\min}_{\bbeta}{(\by-\bX\bbeta)^\top(\by-\bX\bbeta) + \lambda \bbeta^\top\bbeta}. \nonumber
		\end{equation}
		By differentiating and setting the derivative to zero we get 
		\begin{equation}
			\hat{\bbeta}_{ridge} = (\bX^\top\bX + \lambda \bI)^{-1} \bX^\top\by, \nonumber
		\end{equation}
		in which case, $(\bX^\top\bX + \lambda \bI)$ is invertible even when $\bX$ does not have full rank. We leave more details about ridge regression to the readers. 
	\end{remark}
\end{svgraybox}

\begin{mdframed}[hidealllines=\mdframehidelineNote,backgroundcolor=\mdframecolorNote,frametitle={Connection to Ridge Regression}]
	Realize that when we set $\bSigma_0 = \bI$, we obtain $\bbeta_1 = (\bX^\top\bX + \sigma^2 \bI)^{-1} \bX^\top\by$ and $\bSigma_1 = (\frac{1}{\sigma^2}\bX^\top\bX+ \bI)^{-1}$. Since $\mathrm{posterior} = \bbeta|\by,\bX, \sigma^2 \sim \normal(\bbeta_1, \bSigma_1)$. The MAP estimator of $\bbeta = \bbeta_1 =  (\bX^\top\bX + \sigma^2 \bI)^{-1} \bX^\top\by$, which shares the same form as ridge regression by letting $\sigma^2 = \lambda$. Thus we notice ridge regression is a special case of Bayesian linear model with zero-mean prior. And ridge regression has a nice interpretation from the Bayesian approach - finding the mode of the posterior.
	
	An example is shown in \citep{rasmussen2003gaussian} where the ``well determined" (i.e., the distribution around the slope is more compact) slope of $\bbeta$ is almost unchanged after the posterior process while the intercept which is more dispersed shrunk towards zero. This is actually a regularization effect on the parameter like ridge regression.
\end{mdframed}

\subsection{An appetizer: Bayesian linear model with semi-conjugate prior Distribution}\label{sec:semiconjugate}

We will use gamma distribution as the prior of the inverse variance (precision) of Gaussian distribution. Before the discussion about gamma distribution, we first introduce a special gamma distribution, which is often used and known as chi-square distribution.
\begin{svgraybox}
	\begin{definition}[Chi-Square Distribution]
		Let $\bA \sim \normal(0, \bI_{p\times p})$. Then $X=\sum_i^p \bA_{ii}$ has the Chi-square distribution with $p$ degrees of freedom. We write $X \sim \chi_{(p)}^2$, and we will see this is equivalent to $X\sim \gammadist(p/2, 1/2)$.
		
		$$ f(x; p)=\left\{
		\begin{aligned}
			&\frac{1}{2^{p/2}\Gamma(\frac{p}{2})} x^{\frac{p}{2}-1} \exp(-\frac{x}{2}) ,& \mathrm{\,\,if\,\,} x \geq 0.  \\
			&0 , &\mathrm{\,\,if\,\,} x <0.
		\end{aligned}
		\right.
		$$
		The mean, variance of $X\sim \chi_{(p)}^2$ are given by $\Exp[X]=p$, $\Var[X]=2p$.
		
		The function $\Gamma(\alpha) = \int_{0}^{\infty}  t^{\alpha-1} e^{-t} dt $ is the gamma function and we can just take it as a function to normalize the distribution into sum to 1. In special case when $y$ is a positive integer, $\Gamma(y) = (y-1)!$.
	\end{definition}
\end{svgraybox}

\begin{svgraybox}
\begin{definition}[Gamma Distribution]\label{definition:gamma-distribution}
A random variable $X$ is said to follow the gamma distribution with parameter $r>0$ and $\lambda>0$, denoted by $X \sim \gammadist(r, \lambda)$ if 

$$ f(x; r, \lambda)=\left\{
\begin{aligned}
&\frac{\lambda^r}{\Gamma(r)} x^{r-1} \exp(-\lambda x) ,& \mathrm{\,\,if\,\,} x \geq 0.  \\
&0 , &\mathrm{\,\,if\,\,} x <0.
\end{aligned}
\right.
$$
So if $X \sim \chi_{(p)}^2$, then $X\sim \gammadist(p/2, 1/2)$, i.e., Chi-square distribution is a special case of Gamma distribution.
The mean and variance of $X \sim \gammadist(r, \lambda)$ are given by 
\begin{equation}
\Exp[X] = \frac{r}{\lambda}, \qquad \Var[X] = \frac{r}{\lambda^2}. \nonumber
\end{equation}
Specially, let $X_1, X_2, \ldots, X_n$ be i.i.d., random variables drawn from $\gammadist(r_i, \lambda)$ for each $i \in \{1, 2, \ldots, n\}$. Then $Y = \sum_{i=1}^{n} X_i$ is a random variable following from $\gammadist(\sum_{i=1}^{n}r_i, \lambda)$.
\end{definition}
\end{svgraybox}
As for the reason of using the gamma distribution as the prior for precision, we quote the description from \citep{kruschke2014doing}:
\begin{mdframed}[hidealllines=\mdframehidelineNote,backgroundcolor=\mdframecolorNote]
	Because of its role in conjugate priors for normal likelihood function, the gamma distribution is routinely used as a prior for precision (i.e., inverse variance). But there is no logical necessity to do so, and modern MCMC methods permit more flexible specification of priors. Indeed, because precision is less intuitive than standard deviation, it can be more useful to give standard deviation a uniform prior that spans a wide range.
\end{mdframed}

Same setting as Section~\ref{sec:bayesian-zero-mean}, but we assume now $\sigma^2$ is not fixed. Again, we have likelihood function by  
\begin{equation}
	\mathrm{likelihood} = \by | \bX, \bbeta, \sigma^2 \sim \normal(\bX\bbeta, \sigma^2\bI). \nonumber
\end{equation}
We specify a non zero-mean Gaussian prior over the weight parameter 
\begin{equation}
	\begin{aligned}
		{\color{blue}\mathrm{prior:\,}} &\bbeta \sim \normal({\color{blue}\bbeta_0}, \bSigma_0) \\
		&{\color{blue}\gamma = 1/\sigma^2 \sim \gammadist(a_0, b_0)}, \nonumber
	\end{aligned}
\end{equation}
where we differentiate from previous descriptions by blue text.

(1). Then, given $\sigma^2$, by the Bayes' theorem ``$\mathrm{posterior} \propto \mathrm{likelihood} \times \mathrm{prior} $", we get the posterior
\begin{equation}
	\begin{aligned}
		\mathrm{posterior}&= p(\bbeta|\by,\bX, \sigma^2) \\
		&\propto p(\by|\bX, \bbeta, \sigma^2) p(\bbeta | \bbeta_0, \bSigma_0) \\
		&=  \frac{1}{(2\pi \sigma^2)^{n/2}} \exp\left(-\frac{1}{2\sigma^2} (\by-\bX\bbeta)^\top(\by-\bX\bbeta)\right) \\
		&\,\,\,\,\,\, \times \frac{1}{(2\pi)^{n/2}|\bSigma_0|^{1/2}}\exp\left(-\frac{1}{2} (\bbeta-\bbeta_0)^\top\bSigma_0^{-1}(\bbeta-\bbeta_0)\right) \\
		&\propto \exp\left(-\frac{1}{2} (\bbeta - \bbeta_2)^\top \bSigma_2^{-1} (\bbeta - \bbeta_2)\right), \nonumber
	\end{aligned}
\end{equation}
where $\bSigma_2 = (\frac{1}{\sigma^2} \bX^\top\bX + \bSigma_0^{-1})^{-1}$ and  
$$
\bbeta_2 = \bSigma_2 (\bSigma_0^{-1}\bbeta_0+\frac{1}{\sigma^2}\bX^\top\by)= (\frac{1}{\sigma^2}\bX^\top\bX + \bSigma_0^{-1})^{-1}(\textcolor{blue}{\bSigma_0^{-1}\bbeta_0}+\frac{1}{\sigma^2}\bX^\top\by).
$$
Therefore, the posterior is from a Gaussian distribution:
\begin{equation}
	\mathrm{posterior} = \bbeta|\by,\bX, \sigma^2 \sim \normal(\bbeta_2, \bSigma_2). \nonumber
\end{equation}

\begin{mdframed}[hidealllines=\mdframehidelineNote,backgroundcolor=\mdframecolorNote,frametitle={Connection to Zero-Mean Prior}]
	\item 1. $\bSigma_0$ here is a fixed hyperparameter.
	\item 2. We note that $\bbeta_1$ in Section~\ref{sec:bayesian-zero-mean} is a special case of $\bbeta_2$ when $\bbeta_0=\bzero$. 
	\item 3. And if we assume further $\bX$ has full rank. When $\bSigma_0^{-1} \rightarrow \bzero$, $\bbeta_2 \rightarrow \hat{\bbeta} = (\bX^\top\bX)^{-1}\bX\by$ which is the OLS estimator. 
	\item 4. When $\sigma^2 \rightarrow \infty$, $\bbeta_2$ is approximately approaching to $\bbeta_0$, the prior expectation of parameter. However, in zero-mean prior, $\sigma^2 \rightarrow \infty$ will make $\bbeta_1$ approach to $\bzero$.
	\item 5. \textbf{Weighted average}: we reformulate
	\begin{equation}
		\begin{aligned}
			\bbeta_2 &=  (\frac{1}{\sigma^2}\bX^\top\bX + \bSigma_0^{-1})^{-1}(\bSigma_0^{-1}\bbeta_0+\frac{1}{\sigma^2}\bX^\top\by) \\
			&= (\frac{1}{\sigma^2}\bX^\top\bX + \bSigma_0^{-1})^{-1} \bSigma_0^{-1}\bbeta_0 + (\frac{1}{\sigma^2}\bX^\top\bX + \bSigma_0^{-1})^{-1} \frac{\bX^\top\bX}{\sigma^2} (\bX^\top\bX)^{-1}\bX^\top\by \\
			&=(\bI-\bA)\bbeta_0 + \bA \hat{\bbeta}, \nonumber
		\end{aligned}
	\end{equation}
	where $\hat{\bbeta}=(\bX^\top\bX)^{-1}\bX^\top\by$ is the OLS estimator of $\bbeta$ and $\bA=(\frac{1}{\sigma^2}\bX^\top\bX + \bSigma_0^{-1})^{-1} \frac{\bX^\top\bX}{\sigma^2}$. We see that the posterior mean of $\bbeta$ is a weighted average of prior mean and OLS estimator of $\bbeta$. Thus, if we set the prior parameter $\bbeta_0 = \hat{\bbeta}$, the posterior mean of $\bbeta$ will be exactly $\hat{\bbeta}$.
\end{mdframed}

(2). Given $\bbeta$, again, by Bayes' theorem, we obtain the posterior

\begin{equation}
	\begin{aligned}
		\mathrm{posterior}&= p(\gamma=\frac{1}{\sigma^2}|\by,\bX, \bbeta) \\
		&\propto p(\by|\bX, \bbeta, \gamma) p(\gamma | a_0, b_0) \\
		&=  \frac{\gamma^{n/2}}{(2\pi )^{n/2}} \exp\left(-\frac{\gamma}{2} (\by-\bX\bbeta)^\top(\by-\bX\bbeta)\right) \\
		&\,\,\,\,\,\,\, \times \frac{{b_0}^{a_0}}{\Gamma(a_0)} \gamma^{a_0-1} \exp(-b_0 \gamma) \\
		&\propto \gamma(a_0+\frac{n}{2}-1) \exp\left(-\gamma\left[b_0+\frac{1}{2}(\by-\bX\bbeta)^\top(\by-\bX\bbeta)\right]\right), \nonumber
	\end{aligned}
\end{equation}
and the posterior is a Gamma distribution:
\begin{equation}
	\mathrm{posterior\,\, of\,\,} \gamma \mathrm{\,\,given\,\,} \bbeta  = \gamma|\by,\bX, \bbeta \sim \gammadist\left(a_0+\frac{n}{2}, [b_0+\frac{1}{2}(\by-\bX\bbeta)^\top(\by-\bX\bbeta)]\right). \nonumber
\end{equation}

\begin{mdframed}[hidealllines=\mdframehidelineNote,backgroundcolor=\mdframecolorNote,frametitle={Prior Information on the Noise}]
	\item 1. We notice that the prior mean and posterior mean of $\gamma$ are $\Exp[\gamma]=\frac{a_0}{b_0}$ and $\Exp[\gamma |\bbeta]=\frac{a_0 + \frac{n}{2}}{b_0 +\frac{1}{2}(\by-\bX\bbeta)^\top(\by-\bX\bbeta)}$ respectively. So the inside meaning of $2 a_0$ is the prior sample size for the noise $\sigma^2 = \frac{1}{\gamma}$. 
	
	\item 2. As we assume $\by=\bX\bbeta +\bepsilon$ where $\bepsilon \sim \normal(\bzero, \sigma^2\bI)$, then $\frac{(\by-\bX\bbeta)^\top(\by-\bX\bbeta)}{\sigma^2} \sim \chi_{(n)}^2$ and $\Exp[\frac{1}{2}(\by-\bX\bbeta)^\top(\by-\bX\bbeta)] = \frac{n}{2}\sigma^2$. So the inside meaning of $\frac{2b_0}{a_0}$ is the prior variance of the noise.
	
	\item 3. Some textbooks would write $\gamma \sim  \gammadist(n_0/2, n_0\sigma_0^2/2)$ to make this explicit (in which case, $n_0$ is the prior sample size, and $\sigma_0^2$ is the prior variance). But a prior in this form seems coming from nowhere at first glance.
\end{mdframed}

By this Gibbs sampling method introduced in Section~\ref{section:gibbs-sampler}, we can construct a Gibbs sampler for Bayesian linear model with semi-conjugate prior in Section~\ref{sec:semiconjugate}:

0. Set initial values to $\bbeta$ and $\gamma = \frac{1}{\sigma^2}$;

1. update $\bbeta$: $\mathrm{posterior} = \bbeta|\by,\bX, \gamma \sim \normal(\bbeta_2, \bSigma_2)$;

2. update $\gamma$: $\mathrm{posterior}  = \gamma|\by,\bX, \bbeta \sim \gammadist\left(a_0+\frac{n}{2}, [b_0+\frac{1}{2}(\by-\bX\bbeta)^\top(\by-\bX\bbeta)]\right)$.

\subsection{An appetizer: Bayesian linear model with full conjugate prior}\label{section:blm-fullconjugate}
Putting a gamma prior on the inverse variance is equivalent to putting a inverse-gamma prior on the variance: 
\begin{svgraybox}
	\begin{definition}[Inverse-Gamma Distribution]\label{definition:inverse-gamma}
		A random variable $Y$ is said to follow the inverse-gamma distribution with parameter $r>0$ and $\lambda>0$ if
		
		$$ f(y; r, \lambda)=\left\{
		\begin{aligned}
			&\frac{\lambda^r}{\Gamma(r)} y^{-r-1} \exp(- \frac{\lambda}{y} ) ,& \mathrm{\,\,if\,\,} y > 0.  \\
			&0 , &\mathrm{\,\,if\,\,} y \leq 0.
		\end{aligned}
		\right.
		$$
		And it is denoted by $Y \sim \inversegammadist(r, \lambda)$.
		The mean and variance of inverse-gamma distribution are given by 
		$$ \Exp[Y]=\left\{
		\begin{aligned}
			&\frac{\lambda}{r-1}, \, &\mathrm{if\,} r\geq 1. \\
			&\infty, \, &\mathrm{if\,} 0<r<1.
		\end{aligned}
		\right.\qquad
		\Var[Y]=\left\{
		\begin{aligned}
			&\frac{\lambda^2}{(r-1)^2(r-2)}, \, &\mathrm{if\,} r\geq 2. \\
			&\infty, \, &\mathrm{if\,} 0<r<2.
		\end{aligned}
		\right.
		$$
	\end{definition}
\end{svgraybox}
Note that the inverse-gamma density is not simply the gamma density with
$x$ replaced by $\frac{1}{y}$. There is an additional factor of $y^{-2}$. \footnote{Which is from the Jacobian in the change-of-variables formula. A short proof is provided here. Let $y=\frac{1}{x}$ where $y\sim \inversegammadist(r, \lambda)$ and $x\sim \gammadist(r, \lambda)$. Then, $f(y) |dy| = f(x) |dx|$ which results in $f(y) = f(x) |\frac{dx}{dy}| = f(x)x^2 \xlongequal{ \mathrm{y}=\frac{1}{x}} \frac{\lambda^r}{\Gamma(r)} y^{-r-1} exp(- \frac{\lambda}{y})$ for $y>0$. }

Same setting as semiconjugate prior distribution in Section~\ref{sec:semiconjugate}. We have the likelihood function:
\begin{equation}
	\mathrm{likelihood} = \by | \bX, \bbeta, \sigma^2 \sim \normal(\bX\bbeta, \sigma^2\bI). \nonumber
\end{equation}
But now we specify a Gaussian prior over the weight parameter by
\begin{equation}
	\begin{aligned}
		{\color{blue}\mathrm{prior:\,}} &\bbeta|\sigma^2 \sim \normal(\bbeta_0, {\color{blue}\sigma^2} \bSigma_0) \\
		&{\color{blue}\sigma^2 \sim \inversegammadist(a_0, b_0)}, \nonumber
	\end{aligned}
\end{equation}
where again we differentiate from previous descriptions by blue text. 
Equivalently, we can formulate the prior into one which is called the normal-inverse-gamma (NIG) distribution:
\begin{equation}
	\begin{aligned}
		\mathrm{prior:\,} &\bbeta,\sigma^2 \sim \nig(\bbeta_0, \bSigma_0, a_0, b_0) =  \normal(\bbeta_0, \sigma^2 \bSigma_0)\cdot \inversegammadist(a_0, b_0) . \nonumber
	\end{aligned}
\end{equation}
Again by the Bayes' theorem ``$\mathrm{posterior} \propto \mathrm{likelihood} \times \mathrm{prior} $", we obtain the posterior

\begin{equation}
	\begin{aligned}
		\mathrm{posterior}&= p(\bbeta,\sigma^2|\by,\bX) \\
		&\propto p(\by|\bX, \bbeta, \sigma^2) p(\bbeta, \sigma^2 | \bbeta_0, \bSigma_0, a_0, b_0) \\
		&=  \frac{1}{(2\pi \sigma^2)^{n/2}} \exp\left\{-\frac{1}{2\sigma^2} (\by-\bX\bbeta)^\top(\by-\bX\bbeta)\right\} \\
		&\,\,\,\,\,\, \times \frac{1}{(2\pi \sigma^2)^{p/2} |\bSigma_0|^{1/2}} \exp\left\{-\frac{1}{2\sigma^2} (\bbeta - \bbeta_0)^\top\bSigma_0^{-1} (\bbeta - \bbeta_0)\right\} \\
		&\,\,\,\,\,\, \times  \frac{{b_0}^{a_0}}{\Gamma(a_0)} \frac{1}{(\sigma^2)^{a_0+1}} \exp\{-\frac{b_0}{\sigma^2} \} \\
		&\propto \frac{1}{(2\pi \sigma^2)^{p/2} }  \exp\left\{  \frac{1}{2\sigma^2} (\bbeta -\bbeta_3)^\top\bSigma_3^{-1}(\bbeta -\bbeta_3) \right\} \\
		&\,\,\,\,\,\, \times \frac{1}{(\sigma^2)^{a_0 +\frac{n}{2}+1}} \exp\left\{-\frac{1}{\sigma^2} [b_0+\frac{1}{2} (\by^\top\by +\bbeta_0^\top\bSigma_0^{-1}\bbeta_0 -\bbeta_3^\top\bSigma_3^{-1}\bbeta_3) ]\right\}, \nonumber
	\end{aligned}
\end{equation}
where $\bSigma_3 = ( \bX^\top\bX + \bSigma_0^{-1})^{-1}$ and 
$$
\bbeta_3 = \bSigma_3(\bX^\top\by + \bSigma_0^{-1}\bbeta_0) = ( \bX^\top\bX + \bSigma_0^{-1})^{-1}(\bSigma_0^{-1}\bbeta_0 + \bX^\top\by).
$$ 
Let $a_n = a_0 +\frac{n}{2}+1$ and $b_n=b_0+\frac{1}{2} (\by^\top\by +\bbeta_0^\top\bSigma_0^{-1}\bbeta_0 -\bbeta_3^\top\bSigma_3^{-1}\bbeta_3) $. The posterior is thus a NIG distribution:
\begin{equation}
	\begin{aligned}
		\mathrm{posterior}&=  \bbeta, \sigma^2 | \by, \bX \sim \nig(\bbeta_3, \bSigma_3, a_n, b_n). \nonumber
	\end{aligned}
\end{equation}
\begin{mdframed}[hidealllines=\mdframehidelineNote,backgroundcolor=\mdframecolorNote,frametitle={Connection to Zero-Mean Prior and Semiconjugate Prior}]
	\item 1. $\bSigma_0$ here is a fixed hyperparameter.
	
	
	\item 2. If we assume further $\bX$ has full rank, when $\bSigma_0^{-1} \rightarrow \bzero$, $\bbeta_3 \rightarrow \hat{\bbeta} = (\bX^\top\bX)^{-1}\bX\by$ which is the OLS estimator. 
	
	\item 3. When $b_0 \rightarrow \infty$, then $\sigma^2 \rightarrow \infty$ and $\bbeta_3$ is approximately $\bbeta_0$, the prior expectation of parameter. Compared to $\bbeta_2$ in Section~\ref{sec:semiconjugate}, $\sigma^2 \rightarrow \infty$ will make $\bbeta_2$ approach to $\bbeta_0$ where $\sigma^2$ is a fixed hyperparameter.

	\item 4. \textbf{Weighted average}: we reformulate
	\begin{equation}
		\begin{aligned}
			\bbeta_3 &=  (\bX^\top\bX + \bSigma_0^{-1})^{-1}(\bSigma_0^{-1}\bbeta_0+\bX^\top\by) \\
			&= (\bX^\top\bX + \bSigma_0^{-1})^{-1} \bSigma_0^{-1}\bbeta_0 + (\bX^\top\bX + \bSigma_0^{-1})^{-1} (\bX^\top\bX) (\bX^\top\bX)^{-1}\bX^\top\by \\
			&=(\bI-\bC)\bbeta_0 + \bC \hat{\bbeta}, \nonumber
		\end{aligned}
	\end{equation}
	where $\hat{\bbeta}=(\bX^\top\bX)^{-1}\bX^\top\by$ is the OLS estimator of $\bbeta$ and $\bC=(\bX^\top\bX + \bSigma_0^{-1})^{-1} (\bX^\top\bX)$. We see that the posterior mean of $\bbeta$ is a weighted average of the prior mean and the OLS estimator of $\bbeta$. Thus, if we set $\bbeta_0 = \hat{\bbeta}$, the posterior mean of $\bbeta$ will be exactly $\hat{\bbeta}$.
	
	\item 5. From $a_n = a_0 +\frac{n}{2}+1$, we know that $2a_0$ is the prior sample size for $\sigma^2$.
	
	\item 6. $\bSigma_3^{-1} = \bX^\top\bX + \bSigma_0^{-1}$: The posterior inverse covariance is equal to $\bX^\top\bX$ + prior inverse covariance.
\end{mdframed}

A Bayesian and non-Bayesian variable selection procedure can be referred to \citep{hoff2009first, lu2021rigorous}. In the Bayesian case, the Zeller's g-prior is taken to give rise to a mask vector on the variables such that the selected variables will have mask 1 and 0 otherwise. From the three different priors on the same model, we have a taste for different situations in Bayesian approaches. Especially, the priors may be semi-conjugate or full conjugate which result in different sampling algorithms.


\newpage
\part{Conjugate priors for Gaussian mixture model}\label{chapter:conjugate_gmm}

\section{Conjugate priors}
In Section~\ref{sec:beta-bernoulli}, we discussed about conjugate priors. We now give the formal definition as follows.
\begin{svgraybox}
\begin{definition}[Conjugate Prior]
Given a family $\{p(\mathcalX | \btheta): \btheta \in \bTheta\}$ of generating distributions, a collection of priors $p_\omega (\btheta)$ indexed by $\bomega \in \bOmega$ is called a conjugate prior family if for any $\bomega$ and any data, the resulting posterior equals to $p_{\bomega^\prime} (\btheta | \mathcalX)$ for some $\bomega^\prime \in \bOmega$. 
\end{definition}
\end{svgraybox}

\begin{example}[Beta-Bernoulli]
	Suppose $x_1, x_2, ..., x_N$ are drawn i.i.d. from $Bernoulli(x|\theta)$.
$Beta(\theta | a,b)$ distribution, with $a, b >0$, is conjugate to $Bernoulli(x|\theta)$, since the posterior is $p(\theta | x_{1:N}) = Beta(\theta | a+\sum x_i, b+N-\sum x_i)$. 
\exampbar
\end{example}

Conjugate priors make it possible to do Bayesian reasoning in a computationally efficient manner, as well as having the philosophically satisfying interpretation of representing real or imaginary prior data.

\section{Conjugate prior for the multinomial distribution}\label{sec:dirichlet_prior_on_multinomial}
This section and the next section serve as reference for the rest of the document. 

\subsection{Multinomial distribution}
The multinomial distribution is widely used in Bayesian mixture model to introduce latent variable. And the use of conjugate priors allows all the results to be derived in closed form. 
The multinomial distribution is parametrized by an integer $N$ and a p.m.f. $\bpi = \{\pi_1, \pi_2, \ldots , \pi_K\}$, and can be thought of as following: If we have $N$ independent events, and for each event, the probability of outcome $k$ is $\pi_k$, then the multinomial distribution specifies the probability that outcome $k$ occurs $N_k$ times,  for $k = 1, 2, \ldots , K$. For example, the multinomial distribution can model the probability of an $N$-sample empirical histogram, if each sample is drawn i.i.d., from $\bpi$. Formally, we have the following definition of Multinomial distribution.

\begin{svgraybox}
\begin{definition}[Multinomial Distribution]
A random vector $\bN=[N_1, N_2, \ldots, N_K]\in \{0, 1, 2, \ldots, N\}^K$ where $\sum_{k=1}^{K} N_k=N$ is said to follow the multinomial distribution with parameter $N\in \natu$ and $\bpi =[\pi_1, \pi_2, \ldots, \pi_K]\in [0,1]^K$ such that $\sum_{k=1}^{K} \pi_k=1$. Denoted by $\bN \sim $ $\multinomial_K(N, \bpi)$. Then its probability mass function is given by 
\begin{equation*}
	p(N_1, N_2, \ldots , N_K | N, \bpi = (\pi_1, \pi_2, \ldots , \pi_K)) = \frac{N!}{N_1! N_2!  \ldots  N_K!} \prod^K_{k=1}\pi_k^{N_k} \cdot \indicator\left\{\sum_{k=1}^{K}N_k = N\right\},
\end{equation*}
where $\{0, 1, 2, \ldots, N\}$ is a set of $N+1$ elements and $[0,1]$ is an closed set with values between 0 and 1.
The mean, variance, covariance are
$$
\Exp[N_k] = N\pi_k, \qquad \Var[N_k] = N\pi_k(1-\pi_k), \qquad \Cov[N_k, N_m] = -N\pi_k\pi_m.
$$
When $K=2$, the multinomial distribution reduces to the binomial distribution. 
\end{definition}
\end{svgraybox}

\subsection{Dirichlet distribution}\label{section:dirichlet-dist}
\begin{figure}[htp]
	\centering
	\subfigure[ $ \balpha=\begin{bmatrix}
		10,10,10
	\end{bmatrix} $, z-axis is pdf. ]{\includegraphics[width=0.33
		\textwidth]{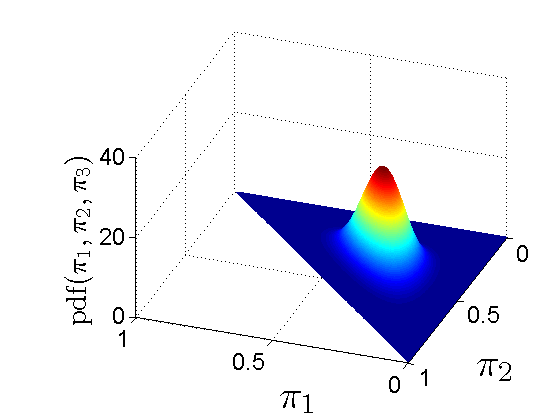} \label{fig:dirichlet_pdf}}
	~
	\subfigure[$\balpha=\begin{bmatrix}
		10,10,10
	\end{bmatrix}$, z-axis is $\pi_3$.]{\includegraphics[width=0.33
		\textwidth]{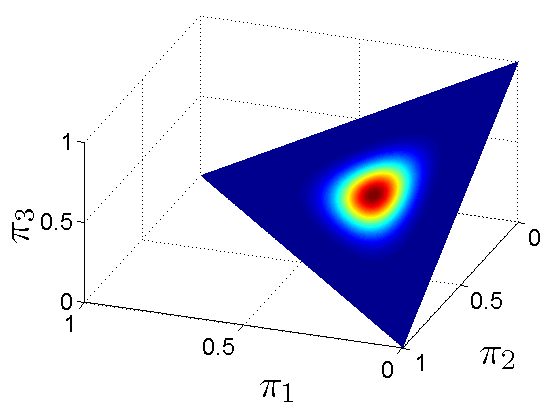} \label{fig:dirichlet_surface}}
	\centering
	\subfigure[ $ \balpha=\begin{bmatrix}
		1,1,1
	\end{bmatrix} $ ]{\includegraphics[width=0.4
		\textwidth]{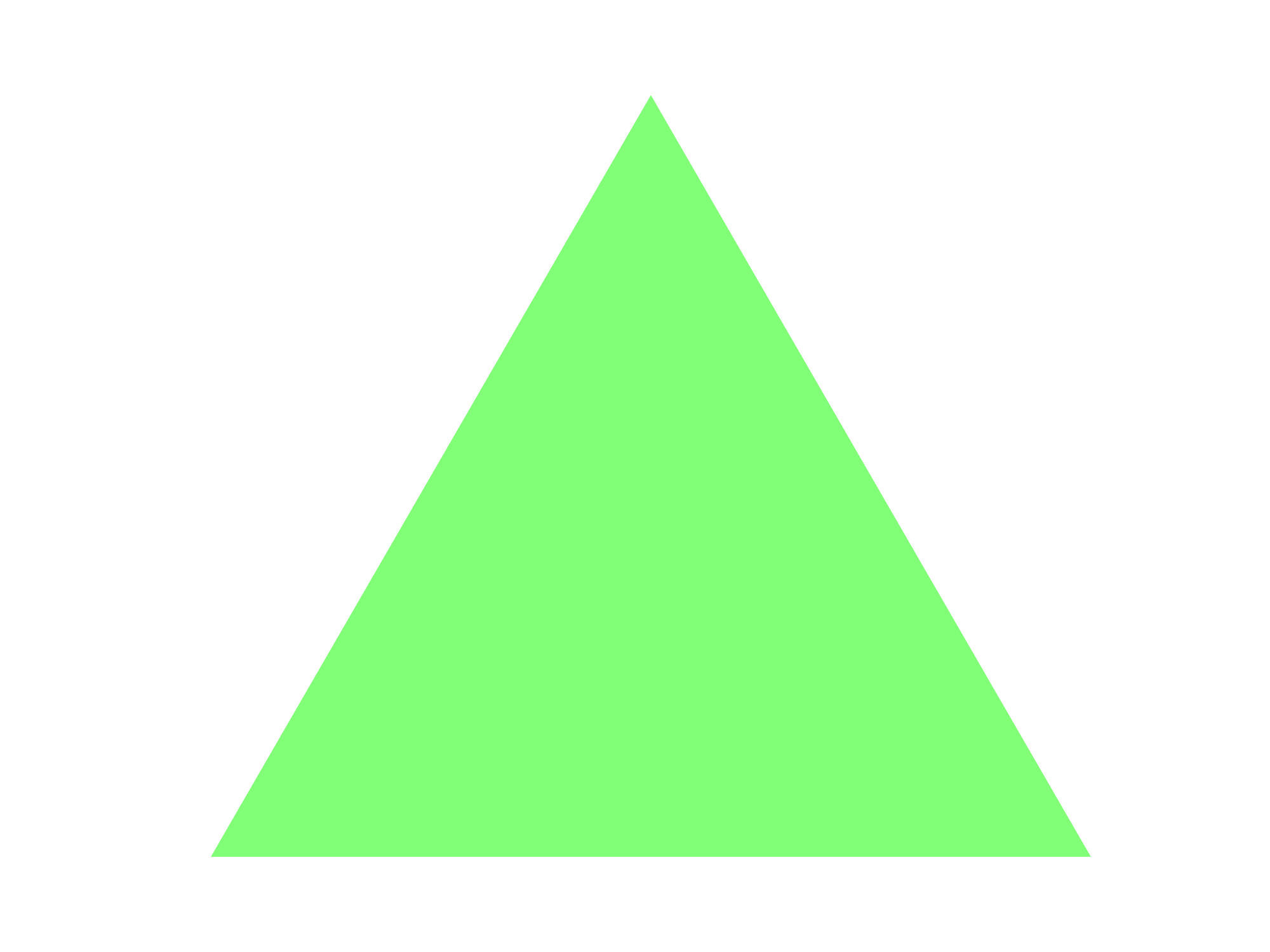} \label{fig:dirichlet_sample_111}}
	~
	\subfigure[$\balpha=\begin{bmatrix}
		0.9,0.9,0.9
	\end{bmatrix}$]{\includegraphics[width=0.4
		\textwidth]{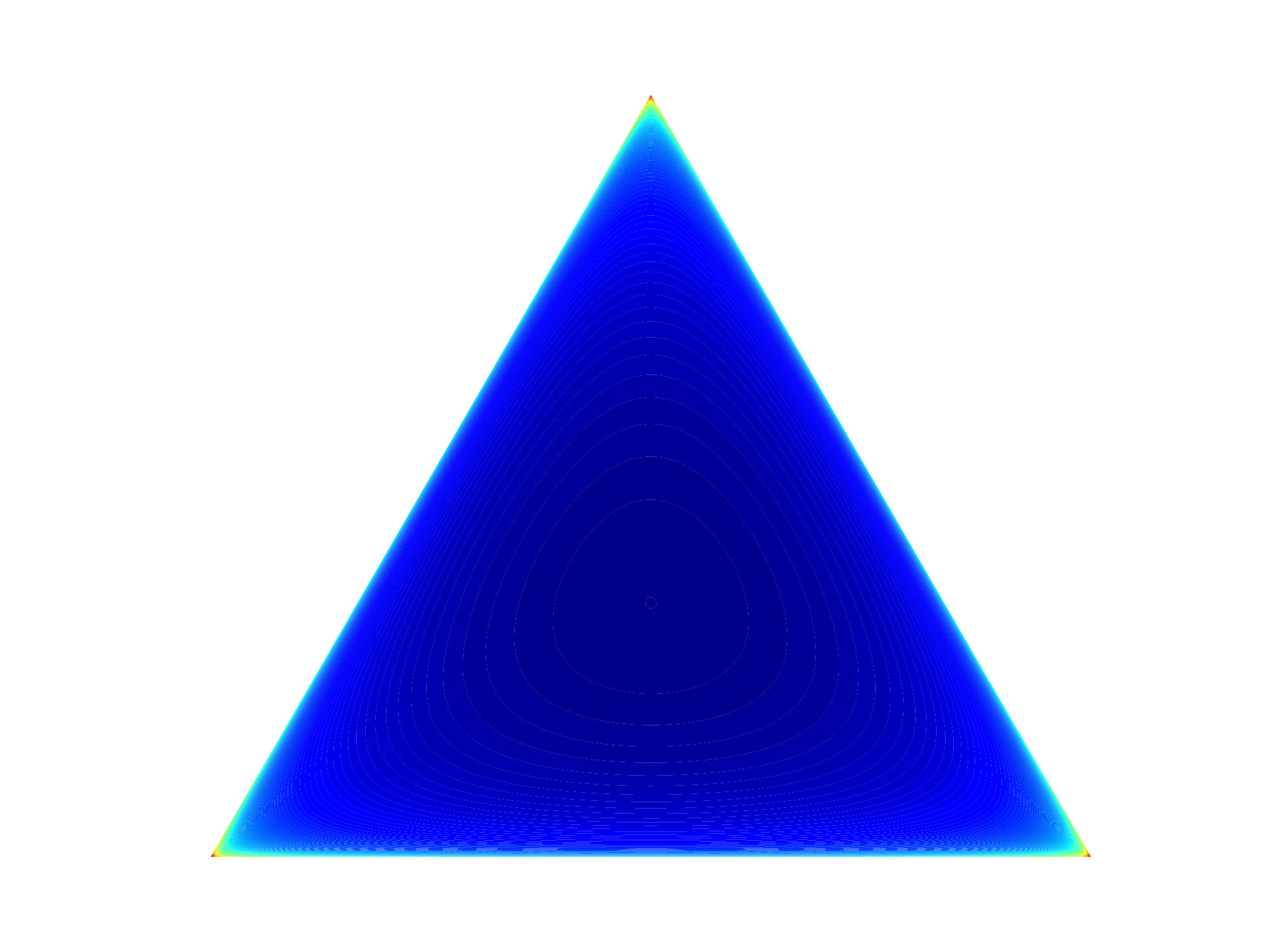} \label{fig:dirichlet_sample_090909}}
	\centering
	\subfigure[$\balpha=\begin{bmatrix}
		10,10,10
	\end{bmatrix}$]{\includegraphics[width=0.4
		\textwidth]{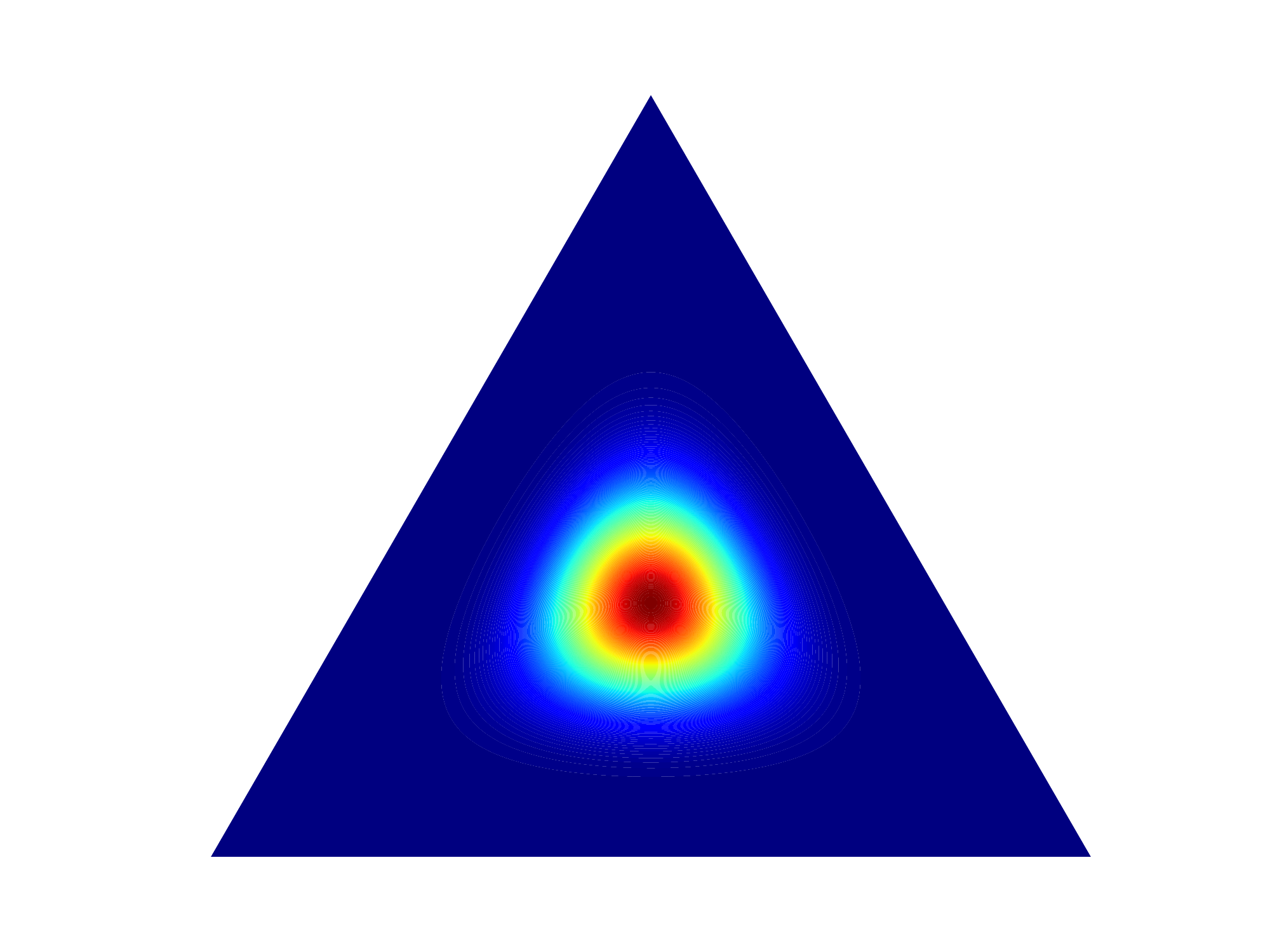} \label{fig:dirichlet_sample_101010}}
	~
	\subfigure[$\balpha=\begin{bmatrix}
		15,5,2
	\end{bmatrix}$]{\includegraphics[width=0.4
		\textwidth]{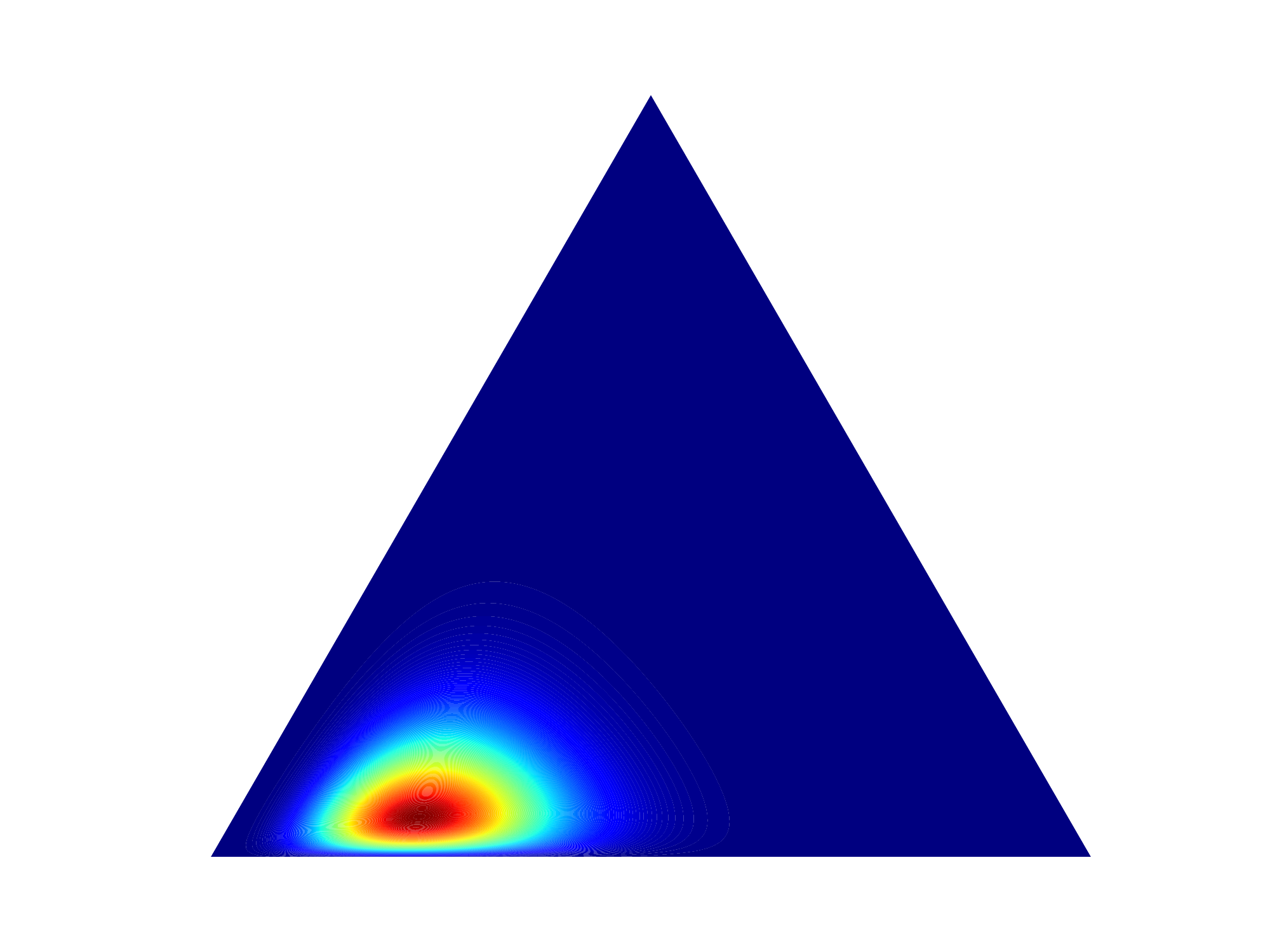} \label{fig:dirichlet_sample_1552}}
	\centering
	\caption{Density plots (blue=low, red=high) for the Dirichlet distribution over the probability simplex in $\mathbb{R}^3$ for various values of the concentration parameter $\balpha$. 
		When $\balpha=[c, c, c]$, the distribution is called a \textbf{symmetric Dirichlet distribution} and the density is symmetric about the uniform probability mass function (i.e., occurs in the middle of the simplex). When $0<c<1$, there are sharp peaks of density almost at the vertices of the simplex. When $c>1$, the density becomes monomodal and concentrated in the center of the simplex. And when $c=1$, it is uniform distributed over the simplex. Finally, if $\balpha$ is not a constant vector, the density is not symmetric.}\centering
	\label{fig:dirichlet_samples}
\end{figure}
The Dirichlet distribution serves as a conjugate prior for the probability parameter $\bpi$ of the multinomial distribution. 
\begin{svgraybox}
\begin{definition}[Dirichlet Distribution]
A random vector $\bX=[x_1, x_2, \ldots, x_K]\in [0,1]^K$ is said to follow Dirichlet distribution if 
\begin{equation} 
	\dirichlet(\bX | \balpha) \triangleq  \frac{1}{D(\balpha)}  \prod_{k=1}^K x_k ^ {\alpha_k - 1},
	\label{equation:dirichlet_distribution2}
\end{equation}  
such that $\sum_{k=1}^K x_k = 1$, $x_k \in$ [0, 1] and 
\begin{equation} 
	D(\balpha) = \frac{\prod_{k=1}^K \Gamma(\alpha_k)}{\Gamma(\alpha_+)},
	\label{equation:dirichlet_distribution3}
\end{equation}  
where $\balpha = [\alpha_1, \alpha_2, \ldots, \alpha_K]$ is a vector of reals with $\alpha_k>0, \forall k$, $\alpha_+ = \sum_{k=1}^K \alpha_k$. The $\balpha$ is also known as the \textbf{concentration parameter} in Dirichlet distribution. $\Gamma(\cdot)$ is the Gamma function which is a generalization of the factorial function. For $m>0$, $\Gamma(m+1) = m\Gamma(m)$ which implies for positive integers $n$, $\Gamma(n)=(n-1)!$ since $\Gamma(1)=1$. The mean, variance, covariance are 
$$
\Exp[x_k] = \frac{\alpha_k}{\alpha_+}, \qquad \Var[x_k] = \frac{\alpha_k(\alpha_+-\alpha_k)}{\alpha_+^2(\alpha_++1)}, \qquad \Cov[x_k, x_m]= \frac{-\alpha_k\alpha_m}{\alpha_+^2(\alpha_++1)}.
$$
When $K=2$, the Dirichlet distribution reduces to the Beta distribution, The Beta distribution $Beta(\alpha, \beta)$ is defined on $[0,1]$ with the probability density function given by 
$$
\betadist(x| \alpha, \beta) = \frac{\Gamma(\alpha+\beta)}{\Gamma(\alpha)\Gamma(\beta)} x^{\alpha-1}(1-x)^{\beta-1}.
$$
That is, if $X\sim \betadist(\alpha, \beta)$, then $\bX=[X, 1-X] \sim \dirichlet(\balpha) $, where $\balpha=[\alpha, \beta]$.
\end{definition}
\end{svgraybox}
Interesting readers can refer to Appendix~\ref{appendix:drive-dirichlet} for a derivation of the Dirichlet distribution.
The sample space of the Dirichlet distribution lies on the $(K-1)$-dimensional probability simplex, which is a surface in $\real^K$ denoted by $\triangle_K$. That is a set of vectors in $\real^K$ whose components are non-negative and sum to 1.
$$
\triangle_K = \{ \bpi: 0 \leq \pi_k\leq 1, \sum_{k=1}^{K}\pi_k=1 \}.
$$
Notice that $\triangle_K$ lies on a $(K-1)$-dimensional space since each component is non-negative, and the components sum to 1.

Figure~\ref{fig:dirichlet_samples} shows various plots of the density of the Dirichlet distribution over the two-dimensional simplex in $\mathbb{R}^3$ for a handful of values of the parameter vector $\balpha$ and Figure~\ref{fig:dirichlet_points} shows the draw of 5, 000 points for each setting. 
In specific, the density plots of Dirichlet in $\real^3$ is a surface plot in 4$d$-space. The Figure~\ref{fig:dirichlet_pdf} is a projection of a surface into 3$d$-space where the z-axis is the probability density function and Figure~\ref{fig:dirichlet_surface} is a projection of a surface into 3$d$-space where the z-axis is $\pi_3$. Figure~\ref{fig:dirichlet_sample_111} to Figure~\ref{fig:dirichlet_sample_1552} are the projections into a 2$d$-space.

When the concentration parameter $\balpha=[1,1,1]$, the Dirichlet distribution reduces to the uniform distribution over the simplex. This can be easily verified that $Dirichlet(\bX | \balpha=[1,1,1]) =  \frac{\Gamma(3)}{(\Gamma(1))^3}= 2 $ which is a constant that does not depend on the specific value of $\bX$. When $\balpha=[c,c,c]$ with $c>1$, the density becomes a monomodal and concentrated in the center of the simplex. This can be seen from $\dirichlet(\bX | \balpha=[c,c,c]) =  \frac{\Gamma(3c)}{(\Gamma(c))^3}\prod_{k=1}^3 x_k ^ {c - 1} $ such that small value of $x_k$ will make the probability density approach to zero. On the contrary, when $\balpha=[c,c,c]$ with $c<1$, the density has sharp peaks almost at the vertices of the simplex.

More properties of the Dirichlet distribution is provided in Table~\ref{table:dirichlet-property}, and the proof can be found in Appendix~\ref{appendix:drive-dirichlet}. And the derivation on the Dirichlet distribution in Appendix~\ref{appendix:drive-dirichlet} can also be utilized to generate samples from the Dirichlet distribution by a set of samples from a set of Gamma distributions.
\begin{table}[]
	\begin{tabular}{l|l}
		\hline
		\begin{tabular}[c]{@{}l@{}}Marginal \\ Distribution\end{tabular}    & $X_i \sim \betadist(\alpha_i, \alpha_+-\alpha_i)$.                                                                                                                                                                                                                                                                                                                                                                                                                                                                                                                                                           \\ \hline
		\begin{tabular}[c]{@{}l@{}}Conditional \\ Distribution\end{tabular} & 
		\begin{tabular}[c]{@{}l@{}}$\bX_{-i} | X_i \sim (1-X_i)\dirichlet(\alpha_{-i})$, \\ 
			 where $\bX_{-i}$ is a random vector excluding $X_i$. \end{tabular}                                                                                                                                                                                                                                                                                                                                                                                                                                                                      \\ \hline
		\begin{tabular}[c]{@{}l@{}}Aggregation\\ Property\end{tabular}      & \begin{tabular}[c]{@{}l@{}}If $M=X_i+X_j$, then $[X_1, \ldots X_{i-1}, X_{i+1}, \ldots, X_{j-1}, X_{j+1}, \ldots, X_K, M] \sim $\\ 
			$\dirichlet([\alpha_1, \ldots, \alpha_{i-1}, \alpha_{i+1}, \ldots, \alpha_{j-1}, \alpha_{j+1}, \ldots, \alpha_K, \alpha_i+\alpha_j])$.\\ 
			In general, If $\{A_1, A_2, \ldots, A_r\}$ is a partition of $\{1, 2, \ldots, K\}$, then \\ 
			$\left[\sum_{i\in A_1} X_i, \sum_{i\in A_2} X_i, \ldots, \sum_{i\in A_r} X_i\right] \sim$\\ $\dirichlet\left(\left[\sum_{i\in A_1} \alpha_i, \sum_{i\in A_2} \alpha_i, \ldots, \sum_{i\in A_r} \alpha_i\right]\right)$.\end{tabular} \\ \hline
	\end{tabular}
	\caption{Properties of Dirichlet distribution.}
\label{table:dirichlet-property}
\end{table}

\begin{figure}[!h]
\center
\subfigure[ $ \balpha=\begin{bmatrix}1,1,1\end{bmatrix}$ ]{\includegraphics[width=0.4
\textwidth]{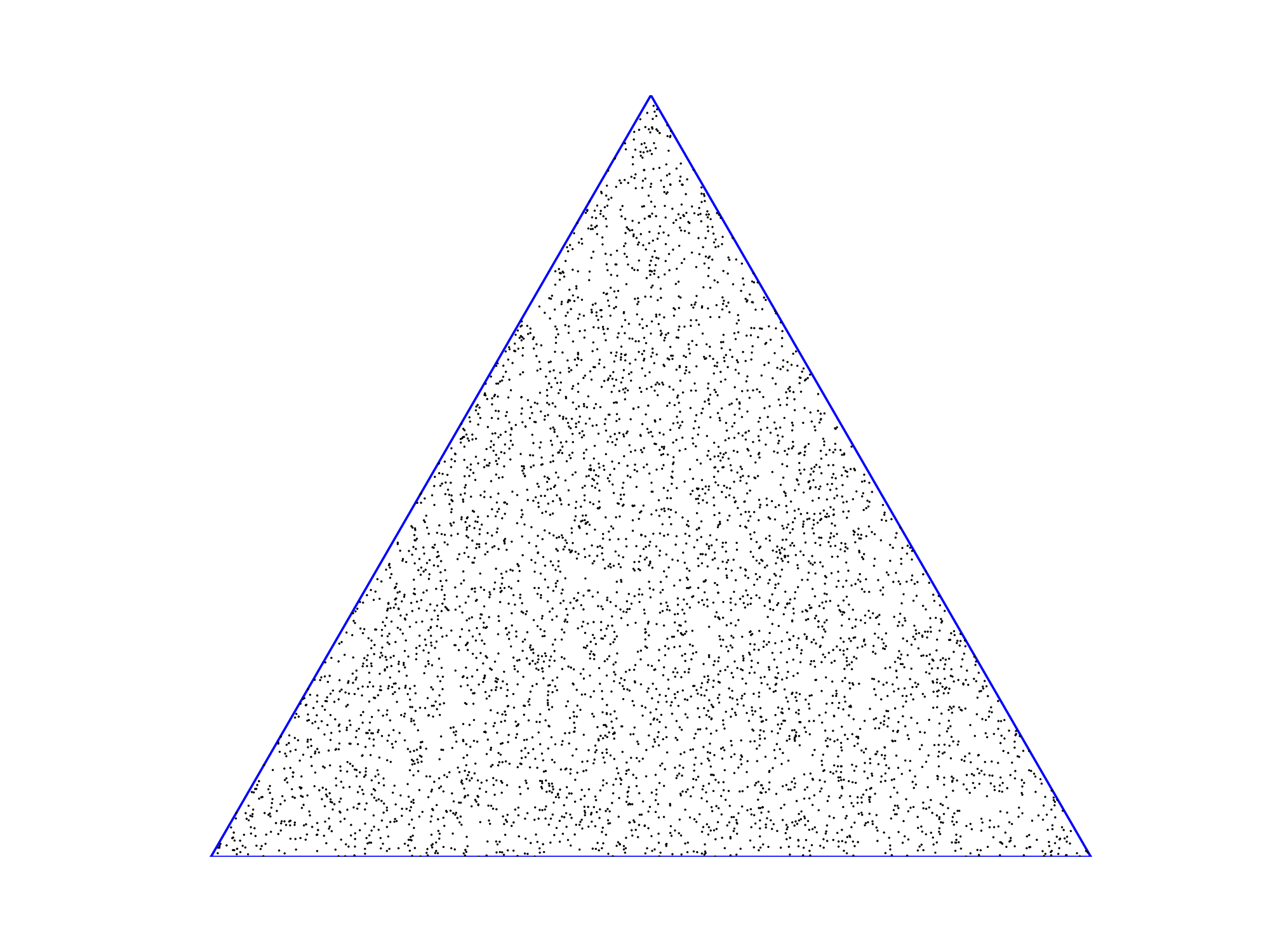} \label{fig:dirichlet_points_111}}
~
\subfigure[$\balpha=\begin{bmatrix}0.9,0.9,0.9\end{bmatrix}$]{\includegraphics[width=0.4
\textwidth]{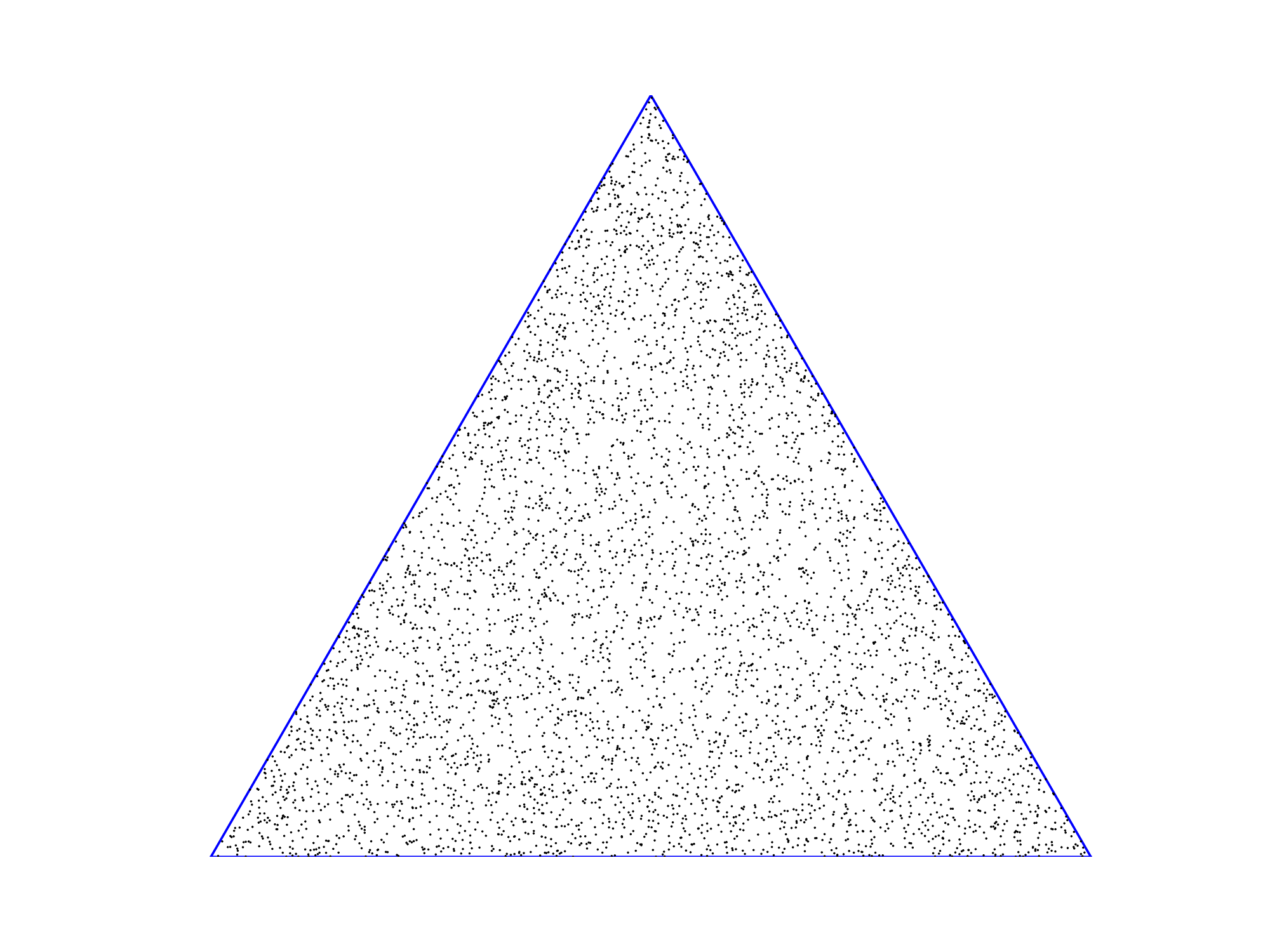} \label{fig:dirichlet_points_090909}}
\center
\subfigure[$\balpha=\begin{bmatrix}10,10,10\end{bmatrix}$]{\includegraphics[width=0.4
\textwidth]{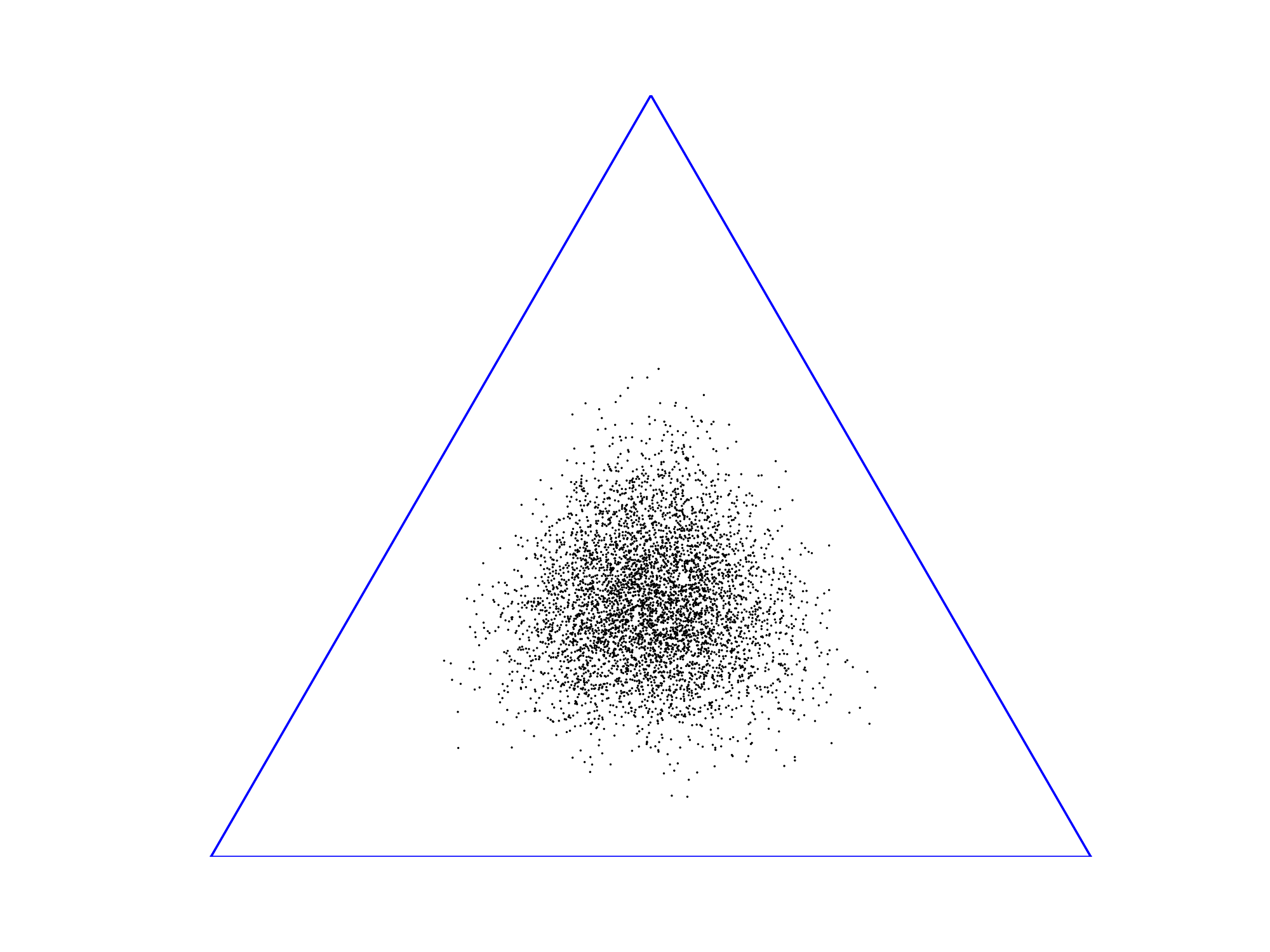} \label{fig:dirichlet_points_101010}}
~
\subfigure[$\balpha=\begin{bmatrix}15,5,2\end{bmatrix}$]{\includegraphics[width=0.4
\textwidth]{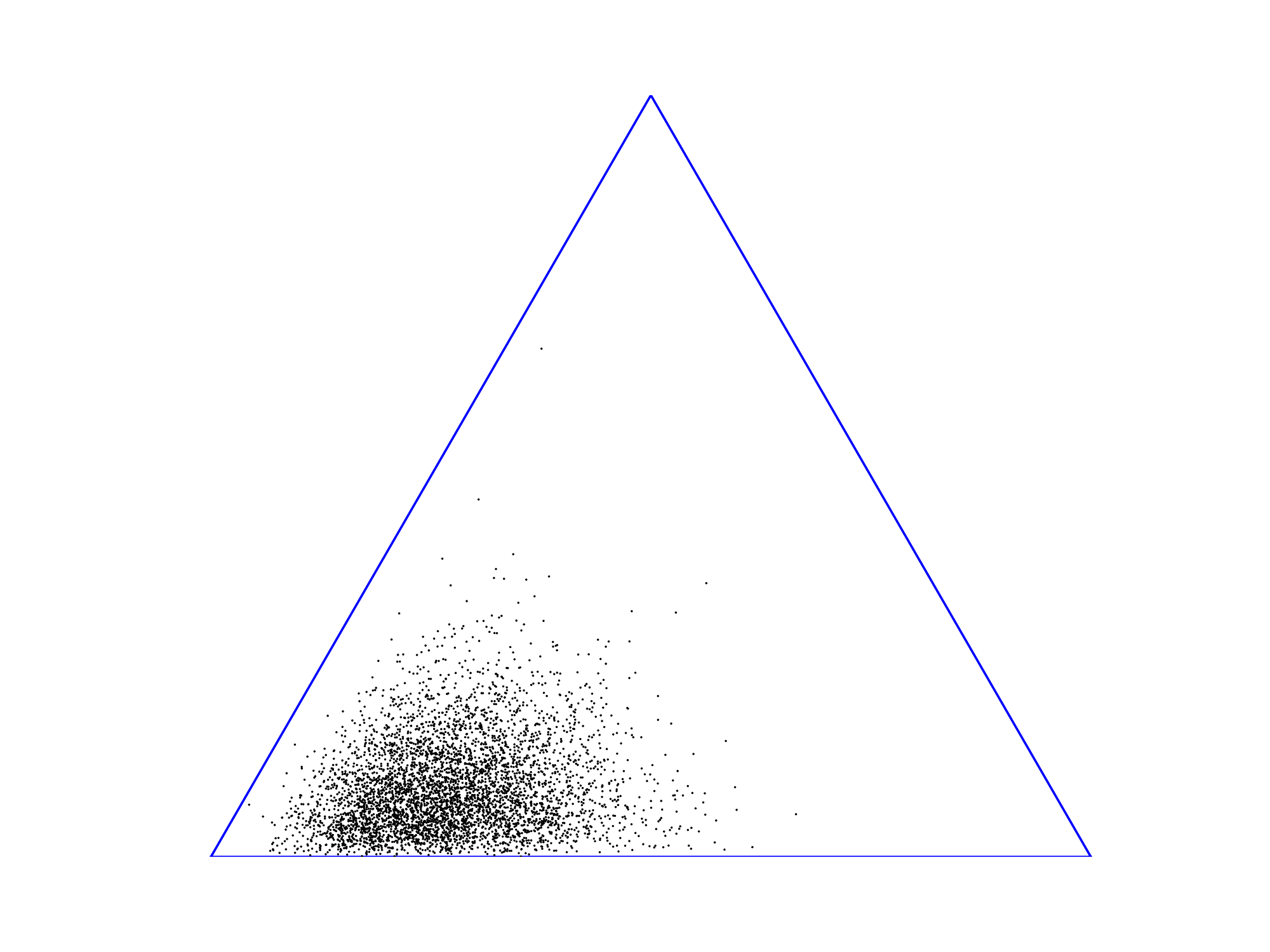} \label{fig:dirichlet_points_1552}}
\center
\caption{Draw of 5, 000 points from Dirichlet distribution over the probability simplex in $\mathbb{R}^3$ for various values of the concentration parameter $\balpha$.}
\label{fig:dirichlet_points}
\end{figure}

\subsection{Posterior distribution for multinomial Distribution}\label{section:dirichlet-dist-post}
For the conjugacy, that is, if $(\bN | \bpi) \sim $ $\multinomial_K(N, \bpi)$ and $\bpi \sim$ $\dirichlet(\balpha)$, then $(\bpi |\bN) \sim$ $\dirichlet(\balpha+\bN)$ = $\dirichlet(\alpha_1+N_1, \ldots, \alpha_K+N_K)$.
\begin{proof}[of conjugate prior of multinomial distribution]
	By the Bayes' theorem ``$\mathrm{posterior} \propto \mathrm{likelihood} \times \mathrm{prior} $", we get the posterior
	$$
	\begin{aligned}
		\mathrm{posterior}&= p(\bpi|\balpha, \bN) \\
		&\propto \multinomial_K(\bN|N, \bpi) \cdot \dirichlet(\bpi|\balpha) \\
		&= \left(\frac{N!}{N_1! N_2!  \ldots  N_K!} \prod^K_{k=1}\pi_k^{N_k}\right) \cdot  \left(\frac{1}{D(\balpha)}  \prod_{k=1}^K \pi_k ^ {\alpha_k - 1}\right)\\
		&\propto   \prod_{k=1}^K \pi_k ^ {\alpha_k +N_k - 1} \propto \dirichlet(\bpi|\balpha+\bN).
	\end{aligned}
	$$
	Therefore, $(\bpi | \bN) \sim$ $\dirichlet(\balpha+\bN)$ = $\dirichlet(\alpha_1+N_1, \ldots, \alpha_K+N_K)$.
\end{proof}

A comparison between the prior and posterior distribution reveals that the relative sizes of the Dirichlet parameters $\alpha_k$ describe the mean of the prior distribution of $\bpi$, and the sum of $\alpha_k$'s is a measure of the strength of the prior distribution. The prior distribution is mathematically equivalent to a likelihood resulting from $\sum_{k=1}^K(\alpha_k - 1)$ observations with $\alpha_k - 1$ observations of the $k^{th}$ group. 

To be noted, the Dirichlet distribution is a multivariate generalization of the beta distribution, which is the conjugate prior for binomial distribution \citep{hoff2009first, frigyik2010introduction}.

To conclude, here are some important points on Dirichlet distribution:
\begin{itemize}
\item A sample from a Dirichlet distribution is a probability vector (positive and sum to 1). In other words, a Dirichlet distribution is a probability distribution over all possible multinomial distributions with $K$ dimensions.
\item Dirichlet distribution is a conjugate prior of multinomial distribution as mentioned in the beginning of this section.
\end{itemize}

\section{Conjugate prior for multivariate Gaussian distribution}\label{sec:multi_gaussian_conjugate_prior}
The content is based on \citep{murphy2012machine, teh2007exponential, kamper2013gibbs, das2014dpgmm}. And also, \citep{murphy2007conjugate} provides us all other kinds of prior on Gaussian distribution as well.

\subsection{Multivariate Gaussian distribution}
\begin{svgraybox}
	\begin{definition}[Multivariate Guassian Distribution]
		A random vector $\bx \in \real^D$ is said to follow the multivariate Gaussian distribution with parameter $\bmu$ and $\bSigma$ if
		$$
		\begin{aligned}
			\normal(\bx| \bmu, \bSigma)&= (2\pi)^{-D/2} |\bSigma|^{-1/2}\exp\left(-\frac{1}{2}(\bx - \bmu)^\top \bSigma^{-1}(\bx - \bmu)\right),
		\end{aligned}
		$$
		where $\bmu \in \real^D$ is called the \textbf{mean vector}, and $\bSigma\in \real^{D\times D}$ is positive definite and is called the \textbf{covariance matrix}.
		The mean, mode, and covariance of the multivariate Gaussian distribution are given by 
		\begin{equation*}
			\begin{aligned}
				\Exp [\bx] &= \bmu,\\
				\mathrm{Mode}[\bx] &= \bmu, \\
				\Cov [\bx] &= \bSigma. 
			\end{aligned}
		\end{equation*}
	\end{definition}
\end{svgraybox}

The likelihood of $N$ random observations $\mathcal{X} = \{\bx_1, \bx_2, \ldots , \bx_N \}$ being generated by a multivariate Gaussian with mean vector $\bmu$ and covariance matrix $\bSigma$ is given by 
\begin{equation}
\begin{aligned}
&\gap p(\mathcal{X} | \bmu, \bSigma) =\prod^N_{n=1} \mathcal{N} (\bxn| \bmu, \bSigma) \\
&\overset{(a)}{=} (2\pi)^{-ND/2} |\bSigma|^{-N/2}\exp\left(-\frac{1}{2} \sum^N_{n=1}(\bxn - \bmu)^\top \bSigma^{-1}(\bxn - \bmu)\right) \\
&\overset{(b)}{=} (2\pi)^{-ND/2} |\bSigma|^{-N/2}\exp\left(-\frac{1}{2} \tr( \bSigma^{-1}\bS_{\bmu} )  \right)\\
&\overset{(c)}{=} (2\pi)^{-ND/2} |\bSigma|^{-N/2}\exp\left(-\frac{N}{2}(\bmu - \overline{\bx})^\top \bSigma^{-1}(\bmu - \overline{\bx})\right)  \exp\left(-\frac{1}{2}\tr( \bSigma^{-1}\bS_{\overline{x}} )\right),
\end{aligned}
\label{equation:multi_gaussian_likelihood}
\end{equation}
where 
\begin{equation}\label{equation:mvu-sample-covariance}
\begin{aligned}
\bS_{\bmu}            &\triangleq \sum^N_{n=1}(\bxn - \bmu)(\bxn - \bmu)^\top,\\
\bS_{\overline{x}}            &\triangleq \sum^N_{n=1}(\bxn - \overline{\bx})(\bxn - \overline{\bx})^\top, \\
\overline{\bx}		&\triangleq \frac{1}{N}\sum^N_{n=1}\bxn, 
\end{aligned}
\end{equation}
The equivalence of Equation (a) and Equation (c) above in Equation~\eqref{equation:multi_gaussian_likelihood} follows from the identity (similar reason for the equivalence of Equation (a) and Equation (b)):
\begin{align}
\sum^N_{n=1}(\bxn - \bmu)^\top\bSigma^{-1}(\bxn - \bmu) = \tr(\bSigma^{-1}\bS_{\overline{x}}) + N \cdot (\overline{\bx} - \bmu)^\top\bSigma^{-1}(\overline{\bx} - \bmu).
\label{equation:multi_gaussian_identity}
\end{align}
where the trace of a square matrix $\bm{A}$ is defined to be the sum of the diagonal elements $a_{ii}$ of $\bm{A}$:
\begin{equation}
\tr(\bm{A}) \triangleq \sum_i a_{ii}.
\end{equation}
The formulation in Equation (b) is useful for the separated view of the conjugate prior for $\bSigma$, and Equation (c) is useful for the unified view of the conjugate prior for $\bmu, \bSigma$ in the sequel.
\begin{proof}[Proof of Identity~\ref{equation:multi_gaussian_identity}]
There is a “trick” involving the trace that makes such calculations easy (see also Chapter 3 of \citep{gentle2007matrix})
\begin{equation}
\bx^\top \bm{A} \bx = \tr(\bx^\top \bm{A} \bx) = \tr(\bx \bx^\top \bm{A}) =  \tr(\bm{A} \bx \bx^\top )
\end{equation}
where the first equality follows from the fact that $\bx^\top \bm{A} \bx$ is a scalar and the trace of a product is invariant under cyclical permutations of the factors \footnote{Trace is invariant under cyclical permutation: $\tr(\bA\bB\bC) = \tr(\bB\bC\bA) = \tr(\bC\bA\bB)$ if all $\bA\bB\bC$, $\bB\bC\bA$, and $\bC\bA\bB$ exist.} . 

We can then rewrite $\sum^N_{n=1}(\bxn - \bmu)^\top\bSigma^{-1}(\bxn - \bmu)$ by 
\begin{equation}
\begin{aligned}
&\, \sum^N_{n=1}(\bxn - \overline{\bx})^\top\bSigma^{-1}(\bxn - \overline{\bx}) + \sum^N_{n=1}(\overline{\bx} - \bmu)^\top\bSigma^{-1}(\overline{\bx} - \bmu) \\
&= \tr(\bSigma^{-1} \bS_{\overline{x}}  ) + N \cdot (\overline{\bx} - \bmu)^\top\bSigma^{-1}(\overline{\bx} - \bmu).
\end{aligned}
\end{equation}
This concludes the proof.
\end{proof}

By equivalence from the Identity~\eqref{equation:multi_gaussian_identity}, we cannot reduce the complexity, but it is useful to show the the conjugacy in Section~\ref{sec:niw_posterior_conjugacy} below.

\subsection{Multivariate Student $t$ distribution}
The multivariate Student $t$ distribution will be often used in the posterior predictive distribution of multivariate Gaussian parameters. We rigorously define the distribution as follows.
\begin{figure}[htp]
	\subfigure[Gaussian, $\bSigma =\begin{bmatrix}
		1&0\\
		0&1
	\end{bmatrix}. $ ]{\includegraphics[width=0.31
		\textwidth]{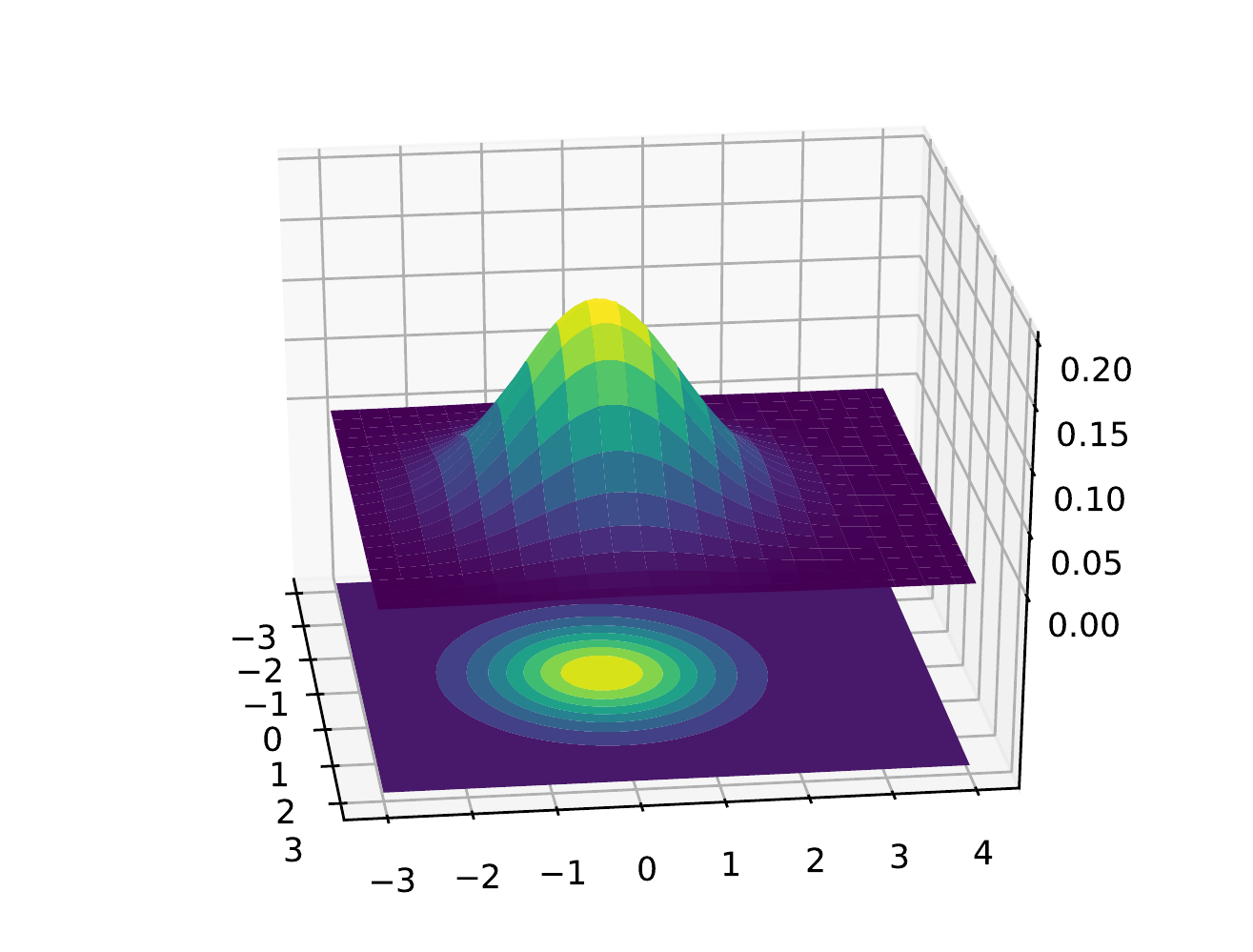} \label{fig:gauss-diagonal}}
	\subfigure[Gaussian, $\bSigma =\begin{bmatrix}
		1&0\\
		0&3
	\end{bmatrix}.$]{\includegraphics[width=0.31
		\textwidth]{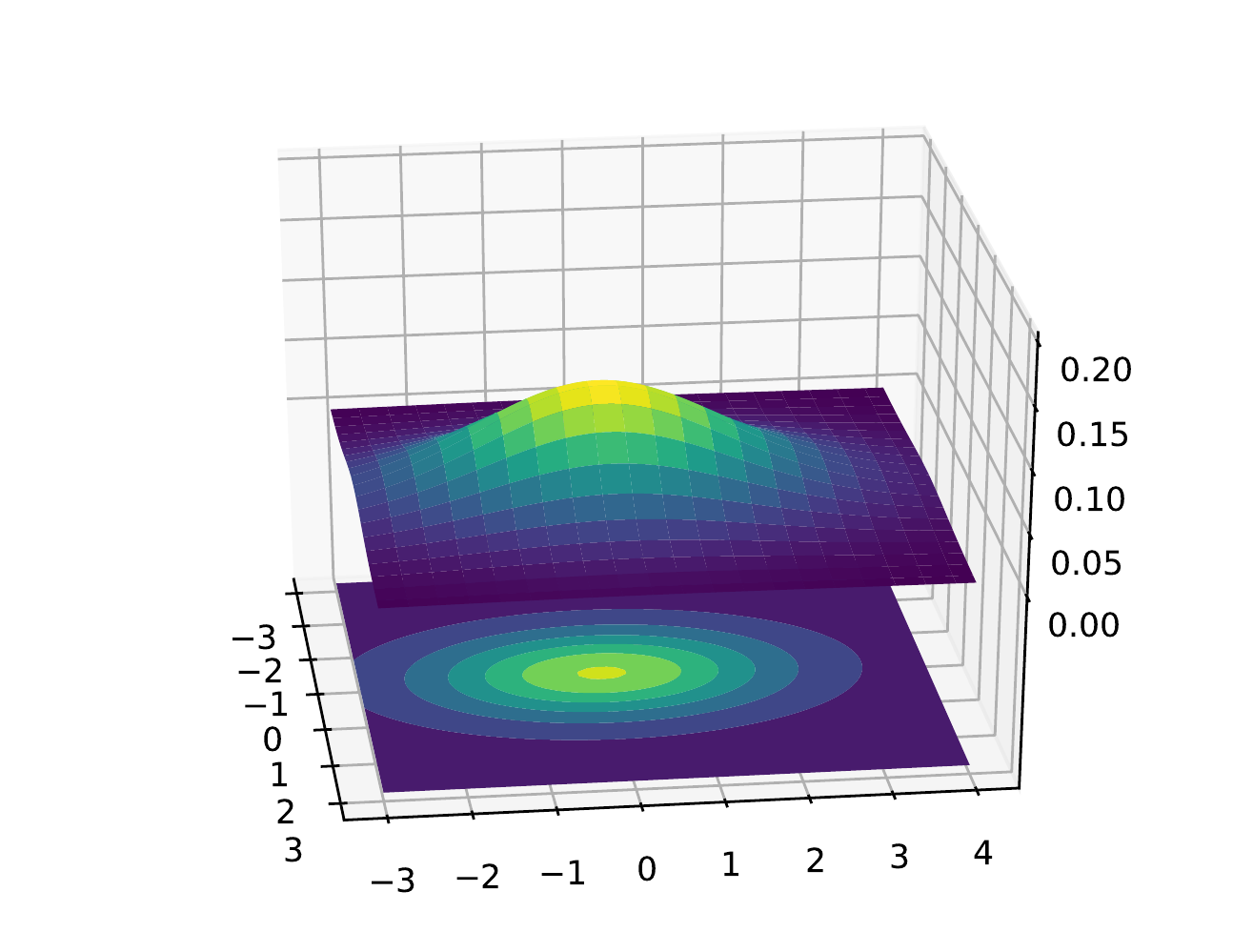} \label{fig:gauss-spherical}}
	\subfigure[Gaussian, $\bSigma =\begin{bmatrix}
		1&-0.5\\
		-0.5&1.5
	\end{bmatrix}.$]{\includegraphics[width=0.31 
		\textwidth]{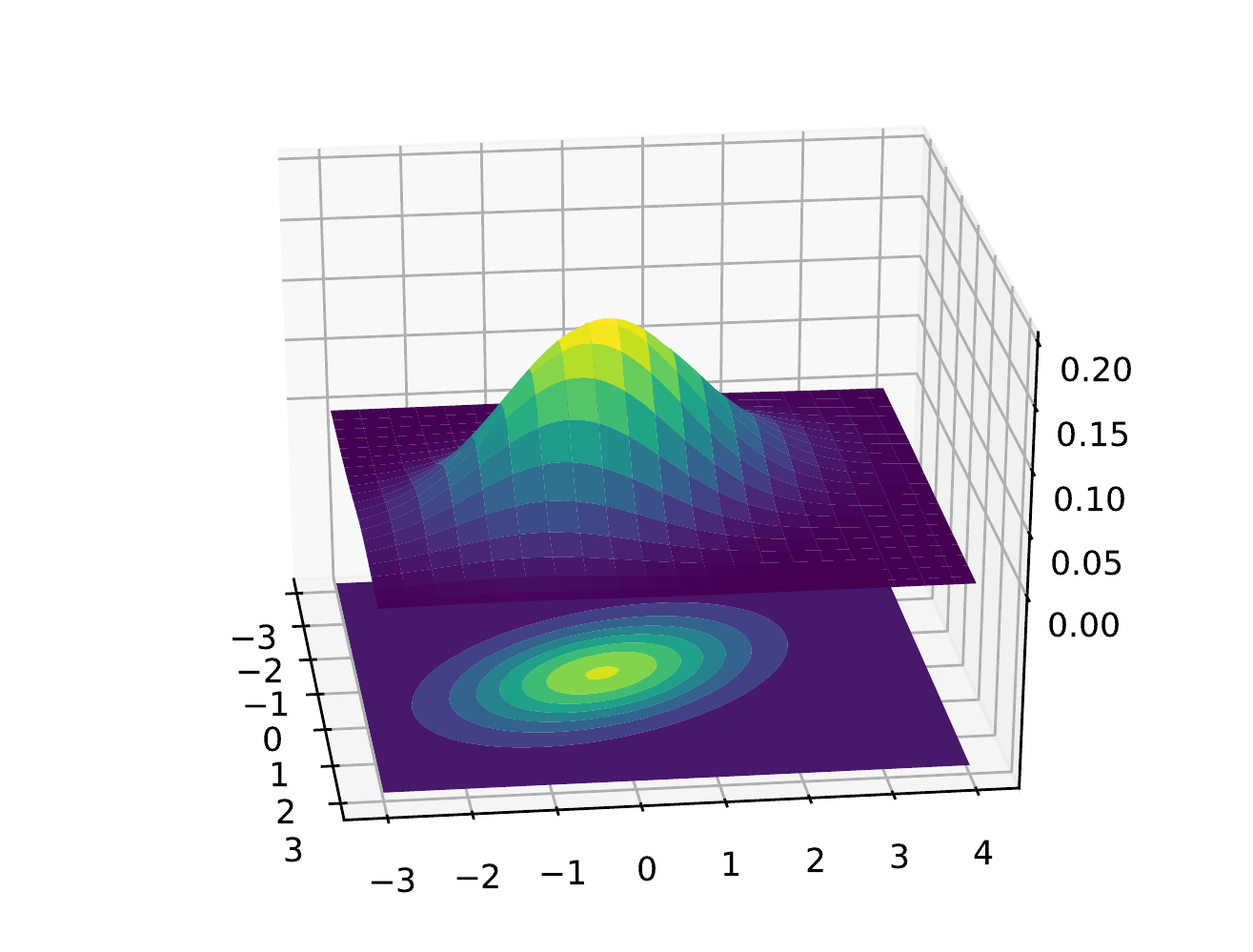} \label{fig:gauss-full}}
	\subfigure[Student $t$, $\bSigma =\begin{bmatrix}
		1&0\\
		0&1
	\end{bmatrix}, \nu=1. $ ]{\includegraphics[width=0.31
		\textwidth]{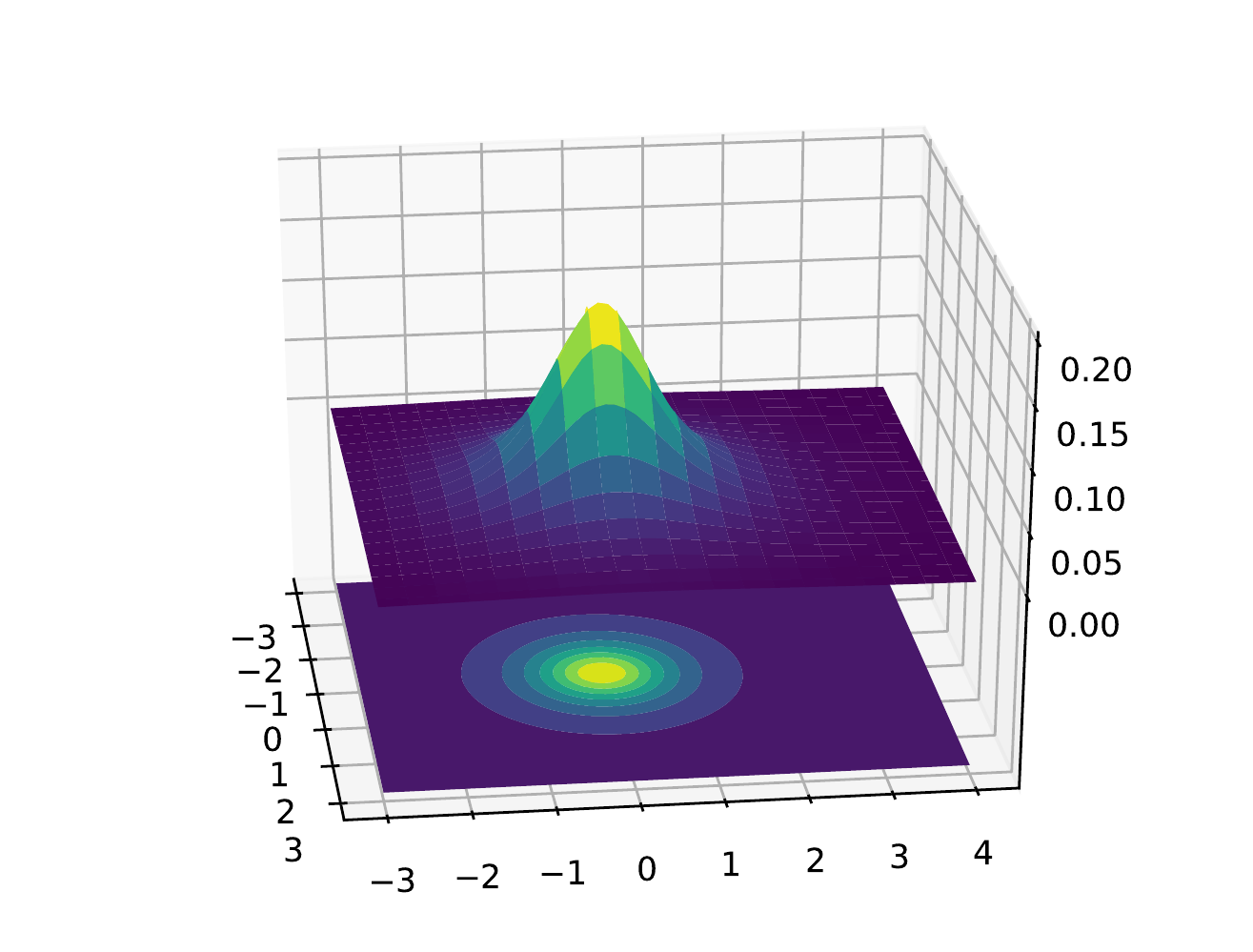} \label{fig:student-1}}
	\subfigure[Student $t$, $\bSigma =\begin{bmatrix}
		1&0\\
		0&1
	\end{bmatrix}, \nu=3. $ ]{\includegraphics[width=0.31
		\textwidth]{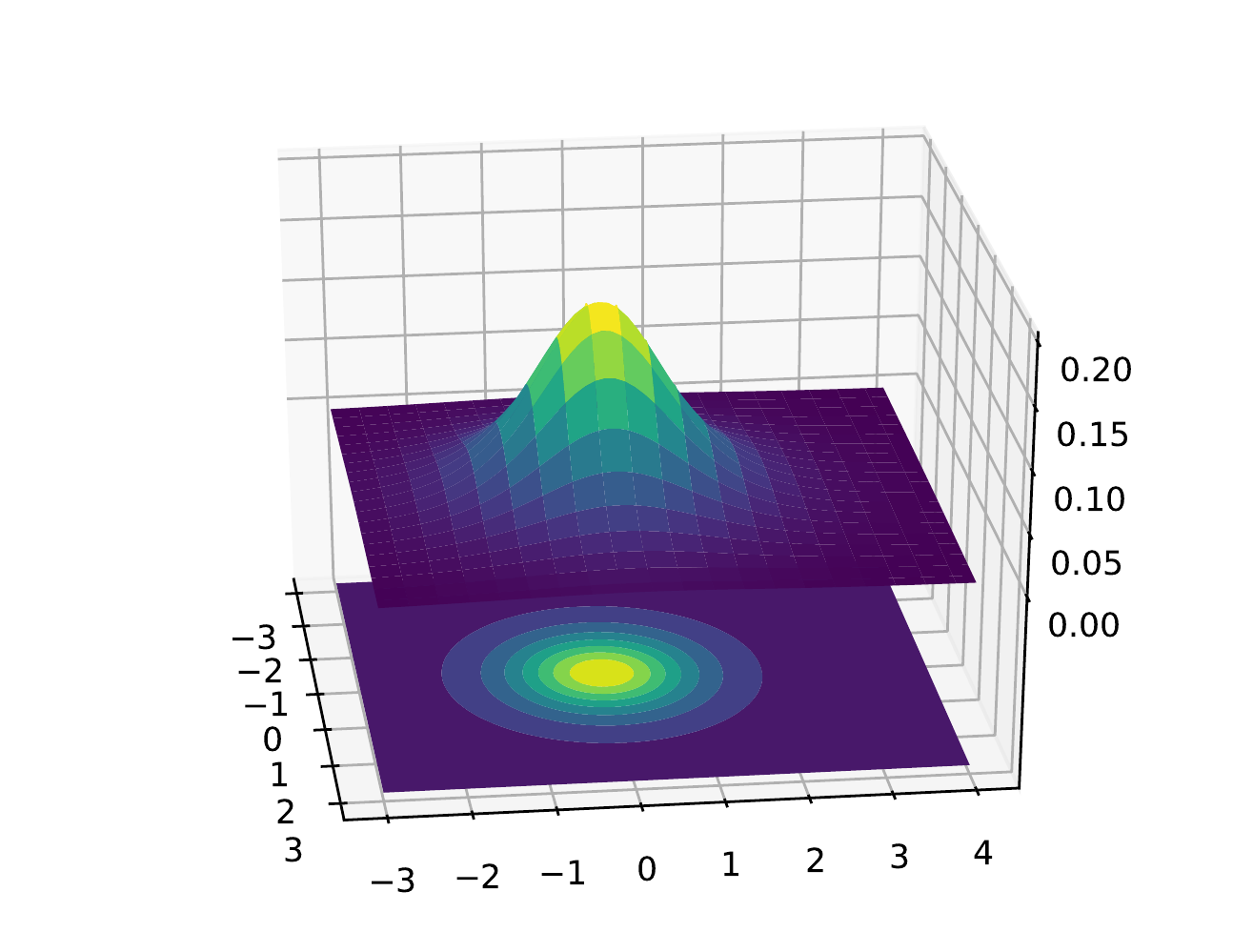} \label{fig:student-3}}
	\subfigure[Stu $t$, $\bSigma =\begin{bmatrix}
		1&0\\
		0&1
	\end{bmatrix}, \nu=200. $ ]{\includegraphics[width=0.31
		\textwidth]{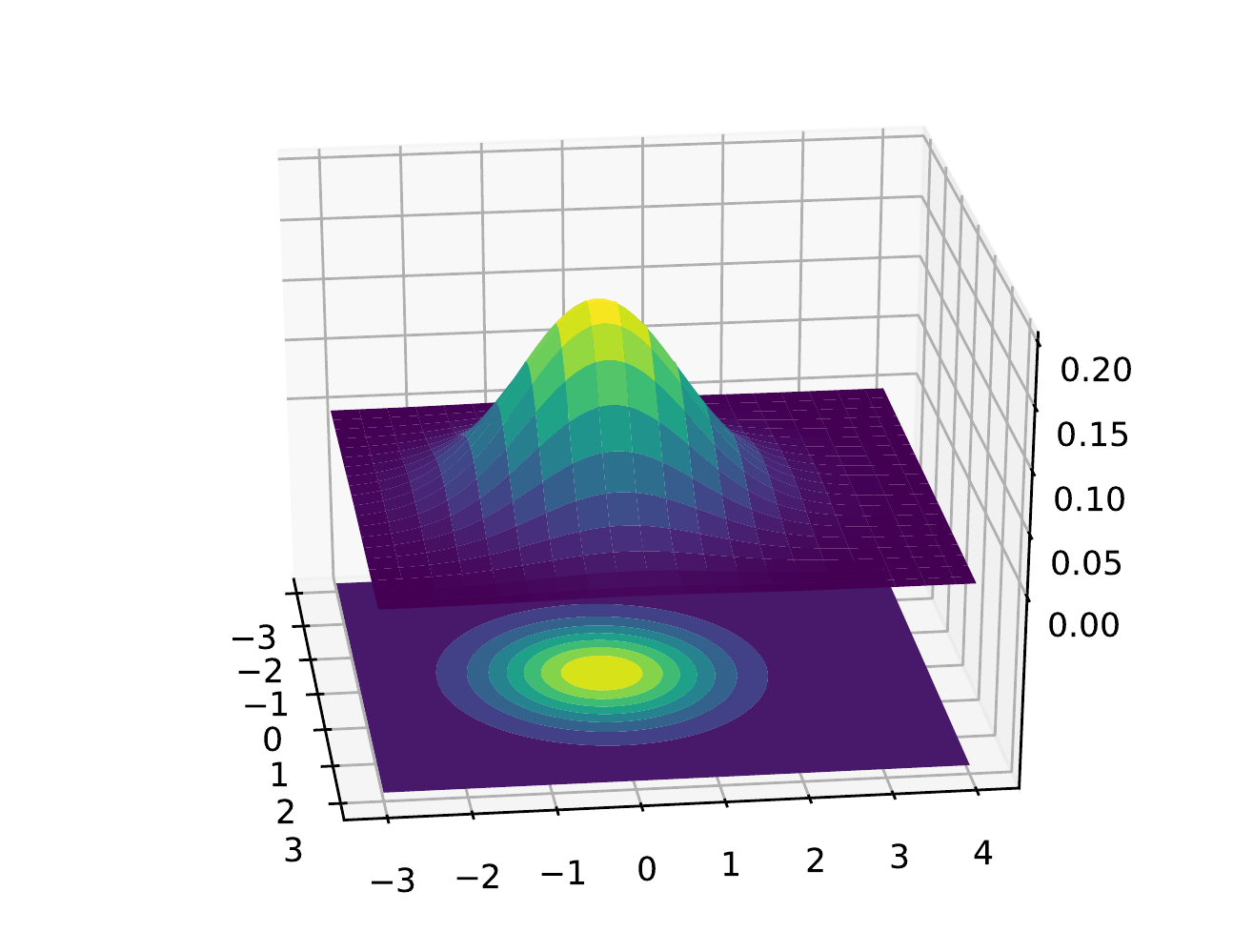} \label{fig:student200}}
	\subfigure[Diff between (a) and (d)]{\includegraphics[width=0.31
		\textwidth]{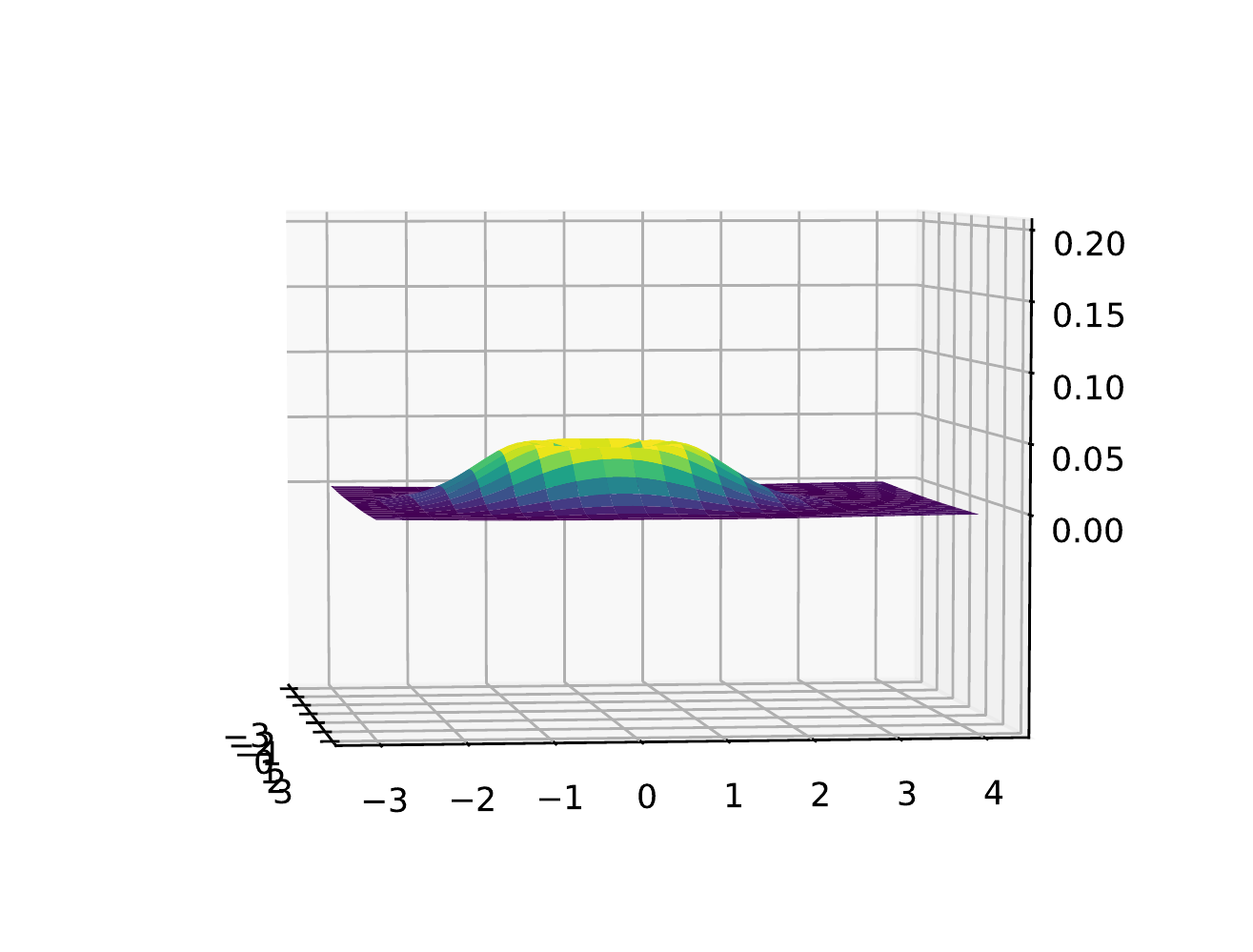} \label{fig:gauss-stu-diff1}}
	\subfigure[Diff between (a) and (e)]{\includegraphics[width=0.31
		\textwidth]{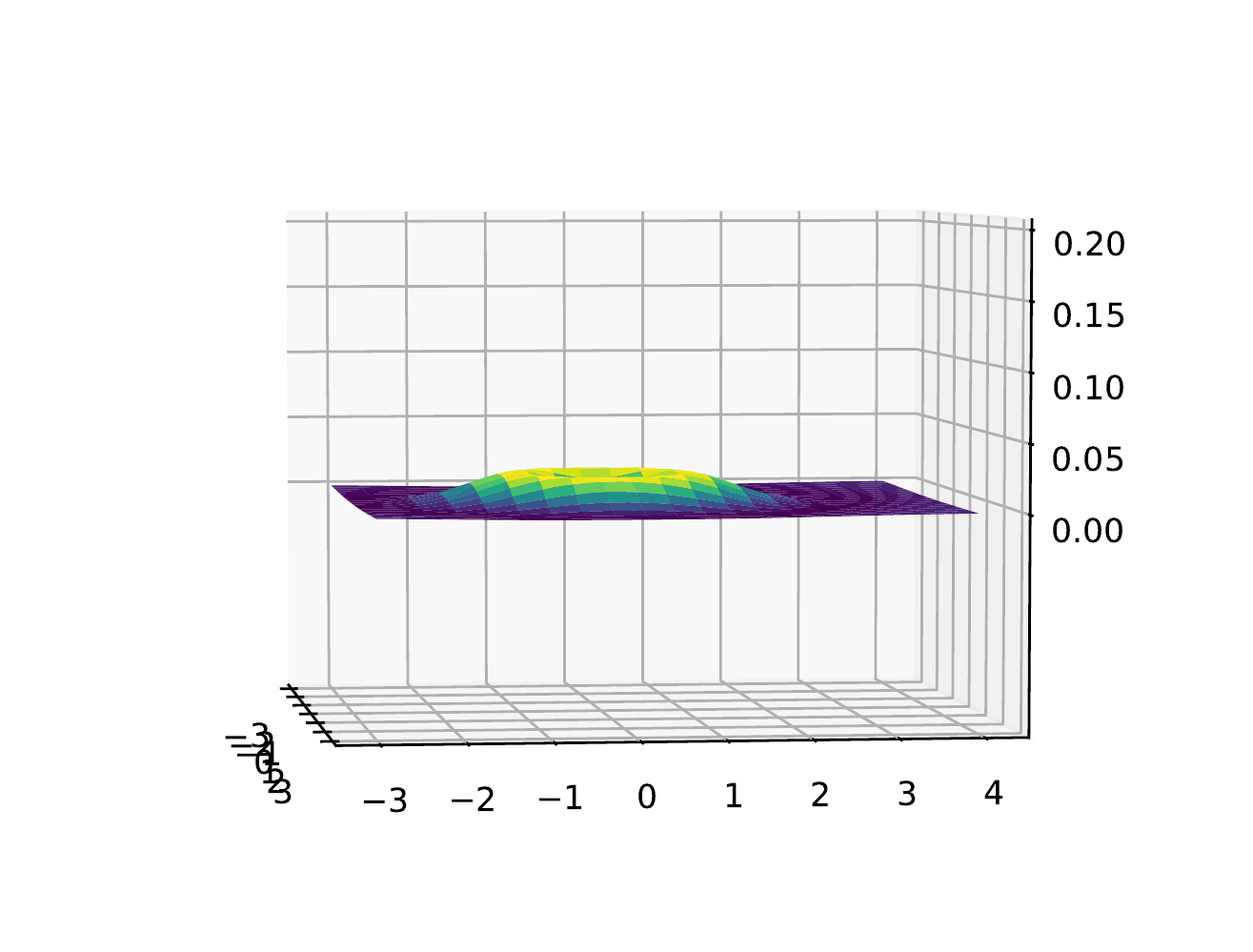} \label{fig:gauss-stu-diff3}}
	\subfigure[Diff between (a) and (f)]{\includegraphics[width=0.31
		\textwidth]{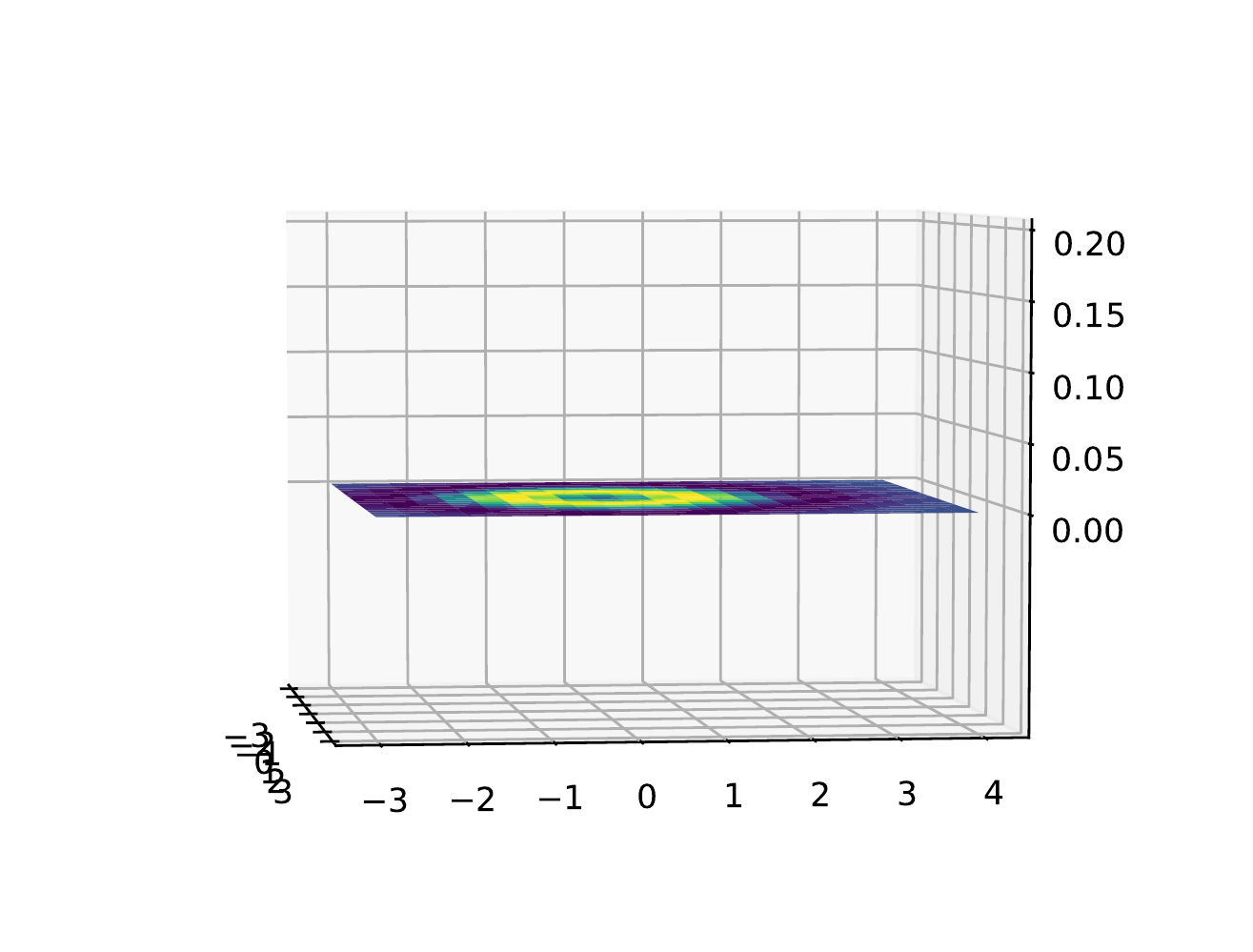} \label{fig:gauss-stu-diff200}}
	\centering
	\caption{Density and contour plots (blue=low, yellow=high) for the multivariate Gaussian distribution and multivariate Student $t$ distribution over the $\mathbb{R}^2$ space for various values of the covariance/scale matrix with zero mean vector.  Fig~\ref{fig:gauss-diagonal}: A spherical covariance matrix has a circular shape; Fig~\ref{fig:gauss-spherical}: A diagonal covariance matrix is an \textbf{axis aligned} ellipse; Fig~\ref{fig:gauss-full}: A full covariance matrix has a elliptical shape; \\
		Fig~\ref{fig:student-1} to Fig~\ref{fig:student200} for Student $t$ distribution with same scale matrix and increasing $\nu$ such that the difference between (a) and (f) in Fig~\ref{fig:gauss-stu-diff200} is approaching to zero.}\centering
	\label{fig:studentt_densitys-1}
\end{figure}
\begin{svgraybox}
\begin{definition}[Multivariate Student $t$ Distribution]\label{definition:multivariate-stu-t}
A random vector $\bx$ is said to follow the multivariate Student's $t$ distribution with parameter $\bmu$, $\bSigma$, and $\nu$ if
$$
\begin{aligned}
\tau(\bx| \bmu, \bSigma, \nu)&= \frac{\Gamma(\nu/2 + D/2)}{\Gamma(\nu/2)} \frac{|\bSigma|^{-1/2}}{\nu^{D/2} \pi^{D/2}} \times \left[ 1+ \frac{1}{\nu} (\bx-\bmu)^\top \bSigma^{-1} (\bx-\bmu)  \right]^{-(\frac{\nu+D}{2})}\\
&= \frac{\Gamma(\nu/2 + D/2)}{\Gamma(\nu/2)} |\pi\bV|^{-1/2} \times \left[ 1+ \frac{1}{\nu} (\bx-\bmu)^\top \bV^{-1} (\bx-\bmu)  \right]^{-(\frac{\nu+D}{2})},
\end{aligned}
$$
where $\bSigma$ is called the \textbf{scale matrix} and $\bV=\nu\bSigma$, and $\nu$ is the \textbf{degree of freedom}. This distribution has fatter tails than a Gaussian one. The smaller the $\nu$ is, the fatter the tails. As $\nu \rightarrow \infty$, the distribution converges towards a Gaussian.
The mean, mode, and covariance of the multivariate Student's $t$ distribution are given by 
\begin{equation*}
\begin{aligned}
\Exp [\bx] &= \bmu,\\
\mathrm{Mode}[\bx] &= \bmu, \\
\Cov [\bx] &= \frac{\nu}{\nu-2}\bSigma. 
\end{aligned}
\end{equation*}
Note that the $\bSigma$ is called the scale matrix since it is not exactly the covariance matrix as that in multivariate Gaussian distribution. 

Specifically, When $D=1$, it follows that
\begin{equation}\label{equation:uni-stu-nonzero}
\begin{aligned}
	\tau(x| \mu, \sigma^2, \nu)&= \frac{\Gamma(\frac{\nu+1}{2})}{\Gamma(\frac{\nu}{2})} \frac{1}{\sigma\sqrt{\nu\pi}} \times \left[ 1+ \frac{(x-\mu)^2}{\nu \sigma^2}   \right]^{-(\frac{\nu+1}{2})}.
\end{aligned}
\end{equation}
When $D=1, \bmu=0, \bSigma=1$, then the p.d.f., defines the \textbf{univariate $t$ distribution}.
\begin{equation*}
\begin{aligned}
\tau(x| , \nu)&= \frac{\Gamma(\frac{\nu+1}{2})}{\Gamma(\frac{\nu}{2})} \frac{1}{\sqrt{\nu\pi}} \times \left[ 1+ \frac{x^2}{\nu }   \right]^{-(\frac{\nu+1}{2})}.
\end{aligned}
\end{equation*}
\end{definition}
\end{svgraybox}
Figure~\ref{fig:studentt_densitys-1} compares the Gaussian and the Student $t$ distribution for various values such that when $\nu\rightarrow \infty$, the difference between the densities is approaching to zero. For same parameters in the densities, Student $t$ in general has longer ``tails" than a Gaussian which can be seen from the comparison between Figure~\ref{fig:gauss-diagonal} and Figure~\ref{fig:student-1}. This gives the Student $t$ distribution an important property called \textbf{robustness}, which means that it is much less sensitive than the Gaussian to the presence of a few data points which are outliers \citep{bishop2006pattern, murphy2012machine}.

A Student $t$ distribution can be written as a \textbf{Gaussian scale mixture}
\begin{equation}\label{equation:gauss-scale-mixture}
\tau(\bx| \bmu, \bSigma, \nu) = \int_0^{\infty} \normal(\bx | \bmu, \bSigma/z)\cdot \gammadist(z|\frac{\nu}{2}, \frac{\nu}{2})dz.
\end{equation}
This can be thought of as an ``infinite'' mixture of Gaussians, each with a slightly different covariance matrix. That is, Student $t$ distribution is obtained by adding up an
infinite number of Gaussian distributions having the same mean vector but different precision matrices. 
From this Gaussian scale mixture view, when $\nu \rightarrow \infty$, the gamma distribution becomes a degenerate random variable with all the non-zero mass at the point unity such that the multivariate Student $t$ distribution converges to multivariate Gaussian distribution.


\subsection{Prior on parameters of multivariate Gaussian distribution}
In Section~\ref{section:blm-fullconjugate}, we have shown that the inverse-gamma distribution is a conjugate prior to the magnitude of covariance matrix of multivariate Gaussian distribution. A generalization to this is the inverse-Wishart distribution which is a conjugate prior to the full covariance matrix of multivariate Gaussian distribution. That is, the inverse-Wishart distribution is a probability distribution of random positive definite matrices that is used to model random covariance matrices.

Before the discussion about inverse-Wishart distribution. We shall notice that it derives from the Wishart distribution. \citep{anderson1962introduction} has said ``The Wishart distribution ranks next to the (muiltivariate) normal distribution in order of importance and usefulness in multivariate statistics".
\begin{svgraybox}
	\begin{definition}[Wishart Distribution]
		A random symmetric positive definite matrix $\bLambda\in \real^{D\times D}$ is said to follow the Wishart distribution with parameter $\bM_0$ and $\nu_0$ if
		$$
		\begin{aligned}
			\mathrm{Wi}(\bLambda| \textcolor{red}{\bM_0}, \nu_0)&= |\bLambda|^{\textcolor{blue}{\frac{\nu_0-D-1}{2}}} \exp\left(-\frac{1}{2}\tr(\textcolor{blue}{\bLambda} \textcolor{red}{\bM_0^{-1}})\right)\\
			&\gap \times \left[2^{\frac{\nu_0 D}{2}}  \pi^{D(D-1)/4}  \textcolor{red}{|\bM_0|^{\nu_0/2}  } \prod_{d=1}^D\Gamma(\frac{\nu_0+1-d}{2}) \right]^{-1},
		\end{aligned}
		$$
		where $\nu_0 > D$ and $\bM_0$ is a $D\times D$ symmetric positive definite matrix, and $|\bLambda| = \det(\bLambda)$.
		The $\nu_0$ is called the \textbf{number of degrees of freedom}, and $\bM_0$ is called the \textbf{scale matrix}.
		And it is denoted by $\bLambda \sim \mathrm{W}(\bM_0, \nu_0)$.
		The mean and variance of Wishart distribution are given by 
		$$
		\begin{aligned}
			\E [\bLambda] &= \nu_0 \bM_0,\\
			\Var[\bLambda_{i,j}] &= \nu_0 (m_{i,j}^2 + m_{i,i}m_{j,j}),
			\label{equation:wishart_expectation}
		\end{aligned}
		$$
		where $m_{i,j}$ is the $i$-th row $j$-th column element of $\bM_0$.
	\end{definition}
\end{svgraybox}
An interpretation of the Wishart distribution is as follows. Suppose we sample i.i.d., $\bz_1, \bz_2, \ldots, \bz_{\nu_0}$ from $\normal(\bzero, \bM_0)$. The sum of squares matrix of the collection of multivariate vectors is given by 
$$
\sum_{i=1}^{\nu_0} \bz_i\bz_i^\top = \bZ^\top\bZ,
$$
where $\bZ$ is the $\nu_0 \times D$ matrix whose $i$-th row is $\bz_i^\top$. It is trivial that $\bZ^\top\bZ$ is positive semidefinite and symmetric. If $\nu_0 >D$ and the $\bz_i$'s are linearly independent, then $\bZ^\top \bZ$ will be positive definite and symmetric. That is $\bZ\bx=\bzero$ only happens when $\bx=\bzero$. We can repeat over and over again, generating matrices $\bZ_1^\top\bZ_1, \bZ_2^\top\bZ_2, \ldots, \bZ_l^\top\bZ_l$. The population distribution of these matrices has a Wishart distribution with parameters $(\bM_0, \nu_0)$. By definition, 
$$
\begin{aligned}
\bLambda&=\bZ^\top\bZ = \sum_{i=1}^{\nu_0} \bz_i\bz_i^\top \\
\Exp[\bLambda]&=\Exp[\bZ^\top\bZ] = \Exp\left[\sum_{i=1}^{\nu_0} \bz_i\bz_i^\top\right] = \nu_0 \Exp[\bz_i\bz_i^\top] = \nu_0\bM_0. \\
\end{aligned}
$$

When $D=1$, this reduces to the case that if $z$ is a mean-zero univariate normal random variable, then $z^2$ is a Gamma random variable. To be specific, 
$$
\mathrm{suppose } \qquad  z \sim \normal(0, a) , \qquad \mathrm{then } \qquad z^2\sim \gammadist(a/2, 1/2).
$$
Just like the relationship between inverse-Gamma distribution and Gamma distribution that if $x \sim \gammadist(r, \lambda)$, then $y=\frac{1}{x} \sim \inversegammadist(r, \lambda)$. There is a similar connection between the inverse-Wishart distribution and Wishart-distribution.

Since we often use the inverse-Wishart distribution as a prior distribution for a covariance matrix, it is often useful to replace $\bM_0$ in the Wishart distribution by $\bso=\bM_0^{-1}$. This results in that
A random $D\times D$ symmetric positive definite matrix $\bSigma$ has an $\mathrm{IW}(\bSigma| \bso, \nu_0)$ distribution if $\bSigma^{-1}=\bLambda$ has a Wishart $\mathrm{Wi}(\bLambda| \bM_0, \nu_0)$ distribution. \footnote{However, if we do not replace $\bM_0$ by $\bso$, the relationship will be: A random $D\times D$ symmetric positive definite matrix $\bSigma$ has an $\mathrm{IW}(\bSigma| \textcolor{blue}{\bM_0}, \nu_0)$ distribution if $\bSigma^{-1}$ has a Wishart $\mathrm{Wi}(\bLambda| \bM_0, \nu_0)$ distribution.}
\begin{svgraybox}
\begin{definition}[Inverse-Wishart Distribution]
A random symmetric positive definite matrix $\bSigma\in \real^{D\times D}$ is said to follow the inverse-Wishart distribution with parameter $\bso$ and $\nu_0$ if
$$
\begin{aligned}
\mathrm{IW}(\bSigma| \textcolor{red}{\bso}, \nu_0)&= |\bSigma|^{\textcolor{blue}{-\frac{\nu_0+D+1}{2}}} \exp\left(-\frac{1}{2}\tr(\textcolor{blue}{\bSigma^{-1}} \textcolor{red}{\bso})\right)\\
&\gap \times \left[2^{\frac{\nu_0 D}{2}}  \pi^{D(D-1)/4}  \textcolor{red}{|\bso|^{-\nu_0/2}}   \prod_{d=1}^D\Gamma(\frac{\nu_0+1-d}{2}) \right]^{-1},
\end{aligned}
$$
where $\nu_0 > D$ and $\bso$ is a $D\times D$ symmetric positive definite matrix, and $|\bSigma| = \det(\bSigma)$.
The $\nu_0$ is called the \textbf{number of degrees of freedom}, and $\bso$ is called the \textbf{scale matrix}.
And it is denoted by $\bSigma \sim \mathrm{IW}(\bso, \nu_0)$.
The mean and mode of inverse-Wishart distribution is given by 
\begin{equation}
\begin{aligned}
\E [\bSigma ^{-1}] &= \nu_0 \bso^{-1}=\nu_0 \bM_0,\\
\E [\bSigma] &= \frac{1}{\nu_0 - D - 1} \bso, \\
\mathrm{Mode}[\bSigma] &= \frac{1}{\nu_0 + D + 1} \bso.
\label{equation:iw_expectation}
\end{aligned}
\end{equation}
Note that, sometimes, we replace $\bso$ by $\bM_0=\bso^{-1}$ such that $\E [\bSigma ^{-1}] = \nu_0 \bM_0$ which does not involve inverse of the matrix.

When $D=1$, the inverse-Wishart distribution reduces to the inverse Gamma such that $\frac{\nu_0}{2} = r$ and $\frac{S_0}{2}=\lambda$, see Definition~\ref{definition:inverse-gamma}:
$$
\mathrm{IW}(y| S_0, \nu_0) = \inversegammadist(y|r, \lambda).
$$
\end{definition}
\end{svgraybox}

Note that the Wishart density is not simply the inverse-Wishart density with
$\bSigma$ replaced by $\bLambda = \bSigma^{-1}$. There is an additional factor of $|\bSigma|^{-(D+1)}$. See \citep{anderson1962introduction} Theorem 7.7.1 that the Jacobian of the transformation $\bLambda = \bSigma^{-1}$ is $|\bSigma|^{-(D+1)}$. Substitution of $\bSigma^{-1}$ in the definition of Wishart distribution and multiply by $|\bSigma|^{-(D+1)}$ can yield the inverse-Wishart distribution.
 \footnote{Which is from the Jacobian in the change-of-variables formula. A short proof is provided here. Let $\bLambda = g(\bSigma)=\bSigma^{-1}$ where $\bSigma\sim \inversewishart(\bso, \nu_0)$ and $\bLambda\sim \wishartdist(\bso, \nu_0)$. Then, $f(\bSigma)  = f(\bLambda) |J_g|$ where $J_g$ is the Jacobian matrix which results in $f(\bSigma) = f(\bLambda) |J_g| = f(\bLambda)|\bSigma|^{-(D+1)} $. }

We will see that a sample drawn from a normal-inverse-Wishart distribution gives a mean vector and a covariance matrix which can define a multivariate gaussian distribution. Separately, we can first sample a matrix $\bSigma$ from an inverse-Wishart distribution parameterized by ($\bS_0, \nu_0$, $\bmu$) which is called a semi-conjugate prior, and then sample a mean vector from a Gaussian distribution parameterized by ($\bmm_0, \bV_0, \bSigma$).

\subsection{Posterior distribution of $\bmu$: Separated view}
Suppose the covariance matrix $\bSigma$ is known in Equation~\eqref{equation:multi_gaussian_likelihood}, the likelihood is
$$
\begin{aligned}
\mathrm{\textbf{likelihood}}&=p(\mathcalX | \bmu) =\normal(\mathcalX| \bmu, \bSigma)\\
&= (2\pi)^{-ND/2} |\bSigma|^{-N/2}\exp\left(-\frac{1}{2} \sum^N_{n=1}(\bxn - \bmu)^\top \bSigma^{-1}(\bxn - \bmu)\right)\\
&\propto \exp\left( N\overline{\bx}^\top \bSigma^{-1} \bmu - \frac{1}{2}N \bmu^\top \bSigma^{-1}\bmu    \right).
\end{aligned}
$$
The conjugate prior of the mean vector is also a Gaussian $p(\bmu)= \normal(\bmu | \bmo, \bV_0)$. 
$$
\begin{aligned}
\mathrm{\textbf{prior}}&=p(\bmu)= \normal(\bmu | \bmo, \bV_0)\\
&= (2\pi)^{-D/2} |\bV_0|^{-1/2}\exp\left(-\frac{1}{2} (\bmu - \bmo)^\top \bV_0^{-1}(\bmu - \bmo)\right) \\
&= (2\pi)^{-D/2} |\bV_0|^{-1/2}\exp \left( -\frac{1}{2}\bmu^\top\bV_0^{-1}\bmu + \bmu^\top \bV_0^{-1}\bmo - \frac{1}{2} \bmo^\top\bV_0^{-1} \bmo     \right)\\
&\propto  \exp \left( -\frac{1}{2}\bmu^\top\bV_0^{-1}\bmu + \bmu^\top \bV_0^{-1}\bmo\right).
\end{aligned}
$$
By the Bayes' theorem ``$\mathrm{posterior} \propto \mathrm{likelihood} \times \mathrm{prior} $", we can derive a Gaussian posterior for $\bmu$:
$$
\begin{aligned}
\mathrm{\textbf{posterior}}&=p(\bmu | \mathcalX, \bSigma) \propto p(\mathcalX | \bmu ) \times  p(\bmu)\\
&=\exp\left( N\overline{\bx}^\top \bSigma^{-1} \bmu - \frac{1}{2}N \bmu^\top \bSigma^{-1}\bmu    \right) \times \exp \left( -\frac{1}{2}\bmu^\top\bV_0^{-1}\bmu + \bmu^\top \bV_0^{-1}\bmo\right)\\
&= \exp\left(  -\frac{1}{2}\bmu^\top(\bV_0^{-1} + N\bSigma^{-1})\bmu + \bmu^\top( \bV_0^{-1}\bmo +  N\bSigma^{-1}\overline{\bx} )   \right) \\
&\propto \normal(\bmu | \bmm_N, \bV_N) 
\end{aligned}
$$
where $\bV_N^{-1} = \bV_0^{-1} + N\bSigma^{-1}$, and $\bmm_N = \bV_N ( \bV_0^{-1}\bmo +  N\bSigma^{-1}\overline{\bx} )$. In which case, the posterior precision matrix is the sum of the prior precision matrix and data precision matrix. 
By letting $\bV_0 \rightarrow \infty \bI$, we can model an uninformative prior such that the posterior distribution of the mean is $p(\bmu | \mathcalX, \bSigma) =\normal(\bmu | \overline{\bx}, \frac{1}{N}\bSigma)$.

\subsection{Posterior distribution of $\bSigma$: Separated view}
Suppose the mean vector $\bmu$ is known in Equation~\eqref{equation:multi_gaussian_likelihood}, the likelihood is
\begin{equation*}
\begin{aligned}
\mathrm{\textbf{likelihood}}= p(\mathcal{X} | \bmu, \bSigma) =\prod^N_{n=1} \mathcal{N} (\bxn| \bmu, \bSigma) = (2\pi)^{-ND/2} |\bSigma|^{-N/2}\exp\left(-\frac{1}{2} \tr(\bSigma^{-1}\bS_{\bmu} )  \right).
\end{aligned}
\end{equation*}
The corresponding conjugate prior is the inverse-Wishart distribution:
$$
\begin{aligned}
\mathrm{\textbf{prior}}=	\mathrm{IW}(\bSigma| \bso, \nu_0)&= |\bSigma|^{-\frac{\nu_0+D+1}{2}} \exp\left(-\frac{1}{2}\tr(\bSigma^{-1} \bso)\right)\\
	&\gap \times \left[2^{\frac{\nu_0 D}{2}}  \pi^{D(D-1)/4}  |\bso|^{-\nu_0/2}   \prod_{d=1}^D\Gamma(\frac{\nu_0+1-d}{2}) \right]^{-1}.
\end{aligned}
$$
By the Bayes' theorem ``$\mathrm{posterior} \propto \mathrm{likelihood} \times \mathrm{prior} $", we can derive a inverse-Wishart posterior for $\bSigma$:
$$
\begin{aligned}
\mathrm{\textbf{posterior}}  &=p(\bSigma | \mathcalX, \bmu) \propto p(\mathcalX | \bSigma ) \times  p(\bSigma)\\
&\propto |\bSigma|^{-N/2}\exp\left(-\frac{1}{2} \tr(\bSigma^{-1}\bS_{\bmu} )  \right) \times |\bSigma|^{-\frac{\nu_0+D+1}{2}} \exp\left(-\frac{1}{2}\tr(\bSigma^{-1} \bso)\right) \\
&= |\bSigma|^{-\frac{\nu_0+N+D+1}{2}} \exp\left(-\frac{1}{2}\tr(\bSigma^{-1} [\bso+\bS_{\bmu}])\right)\\
&\propto \inversewishart(\bSigma |\bso+\bS_{\bmu}, \nu_0+N).
\end{aligned}
$$
The posterior degree of freedom is the prior degree of freedom $\nu_0$ plus the number of observations $N$. And the posterior scale matrix is the prior scale matrix $\bso$ plus the data scale matrix $\bS_{\bmu}$. 
The mean of the posterior $\bSigma$ is given by
$$
\begin{aligned}
\Exp[\bSigma] &= \frac{1}{\nu_0 +N- D - 1} (\bso+\bS_{\bmu}) \\
&= \frac{\nu_0 -D-1}{\nu_0 +N- D - 1} \cdot (\frac{1}{\nu_0 -D-1} \bso) +    \frac{N}{\nu_0 +N- D - 1} \cdot (\frac{1}{N}\bS_{\bmu})\\
&= \lambda \cdot (\frac{1}{\nu_0 -D-1} \bso) +   (1-\lambda) \cdot (\frac{1}{N}\bS_{\bmu}),
\end{aligned}
$$
where $\lambda=\frac{\nu_0 -D-1}{\nu_0 +N- D - 1}$, $(\frac{1}{\nu_0 -D-1} \bso)$ is the prior mean of $\bSigma$, and $(\frac{1}{N}\bS_{\bmu})$ is an unbiased estimator of the covariance such that $(\frac{1}{N}\bS_{\bmu})$ converges to the true population covariance matrix. Thus, the posterior mean of the covariance matrix can be seen as the weighted average of the prior expectation and the unbiased estimator. The unbiased estimator can also be shown to be equal to the maximum likelihood estimator (MLE) of $\bSigma$. As $N\rightarrow \infty$, it can be shown that the posterior expectation of $\bSigma$ is a consistent \footnote{An estimator $\hat{\theta}_N$ of $\theta$ constructed on the basis of a sample of size $N$ is said to consistent if $\hat{\theta}_N \stackrel{p}{\longrightarrow} \theta$ as $N\rightarrow \infty$. See also \citep{lu2021rigorous}.} estimator of the population covariance. When we set $\nu_0=D+1$, $\lambda=0$ and we recover the MLE.

Similarly, the mode of the posterior $\bSigma$ is given by
\begin{equation}\label{equation:map-covariance-multigauss}
\begin{aligned}
	\mathrm{Mode}[\bSigma] &= \frac{1}{\nu_0 +N+ D + 1} (\bso+\bS_{\bmu})\\
	&= \frac{\nu_0+D+1}{\nu_0 +N+ D + 1} (\frac{1}{\nu_0+D+1} \bso)+  \frac{N}{\nu_0 +N+ D + 1} (\frac{1}{N}  \bS_{\bmu})\\
	&=\beta (\frac{1}{\nu_0+D+1} \bso)+  (1-\beta)(\frac{1}{N}  \bS_{\bmu}),
\end{aligned}
\end{equation}
where $\beta=\frac{\nu_0+D+1}{\nu_0 +N+ D + 1}$, and $(\frac{1}{\nu_0+D+1} \bso)$ is the prior mode of $\bSigma$. The posterior mode is a weighted average of the prior mode and the unbiased estimator. Again, the maximum a posterior (MAP) estimator in Equation~\eqref{equation:map-covariance-multigauss} is a consistent estimator.

\subsection{Gibbs sampling of the mean and covariance: Separated view}
The separated view here is known as a semi-conjugate prior on the mean and covariance of multivariate Gaussian distribution since both conditionals, $p(\bmu|\mathcalX,\bSigma)$ and $p(\bSigma|\mathcalX,\bmu)$, are individually conjugate.
In last two sections, we have shown 
$$
\begin{aligned}
\bmu | \mathcalX, \bSigma &\sim \normal(\bmm_N, \bV_N),\\
\bSigma | \mathcalX, \bmu &\sim \inversewishart(\bso+\bS_{\bmu}, \nu_0+N).
\end{aligned}
$$
The two full conditional distributions can be used to construct a Gibbs sampler. The Gibbs sampler generates the mean and covariance $\{\bmu^{t+1}, \bSigma^{t+1}\}$ in step $t+1$ from $\{\bmu^{t}, \bSigma^{t}\}$ in step $t$ via the following two steps:

1. Sample $\bmu^{t+1}$ from its full conditional distribution: $\bmu^{t+1} \sim \normal(\bmm_N, \bV_N)$, where $(\bmm_N, \bV_N)$ depend on $\bSigma^{t}$.

2. Sample $\bSigma^{t+1}$ from its full conditional distribution: $\bSigma^{t+1} \sim \inversewishart(\bso+\bS_{\bmu}, \nu_0+N)$, where $(\bso+\bS_{\bmu}, \nu_0+N)$ depend on $\bmu^{t+1}$.

\subsection{Posterior distribution of $\bmu$ and $\bSigma$ under NIW: Unified view}\label{sec:niw_posterior_conjugacy}

\paragraph{Likelihood}
The likelihood of $N$ random observations $\mathcal{X} = \{\bx_1, \bx_2, \ldots , \bx_N \}$ being generated by a multivariate Gaussian with mean vector $\bmu$ and covariance matrix $\bSigma$ is given by Equation~\eqref{equation:multi_gaussian_likelihood}
$$
\begin{aligned}
p(\mathcal{X} | \bmu, \bSigma)=  \frac{1}{(2\pi)^{ND/2}} |\bSigma|^{-N/2}\exp\left(-\frac{N}{2}(\bmu - \overline{\bx})^\top \bSigma^{-1}(\bmu - \overline{\bx})-\frac{1}{2}\tr(\bSigma^{-1} \bS_{\overline{x}})\right).
\end{aligned}
$$
\paragraph{Prior}
A trivial prior is to combine the conjugate priors for $\bmu$ and $\bSigma$ respectively in the above sections:
$$
p(\bmu, \bSigma) = \normal(\bmu | \bmo, \bV_0)\cdot \inversewishart(\bSigma | \bso, \nu_0).
$$
However, this is not a conjugate prior to the likelihood with parameters $\bmu, \bSigma$ since $\bmu$ and $\bSigma$ appear together in a non-factorized way in the likelihood.
For the full parameters of a multivariate Gaussian distribution (i.e., mean vector $\bmu$ and covariance matrix $\bSigma$), the normal-inverse-Wishart (NIW) prior is fully conjugate and defined as follows:
\begin{equation}
	\begin{aligned}
&\gap \niw (\bmu, \bSigma| \bmo, \kappa_0, \nu_0, \bso) \\
		&\triangleq \mathcal{N}(\bmu| \bmo, \frac{1}{\kappa_0}\bSigma) \cdot  \mathrm{IW}(\bSigma| \bso, \nu_0) \\
		&=\frac{1}{Z_{\niw}(D, \kappa_0, \nu_0, \bso)} |\bSigma|^{-1/2}\exp\left(\frac{\kappa_0}{2}(\bmu - \bmo)^\top\bSigma^{-1}(\bmu - \bmo)\right) \\
		&\gap \times |\bSigma|^{-\frac{\nu_0+D+1}{2}} \exp\left(-\frac{1}{2}\tr(\bSigma^{-1} \bso)\right) \\
		&= \frac{1}{Z_{\niw}(D, \kappa_0, \nu_0, \bso)} |\bSigma|^{-\frac{\nu_0+D+2}{2}}\\
		&\gap \times \exp\left(-\frac{\kappa_0}{2}(\bmu - \bmo)^\top\bSigma^{-1}(\bmu - \bmo) -\frac{1}{2}\tr(\bSigma^{-1} \bso)\right), 
	\end{aligned}
	\label{equation:multi_gaussian_prior}
\end{equation}
where
\begin{equation}
	Z_{\niw}(D, \kappa_0, \nu_0, \bso) =  2^{\frac{(\nu_0+1)D}{2}} \pi^{D(D+1)/4} \kappa_0^{-D/2} | \bso|^{-\nu_0/2}\prod_{d=1}^D\Gamma(\frac{\nu_0+1-d}{2}).
	\label{equation:multi_gaussian_giw_constant}
\end{equation}
The specific form of the normalization term $Z_{\niw}(D, \kappa_0, \nu_0, \bso)$ will be useful to show the posterior marginal likelihood of the data in Section~\ref{section:posterior-marginal-of-data}.

\paragraph{A ``prior" interpretation for the NIW prior}
The inverse-Wishart distribution will ensure that the resulting covariance matrix is positive definite when $\nu_0 > D$. And if we are confident that the true covariance matrix is near some covariance matrix $\bSigma_0$, then we might choose $\nu_0$ to be large and set $\bso = (\nu_0 - D - 1) \bSigma_0$, making the distribution of the covariance matrix $\bSigma$ concentrated around $\bSigma_0$. On the other hand, choosing $\nu_0 = D+2$ and $\bso = \bSigma_0$ will make $\bSigma$ loosely concentrated around $\bSigma_0$.
More details can be referred to \citep{chipman2001practical, fraley2007bayesian, hoff2009first, murphy2012machine}.


An intuitive interpretation of the hyper-parameters \citep{murphy2012machine, hoff2009first}: $\bmo$ is our prior
mean for $\bmu$, $\kappa_0$ is how strongly we believe this prior for $\bmu$ (the larger the stronger we believe this prior mean), $\bso$ is proportional to our prior mean for $\bSigma$, and $\nu_0$ controls how strongly we believe this prior for $\bSigma$. Because the Gamma function is not defined for negative integers and zero, from Equation~\eqref{equation:multi_gaussian_giw_constant} we require $\nu_0 > D - 1$ (which also can be shown from the expectation of the covariance matrix Equation~\eqref{equation:iw_expectation}. And also $\bso$ needs to be a positive definite matrix, where an intuitive reason can be shown from Equation~\eqref{equation:iw_expectation}. A more detailed reason can be found in \citep{hoff2009first}.

\paragraph{Posterior}
By the Bayes' theorem ``$\mathrm{posterior} \propto \mathrm{likelihood} \times \mathrm{prior} $", the posterior of the $\bmu$ and $\bSigma$ parameters under the NIW prior is
\begin{equation}
p(\bmu, \bSigma| \mathcalX, \bbeta ) \propto p(\mathcalX | \bmu, \bSigma) p(\bmu, \bSigma | \bbeta) = p(\mathcalX, \bmu, \bSigma | \bbeta),  
\label{equation:niw_full_posterior}
\end{equation}
where $\bbeta=(\bmo, \kappa_0, \nu_0, \bso)$ are the hyperparameters and the right hand side of Equation~\eqref{equation:niw_full_posterior} is also known as the full joint distribution $p(\mathcalX, \bmu, \bSigma | \bbeta)$, and is given by 
\begin{equation}
\begin{aligned}
p(\mathcalX, \bmu, \bSigma | \bbeta)&=p(\mathcalX | \bmu, \bSigma) \cdot p(\bmu, \bSigma | \bbeta) \\
&= C\times  |\bSigma|^{- \frac{\nu_0+N+D+2}{2}} \times\\
&\gap \exp \Bigg\{ -\frac{N}{2}(\bmu - \overline{\bx})^\top \bSigma^{-1} (\bmu - \overline{\bx}) - \frac{\kappa_0}{2} (\bmu - \bmo)^\top \bSigma^{-1}(\bmu - \bmo) \\
&\gap -\frac{1}{2} \tr(\bSigma^{-1} \bS_{\overline{x}}) - \frac{1}{2} \tr(\bSigma^{-1} \bS_0)  \Bigg\},\\
\end{aligned}
\label{equation:niw_full_joint}
\end{equation}
where $C =\frac{(2\pi)^{-ND/2}}{Z_{\niw}(D, \kappa_0, \nu_0, \bS_0)}$ is a constant normalization term.
This can be reduced to 
\begin{equation}
\begin{aligned}
&\gap p(\mathcalX, \bmu, \bSigma | \bbeta)\\ 
&=C|\bSigma|^{- \frac{\nu_0+N+D+2}{2}} \times \\
&\gap \exp \Bigg\{-\frac{\kappa_0+N}{2} \left(\bmu - \frac{\kappa_0 \bmo+N \overline{\bx}}{\kappa_N} \right)^\top \bSigma^{-1} \left(\bmu - \frac{\kappa_0 \bmo+N \overline{\bx}}{\kappa_N} \right) \\
&\gap - \frac{1}{2} \tr \left[\bSigma^{-1} \left( \bS_0 + \bS_{\overline{x}} + \frac{\kappa_0 N}{\kappa_0 + N} (\overline{\bx} - \bmo)(\overline{\bx}-\bmo)^\top \right) \right] \Bigg\}, 
\end{aligned}
\label{equation:niw_full_joint2}
\end{equation}
which is calculated to compare with the NIW form in Equation~\eqref{equation:multi_gaussian_prior}, and we can see the reason why we rewrite the multivariate Gaussian distribution into Equation~\eqref{equation:multi_gaussian_identity} by the trace trick. It follows that the posterior is also a NIW density with updated parameters and gives the view of conjugacy for multivariate Gaussian distribution:
\begin{equation}
p(\bmu, \bSigma| \mathcalX , \bbeta) = \niw (\bmu, \bSigma | \bmm_N, \kappa_N, \nu_N, \bS_N), \label{equation:niw_posterior_equation_1}
\end{equation}
where 
\begin{align}
\bmm_N &= \frac{\kappa_0\bmo + N\overline{\bx}}{\kappa_N} = 
\frac{\kappa_0 }{\kappa_N}\bmo+\frac{N}{\kappa_N}\overline{\bx}  \label{equation:niw_posterior_equation_2}\\
\kappa_N        &= \kappa_0 + N  \label{equation:niw_posterior_equation_3}\\
\nu_N           &=\nu_0 + N  \label{equation:niw_posterior_equation_4}\\
\bS_N  		&=\bS_0 + \bS_{\overline{x}} + \frac{\kappa_0N}{\kappa_0 + N}(\overline{\bx} - \bmo)(\overline{\bx} - \bmo)^\top \label{equation:niw_posterior_equation_5}\\
&=\bS_0 + \sum_{n=1}^N \bx_n \bx_n^\top + \kappa_0 \bmo \bmo^\top - \kappa_N \bmm_N \bmm_N^\top . \label{equation:niw_posterior_equation_6}
\end{align}
\paragraph{A ``posterior" interpretation for the NIW prior}
An intuitive interpretation for the parameters in NIW can be obtained from the updated parameters above. $\nu_0$ is the prior number of samples to observe the covariance matrix, and $\nu_N =\nu_0 + N$ is the posterior number of samples. The posterior mean $\bmm_N$ of the model mean $\bmu$ is a weighted average of the prior mean and the sample mean. The posterior scale matrix $\bS_N$ is the sum of the prior scale matrix, empirical covariance matrix $\bS_{\overline{x}}$, and an extra term due to the uncertainty in the mean.

\subsubsection{Parameter choice}
In practice, it is often better to use a weakly informative data-dependent prior. A common choice is to set $\bS_0 = \diag(\bS_{\overline{x}})/N$, and $\nu_0 =D+2$, to ensure $\E[\bSigma]=\bS_0$, and to set $\bmm_0 =\overline{\bx}$ and $\kappa_0$ to some small number, such as 0.01, where $\bS_{\overline{x}}$ is the sample covariance matrix and $\overline{\bx}$ is the sample mean vector as shown in Equation~\eqref{equation:mvu-sample-covariance} \citep{chipman2001practical, fraley2007bayesian, hoff2009first, murphy2012machine}. Equivalently, we can also standardize the observation matrix $\mathcalX$ first to have zero mean and unit variance for every feature, and then let $\bS_0 = \bm{I}_D$, and $\nu_0 =D+2$, to ensure $\E[\bSigma]=\bm{I}_D$, and to set $\bmm_0 =\bm{0}_D$ and $\kappa_0$ to some small number, such as 0.01.


\subsubsection{Reducing sampling time by maintaining squared sum of customers}\label{section:reduce-sampling-sum-square}
In this section, we introduce some tricks to implement NIW in Gaussian mixture model more efficiently. The content can also be found in \citep{das2014dpgmm}. The readers will better understand the Chinese restaurant process terminology in this section after reading Section \ref{sec:fmm_collabsed_gibbs} or Section \ref{sec:ifmm_collabsed_gibbs}. Feel free to skip this section on a first reading.

We have seen the equivalence between the Equation~\eqref{equation:niw_posterior_equation_5} and Equation~\eqref{equation:niw_posterior_equation_6}. The reason why we make a step further to Equation~\eqref{equation:niw_posterior_equation_6} from Equation~\eqref{equation:niw_posterior_equation_5} is to reduce sampling time.
Suppose now that the data is not fixed and some data points can be removed from or added to $\mathcalX$. If we stick to the form in Equation~\eqref{equation:niw_posterior_equation_5}, we need to calculate $\bS_{\overline{x}}$ and $\overline{\bx}$ over and over again whenever the data points are updated.

 In Chinese restaurant process/clustering terminology, if we use Equation~\eqref{equation:niw_posterior_equation_5} instead of Equation~\eqref{equation:niw_posterior_equation_6}, whenever a customer is removed from (or added to) a table, we have to compute the matrix $\bS_{\overline{x}}$, which requires to go over each point in this cluster (or each customer in this table following the term from Chinese restaurant process, this could be clear when you finish reading the collapsed Gibbs sampler for finite Gaussian mixture model or infinite Gaussian mixture model later). Computing this term everytime when a customer is removed or added, could be computationally expensive. 
 
We realize that the data terms in Equation~\eqref{equation:niw_posterior_equation_6} only involves a sum of the outer product which does not contain any cross product (e.g., $\bx_i\bx_j^\top$ for $i \neq j$). 
By reformulating into Equation~\eqref{equation:niw_posterior_equation_6}, whenever a customer is removed or added, we just have to subtract or add $\bx_n \bx_n^\top$. Thus for each table, we only have to maintain the squared sum of customer vectors $\sum_{n=1}^N \bx_n \bx_n^\top$ for $\bS_N$. 

Similarly, for $\bmm_N$, we need to maintain the sum of customer vectors $\sum_{n=1}^N\bx_n$ for the same reason from Equation~\eqref{equation:niw_posterior_equation_2}.

\subsection{Posterior marginal likelihood of parameters}

The posterior marginal for $\bSigma$ is given by
\begin{equation*}
	\begin{aligned}
		p(\bSigma |\mathcalX,\bbeta) &= \int_{\bmu} p(\bmu, \bSigma |\mathcalX,\bbeta) d\bmu\\ &=\inversewishart(\bSigma|\bS_N, \nu_N),
	\end{aligned}
\end{equation*}
where the mean and mode can be obtained by Equation~\ref{equation:iw_expectation}, and they are given by 
$$
\begin{aligned}
	\Exp[\bSigma | \mathcalX, \bbeta] &= \frac{\bS_N}{\nu_N - D-1}, \\
	\mathrm{Mode}[\bSigma | \mathcalX, \bbeta]&=\frac{\bS_N}{\nu_N + D+1}.
\end{aligned}
$$
The posterior marginal for $\bmu$ follows from a Student $t$ distribution. We can show the posterior marginal for $\bmu$ is given by
\begin{equation*}
\begin{aligned}
p(\bmu |\mathcalX,\bbeta) &= \int_{\bSigma} p(\bmu, \bSigma |\mathcalX,\bbeta) d\bSigma\\
&= \int_{\bSigma} \niw (\bmu, \bSigma | \bmm_N, \kappa_N, \nu_N, \bS_N) d\bSigma \\
&=\tau(\bmu | \bmm_N, \frac{1}{\kappa_N(\nu_N-D+1)}\bS_N, \nu_N-D+1),
\end{aligned}
\end{equation*}
which is from the Gaussian scale mixture property of Student $t$ distribution, see Equation~\ref{equation:gauss-scale-mixture} and further discussed in \citep{murphy2012machine}.


\subsection{Posterior marginal likelihood of data}\label{section:posterior-marginal-of-data}
By integrating the full joint distribution in Equation~\eqref{equation:niw_full_joint2}, we can get the marginal likelihood of data under hyper-parameter $\bbeta=(\bmo, \kappa_0, \nu_0, \bso)$:
\begin{equation}
\begin{aligned}
p(\mathcal{X}|\bbeta) &= \int_{\bmu} \int_{\bSigma} p(\mathcal{X}, \bmu, \bSigma | \bbeta) d\bmu d\bSigma \\
	&= \int_{\bmu} \int_{\bSigma} \normal(\mathcal{X}|\bmu, \bSigma) \cdot \niw(\bmu, \bSigma|\bbeta)  d\bmu d\bSigma \\
	&= \frac{(2\pi)^{-ND/2}}{Z_{\niw}(D, \kappa_0, \nu_0, \bS_0)} \int_{\bmu}\int_{\bSigma}|\bSigma|^{-\frac{\nu_0+N+D+2}{2}}  \\
	&\gap \times \exp \left(-\frac{\kappa_N}{2} (\bmu - \bmm_N) \bSigma^{-1} (\bmu - \bmm_N ) - \frac{1}{2} \tr(\bSigma^{-1} \bS_N )  \right)d \bmu d \bSigma \\
	&\overset{(*)}{=} (2\pi)^{-ND/2} \frac{Z_{\niw}(D, \kappa_N, \nu_N, \bS_N)}{Z_{\niw}(D, \kappa_0, \nu_0, \bS_0)} \\
	&= \pi^{-\frac{ND}{2}} \cdot\frac{\kappa_0^{D/2}\cdot |\bS_0|^{\nu_0/2}}{\kappa_N^{D/2}\cdot |\bS_N|^{\nu_N/2}} \prod_{d=1}^D \frac{\Gamma(\frac{\nu_N+1-d}{2})}{\Gamma(\frac{\nu_0+1-d}{2})}, 
\end{aligned}
\label{equation:niw_marginal_data}
\end{equation}
where the Identity (*) above is from the fact that the integral reduces to the normalizing constant of the NIW density given in Equation~\eqref{equation:niw_posterior_equation_1}. 

\subsection{Posterior predictive for data without observations}
Similarly, suppose now we observe a data vector $\bx^{\star}$ without observing any old datas. Then the predictive for the data vector can be obtained by 
\begin{equation}
\begin{aligned}
p(\bx^{\star} | \bbeta) 
&= \int_{\bmu} \int_{\bSigma} p(\bx^{\star}, \bmu, \bSigma | \bbeta) d\bmu d\bSigma \\
&= \int_{\bmu} \int_{\bSigma} \normal(\bx^{\star} | \bmu, \bSigma) \cdot \niw(\bmu, \bSigma | \bbeta) d\bmu d\bSigma\\
&= \pi^{-D/2} \frac{\kappa_0^{D/2}  |\bso|^{\nu_0/2}  }{(\kappa_0 + 1) ^{D/2}  |\bS_1|^{\nu_1/2}}  \prod_{d=1}^D \frac{\Gamma(\frac{\nu_1+ 1-d}{2})}{\Gamma(\frac{\nu_0 + 1-d}{2})}\\
&= \pi^{-D/2} \frac{\kappa_0^{D/2}  |\bso|^{\nu_0/2}  }{(\kappa_0 + 1) ^{D/2}  |\bS_1|^{\nu_1/2}}   \frac{\Gamma(\frac{\nu_0+ 2-D}{2})}{\Gamma(\frac{\nu_0 }{2})},
\end{aligned}
\label{equation:niw_prior_predictive_abstract}
\end{equation}
where $\nu_1 = \nu_0+1$, $\bS_1 = \bso+ \frac{\kappa_0 }{\kappa_0+1} (\bx^{\star}-\bmm_0)(\bx^{\star}-\bmm_0)^\top$.
An alternative form of Equation~\eqref{equation:niw_prior_predictive_abstract} is to rewrite by a multivariate Student $t$ distribution
\begin{equation}
p(\bx^{\star} | \bbeta) = \tau(\bx^{\star} | \bmo, \frac{\kappa_0 + 1}{\kappa_0(\nu_0 - D + 1)}\bS_0, \nu_0 - D + 1).
\end{equation}

\subsection{Posterior predictive for new data with observations}

Similar to posterior predictive for data without observation, now suppose we observe a new data vector $\bx^{\star}$ give old observations $\mathcalX$. Then the posterior predictive for this vector is
\begin{equation}
p(\bx^{\star} | \mathcalX, \bbeta) = \frac{p(\bx^{\star}, \mathcalX | \bbeta) }{p(\mathcalX | \bbeta)}. 
\label{equation:niw_posterior_predictive_abstract}
\end{equation}

The denominator of Equation~\eqref{equation:niw_posterior_predictive_abstract} can be obtained directly from Equation~\eqref{equation:niw_marginal_data}. The numerator of it can be obtained in a similar way from Equation~\eqref{equation:niw_marginal_data} by considering the marginal likelihood of the new set $\{\mathcalX, \bx^{\star}\}$. We just need to replace $N$ by $N^{\star}=N+1$ in Equation~\eqref{equation:niw_posterior_equation_2}, Equation~\eqref{equation:niw_posterior_equation_3}, and Equation~\eqref{equation:niw_posterior_equation_4}, and replace $\bS_N$ by $\bS_{N^{\star}}$  in Equation~\eqref{equation:niw_posterior_equation_5}. Therefore, we obtain
\begin{equation}
\begin{aligned}
p(\bx^{\star} | \mathcalX, \bbeta) &= (2\pi)^{-D/2} \frac{Z_{\niw}(D, \kappa_{N^{\star}}, \nu_{N^{\star}}, \bS_{N^{\star}})}{Z_{\niw}(D, \kappa_{N}, \nu_{N}, \bS_{N})} \\
&= \pi^{-D/2} \frac{(\kappa_{N^{\star}})^{-D/2}|\bS_{N} |^{(\nu_N)/2}}{(\kappa_{N})^{-D/2}  |\bS_{N^{\star}} |^{(\nu_{N^{\star}})/2} }  
\prod_{d=1}^D\frac{ \Gamma(\frac{\nu_{N^{\star}} + 1-d}{2})}{ \Gamma(\frac{\nu_{N} + 1-d}{2})} \\
&= \pi^{-D/2} \frac{(\kappa_{N^{\star}})^{-D/2}|\bS_{N} |^{(\nu_N)/2}}{(\kappa_{N})^{-D/2}  |\bS_{N^{\star}} |^{(\nu_{N^{\star}})/2} }  
\frac{ \Gamma(\frac{\nu_{0} + N+2-D}{2})}{ \Gamma(\frac{\nu_{0} + N}{2})} . 
\end{aligned}
\label{equation:niw_posterior_predictive_equation}
\end{equation}
Again an alternative form of Equation~\eqref{equation:niw_posterior_predictive_equation} is to rewrite by a multivariate Student $t$ distribution:
\begin{equation*}
p(\bx^{\star} | \mathcalX, \bbeta) = \tau (\bx^{\star}  |  \bmm_N, \frac{\kappa_N + 1}{\kappa_N (\nu_N - D + 1)} \bS_N, \nu_N - D + 1).
\end{equation*}
Thus, the mean and covariance of $\bx^\star$ are given by
\begin{equation*}
\begin{aligned}
\Exp [\bx^\star |\mathcalX, \bbeta] &= \bmm_N = \frac{\kappa_0 }{\kappa_0+N}\bmo+\frac{N}{\kappa_0+N}\overline{\bx} ,\\
\Cov [\bx^\star |\mathcalX, \bbeta] &= \frac{\kappa_N + 1}{\kappa_N (\nu_N - D - 1)} \bS_N
= \frac{\kappa_0+N + 1}{(\kappa_0+N) (\nu_0+N - D - 1)} \bS_N,
\end{aligned}
\end{equation*}
where we can find, on average, the new coming data has expectation $\bmm_N$. We mentioned previously, $\kappa_0$ controls how strongly we believe this prior for $\bmu$. When $\kappa_0$ is large enough, $\Exp [\bx^\star |\mathcalX, \bbeta]$ converges to $\bmm_0$, the prior mean, and $\Cov [\bx^\star |\mathcalX, \bbeta]$ converges to $\frac{\bS_N}{(\kappa_0+N) (\nu_0+N - D - 1)} $. In the meantime, if we set $\nu_0$ large enough, the covariance matrix $\bSigma$ concentrated around $\bSigma_0$, and 
$$
\bS_N \rightarrow \frac{\bS_{\overline{x}}}{\nu_0} + \frac{\kappa_0N}{\nu_0(\kappa_0 + N)}(\overline{\bx} - \bmo)(\overline{\bx} - \bmo)^\top ,
$$
which is largely controlled by data sample and data magnitude (rather than the prior hyperparameters), so as the posterior variance $\Cov [\bx^\star |\mathcalX, \bbeta]$.

\subsection{Further optimization via the Cholesky decomposition}
\subsubsection{Definition}
The Cholesky decomposition of a symmetric positive definite matrix $\bS$ is its decomposition into the product of a lower triangular matrix $\bL$ and it's transpose:
\begin{equation}
\bS = \bL \bL^\top, 
\end{equation}
where $\bL$ is called the \textbf{Cholesky factor} of $\bS$. We realize that an alternative form of the Cholesky decomposition is using it's upper triangular $\bU$, i.e., $\bS = \bU^\top \bU$. A triangular matrix is a special kind of square matrix. Specifically, a square matrix is called lower triangular if all the entries are above the main diagonal are zero. Similarly, a square matrix is called upper triangular if all the entries below the main diagonal are zero. 

If the matrix has dimensionality $D$, the complexity of Cholesky decomposition is $O(D^3)$. In specific, it requires $\sim \frac{1}{3}D^3$ floating points operations (flops) to compute a Cholesky decomposition
of a $D\times D$ positive definite matrix \citep{lu2021numerical}, where the symbol ``$\sim$" has the usual asymptotic meaning
\begin{equation*}
	\lim_{D \to +\infty} \frac{\mathrm{number\, of\, flops}}{(1/3)D^3} = 1.
\end{equation*}
The existence of Cholesky decomposition is based on the existence of the LU decomposition and a rigorous proof can be found in \citep{lu2021numerical}.

\subsubsection{Rank one update}
A rank 1 update of matrix $\bS$ by vector $\bx$ is of the form \citep{seeger2004low}
\begin{equation*}
\bS^\prime = \bS + \bx \bx^\top. 
\end{equation*}
If we have already calculated the Cholesky factor $\bL$ of $\bS$, then the Cholesky factor $\bL^\prime$ of $\bS^\prime$ can be calculated efficiently.  Note that $\bS^\prime$ differs from $\bS$ only via three symmetric rank one matrices. Hence we can compute $\bL^\prime$ from $\bL$ using three rank one Cholesky update, which takes $O(D^2)$ operations each saving from $O(D^3)$ if we do know $\bL$, the Cholesky decomposition of $\bS$. 

\subsubsection{Speedup for determinant}

The determinant of a positive definite matrix $\bS$ can be computed from its Cholesky factor $\bL$:
\begin{equation*}
|\bS| = \prod_{d=1}^{D} \bL_{dd}^2,\qquad \log(|\bS|) = 2\log(|\bL|)= 2 \times \sum_{d=1}^D \log(\bL_{dd}), 
\end{equation*}
where $\bL_{dd}$ is the ($d,d$) entry of matrix $\bL$. This is an $O(D)$ operation, i.e., given the Cholesky decomposition, the determinant is just the product of the diagonal terms.

\subsubsection{Update in NIW}
Now we consider computing the marginal likelihood of data in Equation~\eqref{equation:niw_marginal_data} and the posterior predictive for new coming data in Equation~\eqref{equation:niw_posterior_predictive_abstract} of which the two cases are similar.  We will see this optimization will be often used in the Chinese restaurant process like that in Section~\ref{section:reduce-sampling-sum-square}. Feel free to skip this section on a first reading. 

Take the latter as an example, note that to compute posterior predictive for new coming data $p(\bx^{\star} | \mathcalX, \bbeta)$ in Equation~\eqref{equation:niw_posterior_predictive_abstract}, we just need to evaluate $ \frac{p(\bx^{\star}, \mathcalX | \bbeta) }{p(\mathcalX | \bbeta)}$, in which we must calculate $|\bS_N|$ and $|\bS_{N^{\star}}|$ efficiently where $N^{\star} = N+1$.We deal with computing the determinants $|\bS_N|$ and $|\bS_{N^{\star}}|$ by representing $\bS_N$ and $\bS_{N^{\star}}$ using their Cholesky decomposition. In particular, updates to $\bS_N$ and $\bS_{N^{\star}}$ will be carried out by directly updating their Cholesky decompositions given the Cholesky decomposition the determinant is just the product of the diagonal terms. Write out $\bS_{\star}$ by $\bS_N$:
\begin{align}
\bmm_N &= \frac{\kappa_{N^{\star}}\bmm_{N^{\star}} - x^\star}{\kappa_N}=\frac{(\kappa_0 + N + 1)\bmm_{N^{\star}} - x^\star}{\kappa_0 + N} , \\
\bmm_{N^{\star}} &= \frac{\kappa_{N} \bmm_N +\bx^{\star}}{\kappa_{N^{\star}}} 
= \frac{(\kappa_{0}+N) \bmm_N +\bx^{\star}}{\kappa_{0}+N+1},\\
\bS_{N^{\star}} &= \bS_N + \bx^{\star} \bx^{\star T} - \kappa_{N^{\star}} \bmm_{N^{\star}} \bmm_{N^{\star}}^\top + \kappa_N \bmm_N \bmm_N^\top \\
&= \bS_N + \frac{\kappa_0 + N + 1}{\kappa_0 + N}(\bmm_{N^{\star}} - \bx^\star)(\bmm_{N^{\star}} - \bx^\star)^\top, \label{equation:cholesky_rank_1_form}
\end{align}
where Equation~\eqref{equation:cholesky_rank_1_form} implies that Cholesky decomposition of $\bS_{N^\star}$ can be obtained from Cholesky decomposition of $\bS_N$ by a Rank 1 update. Therefore if we know the Cholesky decomposition of $\bS_N$, the Cholesky decomposition of $\bS_{N^\star}$ can be obtained in $O(D^2)$ complexity.

\subsection{Last words on the conjugate prior for Gaussian distribution}
The univariate analog of normal-inverse-Wishart distribution is the mormal-inverse-Chi-squared (NIX) distribution. 
For simplicity, we only write the likelihood and prior distribution for univariate Gaussian distribution here, all the analysis in the following sections will be described in the multivariate case.

\paragraph{Likelihood}
The univariate Gaussian distribution is
\begin{equation}
	\begin{aligned}
		p(\mathcal{X} |\smu, \ssigma^2) &= \prod^N_{n=1} \mathcal{N} (x_n|\smu, \ssigma^2) \\
		&= (2\pi)^{-N/2}  (\ssigma^2)^{-N/2} \exp\left(-\frac{1}{2 \ssigma^2}  \left[ N(\overline{x} - \smu)^2 + N \sum_{n=1}^N(x_n - \overline{x})^2   \right] \right)  \\
		&= (2\pi)^{-N/2}  (\ssigma^2)^{-N/2} \exp\left(-\frac{1}{2 \ssigma^2}  \left[  N(\overline{x} - \smu)^2 +  N S_{\overline{x}} \right] \right),
	\end{aligned}
	\label{equation:uni_gaussian_likelihood}
\end{equation}
where $S_{\overline{x}}=\sum_{n=1}^N(x_n - \overline{x})^2$.
\subsubsection{Normal-inverse-chi-squared prior}
\paragraph{Prior on parameters}
Follow from the definition of inverse-gamma distribution in Definition~\ref{definition:inverse-gamma}, we give the rigorous definition of inverse-chi-squared distribution as follows.
\begin{svgraybox}
\begin{definition}[Inverse-Chi-Squared Distribution]\label{definition:inverse-chi-square}
A random variable $Y$ is said to follow the inverse-chi-squared distribution with parameter $\nu_0>0$ and $\sigma_0^2>0$ if $Y\sim \inversegammadist(\frac{\nu_0}{2}, \frac{\nu_0 \sigma_0^2}{2})$:

$$ f(y; \nu_0, \sigma_0^2)=\left\{
\begin{aligned}
	&\frac{{(\frac{\nu_0 \sigma_0^2}{2})}^{\frac{\nu_0}{2}}}{\Gamma(\frac{\nu_0}{2})} y^{-\frac{\nu_0}{2}-1} \exp(- \frac{\nu_0 \sigma_0^2}{2y} ) ,& \mathrm{\,\,if\,\,} y > 0.  \\
	&0 , &\mathrm{\,\,if\,\,} y \leq 0.
\end{aligned}
\right.
$$
And it is denoted by $Y \sim \inversechidist(\nu_0, \sigma_0^2)$. The parameter $\nu >0$ is called the \textbf{degrees of freedom}, and $\sigma_0^2 > 0$ is the \textbf{scale parameter}. And it is also known as the \textbf{scaled} inverse-chi-squared distribution.
The mean and variance of inverse-gamma distribution are given by 
$$ \Exp[Y]=\left\{
\begin{aligned}
	&\frac{\nu_0 \sigma_0^2}{\nu_0-2}, \, &\mathrm{if\,} \nu_0\geq 2. \\
	&\infty, \, &\mathrm{if\,} 0<\nu_0<2.
\end{aligned}
\right.\qquad
\Var[Y]=\left\{
\begin{aligned}
	&\frac{2\nu_0^2 \sigma_0^4}{(\nu_0-2)^2(\nu_0-4)}, \, &\mathrm{if\,} \nu_0\geq 4. \\
	&\infty, \, &\mathrm{if\,} 0<\nu_0<4.
\end{aligned}
\right.
$$
\end{definition}
\end{svgraybox}

To make connection to inverse-Wishart distribution, we can set $S_0=\nu_0\sigma_0^2$. Then the inverse-Chi-squared distribution can also be denoted by $Y\sim \inversegammadist(\frac{\nu_0}{2}, \frac{S_0}{2})$ if $Y \sim \inversechidist(\nu_0, \sigma_0^2)$ of which the form conforms to the univariate case of inverse-Wishart distribution. And we will see the similarity in the posterior parameters as well.

%
Similarly to the normal-inverse-Wishart prior, the normal-inverse-chi-squared prior is defined as
\begin{equation}
\begin{aligned}
&\gap \nix (\smu, \sigma^2 | m_0, \kappa_0, \nu_0, S_0) \\
&= \mathcal{N} (\mu| m_0, \frac{\ssigma^2}{\kappa_0})  \cdot \inversechidist (\ssigma^2 | \nu_0,  \sigma^2_0) \\
&=\frac{1}{Z_{\nix}(\kappa_0, \nu_0, \sigma^2_0)} (\ssigma^2)^{-(\nu_0/2 + 3/2)} \exp\left( -\frac{1}{2 \ssigma^2} \left[\nu_0 \sigma^2_0 + \kappa_0(m_0-\mu)^2\right] \right), \\
&\stackrel{S_0 = \nu_0\sigma^2_0}{=}\frac{1}{Z_{\nix}(\kappa_0, \nu_0, \sigma^2_0)} (\ssigma^2)^{-(\nu_0/2 + 3/2)} \exp\left( -\frac{1}{2 \ssigma^2} \left[S_0 + \kappa_0(m_0-\mu)^2\right] \right)
\end{aligned}
\label{equation:uni_gaussian_prior}
\end{equation}
where 
\begin{equation}\label{equation:uni_gaussian_giw_constant}
Z_{\nix}(\kappa_0, \nu_0, \sigma^2_0) = \frac{\sqrt{(2\pi)}}{\sqrt{\kappa_0}} \Gamma(\frac{\nu_0}{2}) (\frac{2}{\nu_0 \sigma^2_0})^{\nu_0/2} 
= \frac{\sqrt{(2\pi)}}{\sqrt{\kappa_0}} \Gamma(\frac{\nu_0}{2}) (\frac{2}{S_0})^{\nu_0/2}.
\end{equation}
\paragraph{Posterior under NIX}
Again, by the Bayes' theorem ``$\mathrm{posterior} \propto \mathrm{likelihood} \times \mathrm{prior} $", the posterior of the $\bmu$ and $\bSigma$ parameters under the NIW prior is
\begin{equation}\label{equation:nix-posterior}
\begin{aligned}
p(\mu, \sigma^2| \mathcalX, \bbeta ) &\propto p(\mathcalX | \mu, \sigma^2) p(\mu, \sigma^2 | \bbeta) = p(\mathcalX, \mu, \sigma^2 | \bbeta)\\
&= C \times (\sigma^2)^{-\frac{\nu_0 + 3+N}{2}} \exp\left(-\frac{1}{2 \ssigma^2}  \left[  N(\overline{x} - \smu)^2 +  N S_{\overline{x}} \right] \right) \\
&\gap \times \exp\left( -\frac{1}{2 \ssigma^2} \left[S_0 + \kappa_0(m_0-\mu)^2\right] \right)\\
&= C\times  (\sigma^2)^{-\frac{\nu_N + 3}{2}}\exp\left( -\frac{1}{2 \ssigma^2} \left[ S_N + \kappa_N(m_N-\mu)^2\right] \right)\\
&\propto\nix(\mu, \sigma^2 | m_N, \kappa_{N}, \nu_N, \textcolor{blue}{S_N})= \normal (\mu| m_N, \frac{\ssigma^2}{\kappa_N})  \cdot \inversechidist (\ssigma^2 | \nu_N,  \textcolor{blue}{\sigma^2_N}),
\end{aligned}
\end{equation}
where $\bbeta=( m_0, \kappa_0, \nu_0, S_0=\nu_0\sigma_0^2)$, $C=\frac{(2\pi)^{-N/2}}{Z_{\nix}(\kappa_0, \nu_0, \sigma^2_0)}$, and 

$$
\begin{aligned}
m_N &= \frac{\kappa_0 m_0 + N\overline{x}}{\kappa_{N}} = \frac{\kappa_0 }{\kappa_{N}}m_0 + \frac{N}{\kappa_{N}}\overline{x},\\ 
\kappa_{N}&= \kappa_{0} +N,\\
\nu_N &= \nu_0 +N,\\
S_N &= S_0 +NS_{\overline{x}} + N\overline{x}^2 + \kappa_{0} m_0^2 -\kappa_{N}m_N^2\\
&=S_0 +NS_{\overline{x}} + \frac{\kappa_0 N }{\kappa_{0}+N} (\overline{x} - m_0)^2,\\
\nu_N \sigma_N^2 &= S_N \leadto     \sigma_N^2 = \frac{S_N}{\nu_N} ,
\end{aligned}
$$
which shares same form as that in the multivariate case from Equation~\ref{equation:niw_posterior_equation_1} except the $N$ in $NS_{\overline{x}}$ which arise from the difference between the multivariate Gaussian distribution and the univariate Gaussian distribution. Similarly, in inverse-chi-squared language, we can show the $\nu_N \sigma_N^2 = S_N$.

Suppose $\nu_0\geq 2$, or $N\geq 2$ such that $\nu_N\geq 2$, the posterior expectations are given by 
$$
\Exp[\mu |\mathcalX, \bbeta] = m_N, \qquad
\Exp[\sigma^2 | \mathcalX, \bbeta] = \frac{S_N}{\nu_N-2}.
$$

\paragraph{Marginal posterior of $\sigma^2$}
Integrate out $\mu$, we have
$$
\begin{aligned}
p(\sigma^2 | \mathcalX, \bbeta) &= \int_{\mu} p(\mu, \sigma^2 | \mathcalX, \bbeta) d \mu \\
&= \int_{\mu }  \mathcal{N} (\mu| m_N, \frac{\ssigma^2}{\kappa_N})  \cdot \inversechidist (\ssigma^2 | \nu_N,  \sigma^2_N)d\mu \\
&= \inversechidist (\ssigma^2 | \nu_N,  \sigma^2_N),
\end{aligned}
$$
which is just an integral over a Gaussian distribution.

\paragraph{Marginal posterior of $\mu$}
Integrate out $\sigma^2$, we have
$$
\begin{aligned}
p(\mu | \mathcalX, \bbeta) &= \int_{\sigma^2} p(\mu, \sigma^2 | \mathcalX, \bbeta) d \sigma^2 \\
&= \int_{\sigma^2 }  \mathcal{N} (\mu| m_N, \frac{\ssigma^2}{\kappa_N})  \cdot \inversechidist (\ssigma^2 | \nu_N,  \sigma^2_N)d\sigma^2 \\
&= \int_{\sigma^2}C(\sigma^2)^{-\frac{\nu_N + 3}{2}}\exp\left( -\frac{1}{2 \ssigma^2} \left[ S_N + \kappa_N(m_N-\mu)^2\right] \right)d\sigma^2. 
\end{aligned}
$$
Let $\phi = \sigma^2$ and $\alpha = (\nu_N+1)/2$, $A =  S_N + \kappa_N(m_N-\mu)^2$, and $x = \frac{A}{2\phi}$, we have 
$$
\frac{d \phi}{d x} = -\frac{A}{2}x^{-2}.
$$
where $A$ can be easily verified to be positive and $\phi=\sigma^2>0$.
It follows that
\begin{equation*}
	\begin{aligned}
		p(\mu | \mathcalX, \bbeta) &=\int_{0}^{\infty} C(\phi)^{-\alpha-1} \exp\left(-\frac{A}{2 \phi} \right)d\phi\\
		&=\int_{\textcolor{red}{\infty}}^{\textcolor{red}{0}}  C(\frac{A}{2x})^{-\alpha-1} \exp\left( -x\right)   ( \textcolor{red}{-}\frac{A}{2}x^{-2}) dx \qquad &\text{(since $x=\frac{A}{2\phi}$)}\\
		&=\int_{\textcolor{red}{0}}^{\textcolor{red}{\infty}}  C(\frac{A}{2x})^{-\alpha-1} \exp\left( -x\right)   ( \frac{A}{2}x^{-2}) dx\\
		&= (\frac{A}{2})^{-\alpha} \int_{x}  Cx^{\alpha-1} \exp\left( -x\right)   dx \\
		&= (\frac{A}{2})^{-\alpha}  (C\cdot \Gamma(1)) \int_{x}\gammadist(x|\alpha, 1) dx\qquad &\text{(see Definition~\ref{definition:gamma-distribution})}\\
		&= (C\cdot \Gamma(1))\left[ \nu_N\sigma_N^2 +\kappa_{N}(m_N-\mu)^2 \right]^{-\frac{\nu_N+1}{2}}\\
		&\overset{(a)}{=} (C\cdot \Gamma(1)) (\nu_N\sigma_N^2)^{-\frac{\nu_N+1}{2}} \left[ 1 +\frac{\kappa_{N}}{\nu_N\sigma_N^2}(m_N-\mu)^2 \right]^{-\frac{\nu_N+1}{2}}
	\end{aligned}
\end{equation*}
We notice that $C$ is defined in Equation~\ref{equation:nix-posterior} (in terms of $(\kappa_N, \nu_N, \sigma^2_N)$) that 
$$
C\overset{(b)}{=}\frac{(2\pi)^{-N/2}}{Z_{\nix}(\kappa_N, \nu_N, \sigma^2_N)} = \frac{(2\pi)^{-N/2}}{\frac{\sqrt{(2\pi)}}{\sqrt{\kappa_N}} \Gamma(\frac{\nu_N}{2}) (\frac{2}{\nu_N \sigma^2_N})^{\nu_N/2}} \propto (\nu_N \sigma^2_N)^{\nu_N/2}.
$$
Combine Equation (a) and (b) above, we obtain 
$$
p(\mu | \mathcalX, \bbeta)  \propto \frac{1}{\sigma_N/\sqrt{\kappa_N}} \left[ 1 +\frac{\kappa_{N}}{\nu_N\sigma_N^2}(\mu-m_N)^2 \right]^{-\frac{\nu_N+1}{2}} 
\propto \tau(\mu| m_N, \sigma_N^2/\kappa_N, \nu_N),
$$
which is a univariate Student $t$ distribution see Definition~\ref{definition:multivariate-stu-t}.

\paragraph{Marginal likelihood of data}
By Equation~\ref{equation:nix-posterior}, we can get the marginal likelihood of data under hyper-parameter $\bbeta=( m_0, \kappa_0, \nu_0, S_0=\nu_0\sigma_0^2)$
$$
\begin{aligned}
p(\mathcalX | \bbeta) &= \int_{\mu} \int_{\sigma^2}  p(\mathcalX, \mu, \sigma^2 |\bbeta) d\mu d\sigma^2\\ 
&=\frac{(2\pi)^{-N/2}}{Z_{\nix}(\kappa_0, \nu_0, \sigma^2_0)}  \int_{\mu} \int_{\sigma^2}  
(\sigma^2)^{-\frac{\nu_N + 3}{2}}\exp\left( -\frac{1}{2 \ssigma^2} \left[ S_N + \kappa_N(m_N-\mu)^2\right] \right)
  d\mu d\sigma^2\\
&= (2\pi)^{-N/2}\frac{Z_{\nix}(\kappa_N, \nu_N, \sigma^2_N)}{Z_{\nix}(\kappa_0, \nu_0, \sigma^2_0)} \\
&= (\pi)^{-N/2} \frac{\Gamma(\nu_N/2)}{\Gamma(\nu_0/2)} \sqrt{\frac{\kappa_0}{\kappa_N}} \frac{(\nu_0\sigma^2_0)^{\nu_0/2}}{(\nu_N\sigma^2_N)^{\nu_N/2}}.
\end{aligned}
$$

\paragraph{Posterior predictive for new data with observations}
Let the number of samples for data set $\{\xstar, \mathcalX\}$ be $\Nstar = N+1$, we have
\begin{equation}\label{equation:nix-posterior-new-withobser}
\begin{aligned}
	p(\xstar |\mathcalX, \bbeta) &= \frac{p(\xstar, \mathcalX | \bbeta)}{p(\mathcalX | \bbeta)}\\
	&=\left\{(2\pi)^{-\Nstar/2}\frac{Z_{\nix}(\kappa_\Nstar, \nu_\Nstar, \sigma^2_\Nstar)}{Z_{\nix}(\kappa_0, \nu_0, \sigma^2_0)}\right\}
	/\left\{(2\pi)^{-N/2}\frac{Z_{\nix}(\kappa_N, \nu_N, \sigma^2_N)}{Z_{\nix}(\kappa_0, \nu_0, \sigma^2_0)}\right\}\\
	&=(2\pi)^{-1/2} \frac{Z_{\nix}(\kappa_\Nstar, \nu_\Nstar, \sigma^2_\Nstar)}{Z_{\nix}(\kappa_N, \nu_N, \sigma^2_N)}\\
	&=(\pi)^{-1/2} 
	\sqrt{\frac{\kappa_N}{\kappa_{N^{\star}}}}    \frac{\Gamma(\frac{\nu_{N^{\star}}}{2})}{\Gamma(\frac{\nu_{N}}{2})}
	\frac{(\nu_N \sigma_N^2)^{\frac{\nu_N}{2}}}{(\nu_{\Nstar}\sigma_{\Nstar}^2)^{\frac{\nu_{\Nstar}}{2}}}\\
	&=\frac{\Gamma(\frac{\nu_{N}+1}{2})}{\Gamma(\frac{\nu_{N}}{2})}
	\sqrt{\frac{\kappa_N}{(\kappa_{N}+1)}   \frac{1}{(\pi\nu_{N}\sigma_{N}^2)}}   
	\left(\frac{(\nu_{\Nstar}\sigma_{\Nstar}^2)}{(\nu_N \sigma_N^2)}\right)^{-\frac{\nu_{N}+1}{2}}.
\end{aligned}
\end{equation}
We realize that 
$$
\begin{aligned}
m_N &= \frac{\kappa_{N^{\star}}m_{N^{\star}} - x^\star}{\kappa_N}=\frac{(\kappa_0 + N + 1)m_{N^{\star}} - x^\star}{\kappa_0 + N} , \\
m_{N^{\star}} &= \frac{\kappa_{N} m_N +x^{\star}}{\kappa_{N^{\star}}} 
= \frac{(\kappa_{0}+N) m_N +x^{\star}}{\kappa_{0}+N+1},\\
S_{N^{\star}} &= S_N + x^{\star} x^{\star T} - \kappa_{N^{\star}} m_{N^{\star}}^2  + \kappa_N m_N^2 \\
&= S_N + \frac{\kappa_N + 1}{\kappa_N}(m_{N^{\star}} - x^\star)^2\\
&=S_N + \frac{\kappa_N }{\kappa_N+ 1}(m_{N} - x^\star)^2,
\end{aligned}
$$
Thus, 
\begin{equation}\label{equation:nix-substitute-posterior-withobser}
\begin{aligned}
	\left(\frac{(\nu_{\Nstar}\sigma_{\Nstar}^2)}{(\nu_N \sigma_N^2)}\right)^{-\frac{\nu_{N}+1}{2}}&=
	\left(\frac{S_\Nstar}{S_N}\right)^{-\frac{\nu_{N}+1}{2}} =1 +  \frac{\kappa_N(m_{N} - x^\star)^2 }{(\kappa_N+ 1)\nu_N\sigma_N^2}.
\end{aligned}
\end{equation}
Substitute Equation~\ref{equation:nix-substitute-posterior-withobser} into Equation~\ref{equation:nix-posterior-new-withobser}, it follows that 
$$
\begin{aligned}
p(\bxstar |\mathcalX, \bbeta) 
&=\frac{\Gamma(\frac{\nu_{N}+1}{2})}{\Gamma(\frac{\nu_{N}}{2})}
\sqrt{\frac{\kappa_N}{(\kappa_{N}+1)}   \frac{1}{(\pi\nu_{N}\sigma_{N}^2)}}   
\left(1 +  \frac{\kappa_N(m_{N} - x^\star)^2 }{(\kappa_N+ 1)\nu_N\sigma_N^2}\right)^{-\frac{\nu_{N}+1}{2}}\\
&= \tau(x^\star | m_N, \frac{\kappa_{N}+1}{\kappa_{N}}\sigma^2_N, \nu_N  ).
\end{aligned}
$$

\paragraph{Posterior predictive for new data without observations}
Similarly, we have
$$
\begin{aligned}
p(\xstar | \bbeta) &=  \int_{\mu} \int_{\sigma^2} p(\xstar, \mu, \sigma^2 |\bbeta) d\mu d\sigma^2 \\
&=(2\pi)^{-1/2}\frac{Z_{\nix}(\kappa_1, \nu_1, \sigma^2_1)}{Z_{\nix}(\kappa_0, \nu_0, \sigma^2_0)}\\
&=(\pi)^{-1/2} 
\sqrt{\frac{\kappa_0}{\kappa_{1}}}    \frac{\Gamma(\frac{\nu_{1}}{2})}{\Gamma(\frac{\nu_{0}}{2})}
\frac{(\nu_0 \sigma_0^2)^{\frac{\nu_0}{2}}}{(\nu_{1}\sigma_{1}^2)^{\frac{\nu_{1}}{2}}}\\
&=\frac{\Gamma(\frac{\nu_{0}+1}{2})}{\Gamma(\frac{\nu_{0}}{2})}
\sqrt{\frac{\kappa_0}{(\kappa_{0}+1)}   \frac{1}{(\pi\nu_{0}\sigma_{0}^2)}}   
\left(\frac{(\nu_{1}\sigma_{1}^2)}{(\nu_0 \sigma_0^2)}\right)^{-\frac{\nu_{0}+1}{2}}\\
&=\tau(x^\star | m_0, \frac{\kappa_{0}+1}{\kappa_{0}}\sigma^2_0, \nu_0  ).
\end{aligned}
$$

\subsubsection{Normal-inverse-gamma prior*}
\paragraph{Prior on parameters}
We realize that inverse-chi-squared distribution is a special inverse-gamma distribution (defined in Definition~\ref{definition:inverse-gamma}). The particularity is in the similarity with the inverse-Wishart distribution. Similarly and more generally, we can define the normal-inverse-gamma prior as follows (as we have shown the inverse-gamma distribution is often used as a conjugate prior for the variance parameter in Section~\ref{section:blm-fullconjugate}):
\begin{equation}
	\begin{aligned}
		&\gap \nig (\smu, \sigma^2 | m_0, \kappa_0, r_0, \lambda_0) \\
		&= \mathcal{N} (\mu| m_0, \frac{\ssigma^2}{\kappa_0})  \cdot \inversegammadist (\ssigma^2 | r_0,  \lambda_0) \\
		&=\frac{1}{Z_{\nig}(\kappa_0, r_0, \lambda_0)}  (\ssigma^2)^{-\frac{2r_0 +3}{2}} \exp\left(-\frac{1}{2 \ssigma^2}\left[\kappa_0(m_0-\mu)^2 + 2\lambda_0 \right] \right)  \\
	\end{aligned}
	\label{equation:uni_gaussian_prior-nig}
\end{equation}
where 
\begin{equation}\label{equation:uni_gaussian_giw_constant-nig}
	Z_{\nig}(\kappa_0, r_0, \lambda_0) = \frac{\Gamma(r_0)}{\lambda_0^{r_0}}(2\pi)^{-1/2}.
\end{equation}
This is equivalent to set $r_0 = \frac{\nu_0}{2}$ and $\lambda_0 = \frac{S_0}{2}$ in $\nix$.

\paragraph{Posterior under NIG}
Again, by the Bayes' theorem ``$\mathrm{posterior} \propto \mathrm{likelihood} \times \mathrm{prior} $", the posterior of the $\bmu$ and $\bSigma$ parameters under the NIG prior is
\begin{equation}
	\begin{aligned}
		p(\mu, \sigma^2| \mathcalX, \bbeta ) &\propto p(\mathcalX | \mu, \sigma^2) p(\mu, \sigma^2 | \bbeta) = p(\mathcalX, \mu, \sigma^2 | \bbeta)\\
		&= C \times (\sigma^2)^{-\frac{2r_0 + 3+N}{2}} \exp\left(-\frac{1}{2 \ssigma^2}  \left[  N(\overline{x} - \smu)^2 +  N S_{\overline{x}} \right] \right) \\
		&\gap \times \exp\left( -\frac{1}{2 \ssigma^2} \left[2\lambda_0 + \kappa_0(m_0-\mu)^2\right] \right)\\
		&\propto (\sigma^2)^{-\frac{2r_N + 3}{2}}\exp\left( -\frac{1}{2 \ssigma^2} \left[ \lambda_N + \kappa_N(m_N-\mu)^2\right] \right)\\
		&\propto\nig(\mu, \sigma^2 | m_N, \kappa_{N}, r_N, \lambda_N).
	\end{aligned}
\end{equation}
where $\bbeta=( m_0, \kappa_0, r_0, \lambda_0)$, $C=\frac{(2\pi)^{-N/2}}{Z_{\nig}(\kappa_0, r_0, \lambda_0)}$, and 

$$
\begin{aligned}
m_N &= \frac{\kappa_0 m_0 + N\overline{x}}{\kappa_{N}} = \frac{\kappa_0 }{\kappa_{N}}m_0 + \frac{N}{\kappa_{N}}\overline{x},\\ 
\kappa_{N}&= \kappa_{0} +N,\\
r_N &= r_0 +\frac{N}{2},\\
\lambda_N &=\lambda_0 +\frac{1}{2}(NS_{\overline{x}} + N\overline{x}^2 + \kappa_{0} m_0^2 -\kappa_{N}m_N^2)\\
&= \lambda_0+\frac{1}{2}(NS_{\overline{x}} + \frac{\kappa_0 N }{\kappa_{0}+N} (\overline{x} - m_0)^2).
\end{aligned}
$$
Further discussion on the posterior marginal likelihood can be found in \citep{murphy2007conjugate}. We will leave this to the readers as it is rather similar as that in the NIX prior.

%
%

\newpage
\part{Bayesian inference for mixture model}\label{chapter:review_bayesian_mixture_model}

\section{General mixture model}
\begin{figure}[h!]
\centering
  \includegraphics[width=5cm]{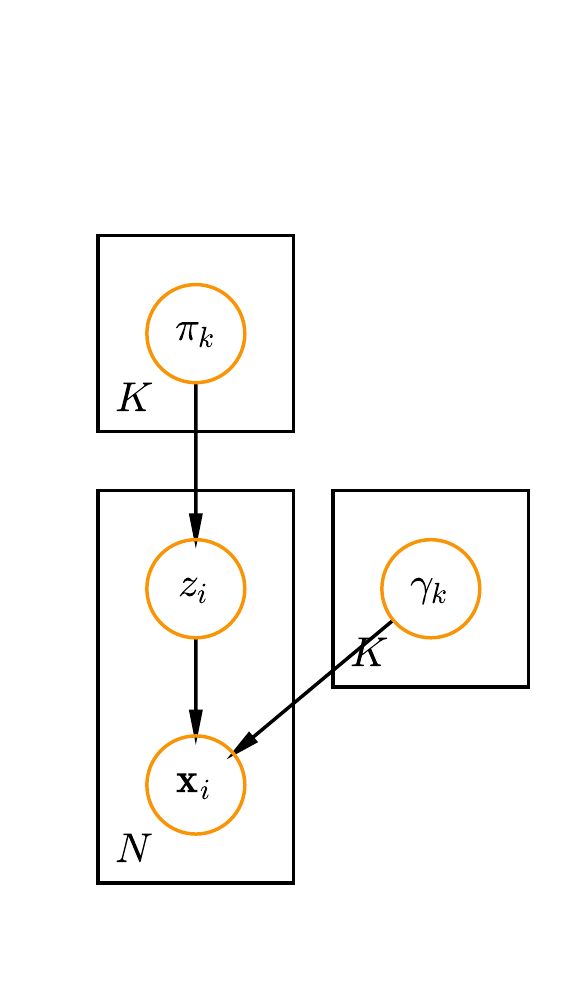}
  \caption{A General finite mixture model, $\bgamma_k$'s are parameters for the specific distribution of each cluster. When it is of Gaussian form, the mixture model is known as Gaussian mixture model (GMM), and the parameters $\bgamma_k$'s include the mean vectors and covariance matrices for the Gaussian distributions. }
  \label{fig:general_mm_finite}
\end{figure}

A typical finite-dimensional mixture model is a hierarchical model shown in Figure~\ref{fig:general_mm_finite} and consists of the following components: \footnote{See also \url{https://en.wikipedia.org/wiki/Mixture_model} for a nice description for the components in mixture models.}
\begin{itemize}
\item N random variables corresponding to observations $\mathcalX = \{\bx_1, \bx_2, \ldots, \bx_N\}$, where each assumed to be distributed according to a mixture of $K$ components, with each component belonging to the same parametric family of distributions (e.g., all Gaussian which have conjugate prior as we have shown previously, all Student $t$ which is not an exponential family and therefore there is no conjugate prior, etc.), but with different parameters;
\item A set of $K$ mixture weight vector $\bpi = \{\pi_1, \pi_2, \ldots, \pi_K\}$, each of which is a probability (a real number between 0 and 1 inclusively), all of which sum to 1 so that $\bpi$ is in a $(K-1)$-dimensional simplex;
\item A set of $K$ parameters, each specifying the parameter of the corresponding mixture component. In many cases, each ``parameter" is actually a set of parameters. For example, observations distributed according to a mixture of one-dimensional Gaussian distributions will have a mean parameter and a variance parameter for each component. And its multivariate version will have a mean vector and covariance matrix for each component.
\end{itemize}

Specifically, assume we have $N$ observations $\mathcalX = \{\bx_1, \bx_2, \ldots, \bx_N\}$ sampled i.i.d., from a finite mixture distribution with density 
\begin{equation}
p(\bx | \bpi, \bgamma) = \sum_{k=1}^{K} \pi_k \Phi(\bx | \bgamma_{k}),
\label{equation:general_mixture_model_expression}
\end{equation}
with $\bgamma_k \in \Gamma$, $\Gamma$ is the metric space of the parameter for some kernel functions, $\Phi$ is the kernel function of each component, and $K$ is finite and known. We wish to make (Bayesian) inference for the model parameters $\btheta = \{\bpi, \bgamma \}$. The likelihood is given by
\begin{equation}
p(\mathcalX |\bpi, \bgamma  ) = \prod_{i=1}^N \sum_{k=1}^K \pi_k \Phi(\bx_i | \bgamma_k), 
\label{equation:mm_maximum_likelihood}
\end{equation}
which is given by $K^N$ terms. This implies a large computational cost for even a not very large sample size, $N$. 

\paragraph{Bayesian Approaches VS EM Algorithm} In this article, we use Bayesian inference to do the calculation. However, an alternative frequentist approach exists to handle clustering based on mixture model which is known as the Expectation-Maximization (EM) algorithm where the parameters of the mixture model are usually estimated into a maximum likelihood estimation (MLE) framework by maximizing the observed data likelihood, i.e., the mixture model parameters $\{\bpi, \bgamma \} $ can be estimated by maximizing the observed data likelihood in Equation~\eqref{equation:mm_maximum_likelihood}. The EM algorithm is advanced in the sense of allowing for different size, shapes, and orientations among the clusters. However, it comes with some limitations that we can overcome with the Bayesian approach. For example, the Bayesian approach will eventually reach the target distribution, even if it takes some time. The EM algorithm estimator runs the risk of getting stuck in a local maximum if present \citep{stephens1997bayesian, fraley2007bayesian}. In addition, the method only outputs point estimates \footnote{\textbf{Estimation method} vs \textbf{estimator} vs \textbf{estimate}: \textbf{Estimation method} is a general algorithm to produce the estimator. \textbf{An estimate} is the specific value that \textbf{an estimator} takes when observing the specific value, i.e., an estimator is a random variable and the realization of this random variable is called an estimate.}, and produces no estimates concerning the \textbf{uncertainty of the parameters}. However, in Bayesian inference approaches, these problems can be avoided by replacing the MLE by the maximum a posterior (MAP) estimation, i.e., a MAP estimation (Bayesian) framework by maximizing the posterior parameter distribution, that is, $p(\bpi, \bgamma|\mathcalX) \propto p(\bpi, \bgamma)p(\mathcalX|\bpi, \bgamma)$, where $p(\bpi, \bgamma)$ is a chosen prior distribution on the model parameters $\{\bpi, \bgamma\}$, and $p(\mathcalX|\bpi, \bgamma)$ is the likelihood of the data under the mixture model. This is namely achieved by introducing a regularization over the model parameters via prior parameter distributions \footnote{See Section~\ref{sec:bayesian-zero-mean} as an example of how Bayesian approach can do the regularization in the linear model context and refer to \citep{lu2021rigorous2} for more details.}, which are assumed to be uniform distributed in the case of MLE.  And the Bayesian approaches generate point estimates for all variables as well as associated uncertainty in the form of the whole estimates' posterior distribution. 

In order to simplify the likelihood, we can introduce latent variables $z_i$ such that:
\begin{equation}
(\bx_i | z_i = k) \sim \Phi(\bx | \bgamma_k) \quad \mathrm{and} \quad p(z_i=k) = \pi_k. 
\end{equation}
These auxiliary variables allow us to identify the mixture component from which each observation has been generated. Therefore, for each sample of data $\mathcalX = \{\bx_1, \bx_2, \ldots, \bx_N\}$, we assume a missing/latent data set $\bz=\{z_1, z_2, \ldots, z_N\}$, which provides the labels indicating the mixture components from which the observations have been generated. Using this missing data set, the likelihood simplifies to 
\begin{equation}\label{equation:mixture-assign-mle}
\begin{aligned}
p(\bpi, \bgamma | \mathcalX, \bz) &= \prod_{i=1}^N \pi_{z_i} \Phi(\bx_i | \bgamma_{z_i})
				&= \prod_{k=1}^K \pi_k^{N_k}\left[  \prod_{i:z_i=k} \Phi(\bx_i | \bgamma_k)   \right], 
\end{aligned}
\end{equation}
where $N_k$ is the count of component $k$ in $\bz$, i.e., $N_k = \# \{z_i = k\}$, and $N = \sum_{k=1}^KN_k$. 

Considering the finite mixture model in Equation~\eqref{equation:general_mixture_model_expression}, a Bayesian approach
is completed by choosing priors for the number of components $K$, the probability weights $\bpi$, and the component-specific parameters $\bgamma=\{\bgamma_1, \bgamma_2, \ldots, \bgamma_K\}$.
Typically, $K$ is assigned a Poisson or multinomial prior, or $K$ can be chosen with an upper bound of mixture components in an over-fitting mixture model setting, $\bpi$ is assigned a Dirichlet$(\balpha)$ prior with $\balpha = \{\alpha_1, \alpha_2, \ldots, \alpha_K\}$, and $\bgamma_k \sim P0$ independently, with $P0$ often chosen to be conjugate to the kernel $\Phi$. As an example, when $\Phi$ is the multivariate Gaussian kernel/distribution and $\bgamma$ is a matrix containing mean vector and covariance matrix, i.e., $\bgamma =\{\bmu, \bSigma\}$, a normal-inverse-Wishart prior can be assigned to $\bgamma$. This mixture model is often referred as the Gaussian mixture model (GMM). We then introduce the mathematical details of this setting in the following sections.

\paragraph{Model-Based Clustering VS Deterministic Clustering} This kind of model-based clustering arised from the Gaussian mixture model has several advantages compared to traditional, deterministic clustering methods (such as k-means). Deterministic methods use different measures between objects, and between objects and centroids, to create cohesive and homogeneous groups. However, they assume equal structure among clusters, and thus cannot handle clusters of different shapes, sizes and directions. Model-based clustering is better able to handle overlapping groups by taking into account cluster membership probabilities in these areas. 

\section{Bayesian finite Gaussian mixture model}\label{section:review_finite_mixture_model} \label{sec:finite_gaussian_mixture}
This section is primarily based on \citep{murphy2012machine, kamper2013gibbs, lu2017robust, lu2017hyperprior, lu2018reducing, franzen2006nbayesian}.

\subsection{Background}\label{sec:fmm_background}
We present a background of Bayesian finite Gaussian mixture model (GMM) here, also the background can be extended to the situation of infinite Gaussian mixture model. In our case, data is assumed to come from a mixture model of $K$ distributions, where each distribution represents a cluster. All clusters have a multivariate Gaussian distribution, but each with its specific mean vector $\bmu_k$ and covariance matrix $\bSigma_k$. Along with the mean vectors and covariance matrices, the probabilities for each cluster, and the probabilities of a single observation is belonging to a given cluster, are estimated. Assume we have $N$ observations $\mathcalX = \{\bx_1, \bx_2, \ldots, \bx_N\}$ sampled i.i.d., from a finite mixture distribution with density 
\begin{equation}
p(\bx | \bpi, \bmu, \bSigma) = \sum_{k=1}^{K} \pi_k \mathcal{N}(\bx | \bmu_{k}, \bSigma_k),
\end{equation}
with $\bmu_k \in \mathbb{R}^D$, $\bSigma_k \in \mathbb{R}^{D \times D} $, and $K$ being finite and known. We wish to make Bayesian inference for the model parameters $\btheta = \{\bpi, \bmu, \bSigma \}$. The likelihood is, 
\begin{equation}
p(\bpi, \bmu, \bSigma | \mathcalX) = \prod_{i=1}^N \sum_{k=1}^K \pi_k \mathcal{N}(\bx_i | \bmu_k, \bSigma_k).
\end{equation}

\subsection{Bayesian finite Gaussian mixture model}
We will work with the following definition of Bayesian finite Gaussian mixture model
\begin{equation}
\begin{aligned}
\bx_i | z_i, \{\bmu_k, \bSigma_k \} &\sim \mathcal{N}(\bmu_{z_i}, \bSigma_{z_i}) \\
z_i | \bpi    									&\sim \discrete(\pi_1, \ldots, \pi_K) \\
\{\bmu_k, \bSigma_k\}					&\sim \niw(\bbeta) \\
\bpi 											&\sim \dirichlet(\alpha_1, \ldots, \alpha_K),
\end{aligned}
\end{equation}
where $\balpha=\{\alpha_1, \ldots, \alpha_K\}$ is a hyper-parameter to generate the probability vector $\bpi$, and $\bbeta$ is a hyper-parameter to generate mean vectors and covariance matrices for multivariate Gaussian distributions.
Using the introduced latent variables $z_i$'s, the Bayesian finite Gaussian mixture model is illustrated in Figure \ref{fig:gmm_finite_without_hyper}, where hyper-parameters are denoted in green cycles. For each observed data vector $\bxi$, we have a latent variable $z_i \in \{1, 2, \ldots, K\}$ indicating which of the K components $\bxi$ belongs to. Using this latent variable by $\pi_k$ = $p$($z_i = k$), we indicate the prior probability that $\bxi$ belongs to component $k$. Given $z_i = k$, $\bxi$ is generated by the $k^{\textit{th}}$ Gaussian mixture component with mean vector $\bmu_k$ and covariance matrix $\bSigma_k$.

\begin{figure}[h!]
\center
\subfigure[General mixture model]{\includegraphics[width=0.33\textwidth]{img_visual/general_mixture_model.pdf} \label{fig:general_gmm_finite}}
\subfigure[A Bayesian finite GMM]{\includegraphics[width=0.33\textwidth]{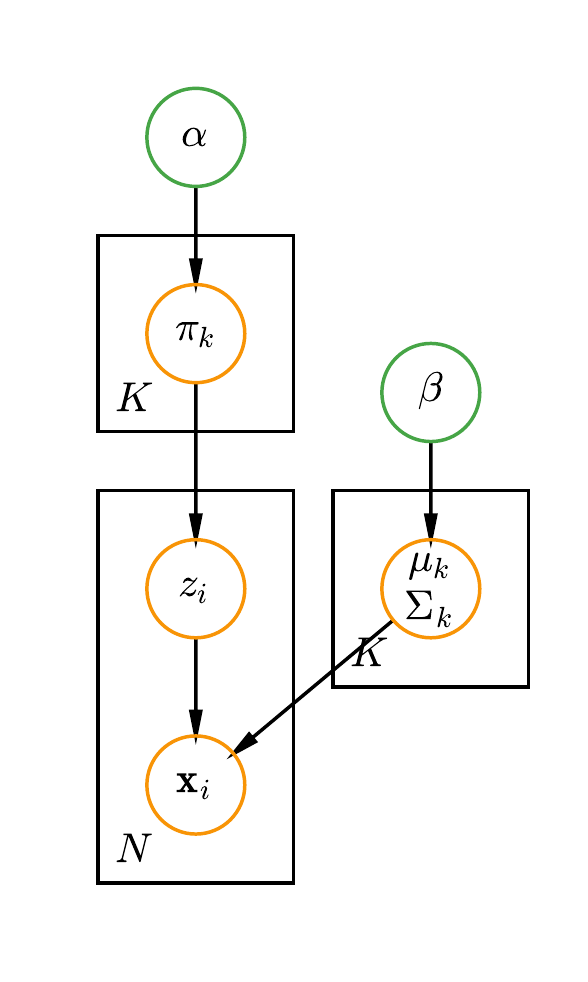} \label{fig:gmm_finite_without_hyper}}
\caption{A Bayesian finite GMM (compared with general mixture models that are not necessarily Gaussian mixture model). Hyper-parameters $\balpha$ and $\bbeta$ are denoted in green cycles. Assume a Dirichlet prior for the probability vector $p(\bpi | \balpha)$  (i.e., the mixture weights). And a NIW prior for the parameters in multivariate Gaussian distributions.}
\label{fig:gmm_finite_without_hyper_allset}
\end{figure}


Here we use a Dirichlet distribution/prior for $p(\bpi | \balpha)$ since Dirichlet distribution is a conjugate prior for the multinomial distribution as introduced in Section \ref{sec:dirichlet_prior_on_multinomial}. The upper left of the Figure~\ref{fig:gmm_finite_without_hyper} shows that we use a Dirichlet distribution as a prior over the mixture weights $\bpi = \{\pi_1, \pi_2, \ldots , \pi_K\}$:
\begin{equation} 
     p(\bpi | \balpha) = \dirichlet(\bpi | \balpha). 
    \label{equation:dirichlet_distribution}
\end{equation}  
If using hyperprior (that is, a prior over a prior) on Dirichlet distribution, we represent the hyper-parameter of the hyperprior as $a, b$. We will give the detail of hyperprior in later sections.

For the mean vector $\bmu_k$ and covariance matrix $\bSigma_k$ of each of the K Gaussian mixture components, again we use a NIW distribution with hyper-parameters $\bbeta  = (\bmo, \kappa_0, \nu_0, \bso)$:
\begin{equation} 
    p(\bmu_k, \bSigma_k | \bbeta) = \niw(\bmu_k, \bSigma_k | \bmo, \kappa_0, \nu_0, \bso). 
    \label{equation:giw}
\end{equation}  
We use NIW as the prior of Gaussian component since the NIW is fully conjugate to the multivariate Gaussian likelihood as introduced in Section~\ref{sec:multi_gaussian_conjugate_prior}.

\subsection{Inference by uncollapsed Gibbs sampling}\label{sec:fmm_uncollapsed_gibbs}
The most widely used posterior inference methods in Bayesian inference models are Markoc Chain Monte Carlo (MCMC) methods as discussed in Section~\ref{sec:monte_carlo_methods}. The idea of MCMC methods is to define a Markov chain on the hidden variables that has the posterior as its equilibrium distribution \citep{andrieu2003introduction}. By drawing samples from this Markov chain, one eventually obtains samples from the posterior. A simple form of MCMC sampling is Gibbs sampling, where the Markov chain is constructed by considering the conditional distribution of each hidden variable given the others and the observations.  

\begin{figure}[h!]
\centering
\includegraphics[width=5cm]{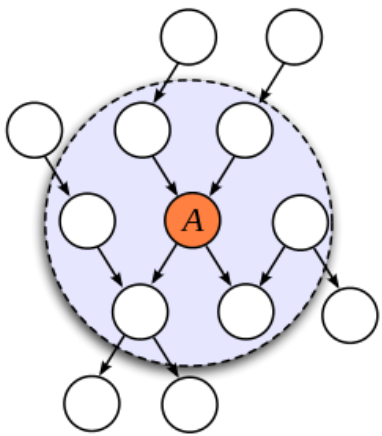}
\caption{The Markov blanket of a directed acyclic graphical model. In a Bayesian network, the Markov blanket of node $A$ includes its \textbf{parents, children and the other parents of all of its children}. That is, the nodes in the cycle are in the Markov blanket of node $A$. The figure is from wikipedia page of Markov blanket.}
\label{fig:markov_blanket}
\end{figure}

To do Gibbs sampling, we need to derive the conditional posterior distributions for each parameters conditioned on all the other parameters $p(\theta_i | \btheta_{-i}, \mathcalX)$, where $\mathcalX$ is again the set of $N$ data points and $\theta_i$'s are the variables for which we want to sample the distributions. But for a graphical model, this conditional distribution is a function only of the nodes in the Markov blanket. For the finite Gaussian mixture model shown in Figure~\ref{fig:gmm_finite_without_hyper}, which is  a directed acyclic graphical (DAG) model, the Markov blanket of a node includes \textbf{the parents, the children, and the co-parents} \citep{jordan2004introduction}, as shown in Figure~\ref{fig:markov_blanket}. The Markov blanket of node A is all nodes in the cycle. 

\paragraph{An Example on the Markov Blanket} This might be mysterious at first glance. Suppose we want to sample $z_i$'s for the distribution of it. From Figure~\ref{fig:gmm_finite_without_hyper}, we find its parents, children, and coparents are $\{\bpi \}$, $\{\mathcalX\}$, and $\{\bmu_k, \bSigma_k\}$, respectively. Therefore the conditional distribution of $z_i$ only depends on the three pairs of parameters:
$$
p(z_i = k | -)  = p(z_i=k | \bznoi, \bpi, \{\bmu_k\}, \{\bSigma_k\}, \mathcalX) .
$$
More specifically, from this graphical representation, we can find the Markov blanket for each parameter in the finite Gaussian mixture model, and then figure out their conditional posterior distributions to be derived:
\begin{align}
	p(z_i = k | -) 					&= p(z_i=k | \bznoi, \bpi, \{\bmu_k\}, \{\bSigma_k\}, \mathcalX)  \label{equation:fmm_blocked_z} \\
	p(\bpi | -)   					    &= p(\bpi | \bz, \balpha) \label{equation:fmm_blocked_pi}\\
	p(\bmu_k, \bSigma_k | -) &= p(\bmu_k, \bSigma_k | \mathcalX, \bz, \bbeta)  \label{equation:fmm_blocked_mu_sigma}.
\end{align} 
In other words, Gibbs sampler moves the chain forward by one step as follows:
\begin{itemize}
\item Sample the cluster assignment $z_i$ for each observation from Equation~\eqref{equation:fmm_blocked_z} which is known as the \textbf{conditional distribution of assignment};
\item Sample the mixture weights $\bpi$ for each cluster from Equation~\eqref{equation:fmm_blocked_pi}, which is known as the \textbf{conditional distribution of mixture weights};
\item Sample the cluster mean vector $\bmu$ and covariance matrix $\bSigma$ for each cluster  from Equation~\eqref{equation:fmm_blocked_mu_sigma} which is known as the \textbf{conditional distribution of cluster parameters}.
\end{itemize}

\subsubsection{Conditional distribution of assignment}
As discussed, the conditional distribution of assignment for each observation is given by 
\begin{equation}
\begin{aligned}
p(z_i = k | -) &= p(z_i=k | \bznoi, \bpi, \{\bmu_k\}, \{\bSigma_k\}, \mathcalX) \\
&=\frac{p(z_i=k , \bznoi, \bpi, \{\bmu_k\}, \{\bSigma_k\}, \mathcalX) }{\cancel{p(\bznoi, \bpi, \{\bmu_k\}, \{\bSigma_k\}, \mathcalX) }}\\
&\propto p(\bz, \bpi,  \{\bmu_k\}, \{\bSigma_k\}, \mathcalX) \\
&= \prod_{n=1}^N \prod_{k=1}^K \left[\pi_k \mathcal{N}(\bx_n | \bmu_k, \bSigma_k)\right]^{\delta(z_n, k)}\\
&\propto \pi_k \normal(\bx_i | \bmu_k, \bSigma_k), 
\end{aligned}
\end{equation}
which comes from the fact that cluster assignements are conditionally independent given the cluster weights and paramters and $\delta(z_n, k)=1$ if $z_n=k$ and $\delta(z_n, k)=0$ if $z_n\neq k$. This equation intuitively makes sense: data point $i$ is more likely to be in cluster $k$ if $k$ is itself probable ($\pi_k\gg 0$ or $\pi_k > \pi_l$ for $k\neq l$) and $x_i$ has large probability in the $k^{th}$ component (i.e., the probability of $\normal(\bx_i | \bmu_k, \bSigma_k)$ is large).

In Gibbs sampling, for each data point $i$, we can compute $p(z_i=k | -)$ by $\pi_k \normal(\bx_i | \bmu_k, \bSigma_k)$ from the above deduction for each of cluster $k$. These values are the unnormalized parameters to a discrete distribution (since we use proportional distribution in the above deduction) from which we can sample assignments by normalizing it. Let $o_m = \pi_m \normal(\bx_i | \bmu_m, \bSigma_m)$, for $m\in \{1, 2, \ldots, K\}$, the probability to output $z_i=k$ is
$$
\frac{o_k}{\sum_{m=1}^{K} o_m}.
$$

\subsubsection{Conditional distribution of mixture weights}
We can similarly derive the conditional distributions of mixture weights by an application of Bayes' theorem. Instead of updating each component of $\bpi$ separately, we update them together (this is also known as the \textbf{blocked Gibbs sampling}).

\begin{equation}
\begin{aligned}
p(\bpi | -) &= p(\bpi | \bz, \balpha) \\
&=\frac{ p(\bpi , \bz, \balpha)}{ \cancel{p( \bz, \balpha)}} \\
&\propto p(\bpi, \bz, \balpha) \\
&= \cancel{p(\balpha)} p(\bpi | \balpha) p(\bz | \bpi, \cancel{\balpha}), \\
\end{aligned}
\end{equation}
where $p(\bpi | \balpha) = \dirichlet(\bpi | \balpha) \propto \prod_{k=1}^K \pi_k^{\alpha_k -1}$, and $p(\bz | \bpi, \cancel{\balpha})$ can be seen as a multinomial distribution $p(\bz | \bpi, \cancel{\balpha})=\multinomial_K(\bz | N, \bpi)\propto \prod_{k=1}^K \pi_k^{N_k}$, where $N_k$ is the number of $z_i$'s assigned in cluster $k$. From Section~\ref{section:dirichlet-dist-post}, therefore, we obtain 
\begin{equation}
	\begin{aligned}
		p(\bpi | -) &= p(\bpi | \bz, \balpha) \\
		&\propto \prod_{k=1}^K \pi_k^{\alpha_k -1} \prod_{k=1}^K \pi_k^{N_k} \\
		&\propto \dirichlet(N_1+\alpha_1, N_2+\alpha_2, \ldots, N_K+\alpha_K). 
	\end{aligned}
\end{equation}

\subsubsection{Conditional distribution of cluster parameters}
Finally, we need to compute the conditional distribution for the cluster means and covariance matrices:

\begin{equation}
\begin{aligned}
p(\bmu_k, \bSigma_k | -) &= p(\bmu_k, \bSigma_k | \mathcalX, \bz, \bbeta)\\
				&= p(\bmu_k, \bSigma_k | \mathcalX_k, \bbeta)
\end{aligned}
\end{equation}
where $\mathcalX_k$ is the data points in $k^{th}$ cluster, and this equation can be easily calculated from Equation~\eqref{equation:niw_posterior_equation_1}.

The pseudo code for uncollapsed Gibbs sampler for a finite Gaussian mixture model is given by Algorithm \ref{algo:fmm_plain_gibbs-uncollapsed}.

\IncMargin{1em}
\begin{algorithm}
\SetKwData{Left}{left}\SetKwData{This}{this}\SetKwData{Up}{up}
\SetKwFunction{Union}{Union}\SetKwFunction{FindCompress}{FindCompress}
\SetKwInOut{Input}{input}\SetKwInOut{Output}{output}
\Input{Choose an initial $\bz$}
\BlankLine

\For{$T$ iterations}{
\For{$i \leftarrow 1$ \KwTo $N$}{
Remove $\bxi$'s statistics from component $z_i$ \;
Calculate weights $p(\bpi | \bz, \balpha)=\dirichlet(N_1+\alpha_1, N_2+\alpha_2, \ldots, N_K+\alpha_K)$\;
\For{$k\leftarrow 1$ \KwTo $K$}{
Calculate Gaussian parameters for each cluster: $p(\bmu_k, \bSigma_k | \mathcalX_k, \bbeta)$ \;
Calculate $p(z_i = k | -) \propto \pi_k \normal(\bx_i | \bmu_k, \bSigma_k)$\;
}
Sample $k_{new}$ from $\frac{o_k}{\sum_{m}^{K} o_m}$ after normalizing \;
Add $\bxi$'s statistics to the component $z_i=k_{new}$ \;
}
}
\caption{Uncollapsed Gibbs sampler for a finite Gaussian mixture model}\label{algo:fmm_plain_gibbs-uncollapsed}
\end{algorithm}\DecMargin{1em}

\subsection{Inference by collapsed Gibbs sampling}\label{sec:fmm_collabsed_gibbs}



Since we choose $p(\bpi | \balpha)$ and $p(\bmu_k, \bSigma_k | \bbeta)$ to be conjugate, we are able to analytically integrate out the model parameters $\bpi$, $\bmu_k$ and $\bSigma_k$ and only sample the component assignments $\bz$. This is known as a \textbf{collapsed Gibbs sampler}, and the discussion about this Gibbs sampler can be found in \citep{neal2000markov, murphy2012machine} and many other articles. The collapsed Gibbs sampler is done as follows:
\begin{equation} 
\begin{aligned}
p(z_i = k| \bznoi, \mathcal{X} , \balpha, \bbeta)  &= 
\frac{p(z_i = k, \bznoi, \mathcal{X} , \balpha, \bbeta) }{\cancel{p(\bznoi, \mathcal{X} , \balpha, \bbeta)} }\\
&\propto p(z_i = k, \bznoi, \mathcal{X} , \balpha, \bbeta) \\
&=p(z_i = k, \bznoi, \balpha, \bbeta) p(\mathcal{X} |z_i = k, \bznoi, \cancel{\balpha}, \bbeta) \\
&=p(z_i = k| \bznoi, \balpha, \bbeta)\cancel{p( \bznoi, \balpha, \bbeta)} p(\mathcal{X} |z_i = k, \bznoi, \cancel{\balpha}, \bbeta) \\
& \varpropto p(z_i = k | \bznoi, \balpha, \cancel{\bbeta})  p(\mathcal{X} |z_i = k, \bznoi, \cancel{\balpha}, \bbeta) \\
& = p(z_i = k| \bznoi, \balpha) p(\bxi |\mathcal{X}_{-i}, z_i = k, \bznoi, \bbeta) p(\mathcal{X}_{-i} |\cancel{z_i = k}, \bznoi, \bbeta)\\
& \varpropto p(z_i = k| \bznoi, \balpha) p(\bxi|\mathcal{X}_{-i}, z_i = k, \bznoi, \bbeta)	 \\
&=  p(z_i = k| \bznoi, \balpha)  p(\bxi | \xknoi, \bbeta),										
\end{aligned}
\label{equation:fmm_collabsed_gibbs}
\end{equation}  
where $\xknoi$ is the set of data points assigned to component $k$ without taking $\bxi$ into account. 

Typically the $\balpha$ hyper-parameter is set to $\alpha_k = \alpha = \alpha_+/K$ (termed as \textbf{standard setting of Dirichlet distribution}, also known as a \textbf{symmetrical Dirichlet prior}). For this setting we also have $\alpha_+ =\sum_{k=1}^K \alpha_k $. In the following section, we will give the solutions and proofs for both the unsymmetrical and symmetrical Dirichlet prior. We give the expressions for the first and second terms on the right hand side of Equation~\eqref{equation:fmm_collabsed_gibbs} respectively in the following two sections.

\subsubsection{First term: $ p(z_i = k| \bznoi, \balpha)$}
We can express $p(z_i = k| \bznoi, \balpha)$ in Equation~\eqref{equation:fmm_collabsed_gibbs} as follows
\begin{align*}
p(z_i = k| \bznoi, \balpha) = \frac{p(z_i = k, \bznoi | \balpha)}{p(\bznoi | \balpha)} = \frac{p(\bz |\balpha)}{p(\bznoi | \balpha)},
\end{align*}
We can find the numerator and denominator of the above equation is marginal $p(\bz |\balpha)$ with $N$ and $N-1$ samples in $\bz$ respectively. Thus we can calculate both the numerator and denominator above if we can find an expression for the marginal $p(\bz |\balpha)$ with appropriate modification for the observation sets.

By marginalizing out $\bpi$ for the following equation:
\begin{align*}
p(\bz |\balpha) = \int_{\bpi} p(\bz | \bpi) p(\bpi | \balpha) d \bpi. 
\end{align*}
Again, the first term in the integrand is from multinomial distribution
\begin{align*}
p(\bz |\bpi) =\multinomial_K(\bz | N, \bpi) \propto \prod_{k=1}^K \pi_k^{N_k}, 
\end{align*}
where $N_k$ is the count of component $k$ in $\bz$. And again, the second term in the integrand is given in Equation~\eqref{equation:dirichlet_distribution}: $p(\bpi | \balpha) = \dirichlet(\bpi | \balpha) \propto \prod_{k=1}^K \pi_k^{\alpha_k -1}$.
We can thus marginalize: 
\begin{align}
p(\bz | \balpha) &= \int_{\bpi} p(\bz | \bpi) p(\bpi | \balpha) d\bpi \\
&\propto \int_{\bpi} \prod_{k=1}^K \pi_k^{N_k} \frac{1}{D(\balpha)} \prod_{k=1}^K \pi_k^{\alpha_{k}-1} d\bpi \\
&\overset{(a)}{=} \frac{1}{D(\balpha)} \int_{\bpi} \prod^K_{k=1} \pi_k^{N_k+\alpha_k-1} d\bpi \\
&\overset{(b)}{=} \frac{\Gamma(\alpha_+)}{\Gamma(N + \alpha_+)} \prod^K_{k=1} \frac{\Gamma(N_k + \alpha_k)}{\Gamma(\alpha_k)} \qquad &\text{(unsymmetric $\balpha$ setting)}  \label{equation:fmm_z_alpha}\\
&\overset{(c)}{=} \frac{\Gamma(\alpha_+)}{\Gamma(N + \alpha_+)} \prod^K_{k=1} \frac{\Gamma(N_k + \alpha_+/K)}{\Gamma(\alpha_+/K)}. \qquad &\text{(symmetric $\balpha$ setting)}  \label{equation:fmm_z_alpha2}
\end{align}
The Equation ($b$) above follows from the Equation ($a$) since the integral reduces to the normalizing constant of the
Dirichlet distribution proportional to $\prod^K_{k=1} \pi_k^{N_k+\alpha_k-1}$ (see Section~\ref{section:dirichlet-dist}). In Equation ($c$) above we use the standard symmetric $\balpha$ setting where $\alpha_k = \alpha = \alpha_+/K$ for $k=1, \ldots, K$. 
Therefore, we can find an expression for the desired term:
\begin{align}
p(z_i = k | \bznoi, \balpha) &= \frac{p(z_i = k, \bznoi | \alpha)}{p(\bznoi | \alpha)} = \frac{p(\bz |\alpha)}{p(\bznoi | \alpha)} \\
&\propto \frac{\frac{\Gamma(\alpha_+)}{\Gamma(N+\alpha_+)}\frac{\Gamma(N_k+\alpha_k)}{\Gamma(\alpha_k)} \prod^K_{j=1,j\neq k}\frac{\Gamma(N_j+\alpha_j )}{\Gamma(\alpha_j)}}{\frac{\Gamma(\alpha_+)}{\Gamma(N+\alpha_+-1)} \frac{\Gamma(\nknoi+\alpha_k)}{\Gamma(\alpha_k)} \prod^K_{j=1,j\neq k} \frac{\Gamma(N_j+\alpha_j)}{\Gamma(\alpha_j )}}\\
&=\frac{\Gamma(N + \alpha_+ - 1)}{\Gamma(N + \alpha_+)} \frac{\Gamma(N_k + \alpha_k)}{\Gamma(\nknoi + \alpha_k)}\\
&= \frac{\nknoi + \alpha_k}{N + \alpha_+ - 1}  \qquad\qquad\qquad \text{(unsymmetric $\balpha$ setting)}\label{equation:fmm_first_term_derivation}\\
&= \frac{\nknoi + \alpha_+/K}{N + \alpha_+ - 1}.  \qquad\qquad \text{(symmetric $\balpha$ setting)}  \label{equation:fmm_first_term_derivation2}
\end{align}
where we used $\Gamma(x+1) = x\Gamma(x)$ and $\nknoi = N_k - 1$ is the number of observations in cluster $k$ except $\bx_i$. Note that the latter statement $\nknoi = N_k - 1$ is not true in general. But in our case, $\nknoi = N_k - 1$ comes from the fact that we set $z_i=k$ for the numerator in the first equation (i.e. $p(\bz |\alpha)$). In Equation~\eqref{equation:fmm_first_term_derivation2}, we still use standard symmetric $\balpha$ setting, in which case, $\alpha_k=\alpha = \alpha_+/K$ for $k=1, \ldots, K$ for the convenience for the following limiting analysis in Section \ref{sec:infinite_gaussian_mixture}.

\subsubsection{Second term: $p(\bxi | \xknoi, \bbeta)$}\label{sec:fmm_second_term}


The second term in Equation~\eqref{equation:fmm_collabsed_gibbs} can be written as
\begin{equation}
p(\bxi| \xknoi, \bbeta) = \frac{p(\bxi, \xknoi | \bbeta)}{p(\xknoi | \bbeta)} = \frac{p(\mathcal{X}_{k} | \bbeta)}{p(\xknoi | \bbeta)}, 
\label{equation:second_equivalent_form}
\end{equation}
where $\bxi$ is assumed to be assigned to component $k$ in the numerator. 

Same as in the previous section, we can calculate both the numerator and denominator above if we can find an expression for the marginal $p(\mathcal{X}_k | \bbeta)$ with $N_k$ and $N_k - 1$ samples in each observation set, where $N_k$ is the number of samples in cluster $k$. We can easily get this equation by marginalizing out $\bmu_k$ and $\bSigma_k$:
\begin{equation}
\begin{aligned}
p(\mathcal{X}_k|\bbeta) &= \int_{\bmu_k} \int_{\bSigma_k} p(\mathcal{X}_k, \bmu_k, \bSigma_k | \bbeta) d\bmu_k d\bSigma_k \\
	&= \int_{\bmu_k} \int_{\bSigma_k} p(\mathcal{X}_k|\bmu_k, \bSigma_k) p(\bmu_k, \bSigma_k|\bbeta)  d\bmu_k d\bSigma_k,
\end{aligned}
\label{equation:fmm_second_term_marginal}
\end{equation}
which is exactly the posterior marginal likelihood of data in multivariate Gaussian distribution under normal-inverse-Wishart prior.
We realize that the marginalization in Equation~\eqref{equation:fmm_second_term_marginal} is exactly equivalent to the marginalization performed in Equation~\eqref{equation:niw_marginal_data} with appropriate modification in the observation set. Let $N_m = N_k-1$, we obtain  
\begin{equation}
\begin{aligned}
p(\mathcalX_k|\bbeta) 
&= (2\pi)^{-N_{k}D/2} \frac{Z_{\niw}(D, \kappa_{N_{k}}, \nu_{N_{k}}, \bS_{N_{k}})}{Z_{\niw}(D, \kappa_0, \nu_0, \bS_0)} \\
&= \pi^{-\frac{{N_{k}}D}{2}} \cdot\frac{\kappa_0^{D/2}\cdot |\bS_0|^{\nu_0/2}}{\kappa_{N_{k}}^{D/2}\cdot |\bS_{N_{k}}|^{\nu_{N_{k}}/2}} \prod_{d=1}^D \frac{\Gamma(\frac{\nu_{N_{k}}+1-d}{2})}{\Gamma(\frac{\nu_0+1-d}{2})},\\ 
p(\mathcalX_{k,-i}|\bbeta) 
&= (2\pi)^{-N_{m}D/2} \frac{Z_{\niw}(D, \kappa_{N_{m}}, \nu_{N_{m}}, \bS_{N_{m}})}{Z_{\niw}(D, \kappa_0, \nu_0, \bS_0)} \\
&= \pi^{-\frac{{N_{m}}D}{2}} \cdot\frac{\kappa_0^{D/2}\cdot |\bS_0|^{\nu_0/2}}{\kappa_{N_{m}}^{D/2}\cdot |\bS_{N_{m}}|^{\nu_{N_{m}}/2}} \prod_{d=1}^D \frac{\Gamma(\frac{\nu_{N_{m}}+1-d}{2})}{\Gamma(\frac{\nu_0+1-d}{2})}, 
\end{aligned}
\end{equation}
This implies
\begin{equation}\label{equation:finite-second-term}
\begin{aligned}
p(\bxi| \xknoi, \bbeta) & = \frac{p(\mathcal{X}_{k} | \bbeta)}{p(\xknoi | \bbeta)}\\
&=(2\pi)^{-D/2} \frac{Z_{\niw}(D, \kappa_{N_{k}}, \nu_{N_{k}}, \bS_{N_{k}})}{Z_{\niw}(D, \kappa_{N_{m}}, \nu_{N_{m}}, \bS_{N_{m}})}\\
&=\pi^{-\frac{D}{2}} \cdot
\frac{\kappa_{N_{m}}^{D/2}\cdot |\bS_{N_{m}}|^{\nu_{N_{m}}/2}}{\kappa_{N_{k}}^{D/2}\cdot |\bS_{N_{k}}|^{\nu_{N_{k}}/2}} 
\prod_{d=1}^D \frac{\Gamma(\frac{\nu_{N_{k}}+1-d}{2})}{\Gamma(\frac{\nu_{N_{m}}+1-d}{2})}.\\
\end{aligned}
\end{equation}
Again, an alternative form is given by
\begin{equation*}
p(\bxi| \xknoi, \bbeta) = \tau (\bxi |  \bmm_{N_m}, \frac{\kappa_{N_m} + 1}{\kappa_{N_m} (\nu_{N_m} - D + 1)} \bS_{N_m}, \nu_{N_m} - D + 1),
\end{equation*}
where $\bmm_{N_m}, \bS_{N_m},  \nu_{N_m}, \kappa_{N_m}$ can be calculated from Equation~\eqref{equation:niw_posterior_equation_1}.
In fact, we may also notice that Equation~\eqref{equation:second_equivalent_form} is equivalent to the posterior predictive for new data with observations in Equation~\eqref{equation:niw_posterior_predictive_abstract}. The full expression for Equation~\eqref{equation:second_equivalent_form} is thus given by Equation~\eqref{equation:niw_posterior_predictive_equation} with appropriate changes to the observation sets in numerator and denominator.

\IncMargin{1em}
\begin{algorithm}
	\SetKwData{Left}{left}\SetKwData{This}{this}\SetKwData{Up}{up}
	\SetKwFunction{Union}{Union}\SetKwFunction{FindCompress}{FindCompress}
	\SetKwInOut{Input}{input}\SetKwInOut{Output}{output}
	\Input{Choose an initial $\bz$}
	\BlankLine
	
	\For{$T$ iterations}{
		\For{$i \leftarrow 1$ \KwTo $N$}{
			Remove $\bxi$'s statistics from component $z_i$ \;
			\For{$k\leftarrow 1$ \KwTo $K$}{
				Calculate $p(z_i=k| \bznoi, \balpha)$ \;
				Calculate $p(\bxi | \xknoi, \bbeta)$\;
				Calculate $p(z_i = k | \bznoi, \mathcal{X}, \balpha, \bbeta) \propto p(z_i=k| \bznoi, \balpha) p(\bxi | \xknoi, \bbeta)$\;
			}
			Sample $k_{new}$ from $p(z_i | \bznoi, \mathcalX, \balpha, \bbeta)$ after normalizing\;
			Add $\bxi$'s statistics to the component $z_i=k_{new}$ \;
		}
	}
	\caption{Collapsed Gibbs sampler for a finite Gaussian mixture model}\label{algo:fmm_plain_gibbs}
\end{algorithm}\DecMargin{1em}

The pseudo code for collapsed Gibbs sampler for a finite Gaussian mixture model is given by Algorithm \ref{algo:fmm_plain_gibbs}.

\subsection{Get the posterior distribution for every parameter}
We leave this derivation in Bayesian infinite mixture model case, i.e., Section \ref{sec:ifmm_posterior_for_all_parameters}. The derivation in finite and infinite cases are the same.

\subsection{Hyperprior on symmetric Dirichlet distribution}\label{section:hyper_fbgmm_background}

\begin{figure}[h!]
\center
\subfigure[Without hyperprior]{\includegraphics[width=0.33\textwidth]{img_visual/bgmm_finite_without_hyper.pdf} \label{fig:gmm_finite_with_hyper_compare}}
\subfigure[With hyperprior]{\includegraphics[width=0.33\textwidth]{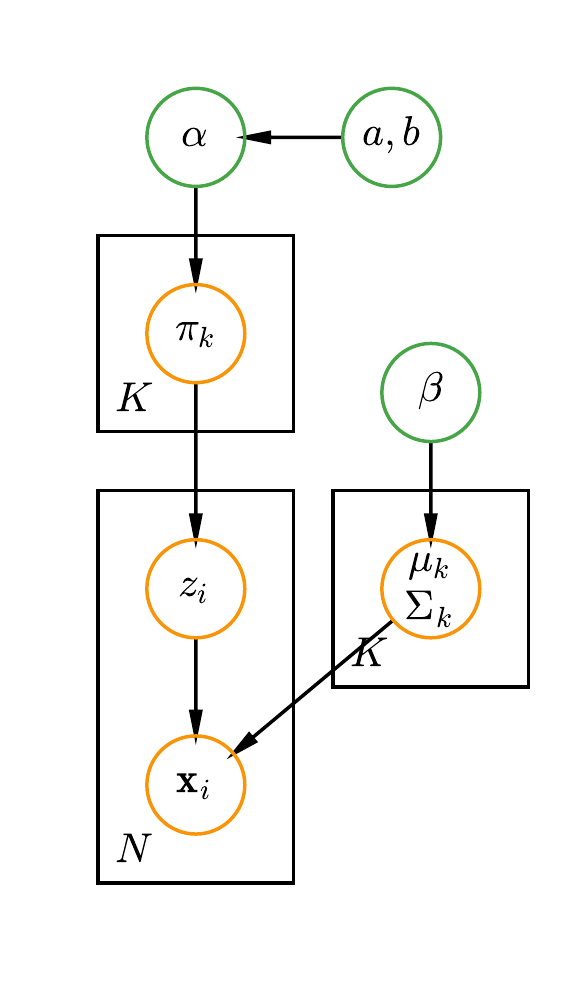} \label{fig:gmm_finite_with_hyper}}
\caption{A Bayesian finite GMM with hyperprior on concentration parameter (compared with non-hyperprior version).}
\label{fig:gmm_finite_with_hyper_allset}
\end{figure}

It has become popular to use over-fitted mixture models in which number of cluster $K$ is chosen as a conservative upper bound on the number of components under the expectation that only relatively few of the components $K^\prime$ will be occupied by data points in the samples $\mathcalX$. This kind of over-fitted mixture models has been successfully due to the ease in computation. 
Previously, \citep{rousseau2011asymptotic} proved that quite generally, the posterior behaviour of overfitted mixtures depends on the chosen prior on the weights, and on the number of free parameters in the emission distributions (here $D$, i.e., the dimension of data). Specifically, they have proved that (a) If $\underline{\alpha}$=min$(\alpha_k, k \leq K)>D/2$ and if the number of components is larger than it should be, asymptotically two or more components in an overfitted mixture model will tend to merge with non-negligible weights. (b) In contrast, if $\overline{\alpha}$=max$(\alpha_k, k \leqslant K)<D/2$, the extra components are emptied at a rate of $N^{-1/2}$. Hence, if none of the components are small, it implies that $K$ is probably not larger than $K_0$. In the intermediate case, if min$(\alpha_k, k \leq K)\leq D/2 \leq$ max$(\alpha_k, k \leqslant K)$, then the situation varies depending on the $\alpha_k$'s and on the difference between $K$ and $K_0$. In particular, in the case where all $\alpha_k$'s are equal to $D/2$, then although the author does not prove definite result, they conjecture that the posterior distribution does not have a stable limit. See also an example conducted in Section~\ref{section:asymptotic}.

As introduced in \citep{rasmussen1999infinite} and further discussed in \citep{gorur2010dirichlet, lu2017hyperprior}, they introduced a hyperprior on symmetric Dirichlet distribution prior. We here put a vague prior of Gamma shape on the concentration parameter $\alpha$ and use the standard symmetric $\balpha$ setting where $\alpha_k = \alpha = \alpha_+/K$ for $k=1, \ldots, K$. The hyperprior is defined by Gamma distribution follows:
\begin{equation*}
\alpha | a, b \sim \gammadist(a, b) \Longrightarrow p(\alpha | a,b) \propto \alpha^{a-1} e^{-b\alpha}. 
\end{equation*}
To get the conditioned posterior distributions on $\alpha$ we need to derive the conditioned posterior distributions on all the other parameters. But for a graphical model, this conditional distribution is a function only of the nodes in the Markov blanket (see Section~\ref{sec:fmm_uncollapsed_gibbs}). In our case, the Bayesian finite Gaussian mixture model, a directed acyclic graphical (DAG) model, the Markov blanket includes the parents ($a, b$), the children ($\pi_k$'s), and the co-parents (none in this case), as shown in Figure \ref{fig:gmm_finite_with_hyper}. From this graphical representation, we can find the Markov blanket for each parameter in the model, and then figure out their conditional posterior distribution to be derived:

\begin{equation}
\begin{aligned}
p(\alpha | \bpi, a, b) &\propto p(\alpha | a,b) p(\bpi | \alpha) \\
					&\propto \alpha^{a-1} e^{-b\alpha} \frac{\Gamma(K \alpha)}{\prod_{k=1}^K \Gamma(\alpha)} \prod_{k=1}^K \pi_k^{\alpha-1} \\
					&= \alpha^{a-1} e^{-b\alpha} (\pi_1\times \pi_2 \times \ldots \times\pi_K)^{\alpha-1} \frac{\Gamma(K\alpha)}{[\Gamma(\alpha)]^K}. 
\end{aligned}
\end{equation}

\begin{svgraybox}
\begin{theorem}\label{theorem:gx_log_concave}
Define the $G$ function:
\begin{equation}
	G(x) = \frac{\Gamma(K x)}{[\Gamma(x)]^K}. 
	\label{equation:gx_log_concave}
\end{equation}
For $x > 0$ and an arbitrary positive integer $K$,  the function $G$ is strictly log-concave. 
\end{theorem}
\end{svgraybox}

\begin{proof}[of Theorem~\ref{theorem:gx_log_concave}]
Follow from \citep{abramowitz1966handbook} we obtain 
$$
\Gamma(Kx) = (2\pi)^{\frac{1}{2}(1-K)} K^{K x-\frac{1}{2} } \prod_{i=0}^{K-1}\Gamma(x + \frac{i}{K}).
$$  
Then, taking log, we have
\begin{equation}
\log G(x) = \frac{1}{2}(1-K) \log(2\pi) + (Kx - \frac{1}{2})\log K + \sum_{i=0}^{K-1} \log \Gamma(x + \frac{i}{K}) - K \log \Gamma(x),
\end{equation}
and the derivative 
\begin{equation}
[\log G(x)]^\prime =K \log K + \sum_{i=0}^{K-1} \Psi(x + \frac{i}{K}) - K \Psi(x),
\end{equation}
where $\Psi(x)$ is the Digamma function, and its derivative is given by
\begin{equation}
\Psi^\prime(x)  = \sum_{h=0}^\infty \frac{1}{(x+h)^2}. 
\label{equation:fbgmm_hyperprior_digamma_derivative}
\end{equation}
Therefore, the second derivative of $\log G(x)$ is given by
\begin{equation}
[\log G(x)]^{\prime \prime} = \left[\sum_{i=0}^{K-1} \Psi^\prime(x + \frac{i}{K}) \right]- K \Psi^\prime(x) < 0, \quad (x>0). 
\end{equation}
The last inequality comes from Equation~\eqref{equation:fbgmm_hyperprior_digamma_derivative} and concludes the theorem. 
\end{proof}
Note that the theorem above is a general case of Theorem 1 in \citep{merkle1997log}. See Appendix~\ref{appendix:revisit} of \citep{lu2017revisit} for further discussion on the convexity of ratio of Gamma functions.

\begin{svgraybox}
\begin{theorem}\label{theorem:finite-GMM-hyper-alpha-logconvex}
	In $p(\alpha | \bpi, a, b)$, when $a \geq 1$, $p(\alpha | \bpi, a, b)$ is log-concave.
\end{theorem}
\end{svgraybox}

\begin{proof}[of Theorem~\ref{theorem:finite-GMM-hyper-alpha-logconvex}]
It is easy to verify that $ \alpha^{a-1} e^{-b\alpha} (\pi_1 \ldots \pi_K)^{\alpha-1} $ is log-concave when $a \geq 1$. In view of that the product of two log-concave functions is log-concave and Theorem \ref{theorem:gx_log_concave}, it follows that $\frac{\Gamma(K\alpha)}{[\Gamma(\alpha)]^K}$ is log-concave. This concludes the proof.
\end{proof}

The conditional posterior for $\alpha$ depends only on the weight of each cluster. The distribution $p(\alpha | \bpi, a, b)$ is log-concave when $a\geq 1$, so we may efficiently generate independent samples from this distribution using Adaptive Rejection Sampling (ARS) technique, see Section~\ref{section:ars-sampling} and \citep{gilks1992adaptive} for more details on ARS.

Although the proposed hyperprior on Dirichlet distribution prior for mixture model  is generic, we focus on its application in Gaussian mixture models for concreteness. We develop a collapsed Gibbs sampling algorithm based on \citep{neal2000markov} for posterior computation.

Again, let $\mathcalX$ be the data observations, assumed to follow a mixture of multivariate Gaussian distributions. We use a conjugate normal-inverse-Wishart (NIW) prior $p(\bmu, \bSigma | \bbeta)$ for the mean vector $\bmu$ and covariance matrix $\bSigma$ in each multivariate Gaussian component, where $\bbeta$ consists of all the hyperparameters in NIW. A key quantity in a collapsed Gibbs sampler is the probability of each customer $i$ sitting with table $k$: $p(z_i = k | \bznoi, \mathcalX, \alpha, \bbeta)$, where $\bznoi$ are the seating assignments of all the other customers and $\alpha$ is the concentration parameter in Dirichlet distribution (a symmetric one, i.e., $\balpha=[\alpha, \alpha, \ldots, \alpha]$). The derivation is exactly the same as that in Section~\ref{sec:fmm_collabsed_gibbs}, except that we are now using symmetric Dirichlet distribution and the concentration parameter is now becoming a scalar: $\balpha \rightarrow \alpha$. This probability is calculated as follows:
\begin{equation} 
	\begin{aligned}
		p(z_i = k| \bznoi, \mathcal{X} , \alpha, \bbeta)  & \varpropto p(z_i = k | \bznoi, \alpha, \cancel{\bbeta})  p(\mathcal{X} |z_i = k, \bznoi, \cancel{\alpha}, \bbeta) \\
		& = p(z_i = k| \bznoi, \alpha) p(\bxi |\mathcal{X}_{-i}, z_i = k, \bznoi, \bbeta) p(\mathcal{X}_{-i} |\cancel{z_i = k}, \bznoi, \bbeta)\\
		& \varpropto p(z_i = k| \bznoi, \alpha) p(\bxi|\mathcal{X}_{-i}, z_i = k, \bznoi, \bbeta) \\
		& \varpropto p(z_i = k| \bznoi, \alpha) p(\bxi | \xknoi, \bbeta), 										
	\end{aligned}
	\label{equation:sdir_fmm_collabsed_gibbs}
\end{equation}  
where $\xknoi$ are the observations in table $k$ excluding the $i^{th}$ observation. Algorithm~\ref{algo:sdir_fmm_plain_gibbs} gives the pseudo code of the collapsed Gibbs sampler to implement hyperprior for Dirichlet distribution prior in Gaussian mixture models. Note that ARS may require even 10-20 times the computational effort per iteration over sampling once from a gamma density and there is the issue of mixing being worse if we don’t marginalize out the $\pi$ in updating $\alpha$.  So this might have a very large impact on effective sample size (ESS) of the Markov chain. Hence, marginalizing out $\bpi$ and using an approximation to the conditional distribution (perhaps with correction through an accept/reject step via usual Metropolis-Hastings or even just using importance weighting without the accept/reject) or even just a Metropolis-Hastings normal random walk for $\log(\alpha)$ may be much more efficient than ARS in practice. We here only introduce the update by ARS.

\IncMargin{1em}
\begin{algorithm}
	\SetKwData{Left}{left}\SetKwData{This}{this}\SetKwData{Up}{up}
	\SetKwFunction{Union}{Union}\SetKwFunction{FindCompress}{FindCompress}
	\SetKwInOut{Input}{input}\SetKwInOut{Output}{output}
	\Input{Choose an initial $\bz$, $\alpha$ and $\bbeta$;}
	\BlankLine
	
	\For{$T$ iterations}{
		\For{$i \leftarrow 1$ \KwTo $N$}{
			Remove $\bxi$'s statistics from component $z_i$ \;
			\For{$k\leftarrow 1$ \KwTo $K$}{
				Calculate $p(z_i=k| \bznoi, \alpha)$ \;
				Calculate $p(\bxi | \xknoi, \bbeta)$\;
				Calculate $p(z_i = k | \bznoi, \mathcal{X}, \alpha, \bbeta) \propto p(z_i=k| \bznoi, \alpha) p(\bxi | \xknoi, \bbeta)$\;
			}
			Sample $k_{new}$ from $p(z_i | \bznoi, \mathcalX, \alpha, \bbeta)$ after normalizing\;
			Add $\bxi$'s statistics to the component $z_i=k_{new}$ \;
		}
		$\star$ Draw current weight variable $\bpi = \{\pi_1, \pi_2, \ldots, \pi_K\}$ \;
		$\star$ Update $\alpha$ using ARS\;
	}
	\caption{Collapsed Gibbs sampler for a finite Gaussian mixture model with hyperprior on Dirichlet distribution}\label{algo:sdir_fmm_plain_gibbs}
\end{algorithm}\DecMargin{1em}

\begin{figure}[!ht]
\centering
\includegraphics[width=0.95\textwidth]{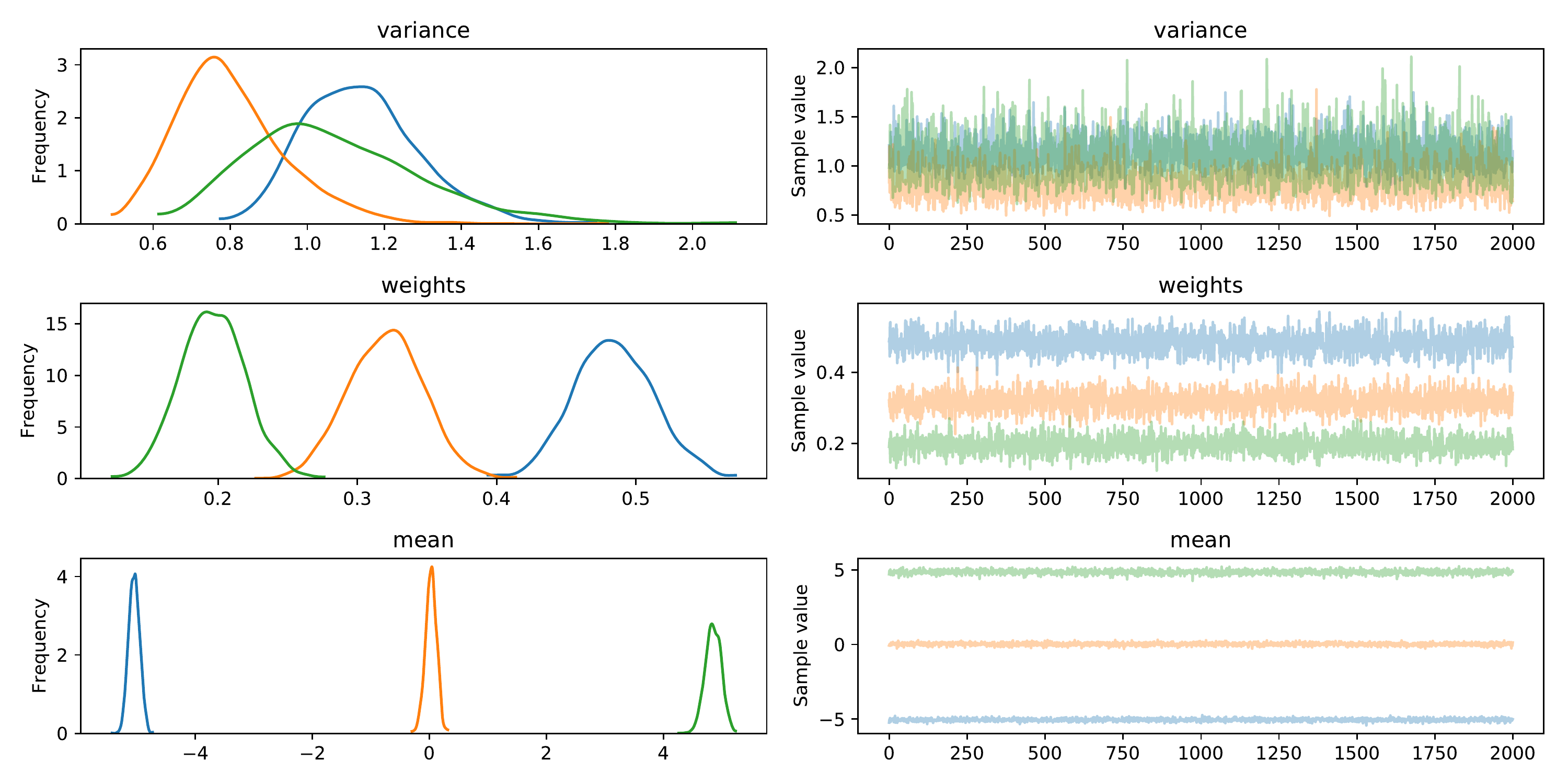}
\qquad
\begin{tabular}[b]{cccccc}\hline
	Sim 1 & \begin{tabular}[c]{@{}c@{}}NMI\\ (SE)\end{tabular}       & \begin{tabular}[c]{@{}c@{}}VI\\ (SE)\end{tabular}         & \begin{tabular}[c]{@{}c@{}}$\overline{K}$\\ (SE)\end{tabular} & \begin{tabular}[c]{@{}c@{}}$\overline{\alpha}$\\ (SE)\end{tabular} & $\overline{\pi}_{-}$ \\ \hline
	$K=3$ & \begin{tabular}[c]{@{}c@{}}0.931\\ (8.5e-5)\end{tabular} & \begin{tabular}[c]{@{}c@{}}0.203\\ (2.5e-4)\end{tabular}  & \begin{tabular}[c]{@{}c@{}}3.0\\ (0.0)\end{tabular}           & \begin{tabular}[c]{@{}c@{}}1.84\\ (5.4e-3)\end{tabular}           & 0.0                  \\
	$K=4$ & \begin{tabular}[c]{@{}c@{}}0.869\\ (2.0e-4)\end{tabular} & \begin{tabular}[c]{@{}c@{}}0.437\\ (7.7e-4)\end{tabular} & \begin{tabular}[c]{@{}c@{}}3.842\\ (2.6e-3)\end{tabular}      & \begin{tabular}[c]{@{}c@{}}1.43\\ (5.2e-3)\end{tabular}           & 0.08                 \\
	$K=5$ & \begin{tabular}[c]{@{}c@{}}0.843\\ (2.7e-4)\end{tabular} & \begin{tabular}[c]{@{}c@{}}0.560\\ (11.1e-4)\end{tabular} & \begin{tabular}[c]{@{}c@{}}4.508\\ (3.2e-3)\end{tabular}      & \begin{tabular}[c]{@{}c@{}}1.01\\ (4.5e-3)\end{tabular}           & 0.12                 \\
	$K=6$ & \begin{tabular}[c]{@{}c@{}}0.846\\ (3.3e-4)\end{tabular} & \begin{tabular}[c]{@{}c@{}}0.564\\ ( 14.9e-4)\end{tabular} & \begin{tabular}[c]{@{}c@{}}4.703\\ (5.2e-3)\end{tabular}      & \begin{tabular}[c]{@{}c@{}}0.618\\ (3.9e-3)\end{tabular}           & 0.11                
	\label{table:hyperdirichlet_sim1_posterior_summary}
\end{tabular}
\caption{\textbf{Upper:} An example of traceplot of Gibbs sampling using hyperprior on Dirichlet distribution for variances, weights and means when $K=3$ in Sim 1 (Upper one: variance; Middle one: weights; Bottom one: mean). Where we find the MAP estimates of the mean and weights parameters are very close to the true ones. There is a bias in the variance estimates although they are in the tolerable error range. \textbf{Bottom:} Summary of posterior distribution in Sim 1: NMI is the normalized mutual information between true clustering and the resulting clustering. VI is the variation of information between true clustering and resulting clustering. SE is the standard error of mean.  $\overline{K}$ is the average occupied number of cluster, $\overline{\alpha}$ is the average $\alpha$ during sampling, $\overline{\pi}_{-}$ is the average of extra weight.}
\label{fig:hyperdirichlet_sim1_traceplot_and_table}
\end{figure}

\begin{figure}[!h]
	\centering     
	\subfigure[$K$=3]{\label{fig:sim1_alpha_k3}\includegraphics[width=0.4\textwidth]{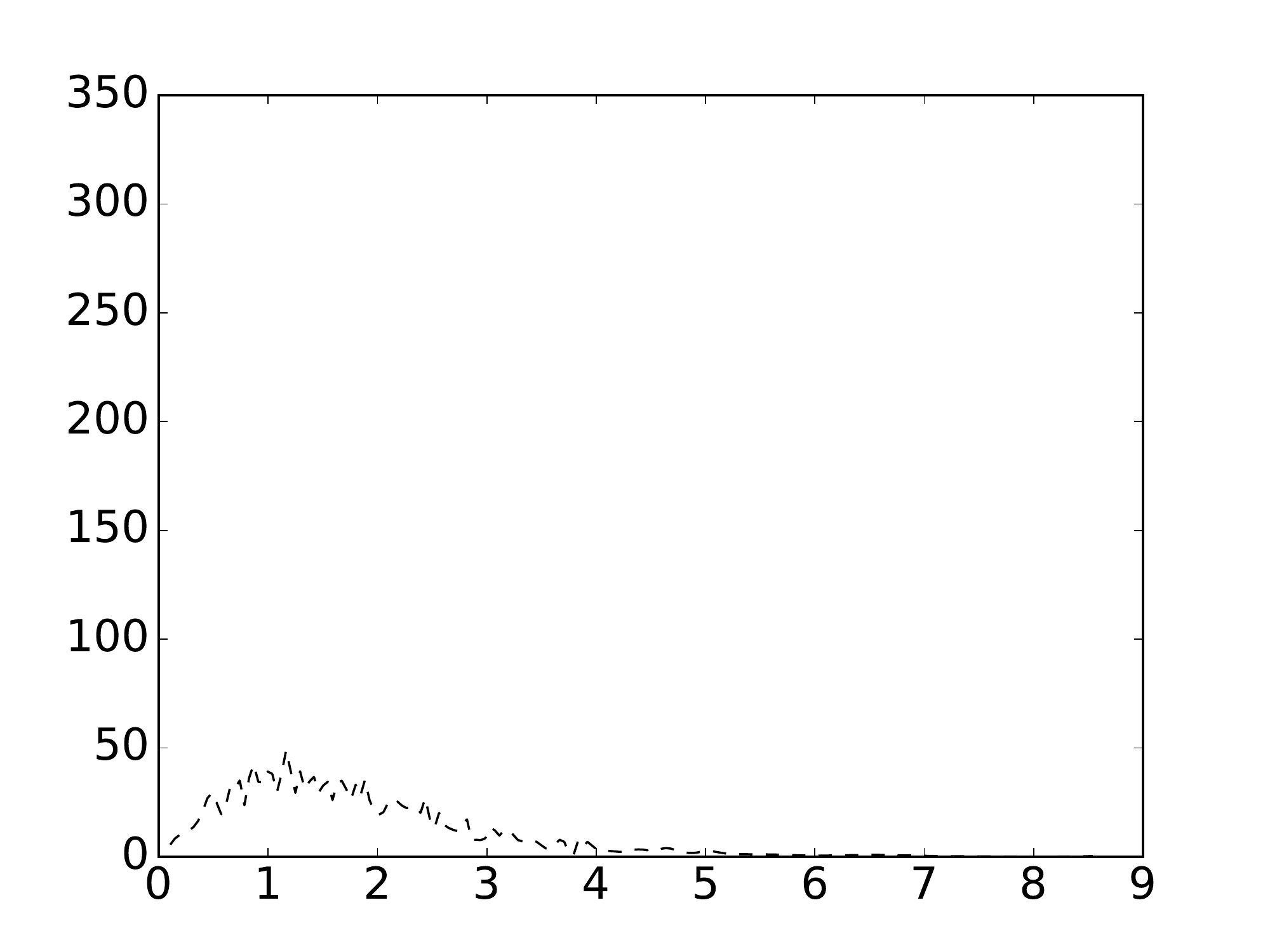}}
	\subfigure[$K$=4]{\label{fig:sim1_alpha_k4}\includegraphics[width=0.4\textwidth]{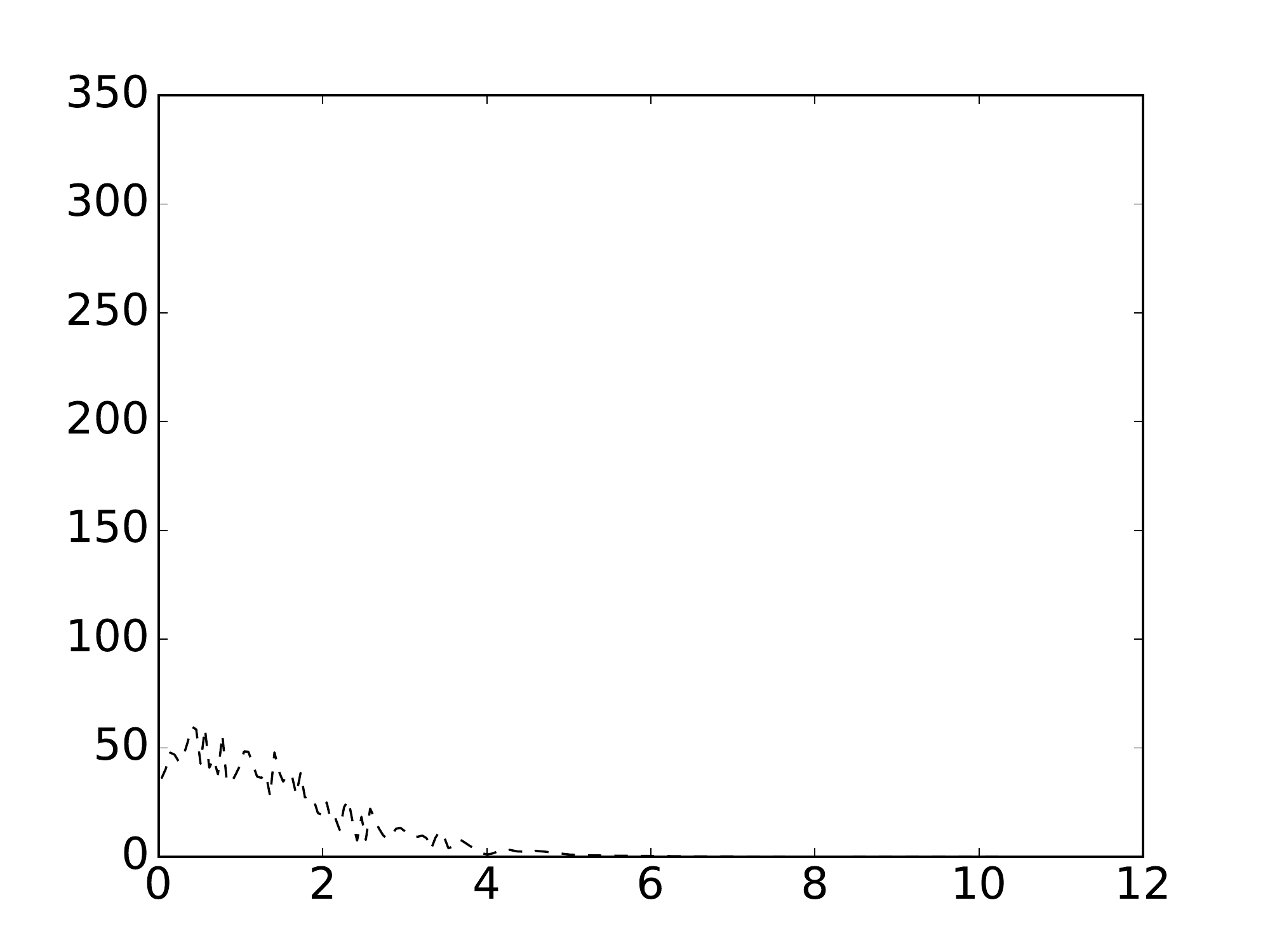}}
	\subfigure[$K$=5]{\label{fig:sim1_alpha_k5}\includegraphics[width=0.4\textwidth]{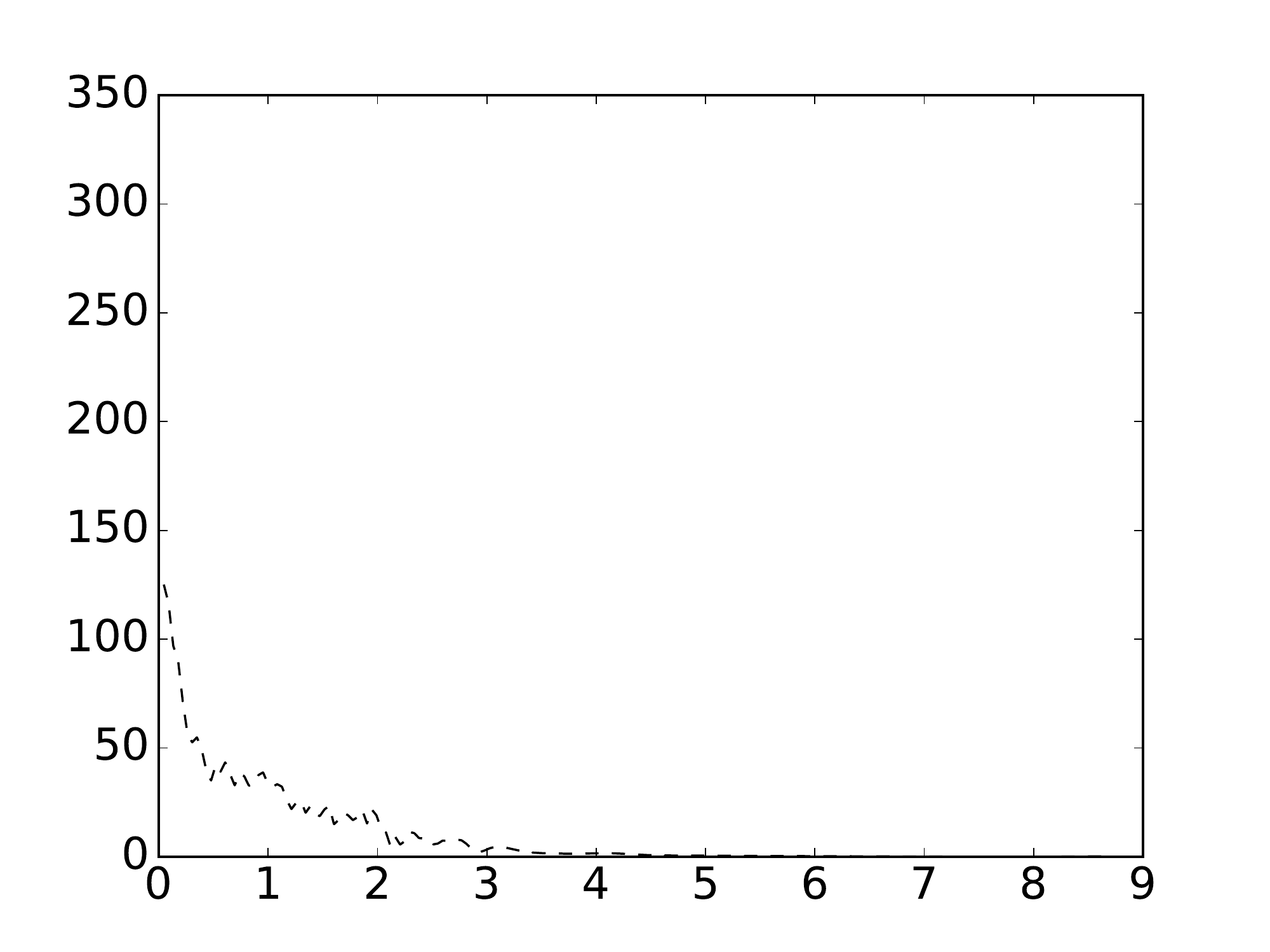}}
	\subfigure[$K$=6]{\label{fig:sim1_alpha_k6}\includegraphics[width=0.4\textwidth]{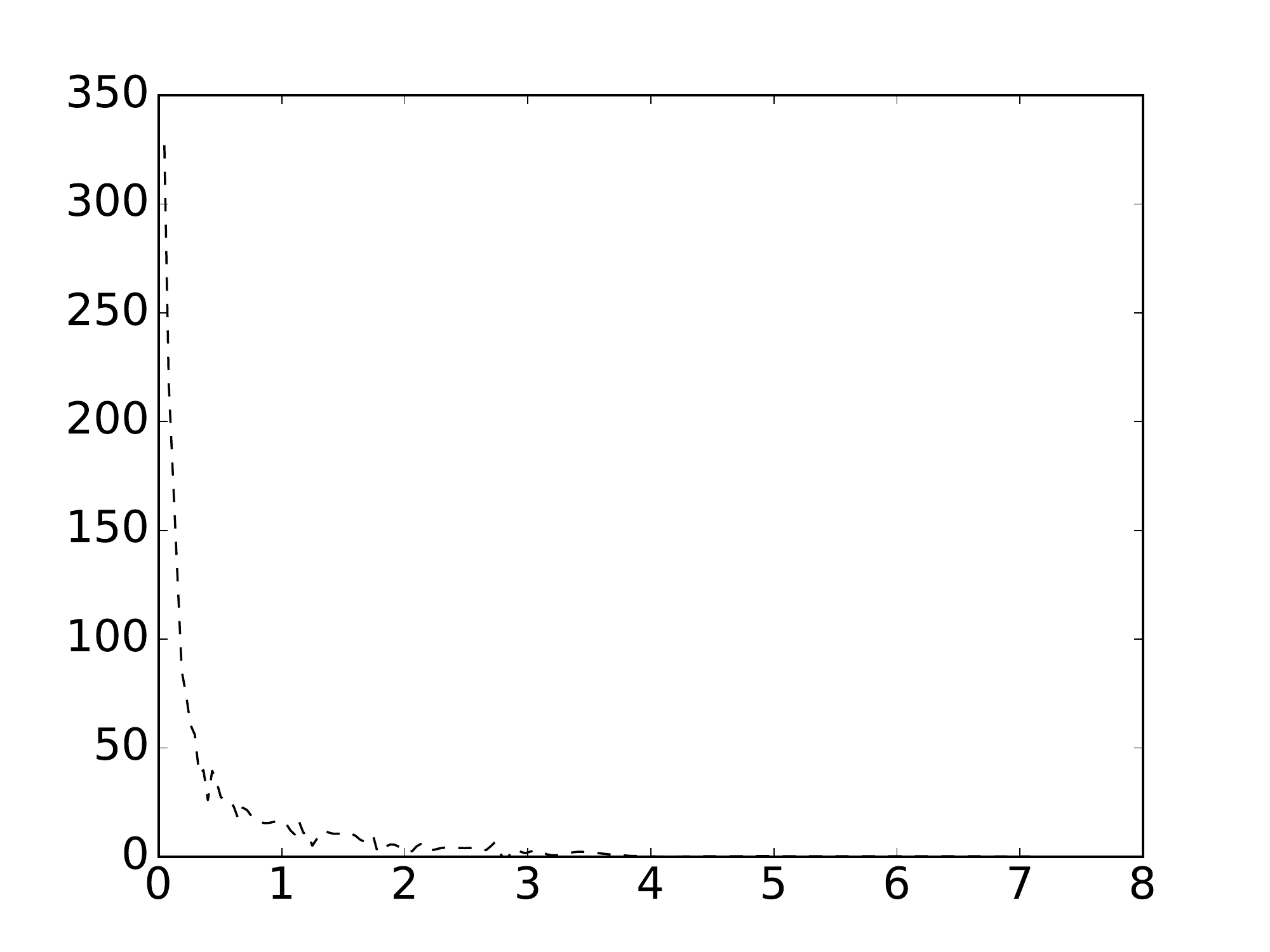}}
	\caption{Posterior distribution for $\alpha$ in different overfitting settings in Sim 1.}
	\label{fig:hyperdirichlet_sim1_alpha_several_k}
\end{figure}

In the following experiments we evaluate the effect of a hyperprior on symmetric Dirichlet prior in finite Bayesian mixture model. See also \citep{lu2017hyperprior}. Some metrics such as normalized mutual information, variation of information are used to evaluate the results where the metrics are discussed later in Section~\ref{section:some-metrics}. Feel free to skip this section for a first reading.

\subsubsection{Synthetic simulation example}

The parameters of the simulations are as follows, where $K_0$ is the true cluster number. And we use $K$ to indicate the cluster number we used in the test:

Sim 1: $K_0=3$, with $N$=300, $\bm{\pi}$=\{0.5, 0.3, 0.2\}, $\bm{\mu}$=\{-5, 0, 5\} and $\bm{\Sigma}$=\{1, 1, 1\};

In the test we put $\alpha \sim \gammadist(1, 1)$ as the hyperprior. Figure~\ref{fig:hyperdirichlet_sim1_traceplot_and_table} shows the result on Sim 1 with different sets of $K$. Figure~\ref{fig:hyperdirichlet_sim1_alpha_several_k} shows the posterior density of $\alpha$ in each set of $K$. We can find that the larger $K-K_0$, the smaller the poserior mean of $\alpha$. This is what we expect, as the larger overfitting, the smaller $\alpha$ will shrink the weight vector in the edge of a probability simplex.

\subsubsection{Conclusion}
We have proposed a new hyperprior on symmetric Dirichlet distribution in finite Bayesian mixture model. This hyperprior can learn the concentration parameter in Dirichlet prior due to over-fitting of the mixture model. The larger the overfitting (i.e., $K-K_0$ is larger, more overfitting), the smaller the concentration parameter. 

Although \citep{rousseau2011asymptotic} proved that $\overline{\alpha}$=max$(\alpha_k, k \leq K)<D/2$ where $D$ is the number of free parameters in the emission distributions which can be simply noted as the dimension, the extra components are emptied at a rate of $N^{-1/2}$ (see discussion in the next section), it is still risky to use such small $\alpha$ in practice, for example, how much do we overfit (i.e., how large the $K-K_0$). If $K-K_0$ is small, we will get very poor mixing from MCMC. Some efforts has been done further by \citep{van2015overfitting}. But simple hyperprior on Dirichlet distribution will somewhat release the burden.

\subsection{Theoretical properties in finite mixture models}

Let $\pi_{(1)} \geq \pi_{(2)} \ldots \geq \pi_{(K)}$ is the ordered sequence of $\pi_{1}, \pi_{2}, \ldots, \pi_{K}$.

\subsubsection{Asymptotic behaviour of the posterior distribution on the weights}\label{section:asymptotic}
\citep{rousseau2011asymptotic} proved that quite generally, the posterior behaviour of overfitted mixtures depends on the chosen prior on the weights, and on the number of free parameters in the emission distributions (here $D$). (a) If $\underline{\alpha}$=min$(\alpha_k, k \leq K)>D/2$ and if the number of components is larger than it should be, asymptotically two or more components in an overfitted mixture model will tend to merge with non-negligible weights. (b) In contrast, if $\overline{\alpha}$=max$(\alpha_k, k \leq  K)<D/2$, the extra components are emptied at a rate of $N^{-1/2}$. Hence, if none of the components are small, it implies that $K$ is probably not larger than $K_0$. In the intermediate case, if min$(\alpha_k, k \leq K)\leq D/2 \leq$ max$(\alpha_k, k \leqslant K)$, then the situation varies depending on the $\alpha_k$'s and on the difference between $K$ and $K_0$. In particular, in the case where all $\alpha_k$'s are equal to $D/2$, then although the author does not prove definite result, they conjecture that the posterior distribution does not have a stable limit. Formally, the author proved the following theorem:
\begin{svgraybox}
\begin{theorem}
	Under the assumption 1-5 in \citep{rousseau2011asymptotic} that the posterior distribution satisfies, let $\mathcal{S}_k$ be the set of permutations of $\{1, \ldots, K\}$, $\overline{\alpha}$=max$(\alpha_k, k \leq K)$, $\underline{\alpha}$=min$(\alpha_k, k \leq K)$. 
	
	(a) If $ \underline{\alpha} > D/2$, set $\rho^\prime=\{D K_0 + K_0-1 + D(D-K_0)/2 \}/ (\overline{\alpha} - D/2)(K-K_0)$; then
	\begin{equation}
		\lim_{\epsilon \rightarrow 0} \lim \sup_N(\E_0^N[  P^{\bpi}\{\min_{\Sigma \in \mathcal{S}_k} (\sum_{i=K_0+1}^K  \pi_{\Sigma(i)} ) < \epsilon \mathrm{log} (N)^{-q(1+\rho^\prime)}   | \mathcal{X} \}   ]) = 0.
	\end{equation}
	
	(b) If $ \overline{\alpha} < D/2$, set $\rho=\{D K_0 + K_0-1 + \overline{\alpha}  (K-K_0) \}/ (D/2 - \overline{\alpha})$; then
	\begin{equation}
		\lim_{M \rightarrow \infty} \lim \sup_N(\E_0^N[  P^{\bpi}\{\min_{\Sigma \in \mathcal{S}_k} (\sum_{i=K_0+1}^K  \pi_{\Sigma(i)} ) > Mn^{-1/2}\log(n)^{q(1+\rho)}     | \mathcal{X} \}   ]) = 0.
	\end{equation}
	where $P^{\bpi}(\cdot |\mathcal{X})$ is the posterior distribution.
\end{theorem}
\end{svgraybox}

We can simply test this theorem by the following simulated data set and set $K=7$ in our test:
$K_0=3$, with $N=50$, $N=100$ or $N=300$, $\bm{\pi}$=\{0.35, 0.4, 0.25\}, $\bm{\mu}$=\{0, 2, 5\} and $\bm{\Sigma}$=\{0.5, 0.5, 1\};

\begin{figure}[h!]
\center
\subfigure[$N$=50]{\includegraphics[width=0.3\textwidth]{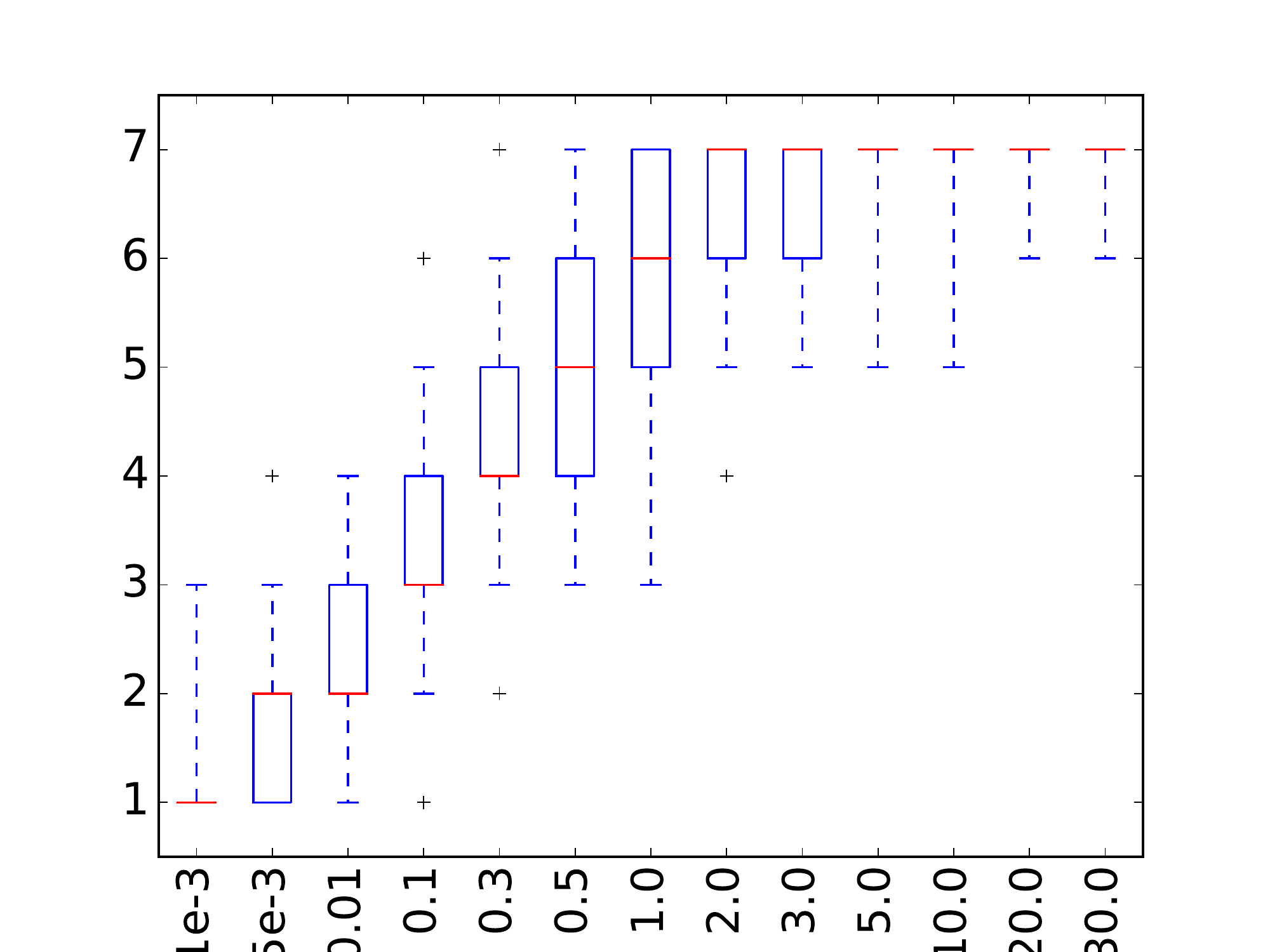} \label{fig:rousseau_alpha_n50}}
~
\subfigure[$N$=100]{\includegraphics[width=0.3\textwidth]{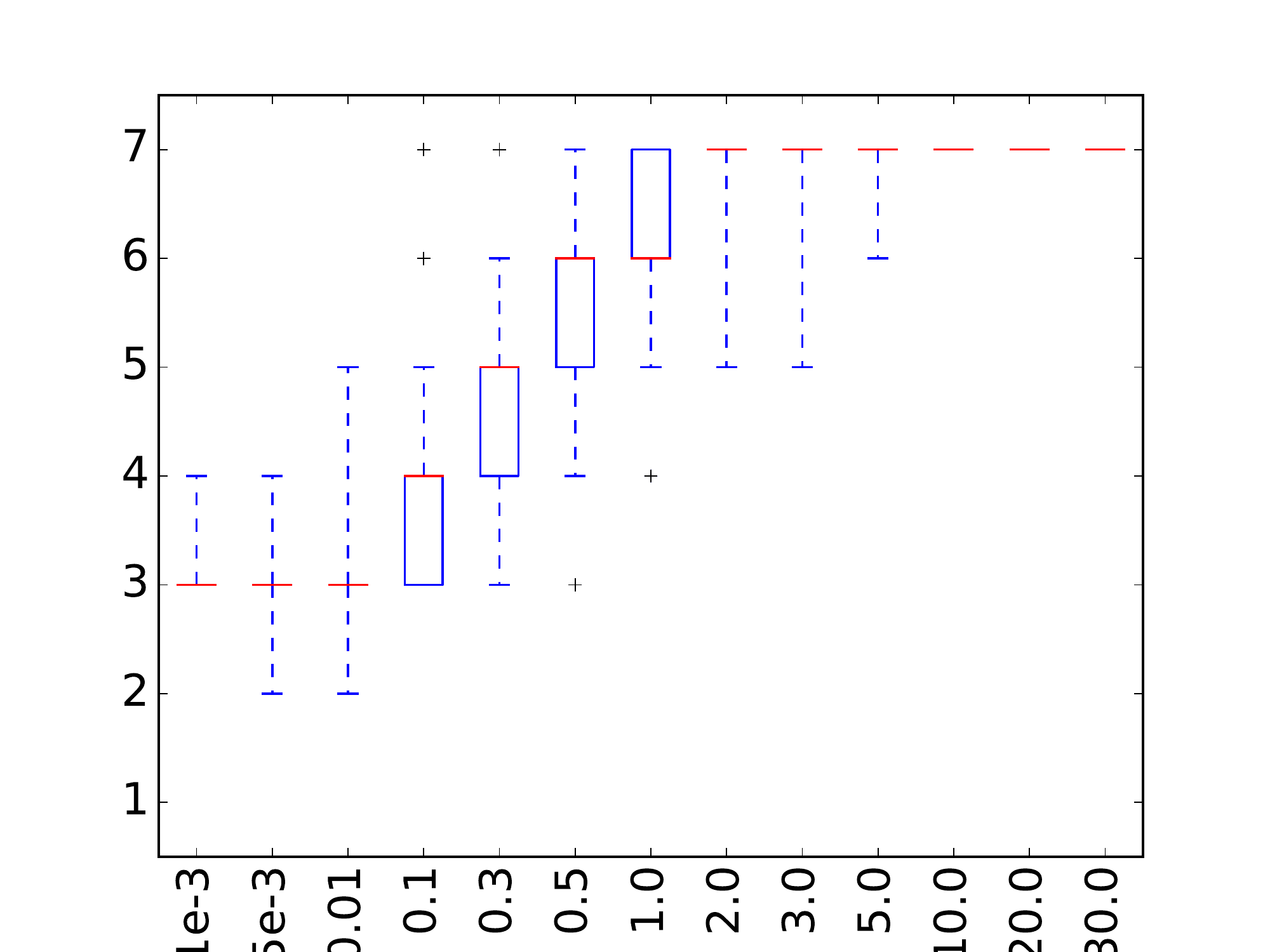} \label{fig:rousseau_alpha_n100}}
~
\subfigure[$N$=300]{\includegraphics[width=0.3\textwidth]{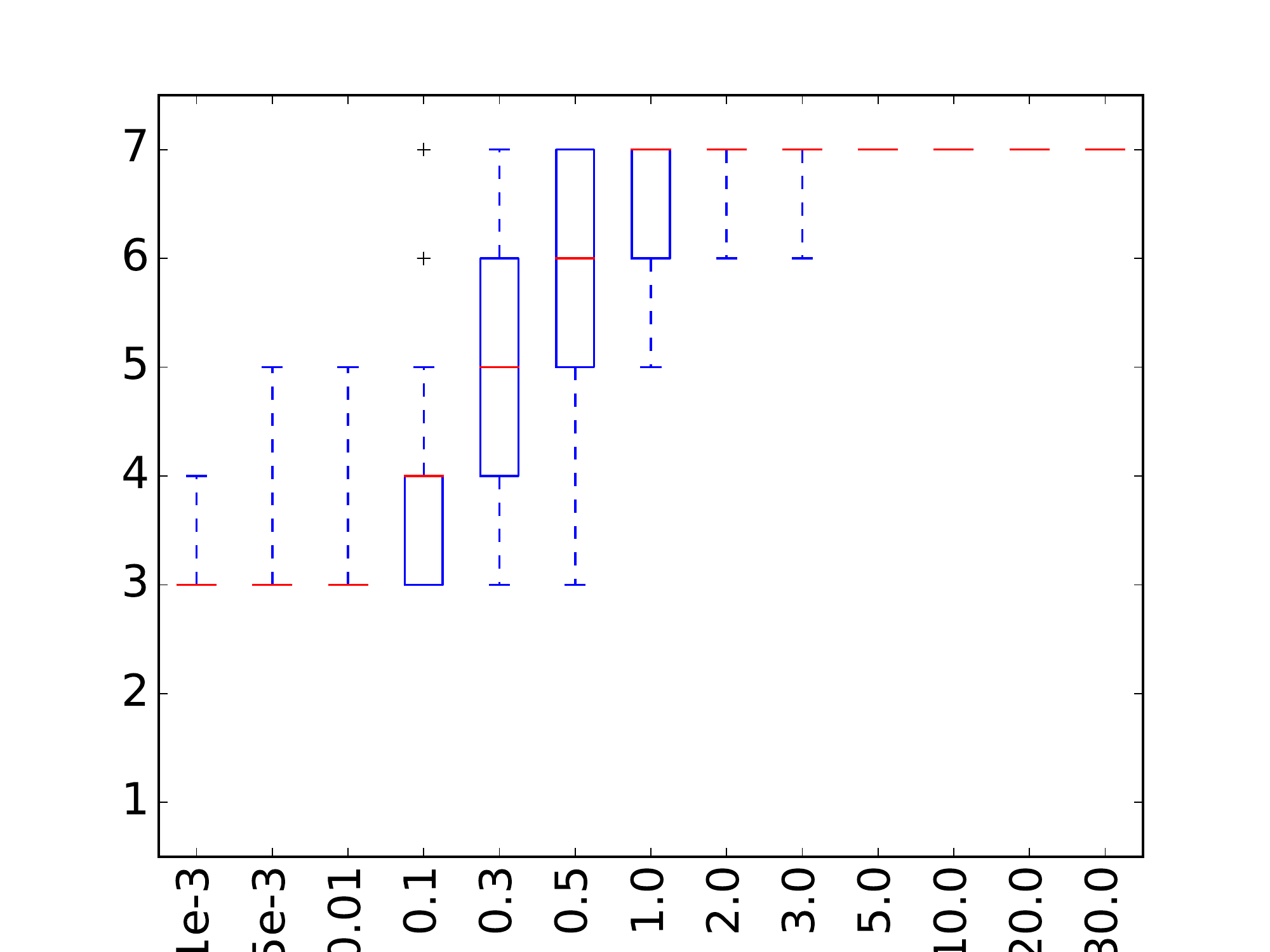} \label{fig:rousseau_alpha_n300} \label{fig:rousseau_alpha_n300} }
\center
\caption{Number of alive (non-empty) groups for a handful of concentration parameter $\alpha$'s. Results are shown for the simulation, $N=50$ (left), $N=100$ (middle) and $N=300$ (right). The x-axis is the choice of $\alpha$, the y-axis is distribution of number of clusters $K^\prime$ with a true cluster number $K_0=3$ and a maximal cluster number $K=7$. }
\label{fig:rousseau_alpha_n50-100-300}
\end{figure}

The result is shown in Figure~\ref{fig:rousseau_alpha_n50-100-300}. We have the following conclusion

\begin{itemize}
\item Once $\alpha$ is smaller than $D/2=0.5$, the posterior distribution of number of clusters $K^\prime$ appears to reach an equilibrium;
\item When the sample size is large enough ($N=300$), the posterior distribution of $K^\prime$ concentrates at 3 once $\alpha$ approaches and smaller than $D/2=0.5$;
\item In the case where $N=50$ or $N=100$, the range of $K^\prime$ includes a small subset of likely configurations, in this case fewer cluster number can be found.
\end{itemize}

\subsubsection{Concentration inequality for symmetric Dirichlet prior}
\citep{yang2014minimax} proved that some particular symmetric Dirichlet prior for probability vectors that can concentrate on sparse subvectors. 

Consider the following set indexed by a tolerance level $\epsilon > 0$ and a sparsity level $s \in \{1, \ldots, K\}: \Pi_{s,\epsilon} = \{ \bpi \in \Pi: \sum_{k=s+1}^K \pi_{(k)} \leq \epsilon \}$, where $\Pi$ is a ($K-1$)-dimensional simplex. 

\begin{svgraybox}
\begin{theorem}
Assume that $\bpi \sim \dirichlet(\alpha, \ldots, \alpha)$ with $\alpha = \rho/K^{\eta}$ and $\eta > 1$. Let $\bpi^\star \in \Pi_s$ be any s-sparse vector in the ($K-1$)-dimensional simplex $\Pi$. Then for any $\epsilon \in (0, 1)$ and some $C>0$, it follows that
\begin{equation}
p(||\bpi - \bpi^{\star}||_1 \leq \epsilon) \geq \{  -C \eta s log \frac{K}{\epsilon} \},
\end{equation}

\begin{equation}
p(\bpi \not\in \Pi_{s,\epsilon} ) \leq \exp\{ -C(\eta - 1)s log \frac{K}{\epsilon}  \}.
\end{equation}
\end{theorem}
\end{svgraybox}


\section{Bayesian infinite Gaussian mixture model}\label{section:review_infinite_mixture_model} \label{sec:infinite_gaussian_mixture}
The infinite Gaussian mixture model is also sometimes referred to as a Dirichlet process Gaussian mixture model (DP GMM).

\subsection{Bayesian nonparametrics for infinite Gaussian mixture model}\label{sec:bayesian_nonparametric}
Let $(\bx_N)_{N \geq 1}$ be an (ideally) infinite sequence of observations, with each $\bx_n$ taking values in a complete and separable metric space $\mathbb{X}$. Let $\mathbf{P}_{\mathbb{X}}$ be the set of all probability measures on $\mathbb{X}$ endowed with the topology of weak convergence. In the most commonly employed Bayesian models, $(\bx_N)_{N \geq 1}$ is assumed to be exchangeable, so that, for some $Q$ on $\mathbf{P}_{\mathbb{X}}$, 
\begin{equation}
\bx_n | \btheta  \overset{iid}{\sim} \btheta, \quad \btheta \sim Q. 
\end{equation}
Hence, $\theta$ is a random probability measure on $\mathbb{X}$ whose probability distribution $Q$ is termed \textit{de Finetti measure} and acts as a prior for Bayesian inference. When $Q$ degenerates on a finite dimensional subspace of $\mathbf{P}_{\mathbb{X}}$, the inferential problem is called \textit{parametric}. On the other hand, when the support of $Q$ is infinite-dimensional, we call it a \textit{nonparametric} inferential problem and it is generally agreed that having a large topological support is a desirable property for a nonparametric prior \citep{ferguson1974prior}.

Generally speaking, for parametric models, the number of parameters is fixed. But for the nonparametric models, the number of parameters can grow with the sample size. I.e., the parameter space is $\infty$-dimensional. Combining with the Bayesian framework, the model complexity can be impacted from the prior which captures the beliefs on them. The nonparametric models can be derived by starting with a finite parametric model and taking the limit as number of parameters go to $\infty$. 

\paragraph{A word on the notation} In the finite mixture models, from Equation~\eqref{equation:fmm_first_term_derivation2}, when considering the collabsed Gibbs sampler, we have the following Gibbs moves on the cluster indicator $z_i$:
\begin{align}
p(z_i = k | \bznoi, \balpha) =\frac{\nknoi + \alpha_k}{N + \alpha_+ - 1}= \frac{\nknoi + \alpha/K}{N + \alpha - 1}.
\end{align}
Note that here we denote $\alpha_+ = \alpha$ for simplicity. We realize that, in finite Gaussian mixture model, we set $\alpha_+ = \alpha \times K$. In this sense, when we deal with finite Gaussian mixture model (Dirichlet distribution prior), the concentration parameter $\alpha$ is $\alpha_+/K$. When we deal with infinite Gaussian mixture model (Dirichlet process prior), the concentration parameter $\alpha$ is $\alpha_+$.

Now let the number of clusters, $K$, go to $\infty$, we will have 
\begin{equation}
\underset{K\rightarrow \infty}{\mathrm{lim}} p(z_i = k | \bznoi, \alpha) = \frac{\nknoi}{ N + \alpha - 1}.
\end{equation}
But when we sum all the clusters, this gives $\sum_k p(z_i = k | \bznoi, \alpha)= \frac{N-1}{N + \alpha -1} < 1$. In order to make this be an actual probability distribution, we need to add a probability $\frac{\alpha}{N + \alpha -1}$, which is corresponded to the assignment to a new cluster, that is 
\begin{equation}
p(z_i = \mathrm{new} | \bznoi, \alpha) = \frac{\alpha}{ N + \alpha - 1}.
\end{equation}

\subsection{The Chinese restaurant process}\label{sec:crp}


Following the nonparametric analysis, we come to the Chinese restaurant process. We here give an formal overview of the Chinese restaurant process (CRP). The CRP is a simple stochastic process that is exchangeable (see discussion below). In the analogy from which this process takes its name, in this process, customers (data points) seat themselves at a restaurant with an infinite number of tables (clusters). 
Each customer sits at a previously occupied table with probability proportional to the number of customers already sitting there, and at a new table with probability proportional to a concentration parameter $\alpha$. For example, the first customer enters and sits at the first table. The second customer enters and sits at the first table with probability $\frac{1}{1+\alpha}$ and at a new table with probability $\frac{\alpha}{1+\alpha}$. The $i^{th}$ customer sits at an occupied table with probability proportional to the number of customers already seated at that table, or sits at a new table with a probability proportional to $\alpha$. From the definition above, we can observe that the CRP is defined by a \textbf{rich-get-richer} property in which the probability of being allocated to a table increases in proportion to the number of customers already at that table. Formally, if $z_i$ is the table chosen by the $i^{th}$ customer, then
\begin{equation}
p(z_i = k | \bz_{1:i-1}, \alpha)=\left\{
                \begin{array}{ll}
                  \frac{N_k}{N+\alpha-1},  \text{ if \textit{k} is occupied, i.e., } N_k > 0,\\
                  \frac{\alpha}{N+\alpha-1},  \text{ if \textit{k} is a new table, i.e., } k = k^{\star} = K + 1,	
                \end{array}
              \right.
\label{equation:ifmm_crp_equation}
\end{equation}
where $\bz_{1:i-1} = (z_1, z_2, \ldots , z_{i-1})$, and $N_k$ is the number of customers already seated at table $k$. Note that there are $N-1$ customers excluding the $i^{th}$ customer in the above definition.

We can see from the above description that CRP is a sequential process. Each table assignment for the new customer is dependent on the table assignment of all the previous customers. And the CRP introduces a partition of customers based on table assignment. For example, the probability for a particular configuration for 3 customers is $p(z_1, z_2, z_3) = p(z_1)p(z_2 | z_1) p(z_3 | z_1,z_2)$.
In a CRP mixture model, each table is assigned a specific parameter in a kernel generating data at the observation level.  Customers assigned to a specific table are given the cluster index corresponding to that table, and have their data generated from the kernel with appropriate cluster/table-specific parameters.  The CRP provides a prior probability model on the clustering process, and this prior can be updated with the observed data to obtain a posterior over the cluster allocations for each observation in a data set.  The CRP provides an exchangeable prior on the partition of indices $\{1,2,\ldots,N\}$ into clusters; exchangeability means that the ordering of the indices has no impact on the probability of a particular configuration -- only the number of clusters $K_N$ and the size of each cluster can play a role.  The CRP implies that $\E[K_N|\alpha] = O(\alpha \log N)$~(see Theorem~\ref{theorem:number-of-clusters} or \citet{teh2011dirichlet}).
In a clustering context, we have the following definition of CRP which is slightly different from Equation~\eqref{equation:ifmm_crp_equation}:
\begin{equation}
	p(z_i = k | \bznoi, \alpha)=\left\{
	\begin{array}{ll}
		\frac{N_{k,-i}}{N+\alpha-1},  \text{ if \textit{k} is occupied, i.e., }  N_k > 0,\\
		\frac{\alpha}{N+\alpha-1},  \text{ if \textit{k} is a new table, i.e., }  k = k^{\star} = K + 1,
	\end{array}
	\right.
	\label{equation:crp_equation_in_pcrp}
\end{equation}
where $\bznoi = (z_1, z_2, \ldots , z_{i-1}, z_{i+1},  \ldots, z_N)$ and $N_{k,-i}$ is the number of customers seated at table $k$ excluding customer $i$.

Important points: 
\begin{itemize}
\item The more the customers at a table, the more probable it is that the next customer will join that table. This is the \textbf{rich-get-richer} property.
\item Probability of a new table (cluster) is proportional to $\alpha$. Thus, we consider $\alpha$ as a concentration parameter which determines the total number of clusters. The higher the $\alpha$, the higher is the number of clusters in a given set of data.
\item CRP specifies a distribution over partitions/table assignments but does not assign parameters to tables.
\end{itemize}

\subsubsection{Exchangeability}
Note that since Equation~\eqref{equation:ifmm_crp_equation} only depends on the number of customers seated at each table $N_k$, the probability of a particular seating arrangement does not depend on the order in which the customers arrived. The random variables $z_i$ in $\bz$ is therefore exchangeable.

\begin{figure}[!h]
  \centering
  \includegraphics[width=0.9\linewidth]{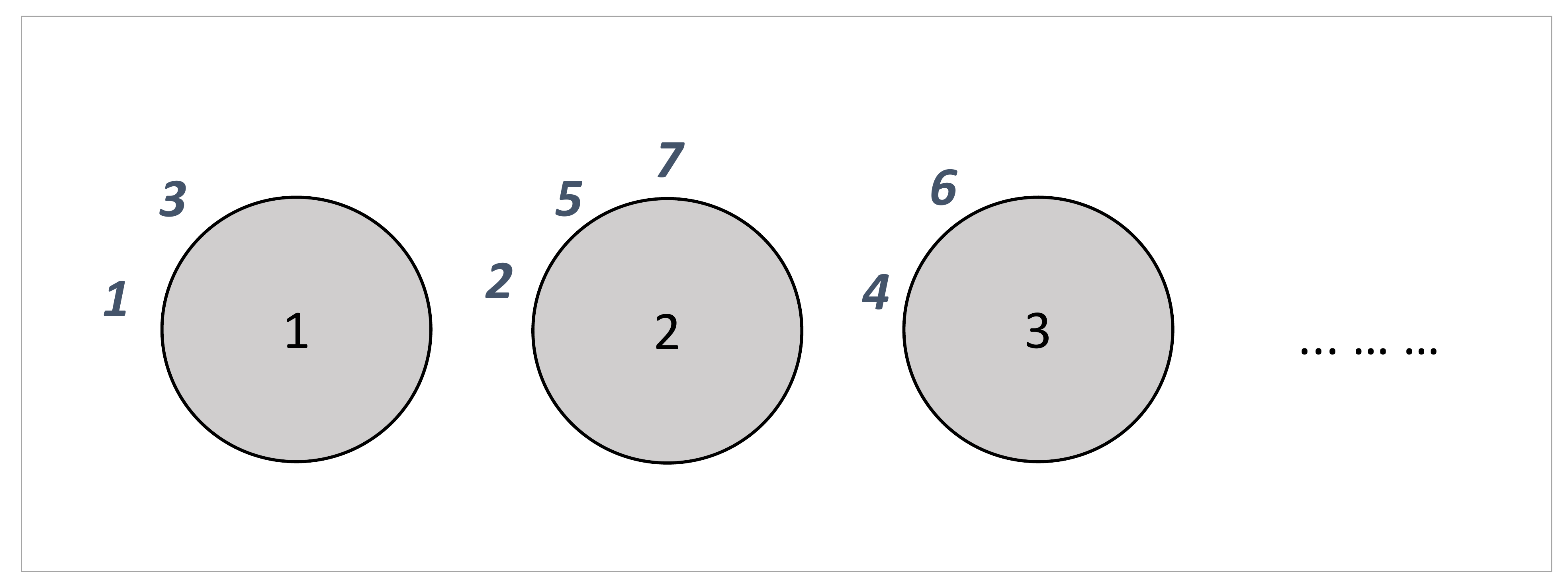}
  \caption{7 customers assignment in a CRP}
  \label{fig:crp_description_figure}
\end{figure}

We here illustrate the CRP by a specific example. For the 7 customers table assignments in Figure \ref{fig:crp_description_figure} where the number on the customer indicates the sequence of arriving, we can find the probability of this particular assignment, 
\begin{equation}
p(z_1, z_2, \ldots, z_7) = \frac{\alpha}{\alpha}  \cdot \frac{\alpha}{1+\alpha} \cdot \frac{1}{2+\alpha} \cdot \frac{\alpha}{3+\alpha} \cdot \frac{1}{4+\alpha} \cdot \frac{1}{5+\alpha} \cdot \frac{2}{6+\alpha}. 
\end{equation}
We realize from the product that the order here does not matter: if the order of customer 1 and 2 is swapped, the probability does not change.

\begin{svgraybox}
\begin{definition}[Infinitely Exchangeable]
	$\bx_1, \bx_2, \ldots$ is infinitely exchangeable if for any $N$, $p(\bx_1, \ldots, \bx_N)$ is invariant under permutation. 
\end{definition}
\end{svgraybox}

To show that the CRP-induced distribution over table assignment is exchangeable, we here first introduce some new notations from \citep{gershman2012tutorial}. Let $K_N$ be the number of groups with $N$ total customers, $I_k$ be the set of indices of customers assigned to the $k^{th}$ group, i.e., $I_{k,i}$ is the total number of customers (including customers at other tables) when $i^{th}$ customer in $k^{th}$ group appears. And the cardinality of $I_k$ is equal to the number of customers at table $k$, $N_k$.

Now, consider the joint distribution over $N$ customers $\bz = \{z_1, z_2, \ldots, z_N\}$. The distribution decomposes according to the chain rule
\begin{equation}
p(\bz | \alpha) = p(z_1) p(z_2 | z_1) p(z_3 | z_2, z_1) \ldots p(z_N | z_{N-1}, z_{N-2}, \ldots, z_1),
\label{equation:ifmm_crp_joint_distribution_for_exchangeable}
\end{equation}
where each term can be calculated from Equation~\eqref{equation:ifmm_crp_equation}. Let's separate the eqaution of table $k$ from Equation~\eqref{equation:ifmm_crp_joint_distribution_for_exchangeable}, in which case we can re-write Equation~\eqref{equation:ifmm_crp_joint_distribution_for_exchangeable} by 
\begin{equation}
\begin{aligned}
p(\bz | \alpha) &= p(z_1) p(z_2 | z_1) p(z_3 | z_2, z_1) \ldots p(z_N | z_{N-1}, z_{N-2}, \ldots, z_1) \\
		&= \prod_{k=1}^K p_k,
\label{equation:ifmm_crp_joint_distribution_for_exchangeable_separated_view}
\end{aligned}
\end{equation}
where $p_k$ is
\begin{equation}
p_k = \frac{\alpha \cdot 1 \cdot 2 \ldots \cdot (N_k-1)}{(I_{k,1}-1+\alpha) (I_{k,2}-1+\alpha) \ldots (I_{k,N_k}-1+\alpha) }. 
\end{equation}
Specifically, the probability for first customer at table $k$ is $\frac{\alpha}{I_{k,1}-1+\alpha}$ because he starts a new table; the probability for second customer at table $k$ is $\frac{1}{I_{k,2}-1+\alpha}$ because he sits at a table with one customer, and so on. With this, we can write the joint probability 
\begin{equation}
\begin{aligned}
p(\bz | \alpha) &= p(z_1) p(z_2 | z_1) p(z_3 | z_2, z_1) \ldots p(z_N | z_{N-1}, z_{N-2}, \ldots, z_1) \\
		&= \prod_{k=1}^K \frac{\alpha (N_k-1)!}{(I_{k,1}-1+\alpha) (I_{k,2}-1+\alpha) \ldots (I_{k,N_k}-1+\alpha) }. 
\label{equation:ifmm_crp_joint_distribution_for_exchangeable_detail}
\end{aligned}
\end{equation}
The probability of a particular sequence of table assignments can be obtained from Equation~\eqref{equation:ifmm_crp_joint_distribution_for_exchangeable_detail} as follows:
\begin{equation}
\begin{aligned}
p(\bz | \alpha) &= \prod^N_{n=1} p(z_n | \bz_{1:n-1})
	&= \frac{\alpha^K \prod^K_{k=1}(N_k - 1)!}{ \prod^N_{n=1}(n - 1 + \alpha)} 
	&= \alpha^K \frac{\Gamma(\alpha)}{\Gamma(N + \alpha)} \prod^K_{k=1}(N_k - 1)! .
\end{aligned}
\label{equation:ifmm_assignment_marginal}
\end{equation}
From the above notation, we can see that the CRP-induced distribution over table assignments is exchangeable. Thus, for any new customer entering the restaurant, we can think him as the last customer entering and apply the generative process for the table assignment.

\begin{svgraybox}
\begin{theorem}[Expected Number of Tables in CRP]\label{theorem:number-of-clusters}
	The expected number of occupied tables for $N$ customers in a CRP grows logarithmically. In particular $\E[K_N|\alpha] = O(\alpha \log N)$. 
\end{theorem}
\end{svgraybox}

\begin{proof}[of Theorem~\ref{theorem:number-of-clusters}]
We introduce a indicator variable $v_n$, which indicates the event that customer $n$ starts a new table. Then the total number of tables after $N$ customers is just $\sum_{n=1}^N = v_n$. The probability of $v_n=1$ is 
\begin{equation*}
p(v_n=1 | \alpha) = \frac{\alpha}{\alpha+n-1}. 
\end{equation*}
Then the expected number of tables after $N$ customers is just
\begin{equation*}
\E[K_N | \alpha] = \E[\sum_{n=1}^N v_n] = \sum_{n=1}^N \E[v_n] = \sum_{n=1}^N \frac{\alpha}{\alpha + n - 1}. 
\end{equation*}
This is a Harmonic series, and is of order $O(\alpha \log N)$. 
\end{proof}

\subsection{The Dirichlet process}
The Dirichlet process (DP) is a distribution over distributions \citep{frigyik2010introduction}. It is parameterized by a concentration parameter $\alpha >0$ and a base distribution $G_0$, which is a distribution over a space $\mathbb{X}$. A random distribution $G$ draw from a DP is denoted $G \sim DP(\alpha, G_0)$. DP can be thought as a random probability measure with Dirichlet marginals, i.e., for any finite decomposition $A_1, \ldots, A_m$ of the whole space $\mathbb{X}$ (i.e., $A_1\cup A_2 \cup \ldots \cup A_m = \mathbb{X}$), we have 
\begin{equation}
(G(A_1), \ldots, G(A_m)) \sim \dirichlet(\alpha G_0(A_1), \ldots, \alpha G_0(A_m) ).
\end{equation}
This means that if we draw a random distribution from the DP and add up the probability mass in a region $A \in \mathbb{X}$, then there will on average be $G_0(A)$ mass in that region (i.e., Base distribution is the ``mean" of DP). The concentration parameter plays the role of an inverse variance; for higher values of $\alpha$, the random probability mass $G(A)$ will concentrate more tightly around $G_0(A)$. I.e., $\E[G(A)] = G_0(A)$ and  $\Var [G(A)] = G_0(A)(1-G_0(A))/( \alpha+ 1) $).

\subsection{CRP V.S. DP}\label{sec:crp_vs_dp}
\begin{svgraybox}
\begin{theorem}{(de Finetti's Theorem).}
	$\bx_1, \bx_2, \ldots, \bx_N, \ldots$ is infinitely exchangeable if and only if there exists a random probability measure $\btheta$, such that 
	\begin{equation}
		p(\bx_1, \ldots, \bx_N) = \int_{\btheta} \prod_ip(\bxi | \btheta) Q(\btheta) d\btheta. 
	\end{equation}
\end{theorem}
\end{svgraybox}

Since CRP is exchangeable, and the underlying parameter, $\btheta$ for CRP is actually Dirichlet process. 

Consider a random distribution draw from a DP followed by repeated draws from that random distribution, 
\begin{equation}
\begin{aligned}
G 		    &\sim DP(\alpha, G_0), \\
\btheta_n &\sim G, \quad i\in \{1, \ldots, N\}. 
\end{aligned}
\end{equation}
Actually, \citep{ferguson1973bayesian} explored the joint distribution of $\btheta_{1:N}$, which is obtained by marginalizing out the random distribution $G$, 
\begin{equation}
p(\btheta_1, \ldots, \btheta_N | \alpha, G_0) = \int \left( \prod_{n=1}^N p(\btheta_n | G) \right) d P(G | \alpha, G_0).
\end{equation}
\citep{ferguson1973bayesian} showed that, under this joint distribution, the $\btheta_n$ will exhibit a clustering property - they will share repeated values with positive probability. The structure of shared values defines a partition of the integers from 1 to $N$, and the distribution of this partition is a Chinese restaurant process. In the following sections, we use Chinese restaurant process and Dirichlet process exchangeably.

\subsection{Bayesian infinite Gaussian mixture model}
\citep{rasmussen1999infinite, anderson1991adaptive, neal2000markov} proposed a solution to the problem of unsupervised clustering based on a probabilistic model known in machine learning as the infinite mixture model and in statistics as the Dirichlet process mixture model. This model intentionally implements an Occam's razor-like tradeoff between two goals: minimizing the number of clusters posited and maximizing the relative similarity of objects within a cluster.
We will work with the following definition of Bayesian infinite Gaussian mixture model
\begin{equation}
\begin{aligned}
\bx_i | z_i, \{\bmu_k, \bSigma_k \} &\sim \mathcal{N}(\bmu_{z_i}, \bSigma_{z_i}), \\
z_i | \bpi    									&\sim \discrete(\pi_1, \ldots, \pi_K), \\
\{\bmu_k, \bSigma_k\}					&\sim \niw(\bbeta), \\
\bpi 											&\sim \dirichlet(\alpha/K, \ldots, \alpha/K).
\end{aligned}
\end{equation}
The Bayesian infinite Gaussian mixture model is obtained by taking the limit as $K \rightarrow \infty$ (See \citep{neal2000markov} or  Section \ref{sec:bayesian_nonparametric} for more details). 
And the Bayesian infinite Gaussian mixture model is illustrated in Figure \ref{fig:gmm_infinite_without_hyper}. The model is very similar to the finite GMM described in Section \ref{sec:finite_gaussian_mixture}. However, in the infinite model the possible number of mixture components could be infinite while in the finite model the number of mixture components $K$ were known beforehand. 
Although any finite samples contain only finitely many clusters, there is no bound on the number of tables (clusters) and any new data point has non-zero probability of being drawn from a new cluster as shown in Equation~\eqref{equation:ifmm_crp_equation}. Therefore, here we use the term “infinite” mixture model.
Here we present a sketch of the model's critical aspects needed to intuitively understand it. Full mathematical details are provided in next sections. The model assumes as input a matrix of objects and features $\mathcal{X}$, where entry $\mathcal{X}_{n,d}$ contains the value of feature $d$ for object $n$. The goal of the model is then to infer likely output of clusters, $\bm{z}$, or sometimes we may infer the single most probable output of clusters (the maximum a posteriori solution or MAP). 

The probability of an assignment of objects to clusters given the data, $p(\bm{z} | \mathcal{X})$ depends on two factors: the prior probability of the assignment of objects to clusters, and the probability of observed data given the cluster assignments. Formally, the probability of an assignment of clusters $\bm{z}$ given the data $\mathcal{X}$ is
\begin{equation}
p(\bm{z} | \mathcal{X}, \alpha, \bbeta) \propto p(\bz | \alpha) p(\mathcalX | \bz, \bbeta)
\end{equation}
where $\alpha$ and $\bbeta$ are hyperparameters, details are provided in Section \ref{sec:ifmm_collabsed_gibbs}. The probability of a particular assignment $p(\bz | \alpha)$ captures a preference for a small number of assignments relative to the total number of objects, and the strength of this preference is governed by the parameter $\alpha$. The term $p(\mathcalX | \bz, \bbeta)$ assesses the probability of the observed feature values, given the assignment of clusters. Thus, the model \textbf{captures a tradeoff} between two competing factors: $p(\bz | \alpha)$ specifies a preference for simple solutions that use a small number of object assignments (see Equation \eqref{equation:ifmm_assignment_marginal}). The term $p(\mathcalX | \bz, \bbeta)$ favors solutions that explain the data well, and tends to prefer more complex solutions. By combining these terms, we arrive at a model that attempts to find the simplest solution that adequately accounts for the data.

Once the prior $p(\bz | \alpha)$ and the likelihood $p(\mathcalX | \bz, \bbeta)$ have been formalized, clustering can be treated as a problem of finding a $\bz$ that has high probability in the posterior distribution $p(\bz | \mathcalX, \alpha, \bbeta)$. We will address this search problem using Gibbs sampler similar to a finite Gaussian mixture model as introduced previously.

\begin{figure}[h!]
\center
\subfigure[A Bayesian finite GMM.]{\includegraphics[width=0.33\textwidth]{img_visual/bgmm_finite_without_hyper.pdf} \label{fig:gmm_infinite_without_hyper_compare}}
\subfigure[A Bayesian infinite GMM.]{\includegraphics[width=0.33\textwidth]{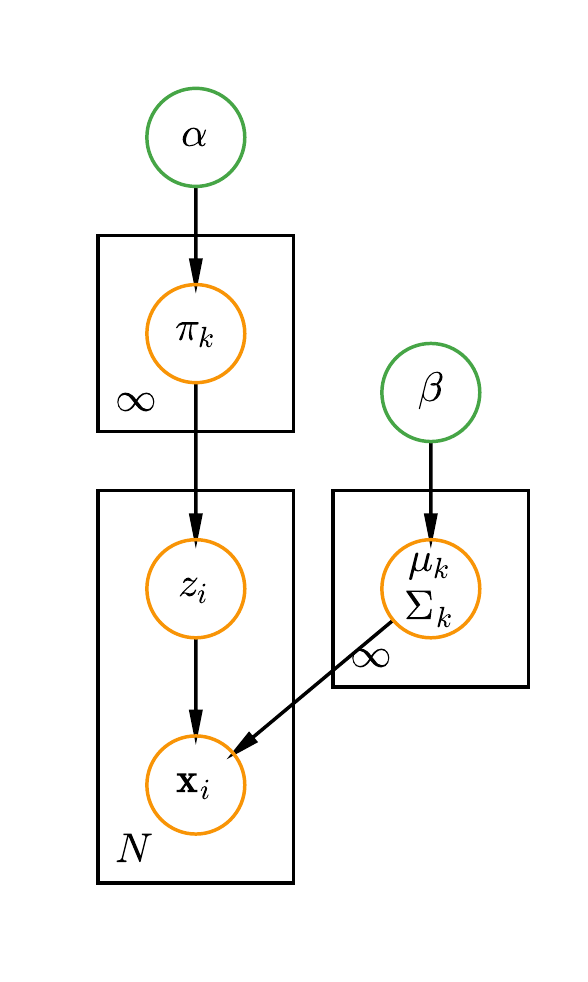} \label{fig:gmm_infinite_without_hyper}}
\caption{A Bayesian infinite GMM compared with a Bayesian finite GMM.}
\label{fig:gmm_infinite_without_hyper_allset}
\end{figure}


In Section~\ref{sec:finite_gaussian_mixture}, we used a Dirichlet distribution as the prior on $\bpi$ for the finite model. Instead, we will use a Dirichlet process (DP) prior with concentration parameter $\alpha$ in the infinite mixture model. It can be shown that by choosing the prior in this way, the model is equivalent to a CRP mixture model, where a short proof is already given by the de Finetti's Theorem in Section \ref{sec:crp_vs_dp}, a more detailed proof can be found in \citep{gershman2012tutorial}. In the case of Gaussian mixture model, we choose a NIW distribution prior with hyperparameters $\bbeta$ for the model parameters of infinite Gaussian components. And again, if using hyperprior on Chinese restaurant process, we represent the hyperparameter of the hyperprior as $a, b$. 

In the following discussion, we thus use the CRP formulation of the DP for simplicity and we also introduce the latent variables to indicate the table assignment as introduced in the Bayesian finite Gaussian mixture model. 


\subsection{Inference by collapsed Gibbs sampling}\label{sec:ifmm_collabsed_gibbs}
Same with finite Gaussian mixture model in Section \ref{sec:fmm_collabsed_gibbs}, we are able to analytically integrate out the parameters $\bpi$, $\bmu_k$ and $\bSigma_k$ due to conjugacy and sample the component assignment $\bz$ directly:
\begin{equation}
\begin{aligned}
p(z_i = k| \bznoi, \mathcal{X} , \alpha, \bbeta)  & \varpropto p(z_i = k | \bznoi, \alpha, \cancel{\bbeta})  p(\mathcal{X} |z_i = k, \bznoi, \cancel{\alpha}, \bbeta) \\
& = p(z_i = k| \bznoi, \alpha) p(\bxi |\mathcal{X}_{-i}, z_i = k, \bznoi, \bbeta) p(\mathcal{X}_{-i} |\cancel{z_i = k}, \bznoi, \bbeta)\\
& \varpropto p(z_i = k| \bznoi, \alpha) p(\bxi|\mathcal{X}_{-i}, z_i = k, \bznoi, \bbeta) \\
& \varpropto p(z_i = k| \bznoi, \alpha) p(\bxi | \xknoi, \bbeta), 										
\label{equation:ifmm_all_term}
\end{aligned}
\end{equation}
This is actually the Algorithm 3 in \citep{neal2000markov}.

\subsubsection{First term: $p(z_i = k| \bznoi, \alpha)$}
The probability $p(z_i = k | \bznoi, \alpha)$ in Equation~\eqref{equation:ifmm_all_term} is the so-called table assignment and is governed by the CRP. From Equation~\eqref{equation:ifmm_crp_equation}, we can thus get
\begin{equation}
p(z_i = k | \bznoi, \alpha)=\left\{
                \begin{array}{ll}
                  \frac{\nknoi}{N+\alpha-1},  \text{ if \textit{k} is an existing component, i.e. } \nknoi > 0, \\
                  \frac{\alpha}{N+\alpha-1},  \text{ if \textit{k} is a new component, i.e. } k = k^{\star} = K + 1	,
                \end{array}
              \right.
\label{equation:crp_equation2}
\end{equation}
where we have assumed  that $z_i$ is the last ``customer” to arrive at the ``restaurant” from exchangeability as shown in above section. For simplicity, we can also denote as $z_i \sim \mathrm{CRP}(\bznoi, \alpha)$.

The first condition in Equation~\eqref{equation:crp_equation2} also follows directly from Equation~\eqref{equation:fmm_first_term_derivation2} (where we called this setting as standard symmetric $\balpha$ setting) as $K \rightarrow \infty$. The second condition also follows from Equation~\eqref{equation:fmm_first_term_derivation2}. A more detailed analysis can be found in Section \ref{sec:bayesian_nonparametric} or \citep{rasmussen1999infinite}. We can thus conclude that Equation~\eqref{equation:crp_equation2} and Equation~\eqref{equation:fmm_first_term_derivation2} are equivalent in the limit as $K \rightarrow \infty$. 

Similarly, from Equation~\eqref{equation:ifmm_assignment_marginal}, the marginal distribution of component assignments of all the data vectors under a CRP prior is given by
\begin{equation}
p(\bz |\alpha)  = \frac{\alpha^K \prod^K_{k=1}(N_k - 1)!}{\prod^N_{n=1}(n - 1 + \alpha)} = \alpha^K \frac{\Gamma(\alpha)}{\Gamma(N+\alpha)} \prod_{k=1}^K (N_k - 1)!. 
\label{equation:ifmm_z_alpha}
\end{equation}
Similar to the discussion above about Equation~\eqref{equation:crp_equation2}, it can be shown that Equation~\eqref{equation:ifmm_z_alpha} results in the limit from Equation~\eqref{equation:fmm_z_alpha2} as $K \rightarrow \infty$ as well \citep{griffiths2005infinite}.

\subsubsection{Second term: $p(\bxi | \xknoi, \bbeta)$}
Similar to the second term discussed in Section~\ref{sec:fmm_second_term}, we can find an expression for $p(\bxi | \mathcal{X}_{-i}, z_i = k, \bznoi, \bbeta)=p(\bxi | \xknoi, \bbeta)$ in Equation~\eqref{equation:ifmm_all_term} by:
\begin{equation}
p(\bxi | \mathcal{X}_{-i}, z_i = k, \bznoi, \bbeta) = p(\bxi | \xknoi, \bbeta) = \frac{p(\mathcal{X}_k | \bbeta)}{p(\xknoi | \bbeta)}. 
\end{equation}
Again, the expression above can be calculated using Equation~\eqref{equation:niw_posterior_predictive_equation} of Equation~\eqref{equation:finite-second-term} if $z_i = k$ is an existing component. When $z_i = k^\star$ is a new component then we have 
\begin{equation}
p(\bxi| \mathcal{X}_{-i}, z_i = k^\star, \bznoi, \bbeta) = p(\bxi | \bbeta) = \int_{\bmu}  \int_{\bSigma} p(\bxi | \bmu, \bSigma) p(\bmu, \bSigma|\bbeta) d\bmu d \bSigma,
\end{equation}
which is just the prior predictive distribution and can be calculated using Equation~\eqref{equation:niw_posterior_predictive_equation} with $\mathcal{X} =\emptyset $ or using Equation~\eqref{equation:niw_prior_predictive_abstract} directly.

The pseudo code for collapsed Gibbs sampler for an infinite Gaussian mixture model is given in Algorithm \ref{algo:ifmm_plain_gibbs}.

\IncMargin{1em}
\begin{algorithm}
\SetKwData{Left}{left}\SetKwData{This}{this}\SetKwData{Up}{up}
\SetKwFunction{Union}{Union}\SetKwFunction{FindCompress}{FindCompress}
\SetKwInOut{Input}{input}\SetKwInOut{Output}{output}
\Input{Choose an initial $\bz$}
\BlankLine

\For{$T$ iterations}{
\For{$i \leftarrow 1$ \KwTo $N$}{
	Remove $\bxi$'s statistics from component $z_i$ \;
	\For{$k\leftarrow 1$ \KwTo $K$}{
		Calculate $p(z_i=k| \bznoi, \alpha) = \frac{\nknoi}{N + \alpha -1}$\;
		Calculate $p(\bxi | \xknoi, \bbeta)$\;
		Calculate $p(z_i = k | \bznoi, \mathcal{X}, \alpha, \bbeta) \propto p(z_i=k| \bznoi, \alpha) p(\bxi | \xknoi, \bbeta)$\;
	}
	Calculate $p(z_i = k^\star | \bznoi, \alpha)=\frac{\alpha}{N + \alpha -1}$\;
	Calculate $p(\bxi | \bbeta)$\;
	Calculate $p(z_i = k^\star | \bznoi, \mathcalX, \alpha, \bbeta) \propto p(z_i = k^\star | \bznoi, \alpha) p(\bxi | \bbeta)$\;
	Sample $k_{new}$ from $p(z_i | \bznoi, \mathcalX, \alpha, \bbeta)$ after normalizing\;
	Add $\bxi$'s statistics to the component $z_i=k_{new}$ \;
	If any component is empty, remove it and decrease $K$.
}
}
\caption{Collapsed Gibbs sampler for an infinite Gaussian mixture model}\label{algo:ifmm_plain_gibbs}
\end{algorithm}\DecMargin{1em}

\subsection{Get the posterior distribution for every parameter}\label{sec:ifmm_posterior_for_all_parameters}

To get the posterior distribution for every parameter we need to derive the conditional posterior distributions on all the other parameters, $p(\theta_i | \btheta_{-i}, \mathcalX)$. But for a graphical model, this conditional distribution is a function only of the nodes in the Markov blanket. In our case, the Bayesian infinite Gaussian mixture model, a directed graphical model, the Markov blanket includes the parents, the children, and the co-parents, as shown in Figure \ref{fig:gmm_infinite_without_hyper}. From this graphical representation, we can find the Markov blanket for each parameter in the model, and then figure out their conditional posterior distribution to be derived:
\begin{align}
&p (\bmu_k, \bSigma_k | \mathcalX_k, \bbeta ), \quad \mathrm{for} \quad k=1, \ldots, K, \\
&p (\bpi | \alpha, \bz). 
\end{align}

\subsubsection{Conditional distribution of cluster mean and covariance}
Because we are using collapsed Gibbs sampler here, we do not get the distribution of mean and covariance from sampling steps. But we can get them from Equation~\eqref{equation:niw_posterior_equation_1}:
\begin{align}
p(\bmu_k, \bSigma_k | \mathcalX_k, \bbeta ) &= \niw (\bmu_k, \bSigma_k | \mathbf{m}_{N_k}, \kappa_{N_k}, \nu_{N_k}, \bS_{N_k}). 
\end{align}
The mode of the joint distribution has the following form
\begin{equation}
\mathrm{arg} \max p(\bmu_k, \bSigma_k | \mathcalX_k) = (\mathbf{m}_{N_k}, \frac{\bS_{N_k}}{\nu_{N_k} + D + 2}),
\end{equation}
where the definition of $\mathbf{m}_{N_k}, \kappa_{N_k}, \nu_{N_k}$ and $\bS_{N_k}$ can be found in Equation~\eqref{equation:niw_posterior_equation_1} by replacing $N$  by $N_k$. 

\subsubsection{Conditional distribution of mixture weights}

We can similarly derive the conditional distributions of mixture weights by an application of Bayes' theorem. Instead of updating each component of $\bpi$ separately, we update them together (this is called a blocked Gibbs sampler):
$$
\begin{aligned}
	p(\bpi | \cdot) &= p(\bpi | \bz, \alpha) \\
	&\propto p(\bpi | \alpha) p(\bz | \alpha, \bpi) \\
	&= \dirichlet(\bpi | \alpha/K) \cdot \mathrm{Multinomial}(\bz | \bpi) \\
	&= \dirichlet(\bpi | \balpha=[\alpha/K, \alpha/K, \ldots, \alpha/K]) \cdot \mathrm{Multinomial}(\bz | \bpi) \\
	&\propto \prod_{k=1}^K \pi_k^{\alpha/K - 1}  \prod_{k=1}^K \pi_k^{N_k}  \\
	&= \prod_{k=1}^K \pi_k^{\alpha + N_k  - 1} \\
	&\propto \dirichlet(N_1 + \alpha/K, \ldots, N_K+\alpha/K). 
\end{aligned}
$$
\begin{figure}[h!]
	\center
	\subfigure[Without hyperprior.]{\includegraphics[width=0.33\textwidth]{img_visual/bgmm_infinite_without_hyper.pdf} \label{fig:gmm_infinite_with_hyper_compare}}
	\subfigure[With hyperprior.]{\includegraphics[width=0.33\textwidth]{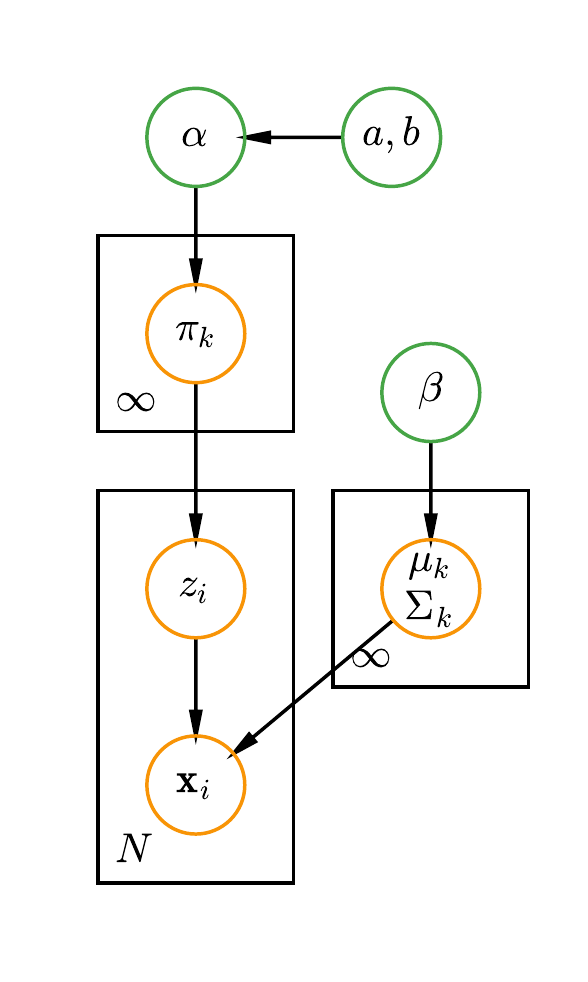} \label{fig:gmm_infinite_with_hyper}}
	\caption{A Bayesian infinite GMM with hyperprior on concentration parameter.}
	\label{fig:gmm_infinite_with_hyper_allset}
\end{figure}

\subsection{Hyperprior on the concentration parameter}


As introduced in \citep{escobar1995bayesian} and further discussed in \citep{west1992hyperparameter}, they put a hyperprior on the concentration parameter of Dirichlet process as shown in Figure~\ref{fig:gmm_infinite_with_hyper}. We here briefly discuss how to put a Gamma prior on the concentration parameter. 
From \citep{antoniak1974mixtures}, the prior distribution of number of clusters $k$ can be written as
\begin{equation}
p(k | \alpha, N)  = z_{n}(k) n! \alpha^k \frac{\Gamma(\alpha)}{\Gamma(\alpha + N)}, \quad (k=1,2,3,\ldots, N)
\label{equation:ifmm_gamma_prior_abstract}
\end{equation} 
and $z_n(k) = p(k|\alpha = 1, N)$ which does not involve $\alpha$. From our model, we can deduce
\begin{equation}
p(\alpha | k, \bpi, \mathcalX) \propto p(\alpha|k) \propto p(\alpha) p(k|\alpha, N). 
\label{equation:ifmm_gamma_prior_posterior}
\end{equation}
For $\alpha > 0$, wen can easily deduce that the Gamma functions in Equation~\eqref{equation:ifmm_gamma_prior_abstract} can be written as,
\begin{equation}
\frac{\Gamma(\alpha)}{\Gamma(\alpha + N)} = \frac{(\alpha + N) \beta(\alpha+1, N)}{\alpha \Gamma(N)},
\end{equation}
where $\beta(.,.)$ is the usual Beta function. Then for Equation~\eqref{equation:ifmm_gamma_prior_posterior}, and for any $k=1,2,\ldots, N$, it follows that 
\begin{equation}
\begin{aligned}
p(\alpha | k, N) &\propto p(\alpha)\alpha^{k-1} (\alpha+N) \beta(\alpha+1, N)\\
					 &\propto p(\alpha)\alpha^{k-1}(\alpha + N) \int_0^1 x^\alpha (1-x)^{N-1} dx,
\end{aligned}
\end{equation}
by using the definition of the Beta function. This implies that $p(\alpha | k, N)$ is the marginal distribution from a joint for $\alpha$ and a continuous quantity $x (0 < x < 1)$ such that
\begin{equation}
p(\alpha, x | k, N) \propto p(\alpha)\alpha^{k-1}(\alpha + N)x^\alpha (1-x)^{N-1}, \quad (0 < \alpha, 0 < x < 1).
\end{equation}
Hence we have conditional posteriors $p(\alpha|x, k, N)$ and $p(x|\alpha, k, N)$ determined as follows. Firstly, under the Gamma(a, b) prior for $\alpha$,
\begin{equation}
\begin{aligned}
p(\alpha|x, k) &\propto \alpha^{a+k-2} (\alpha + N) e^{-\alpha(b-\log(x))} \\
			&\propto \alpha^{a+k-1}e^{-\alpha(b-\log(x))} + N \alpha^{a+k-2}e^{-\alpha(b-\log(x))}. 
\end{aligned}
\end{equation}
for $\alpha > 0$, which reduces easily to a mixture of two gamma densities, i.e.,
\begin{equation}
(\alpha|x, k, N) \sim \pi_x \cdot \gammadist(a + k, b - \log(x)) + (1 - \pi_x) \cdot \gammadist(a + k - 1, b - \log(x)) 
\end{equation}
with weights $\pi_x$ defined by
\begin{equation}
\frac{\pi_x}{1 - \pi_x} = \frac{(a+k-1)}{N(b-\log (x))}. 
\end{equation}
Secondly, 
\begin{equation}
p(x|\alpha, k, N) \propto x^\alpha(1 - x)^{N-1} \quad (0 < x < 1),
\end{equation}
so that $(x|\alpha, k, N) \sim \betadist(\alpha + 1, N)$, a Beta distribution with mean $(\alpha + 1)/(\alpha + N + 1)$.

\subsection{Problem in CRP mixture model}
The development of Markov chain Monte Carlo sampling techniques \citep{ishwaran2001gibbs, ishwaran2002approximate, antoniak1974mixtures, neal2000markov} further popularizes the CRP mixture model in a wide array of applications, such as machine learning, pattern recognition, statistics, etc. Nevertheless, as shown in \citep{xu2016bayesian, miller2013simple}, the CRP mixture models tend to produces relative large number of clusters regardless of whether they are needed to accurately characterize the data - this is particularly true for large data sets. However, some of these clusters are typically redundant and negligible so that interpretability, parsimony, data storage and communication costs all are hampered by having overly many clusters. And when the underlying data generating density is a finite mixture of Gaussians, the posterior number of clusters under the CRP mixture model is inconsistent, i.e., the posterior distribution of the number of clusters does not converge to the point mass at the underlying true number of cluster $K_0$. 

Dirichlet process mixture (DPM) models and closely related formulations have been very widely used for flexible modeling of data and for clustering.  DPMs of Gaussians have been shown to possess frequentist optimality properties in density estimation, obtaining minimax adaptive rates of posterior concentration with respect to the true unknown smoothness of the density \citep{shen2011adaptive}.  DPMs are also very widely used for probabilistic clustering of data.  In the clustering context, it is well known the DPMs favor introducing new components at a log rate as the sample size increases, and tend to produce some large clusters along with many small clusters.  As the sample size $N$ increases, these small clusters can be introduced as an artifact even if they are not needed to characterize the true data generating process; for example, even if the true model has finitely many clusters, the DPM will continue to introduce new clusters as $N$ increases \citep{miller2013simple}.


Continuing to introduce new clusters as $N$ increases can be argued to be an appealing property.  The number of `types’ of individuals is unlikely to be finite in an infinitely large population, and there is always a chance of discovering new types as new samples are collected.  This rationale has motivated a rich literature on generalizations of Dirichlet processes, which have more flexibility in terms of the rate of introduction of new clusters.  For example, the two parameter Poisson-Dirichlet process (aka, the Pitman-Yor process) is a generalization that instead induces a power law rate, which is more consistent with many observed data processes  \citep{perman1992size}. There has also been consideration of a rich class of Gibbs-type processes, which considerably generalize Pitman-Yor to a broad class of so-called exchangeable partition probability functions (EPPFs) \citep{gnedin2006exchangeable, lijoi2010models, de2015gibbs}.  Much of the emphasis in the Gibbs-type process literature has been on data in which `species’ are observed directly, and the goal is predicting the number of new species in a further sample \citep{lijoi2007bayesian}. It remains unclear whether such elaborate generalizations of Dirichlet processes have desirable behavior when clusters/species are latent variables in a mixture model.

\citep{lu2018reducing} proposes a powered Chinese restaurant process to overcome this kind of problem. The emphasis of \citep{lu2018reducing} is on addressing practical problems that arise in implementing DPMs and generalizations when sample sizes and data dimensionality are moderate too large.  In such settings, it is common knowledge that the number of clusters can be too large, leading to a lack of interpretability, computational problems and other issues.  For these reasons, it is well motivated to develop {\em sparser} clustering methods that do not restrict the number of clusters to be finite {\em a priori} but instead favor deletion of small clusters that may not be needed to accurately characterize the true data generating mechanism.  With this goal in mind, we find that the usual focus on exchangeable models, and in particular EPPFs, can limit practical performance.  There has been some previous work on non-exchangeable clustering methods motivated by incorporation of predictor-dependence in clustering \citep{blei2011distance, socher2011spectral}, but the focus is instead on providing a simple approach that tends to delete small and unnecessary clusters produced by a DPM.  Marginalizing out the random measure in the DPM specification produces a Chinese Restaurant Process (CRP).  \citep{lu2018reducing} proposes a simple powered modification to the CRP, which has the desired impact on clustering and develop associated inference methods.

\subsection{Powered Chinese restaurant process (pCRP)}

\subsubsection{Generative powered Chinese restaurant process}
Before our description of powered Chinese restaurant process (pCRP) and to show the properties of pCRP, we first consider a generative process of powered number of customers. The generative process for a pCRP is as follows: each customer sits at a previously occupied table with probability proportional to the powered number of customers already sitting there. For example, the first customer enters and sits at the first table. The second customer enters and sits at the first table with probability $\frac{1}{1+\alpha}$ and at a new table with probability $\frac{\alpha}{1^r+\alpha}$. This power value $r$ will have effect when the table has more than one customer. 

\begin{figure}[!h]
\center
\subfigure[Table index for every customer when $N=10000$]{\includegraphics[width=0.45\textwidth]{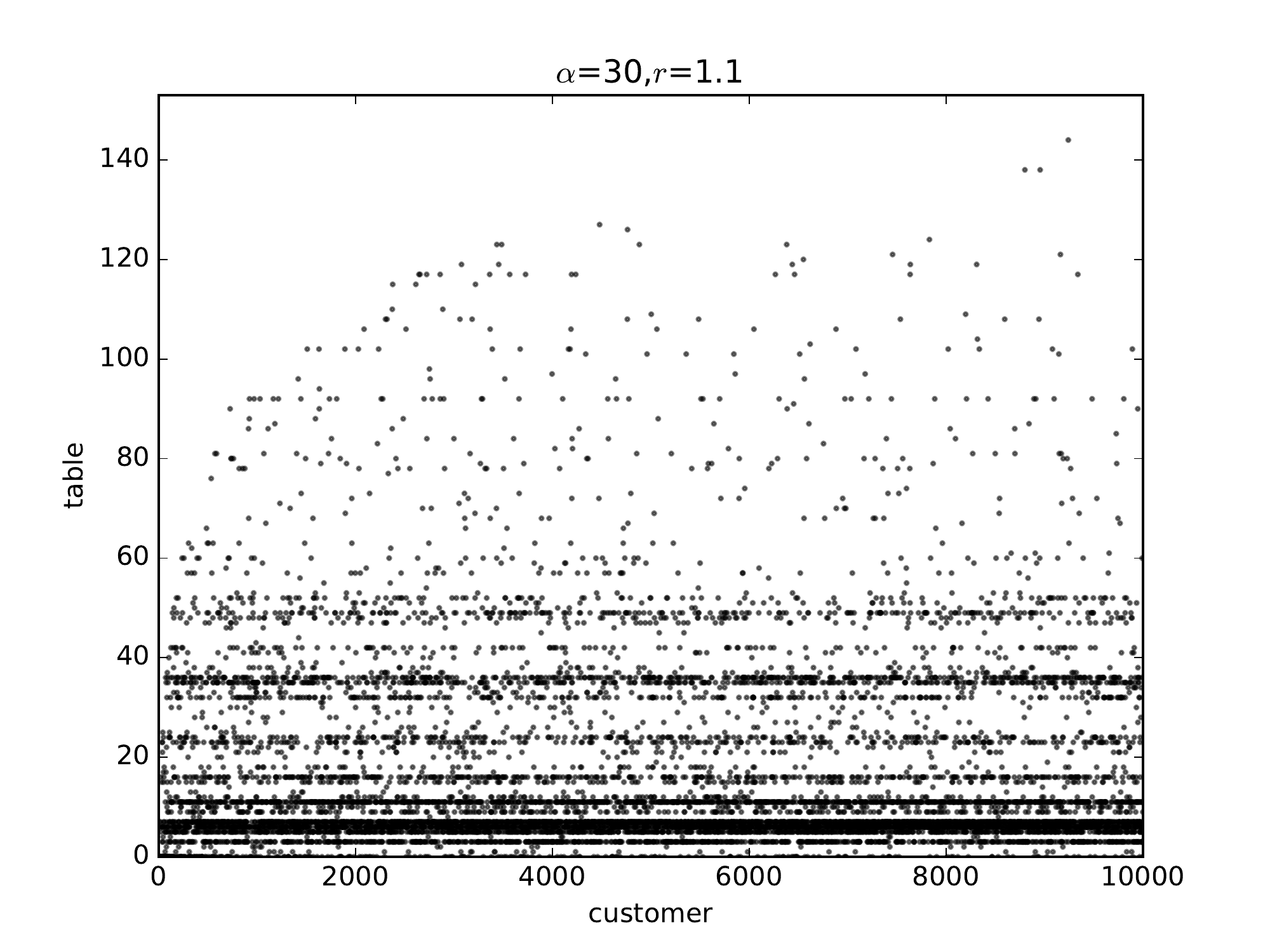} \label{fig:gpcrp_customer_table}}
\subfigure[Number of customers for every table when $N=10000$]{\includegraphics[width=0.45\textwidth]{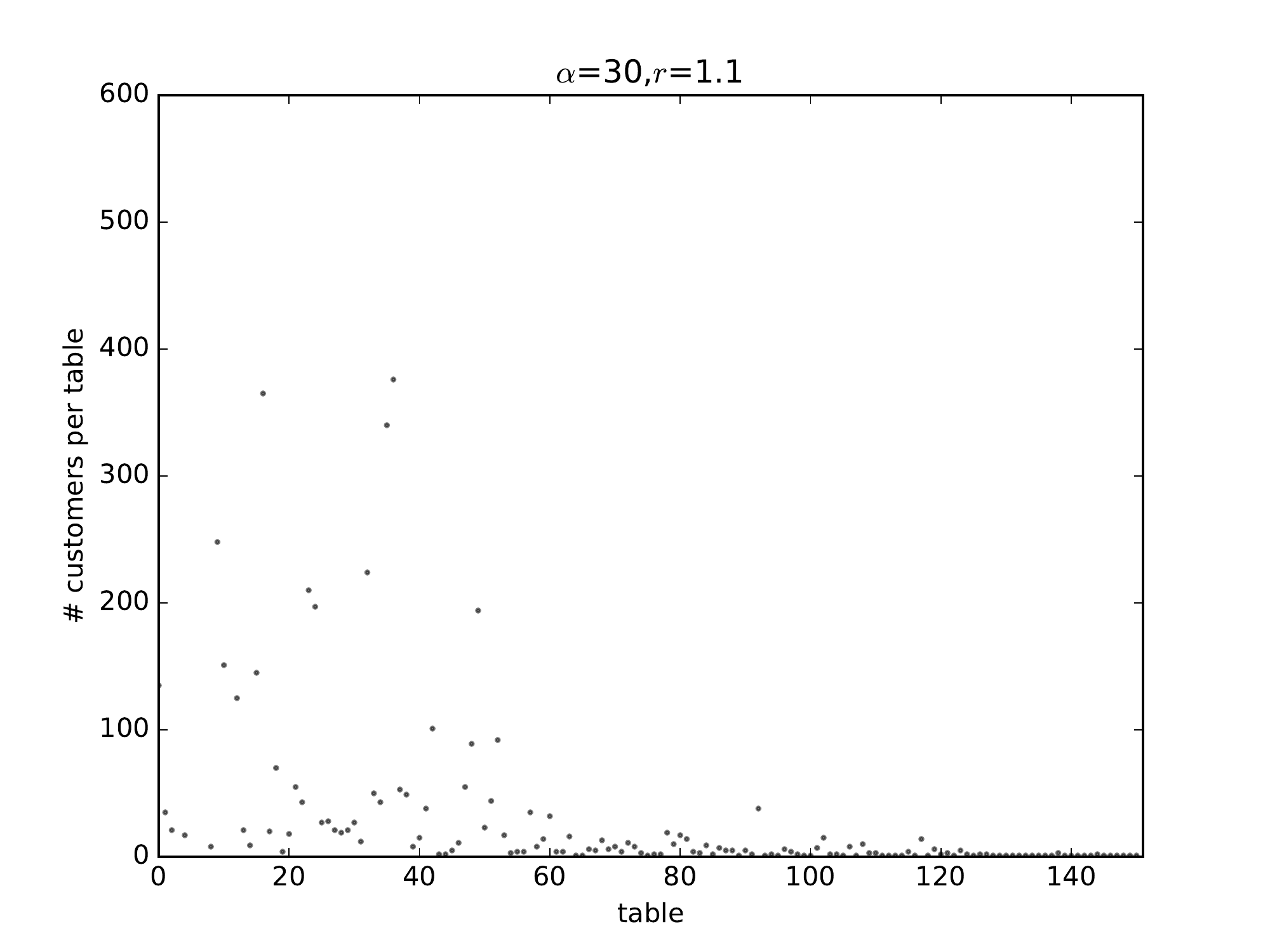} \label{fig:gpcrp_customer_per_table}}
\subfigure[Table index for every customer when $N=10000$]{\includegraphics[width=0.45\textwidth]{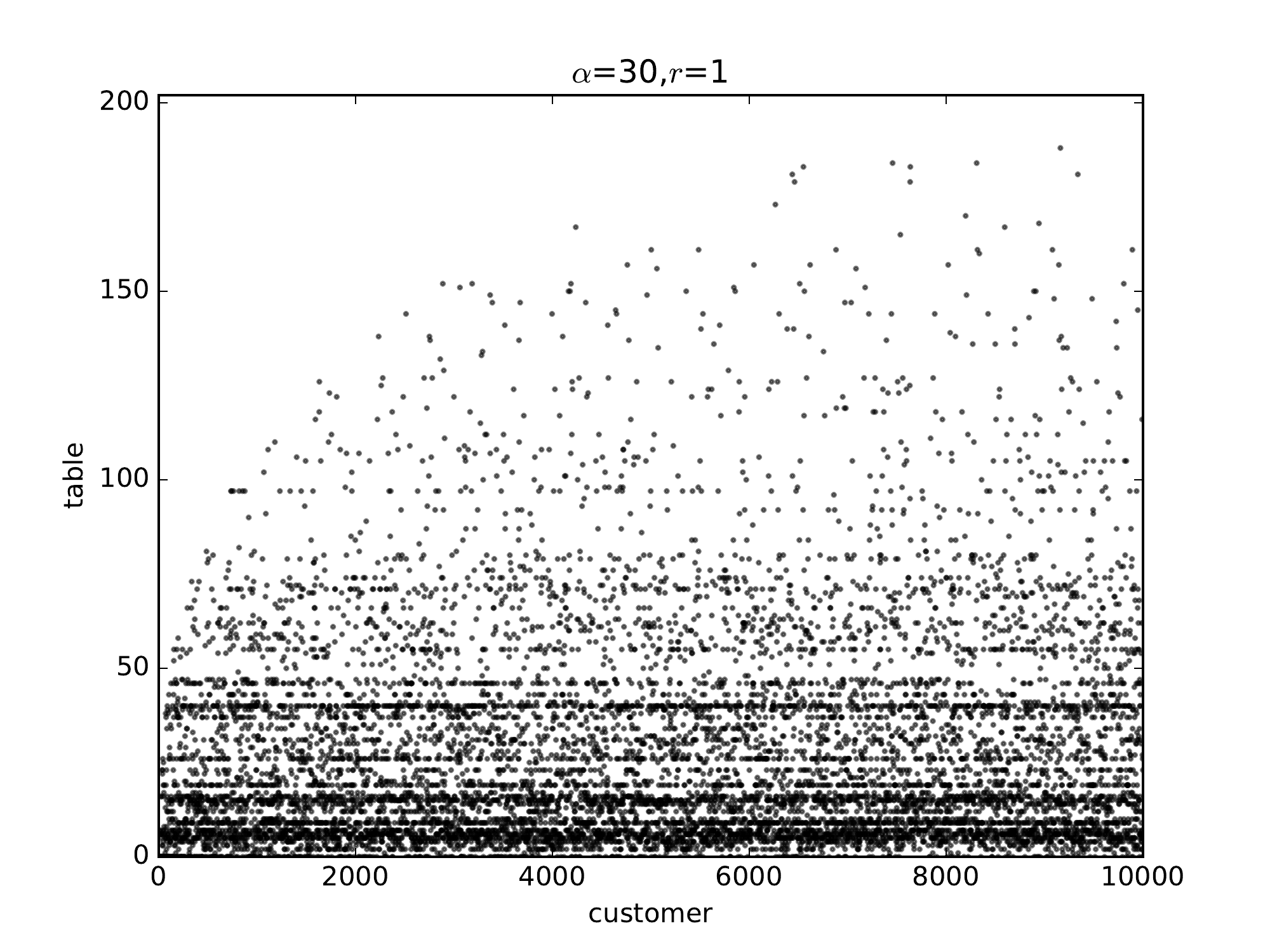} \label{fig:crp_customer_table}}
\subfigure[Number of customers for every table when $N=10000$]{\includegraphics[width=0.45\textwidth]{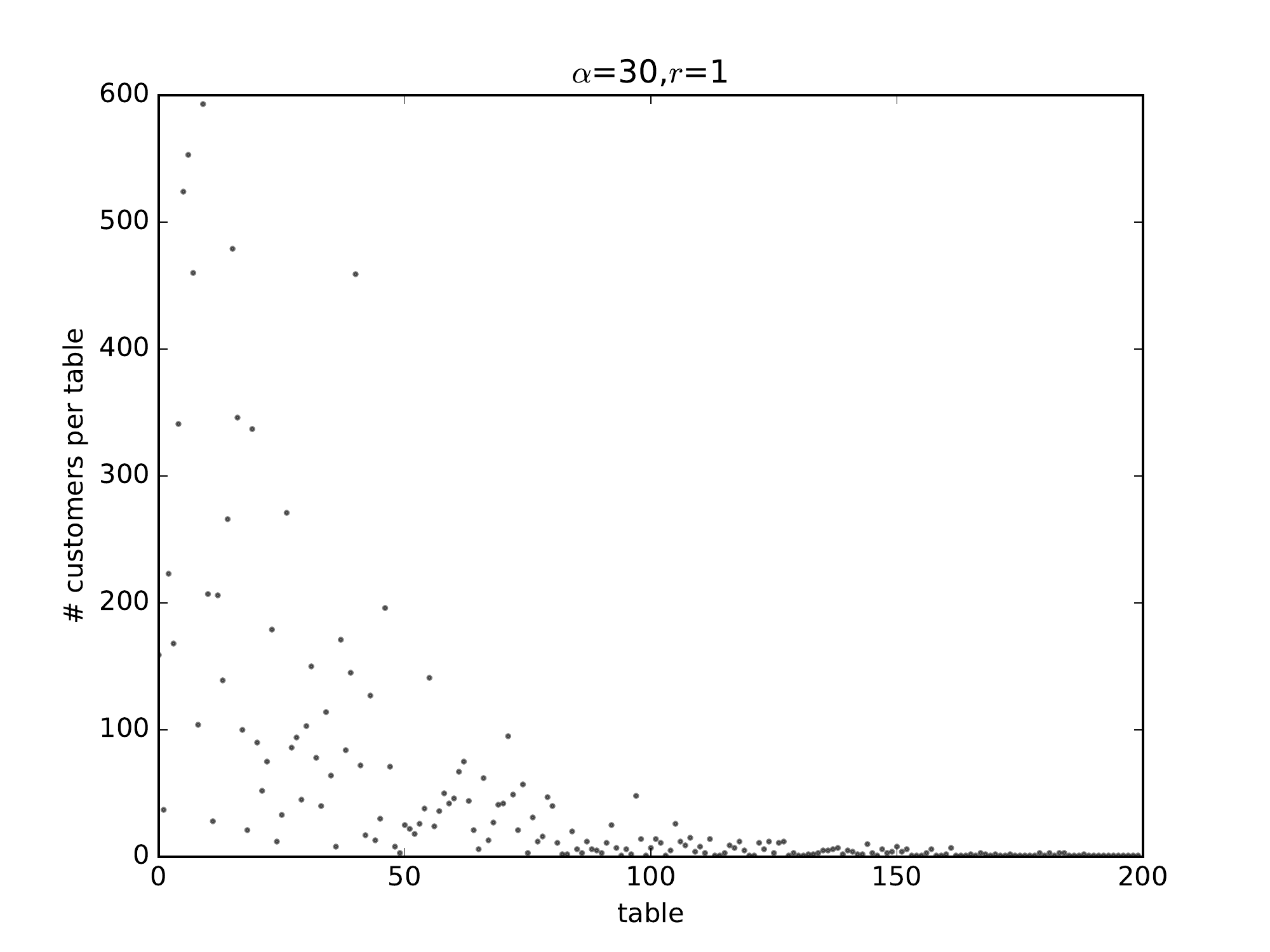} \label{fig:crp_customer_per_table}}
\caption{Draw from a generative powered Chinese restaurant process ($r=1.1$, \textbf{upper-two} figures) and a Chinese restaurant process ($r=1$, \textbf{bottom-two} figures)}
\label{fig:customer_table_pcrp}
\end{figure}

Figure~\ref{fig:customer_table_pcrp} shows each table seated by the customers, and number of customers per table in a draw of pCRP and a draw of CRP. Although this generative powered Chinese restaurant process is not equal to what we propose in Equation~\eqref{equation:powered_crp_equation}, we can see the effect of power value on the number of customers such that it can shrink number of tables (140 in pCRP compared 200 in CRP). Theorem~\ref{theorem:generative_pcrp_expected_num} gives the expected number of tables in a generative power Chinese restaurant process.

\begin{svgraybox}
\begin{theorem}\label{theorem:generative_pcrp_expected_num}
	Assume $N$ customers in a generative pCRP, then the expected number of occupied tables $\E[K_N|\alpha] \in [O(\alpha \varsigma(r)), O(\alpha \log N)]$ when $N \rightarrow \infty$. Where $\varsigma(r)$ is the Riemann zeta function.
\end{theorem}
\end{svgraybox}

\begin{proof}
Again we introduce a indicator variable $v_i$, which indicates the event that customer i starts a new table. Then the total number of tables after $N$ customers is just $\sum_{n=1}^N = v_n$. The probability of $v_n=1$ is 
\begin{equation}
p(v_n=1 | \alpha) = \frac{\alpha}{\alpha+[n-1]^r},
\end{equation}
where $[n-1]^r$ is the sum of powered number of customer at each table. For example, if there are two tables with $(10)$ and $(n-11)$ customers at each table, then $[n-1]^r = 10^r + (n-11)^r$.

It can be easily proved that $[i-1]^r$ ranges from $i-1$ to $(i-1)^r$ when $r>1$. We have
\begin{equation}
\E[K_N | \alpha] = \E \left[\sum_{n=1}^N v_n\right] = \sum_{n=1}^N \E[v_n] \in \left[\sum_{n=1}^N \frac{\alpha}{\alpha + (n-1)^r} ,\sum_{n=1}^N \frac{\alpha}{\alpha + n - 1}\right]. 
\end{equation}
This gives the result.
\end{proof}

\subsubsection{Powered Chinese restaurant process}

Popular Bayesian nonparametric priors, such as the Dirichlet process \citep{ferguson1973bayesian}, Chinese restaurant process, Pitman-Yor process \citep{perman1992size} and Indian buffet process \citep{griffiths2005infinite}, assume infinite exchangeability.  In particular, suppose we have a clustering process for an infinite sequence of data points $i=1,2,3,\ldots,\infty$.  This clustering process will induce a partition of the integers $\{1,2,\ldots,N\}$ into $K_N$ clusters of size $N_1,N_2,\ldots,N_{K_N}$, for $N=1,2,\ldots,\infty$. For an exchangeable clustering process, the probability of a particular partition of $\{1,2,\ldots,N\}$ only depends on $N_1,N_2,\ldots,N_{K_N}$ and $K_N$, and does not depend on the order of the indices $\{1,2,\ldots,N\}$.  In addition, the probability distributions for different choices of $N$ are {\em coherent};  the probability distribution of partitions of $\{1,2,\ldots,N\}$ can be obtained from the probability distribution of partitions of $\{1,2,\ldots,N+1\}$ by marginalizing out the cluster assignment for data point $i=N+1$.  These properties are often highly appealing computationally and theoretically, but it is nonetheless useful to consider processes that violate the infinite exchangeability assumption.  This can occur when the addition of a new data point $i=N+1$ to a sample of $N$ data points can impact the clustering of the original $N$ data points.  For example, we may re-evaluate whether data point $1$ and $2$ are clustered together in light of new information provided by a third data point, a type of {\em feedback} property.  

The proposed new powered Chinese restaurant process (pCRP), which is designed to favor elimination of artifactual small clusters produced by the usual CRP by implicit incorporation of a feedback property violating the usual exchangeability assumptions.  The proposed pCRP makes the random seating assignment of the customers depend on the powered number of customer at each table (i.e., raise the number of each table to power $r$). Formally, we have
\begin{equation}
p(z_i = k | \bznoi, \alpha)=\left\{
                \begin{array}{ll}
                  \frac{N_{k,-i}^r}{\sum_h^K N_{h,-i}^r+\alpha},  \text{ if \textit{k} is occupied, i.e., }  N_k > 0,\\
                  \frac{\alpha}{\sum_h^K N_{h,-i}^r+\alpha},  \text{ if \textit{k} is a new table, i.e., }  k = k^{\star} = K + 1,
                \end{array}
              \right.
\label{equation:powered_crp_equation}
\end{equation}
where $r>1$ and $N_{k,-i}$ is the number of customers seated at table $k$ excluding customer $i$. More generally, one may consider a $g$-CRP to generalize the CRP such that 
\begin{equation}
p(z_i = k | \bznoi, \alpha)=\left\{
\begin{array}{ll}
\frac{g(N_{k,-i})}{\sum_h^K g(N_{h,-i})+\alpha},  \text{ if \textit{k} is occupied, i.e. }  N_k > 0,\\
\frac{\alpha}{\sum_h^K g(N_{h,-i})+\alpha},  \text{ if \textit{k} is a new table, i.e. }  k = k^{\star} = K + 1,
\end{array}
\right.
\label{equation:g_crp_equation}
\end{equation}
where $g(\cdot): \mathbb{R^+} \rightarrow \mathbb{R^+} $ is an increasing function and $g(0) = 0$. We achieve shrinkage of small clusters via a \textbf{\textit{rich-get-(more)-richer}} property by requiring $g(x) \geq x$ for $x > 1$ to `enlarge'  clusters containing more than one element. We require the $g$-CRP to maintain a \textit{\textbf{proportional invariance}} property:
\begin{equation}\label{eq:proportional.invariant}
\frac{g(cN_1)}{g(cN_2)} = \frac{g(N_1)}{g(N_2)}
\end{equation}
for any $c, N_1, N_2 > 0$, so that scaling cluster sizes by a constant factor has no impact on the prediction rule in Equation~\eqref{equation:g_crp_equation}.  The following Lemma~\ref{lemma:unique} shows that the pCRP in Equation~\eqref{equation:powered_crp_equation} using the power function is the only $g$-CRP that satisfies the proportional invariance property. 
\begin{svgraybox}
\begin{lemma}
	\label{lemma:unique}
	If a continuous function $g(x): \mathbb{R^+} \rightarrow \mathbb{R^+}$ satisfies Equation~\eqref{eq:proportional.invariant}, then $g(x) = g(1) \cdot x^r$ for all $x > 0$ and some constant $r \in \mathbb{R}$. 
\end{lemma} 
\end{svgraybox}
\begin{proof}[of Lemma~\ref{lemma:unique}]
	It is easy to verify that  $g(x) = g(1) \cdot x^r$ for some $r > 0$ is a solution to the functional equation~\eqref{eq:proportional.invariant}. We next show its uniqueness.
	
	Equation~\eqref{eq:proportional.invariant} implies that $g(cN_1)/g(N_1) = g(cN_2)/g(N_2)$ for any $N_1, N_2 > 0$. Denote $f(c) = g(cN)/g(N) > 0$ for arbitrary $N > 0$. We then have $f(s t) = g(stN)/g(N) = g(stN)/g(tN) \cdot g(tN) / g(N) = f(s) f(t)$ for any $s, t > 0$. By letting $f^*(x) = f(e^x) > 0$, it follows that $\log f^*(s + t) = \log f^*(s) + \log f^*(t)$, which is the well known Cauchy functional equation and has the unique solution $\log f^*(x) = r x$ for some constant $r$. Therefore, $f(x) = f^*(\log(x)) = x^r$ which gives $g(cN) = g(N) c^r$. We complete the proof by letting $N = 1$. 
\end{proof}

As a generalization of the CRP, which corresponds to the special case in which $r=1$, the proposed pCRP with $r > 1$ generates new clusters following a probability that is configuration dependent and not exchangeable. For example, for three customers $z_1, z_2, z_3$, $p(z_3 = 2 \mid z_1=1, z_2=1) < p( z_3 = 1 \mid z_1=1, z_2=2)$, where $z_i = k$ if the $i^{th}$ customer sits at table $k$. 
This non-exchangeability is a critical feature of pCRP, allowing new cluster generation to learn from existing patterns. Consider two extreme configurations: (i) $K_N = N$ with one member in each cluster, and (ii) $K_N = 1$ with all members in a single cluster. The probabilities of generating a new cluster under (i) and (ii) are both $\alpha/(N + \alpha)$ in CRP, but dramatically different in pCRP: (i) $\alpha/(N + \alpha)$ and (ii) $\alpha/(N^r + \alpha)$, respectively. Therefore, if the previous customers are more spread out, there is a larger probability of continuing this pattern by creating new tables.  Similarly, if customers choose a small
number of tables, then a new customer is more likely to join the dominant clusters rather than open a new table.

The power $r$ is a critical parameter controling how much we penalize small clusters. The larger the power $r$, the greater the penalty.  A method is proposed to choose $r$ in a data-driven fashion: cross validation using a proper loss function to select a fixed $r$.

\subsubsection{Power parameter tuning}\label{sec:pcrp_parameter_calibration}
The proportional invariance property makes it easier to define a cross validation (CV) procedure for estimating $r$.  In particular, one can tune $r$ to obtain good performance on an initial training sample and that $r$ would also be appropriate for a subsequent data set that has a very different sample size.  For other choices of $g(\cdot)$, which do not possess proportional invariance, it may be necessary to adapt $r$ to the sample size for appropriate calibration.   

In evaluating generalization error, we use the following loss function based on within-cluster sum of squares:
\begin{equation}
 \sum_{k=1}^K  \sqrt{ \sum_{j: j \in C_k}^{N_k} ||\bx_j - \overline{\bx}_k||^2 },
\end{equation}
where $C_k$ is the data samples in the $k^{th}$ cluster and $ \overline{\bx}_k$ is the mean vector for cluster $k$. The square root has an important impact in favoring a smaller nunber of clusters (see also the discussion about inertia and squared inertia in Section~\ref{section:squared-inertia}); for example, inducing a price to be paid for introducing two clusters with the same mean.  In implementing CV, we start by choosing a small value of $r$ ($r=1+\epsilon$) and then increasing until we identify an inflection point.

\subsubsection{Posterior inference by collapsed Gibbs sampling}
Although the proposed pCRP is generic, we focus on its application in Gaussian mixture models for concreteness. We here introduce a collapsed Gibbs sampling algorithm~\citep{neal2000markov} for posterior computation. In addition, we permute the data at each sampling iteration to eliminate order dependence as in~\citep{socher2011spectral}.

Again, let $\mathcalX$ be the observations, assumed to follow a mixture of multivariate Gaussian distributions. We use a conjugate normal-inverse-Wishart (NIW) prior $p(\bmu, \bSigma | \bbeta)$ for the mean vector $\bmu$ and covariance matrix $\bSigma$ in each multivariate Gaussian component, where $\bbeta$ consists of all the hyperparameters in NIW. A key quantity in a collapsed Gibbs sampler is the probability of each customer $i$ sitting with table $k$: $p(z_i = k | \bznoi, \mathcalX, \alpha, \bbeta)$, where $\bznoi$ are the seating assignments of all the other customers and $\alpha$ is the concentration parameter in CRP and pCRP. This probability is calculated as follows:
\begin{equation} 
\begin{aligned}
p(z_i = k| \bznoi, \mathcal{X} , \alpha, \bbeta)  & \varpropto p(z_i = k | \bznoi, \alpha, \cancel{\bbeta})  p(\mathcal{X} |z_i = k, \bznoi, \cancel{\alpha}, \bbeta) \\
& = p(z_i = k| \bznoi, \alpha) p(\bxi |\mathcal{X}_{-i}, z_i = k, \bznoi, \bbeta) p(\mathcal{X}_{-i} |\cancel{z_i = k}, \bznoi, \bbeta)\\
& \varpropto p(z_i = k| \bznoi, \alpha) p(\bxi|\mathcal{X}_{-i}, z_i = k, \bznoi, \bbeta) \\
& \varpropto p(z_i = k| \bznoi, \alpha) p(\bxi | \xknoi, \bbeta), 										
\end{aligned}
\label{equation:pcrp_ifmm_collabsed_gibbs22}
\end{equation}  
where $\xknoi$ are the observations in table $k$ excluding the $i^{th}$ observation. Algorithm~\ref{algo:pcrp_ifmm_plain_gibbs} gives the pseudo code of the collapsed Gibbs sampler to implement pCRP in Gaussian mixture models.
\IncMargin{1em}
\begin{algorithm}
\SetKwData{Left}{left}\SetKwData{This}{this}\SetKwData{Up}{up}
\SetKwFunction{Union}{Union}\SetKwFunction{FindCompress}{FindCompress}
\SetKwInOut{Input}{input}\SetKwInOut{Output}{output}
\Input{Choose an initial $\bz$, $r$, $\alpha$, $\bbeta$}
\BlankLine

\For{$T$ iterations}{
Sample random permutation $\tau$ of $1, \ldots, N$ \;
\For{$i \in (\tau(1), \ldots, \tau(N))$}{
	Remove $\bxi$'s statistics from component $z_i$ \;
	\For{$k\leftarrow 1$ \KwTo $K$}{
			Calculate $p(z_i=k| \bznoi, \alpha) =  \frac{N_{k,-i}^r}{\sum_h^K N_{h,-i}^r+\alpha}$\;
			Calculate $p(\bxi | \xknoi, \bbeta)$\;
			Calculate $p(z_i = k | \bznoi, \mathcal{X}, \alpha, \bbeta) \propto p(z_i=k| \bznoi, \alpha) p(\bxi | \xknoi, \bbeta)$\;
	}
	Calculate $p(z_i = k^\star | \bznoi, \alpha)= \frac{\alpha}{\sum_h^K N_{h,-i}^r+\alpha}$\;
	Calculate $p(\bxi | \bbeta)$\;
	Calculate $p(z_i = k^\star | \bznoi, \mathcalX, \alpha, \bbeta) \propto p(z_i = k^\star | \bznoi, \alpha) p(\bxi | \bbeta)$\;
	Sample $k_{new}$ from $p(z_i | \bznoi, \mathcalX, \alpha, \bbeta)$ after normalizing\;
	Add $\bxi$'s statistics to the component $z_i=k_{new}$ \;
	If any component is empty, remove it and decrease $K$.
}
}
\caption{Collapsed Gibbs sampler for a pCRP Gaussian mixture model.}\label{algo:pcrp_ifmm_plain_gibbs}
\end{algorithm}\DecMargin{1em}

\subsubsection{Future work}
Further to powered Chinese restaurant process, we introduce an adaptive version of it.
Adaptive powered Chinese restaurant process (Ada-pCRP) is an extension of pCRP that overcomes the main weekness of pCRP. The idea of Ada-pCRP is simple: it adaptively choose the power $r$ from the proportion of small tables in all tables. In machine learning field, we have a lot of adaptive gradient descent methods \citep{ruder2016overview}: AdaGrad \citep{duchi2011adaptive} is an algorithm for gradient-based optimization that does just this: it adapts the learning rate to the parameters, performing larger updates for infrequent and smaller updates for frequent parameters; AdaDelta \citep{zeiler2012adadelta} is an extension of AdaGrad that seeks to reduce its aggressive, monotonically decreasing learning rate. 

Instead of choosing a power $r$ for all sampling steps, we need to choose an upper bound power $r_{up}>1$. When the proportion of small clusters (noise)  $p_{noise}$ is large, we tend to tune the power $r$ towards $r_{up}$. Otherwise we tend to make it close to $1$ (i.e., Chinese restaurant process). In practice, the $p_{noise}$ can be chosen by the percentage of small clusters. Formally we have 
\begin{equation}
r = 1 + (r_{up} - 1) \times p_{noise}
\end{equation}
and
\begin{equation}
p(z_i = k | z_{-i}, \alpha)=\left\{
                \begin{array}{ll}
                  \frac{N_{k,-i}^r}{\sum_h^K N_{h,-i}^r+\alpha},  \text{ if \textit{k} is occupied, i.e. } N_k > 0\\
                  \frac{\alpha}{\sum_h^K N_{h,-i}^r+\alpha},  \text{ if \textit{k} is a new table, i.e. } k = k^{\star} = K + 1	
                \end{array}
              \right.
\label{equation:ada_powered_crp_equation}
\end{equation}
But the convergence of Ada-pCRP cannot be guaranteed because of the changing of power value.

\subsubsection{Examples}
We conduct some examples to demonstrate the main advantages of the proposed pCRP using both synthetic and real data. In a wide range of scenarios across various sample sizes, pCRP reduces over-clustering of CRP, and leads to performances that are as good or better than CRP in terms of density estimation, out of sample prediction, and overall clustering results. 

In all experiments, we run the Gibbs sampler 20,000 iterations with a burn-in of 10,000. The sampler is thinned by keeping every 5$^{th}$ draw. We use the same concentration parameter $\alpha = 1$ for both CRP and pCRP in all scenarios. In addition, we equip CRP with an unfair advantage to match the magnitude of its prior mean $\alpha \log(N)$ to the true number of clusters, termed as {\it CRP-Oracle}. The power $r$ in pCRP is tuned using cross validation. In order to measure overall clustering performance, we use normalized mutual information (NMI) \citep{mcdaid2011normalized} and variation of information (VI) \citep{meilua2003comparing}, which measures the similarity between the true and estimated cluster assignments. Higher NMI and lower VI indicate better performance.  If applicable, metrics using the true clustering are calculated to provide an upper bound for all methods, coded as `Ground Truth'. The metrics are discussed in Section~\ref{section:some-metrics}. Feel free to skip this section for a first reading.

\subsubsection{Simulation experiments}\label{sec:simulated_examples} 
We first use simulated data to assess the performance of pCRP in emptying extra components, compared to the traditional CRP.  Figure~\ref{fig:pcrp_prior_densities} shows the true data generating density, which represent the two cases of well-mixed Gaussian components and shared mean Gaussian mixture coded as Sim 1 and Sim 2, respectively. The parameters of the simulations are as follows:

\begin{itemize}
\item Sim 1: $K_0=3$, with $N$=300, $\bm{\pi}$=\{0.35, 0.4, 0.25\}, $\bm{\mu}$=\{0, 2, 5\} and $\bm{\Sigma}$=\{0.5, 0.5, 1\};
\item Sim 2: $K_0=2$, with $N$=500, $\bm{\pi}$=\{0.65, 0.35\}, $\bm{\mu}$=\{1, 1\} and $\bm{\Sigma}$=\{10, 1\};
\end{itemize}

The oracle concentration parameters in CRP-Oracle are (0.52, 0.40) in Sim 1 and (0.35, 0.26) in Sim 2 corresponding to the two sample sizes (300, 2000), which are all smaller than the unit concentration parameter used in CRP and pCRP. Figure~\ref{fig:pcrp_cross-validation} shows the cross validation curve to select $r$ in pCRP using a training data set with 200 samples. The representative cases of infection point described in Section \ref{sec:pcrp_parameter_calibration} were observed: the loss curve for cross validation blows up at one point of $r$ value in Sim 1, while the curve decreases rapidly at one point of $r$ value in Sim 2. We choose this change point as the power $r$ in either case. 

\begin{figure}[h!]
\center
\subfigure[Sim 1]{\includegraphics[width=0.4\textwidth]{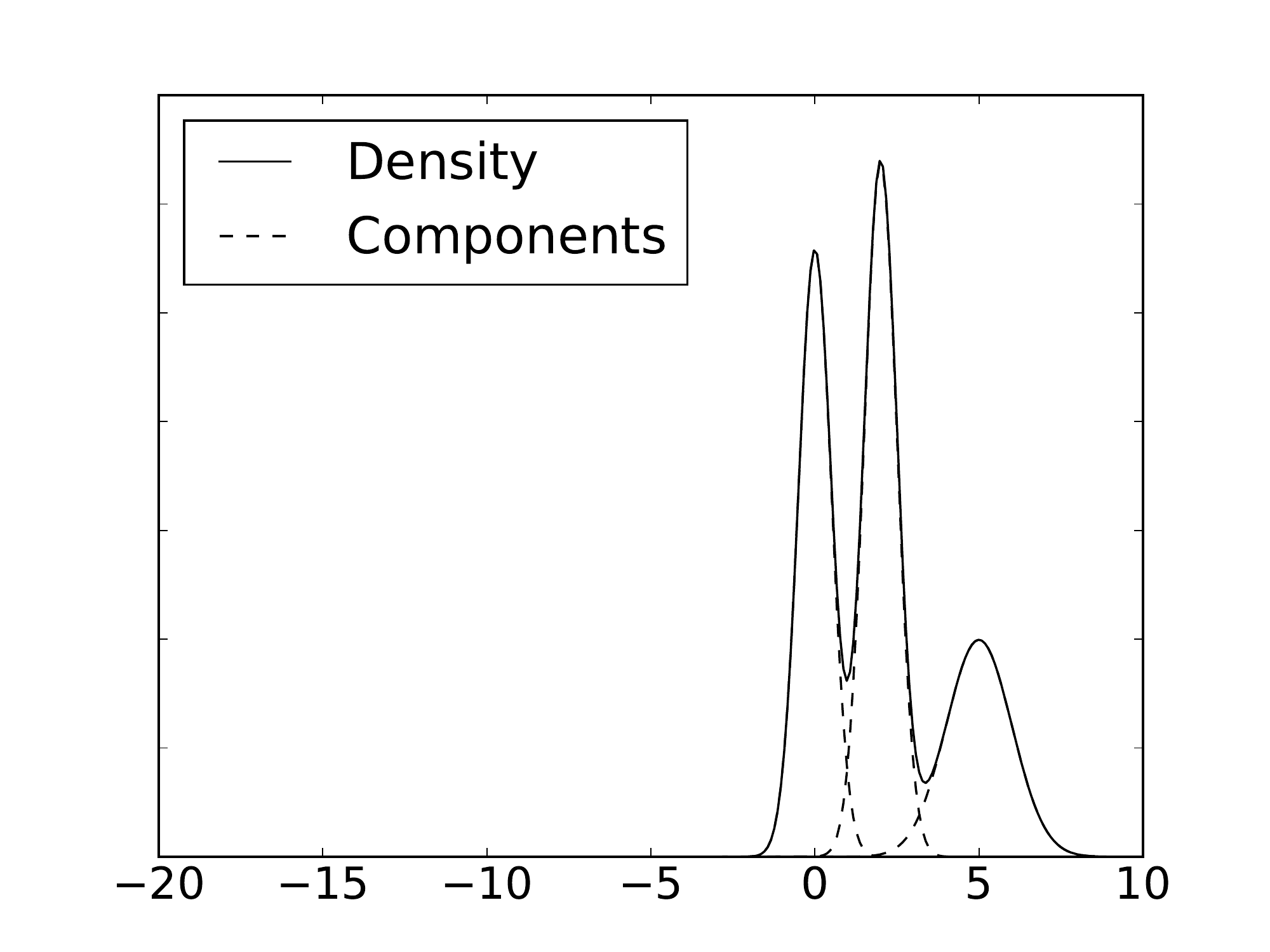} \label{fig:true_density_sim2}}
\subfigure[Sim 2]{\includegraphics[width=0.4\textwidth]{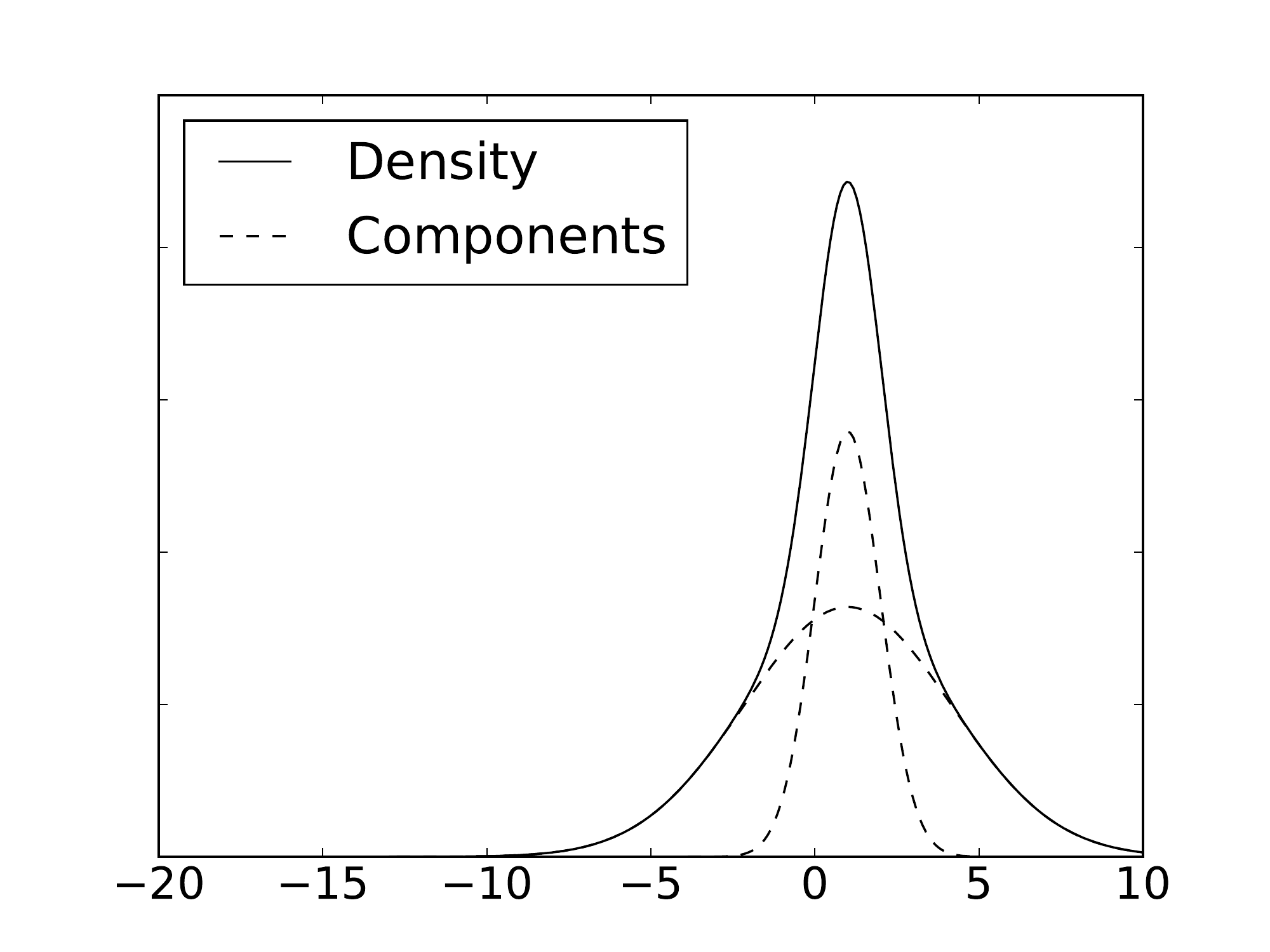} \label{fig:true_density_sim6}}
\caption{Data generating densities for two scenarios: (a) Sim1: a mixture of three poorly separated Gaussian components; (b) Sim 2: a mixture of two components with the same mean value.}
\label{fig:pcrp_prior_densities}
\end{figure}

\begin{figure}[h!]
\center
\subfigure[Sim 1]{\includegraphics[width=0.4\textwidth]{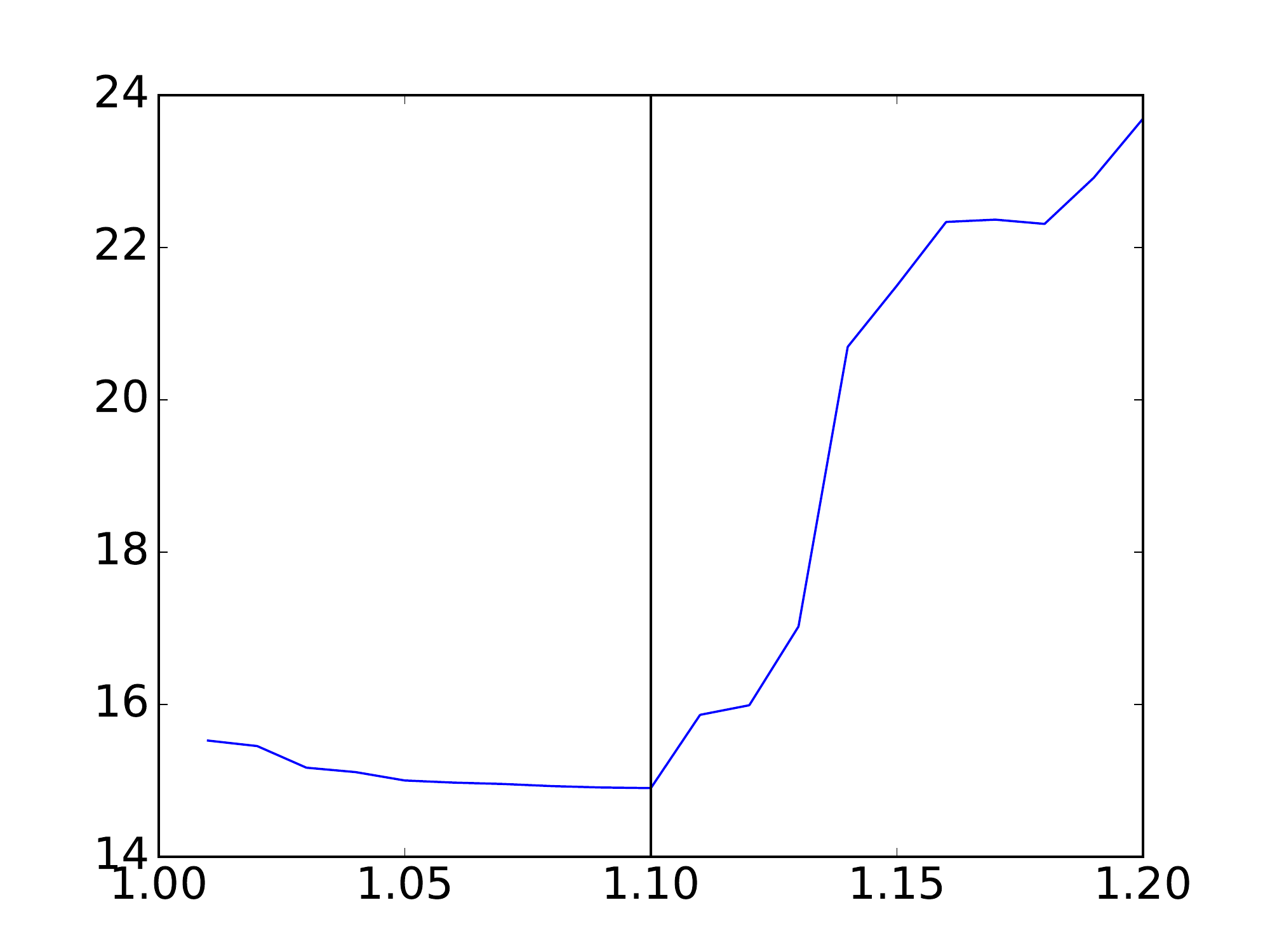} \label{fig:pcrp_cv_sim2}}
\subfigure[Sim 2]{\includegraphics[width=0.4\textwidth]{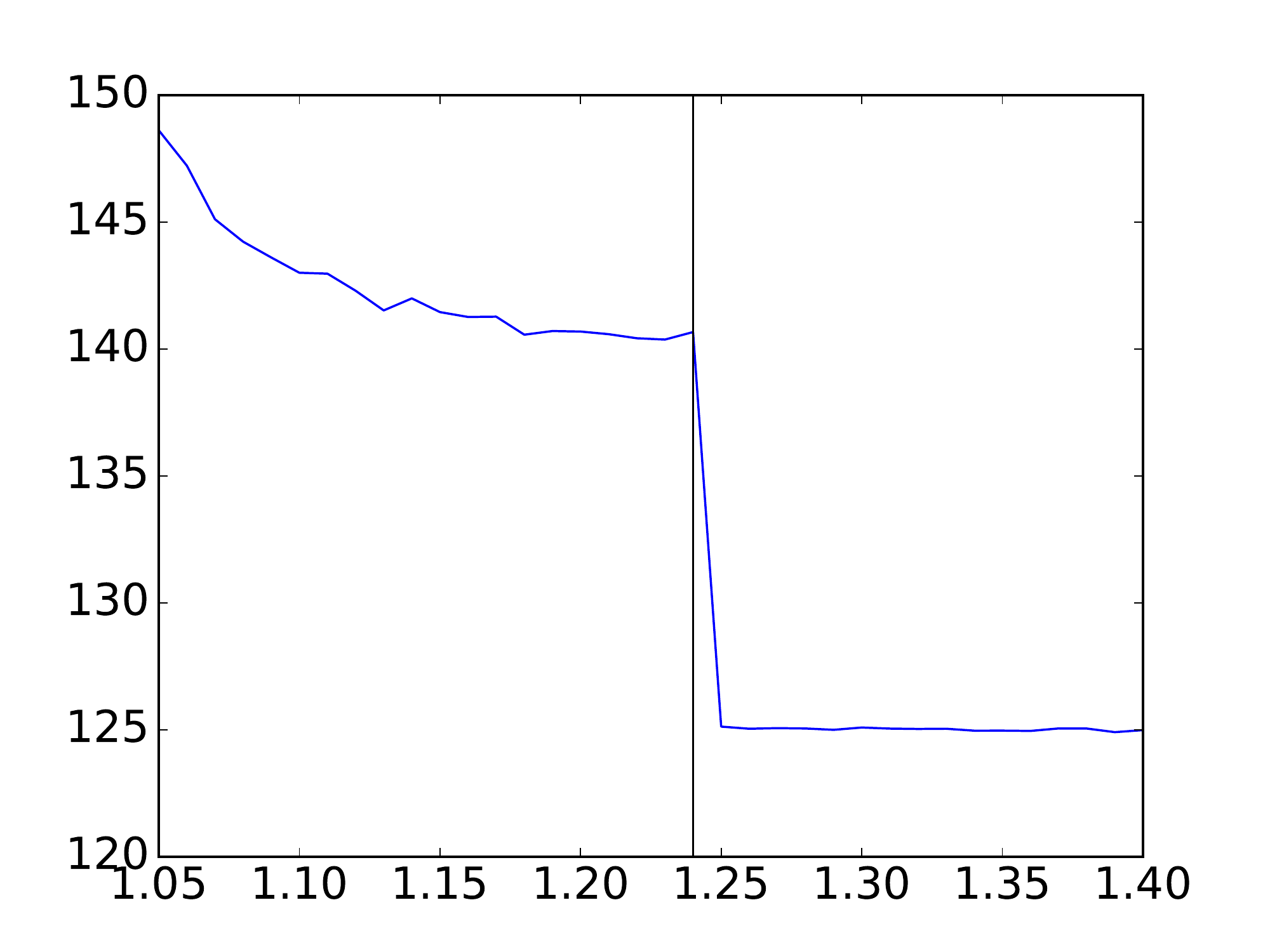} \label{fig:pcrp_cv_sim6}}
\caption{Cross validation curves to choose $r$ for Sim 1 and Sim 2. The $x$-axis is the power value, the $y$-axis is the loss. The vertical line is the chosen power $r$ value.}
\label{fig:pcrp_cross-validation}
\end{figure}

\begin{figure}[!h]
	\center
	\subfigure[ CRP-Oracle]{\includegraphics[width=0.3\textwidth]{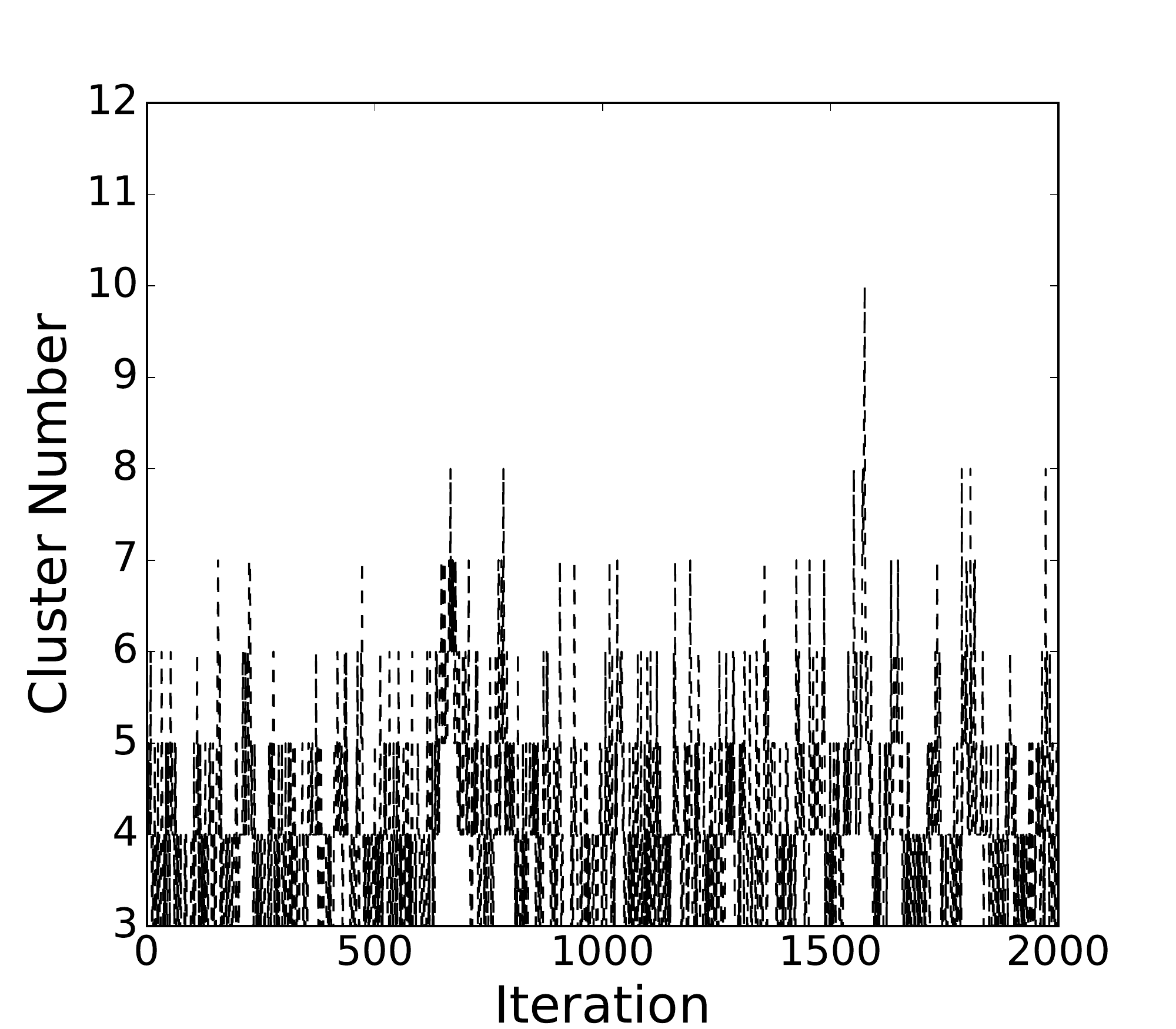} \label{fig:pcrp_clusternum_3methods_traceplot_sim1_crp04}}
	\subfigure[ CRP]{\includegraphics[width=0.3\textwidth]{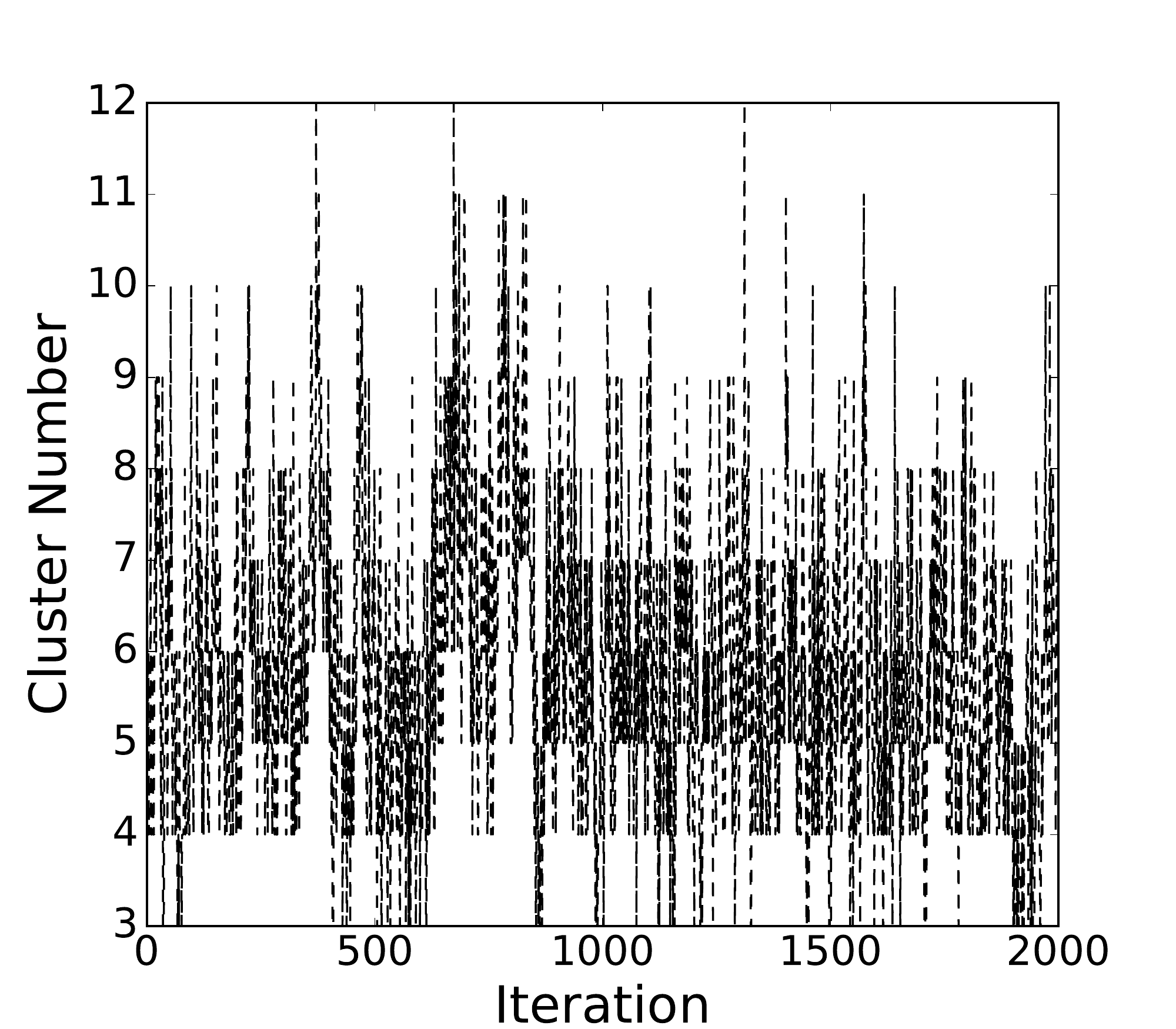} \label{fig:pcrp_clusternum_3methods_traceplot_sim1_crp1}}
	\subfigure[ pCRP]{\includegraphics[width=0.3\textwidth]{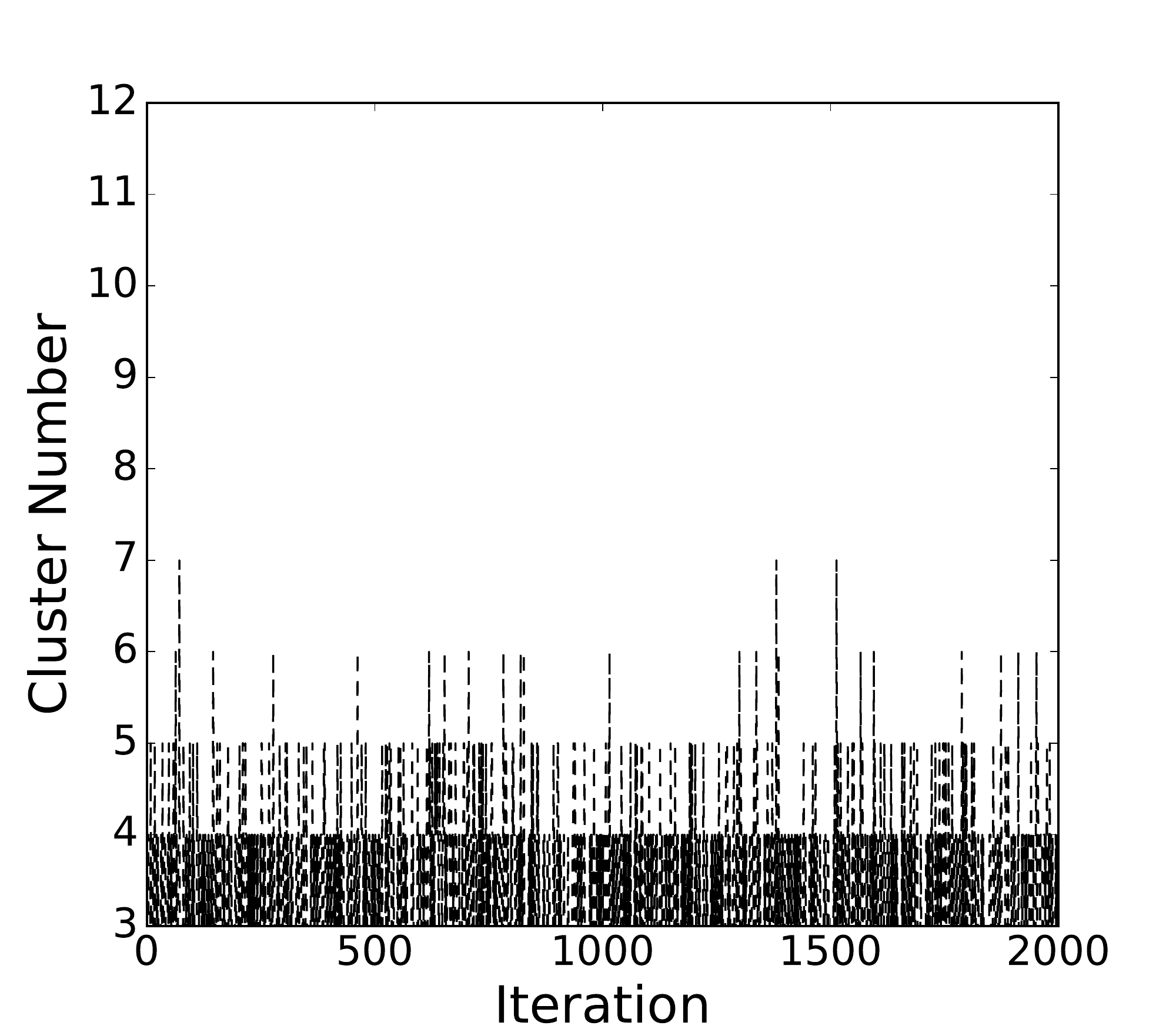} \label{fig:pcrp_clusternum_3methods_traceplot_sim1_pcrp1}}
	\caption{Traceplots of cluster numbers using the three methods in Sim 1 when $N=2000$. The $x$-axis is the sampling iteration, the $y$-axis is the number of clusters.}
	\label{fig:pcrp_clusternum_3methods_traceplot_sim1}
\end{figure}

\begin{figure}[!h]
	\center
	\subfigure[CRP-Oracle]{\includegraphics[width=0.3\textwidth]{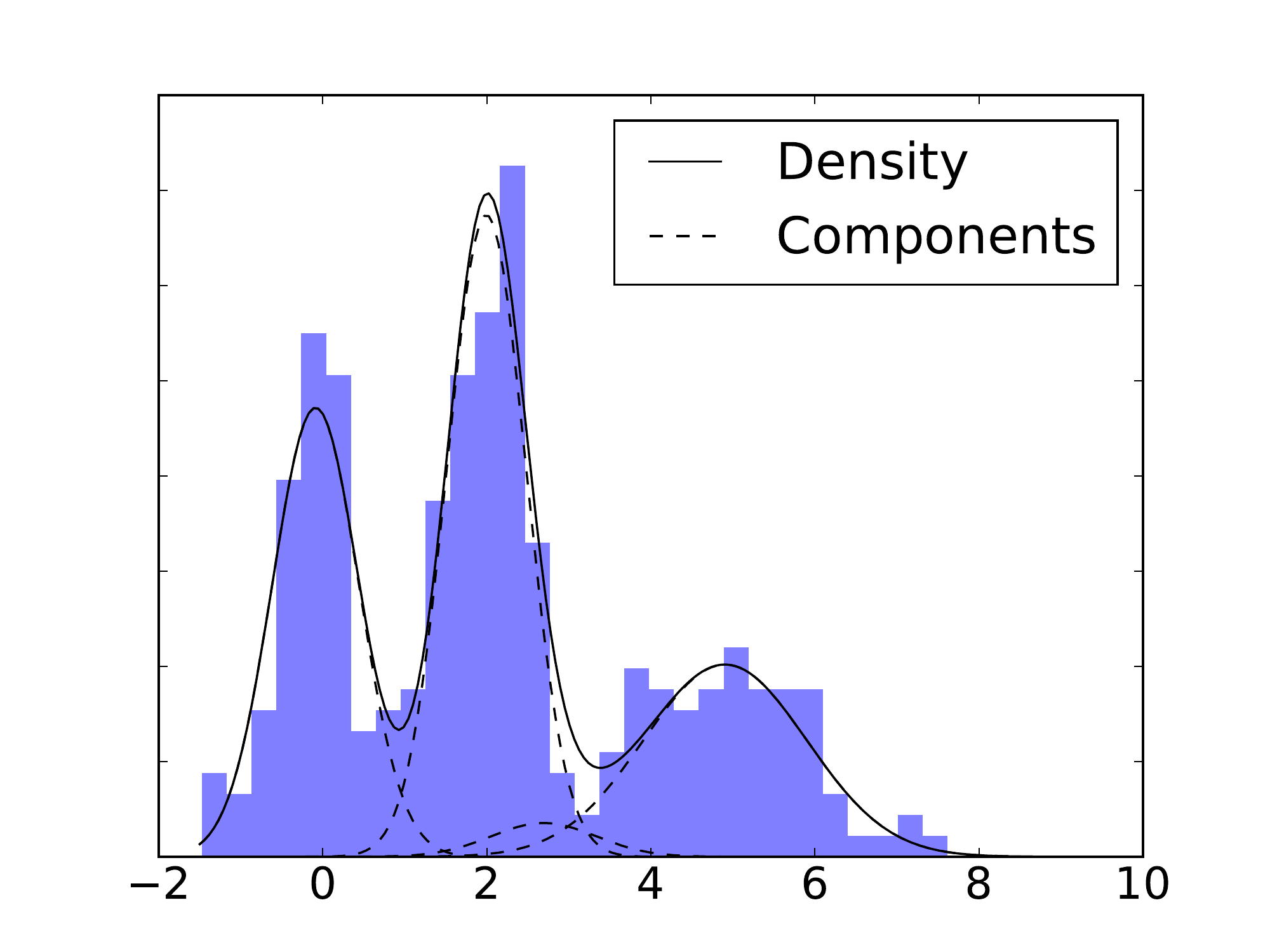} \label{fig:adhoc_3methods_traceplot_sim1_posterior_densities_crp_match}}
	\subfigure[CRP]{\includegraphics[width=0.3\textwidth]{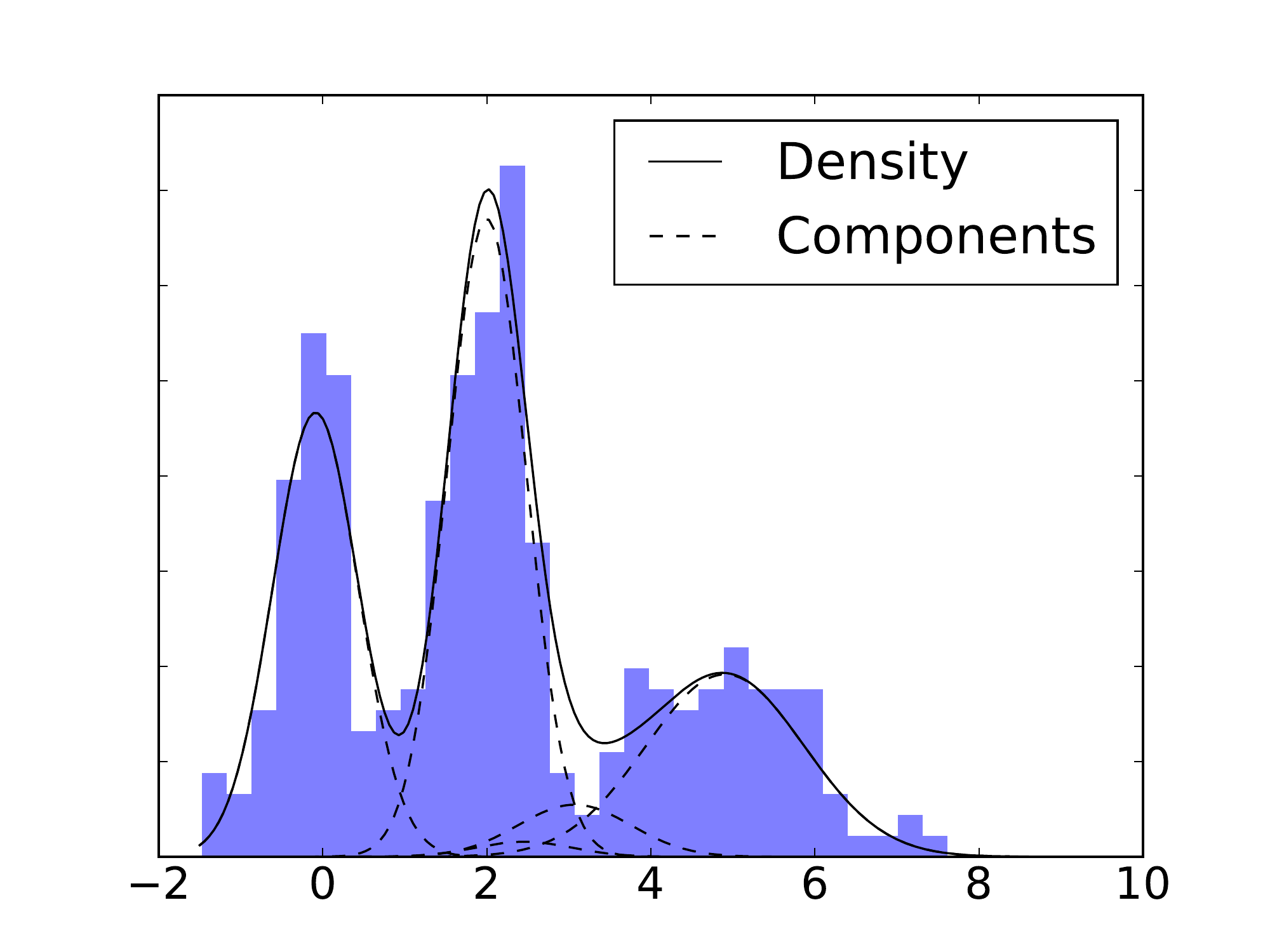} \label{fig:adhoc_3methods_traceplot_sim1_posterior_densities_crp1}}
	\subfigure[pCRP]{\includegraphics[width=0.3\textwidth]{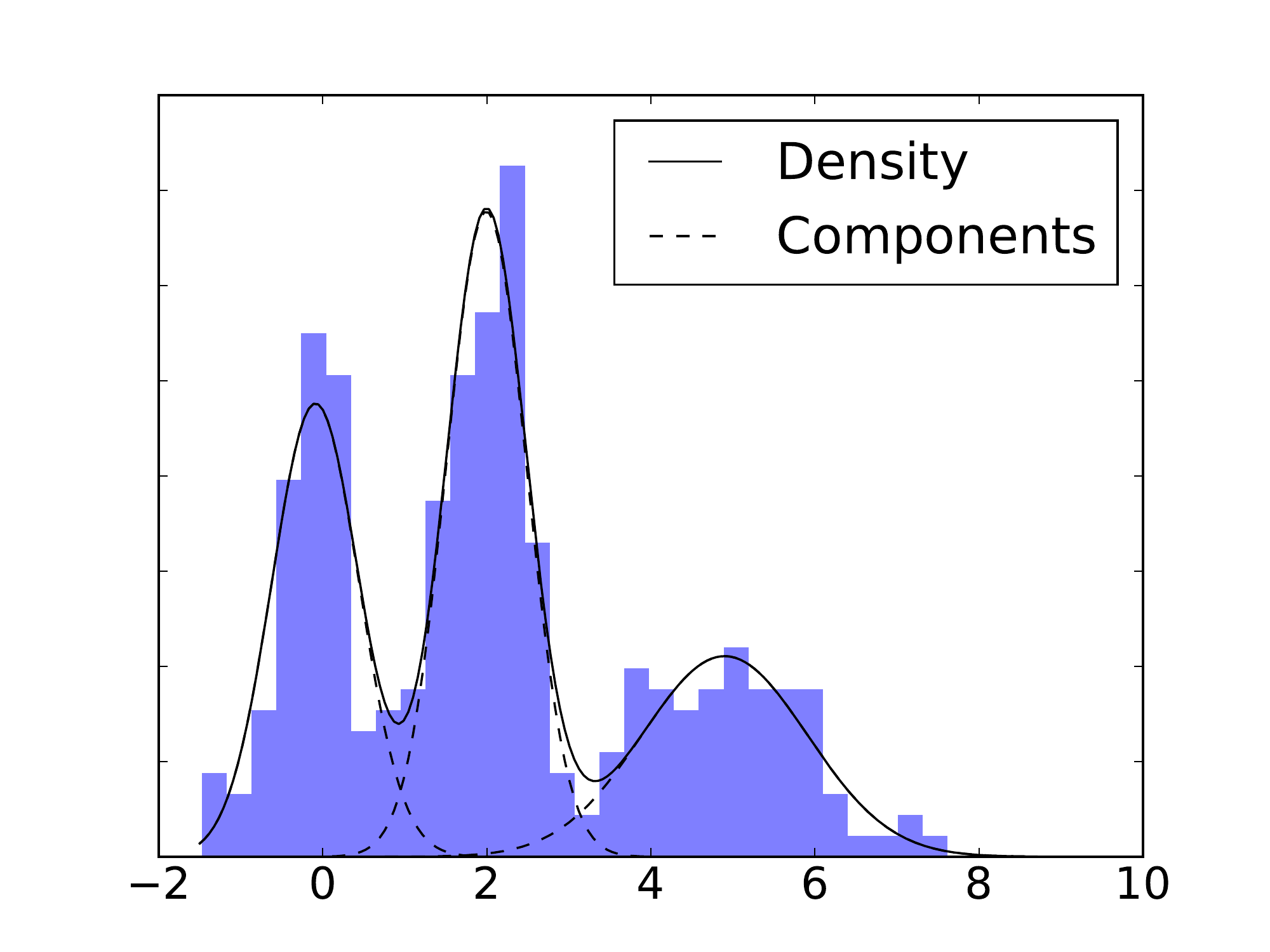} \label{fig:adhoc_3methods_traceplot_sim1_posterior_densities_pcrp1}}
	\caption{Posterior densities for three methods in Sim 1 when $N=2000$. The dashed lines are weighted components. 
	}
	\label{fig:pcrp_3methods_traceplot_sim1_posterior_densities}
\end{figure}

Figure~\ref{fig:pcrp_clusternum_3methods_traceplot_sim1} shows traceplots of posterior samples for the number of clusters for each of the methods in Sim 1.  Clearly pCRP places relatively high posterior probability on three clusters, which is the ground truth. In contrast, CRP has higher posterior variance, systematic over-estimation of the number of clusters, and worse computational efficiency.  The CRP-Oracle has better performance, but does clearly worse than p-CRP, and there is still a tendency for over-estimation.
Figure~\ref{fig:posterior_num_clusters} suggests that CRP will have larger probability on larger cluster numbers especially when the sample size increases, while pCRP tends to have larger probability on the true cluster number as the sample size increases. 
For example, in Sim 1, the probability of selecting three clusters increases from 0.55 to 0.68 in pCRP when $N$ increases from 300 to 2000 and the probability for all the other cluster number decreases. 
However, the probability of finding four clusters stabilizes around 0.37 and 0.38 in CRP-Oracle when $N$ increases from 300 to 2000. CRP has increased probability of selecting larger number of clusters (say 5, 6, 7, 8 clusters) when $N$ increases from 300 to 2000. In fact, the proposed pCRP has the largest concentration probability on the true number of clusters among all the three methods including CRP-Oracle, and this observation is consistent between $N = 300$ and $N = 2000$.  

\begin{figure}[h!]
	\center
	\subfigure[Sim 1, $N$=300]{\includegraphics[width=0.45\textwidth]{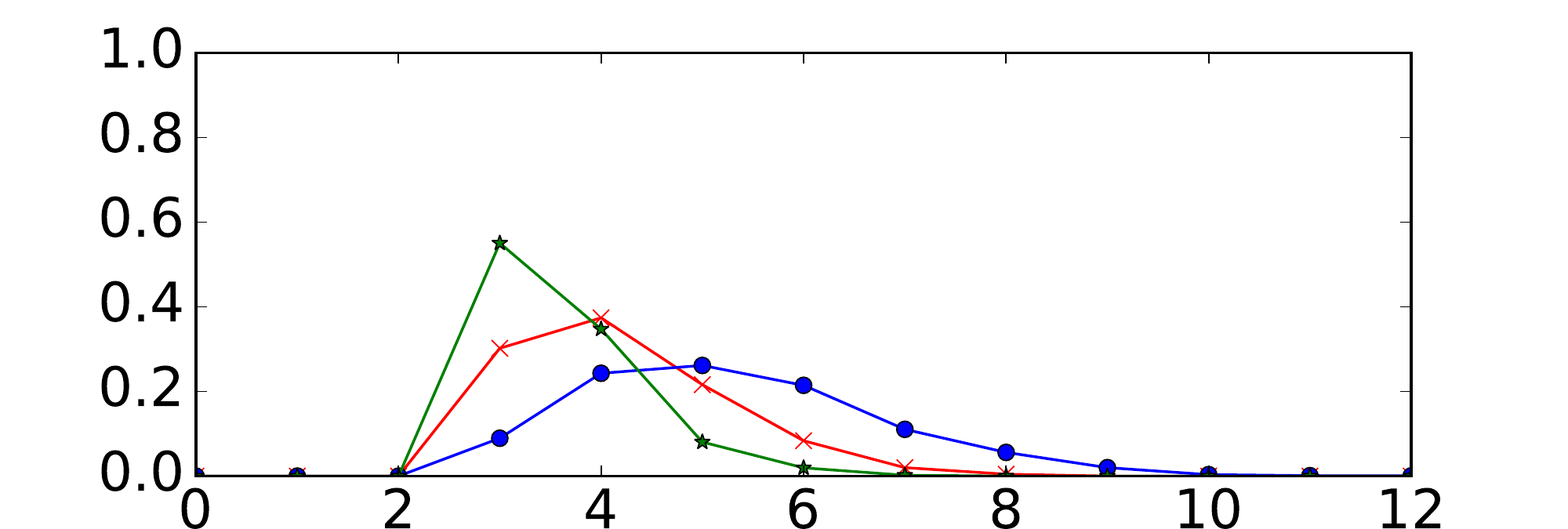} \label{fig:posterior_num_sim2_n300}}
	\subfigure[Sim 2, $N$=300]{\includegraphics[width=0.45\textwidth]{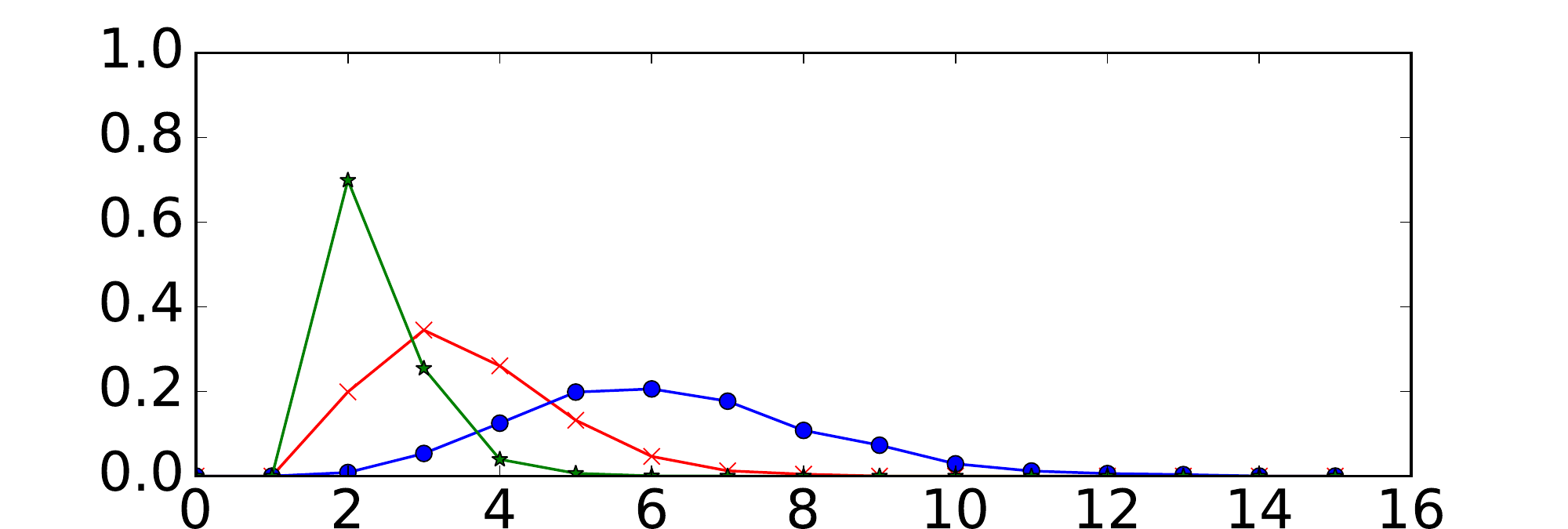} \label{fig:posterior_num_sim6_n300}}
	\subfigure[Sim 1, $N$=2000]{\includegraphics[width=0.45\textwidth]{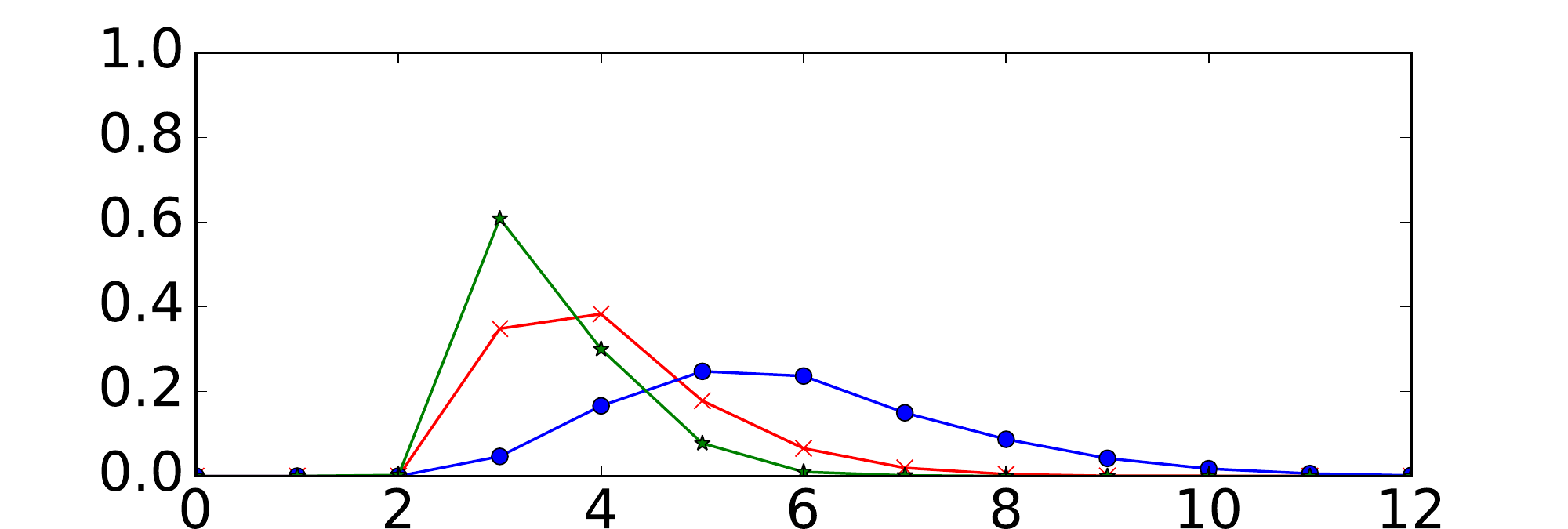} \label{fig:posterior_num_sim2_n2000}}
	\subfigure[Sim 2, $N$=2000]{\includegraphics[width=0.45\textwidth]{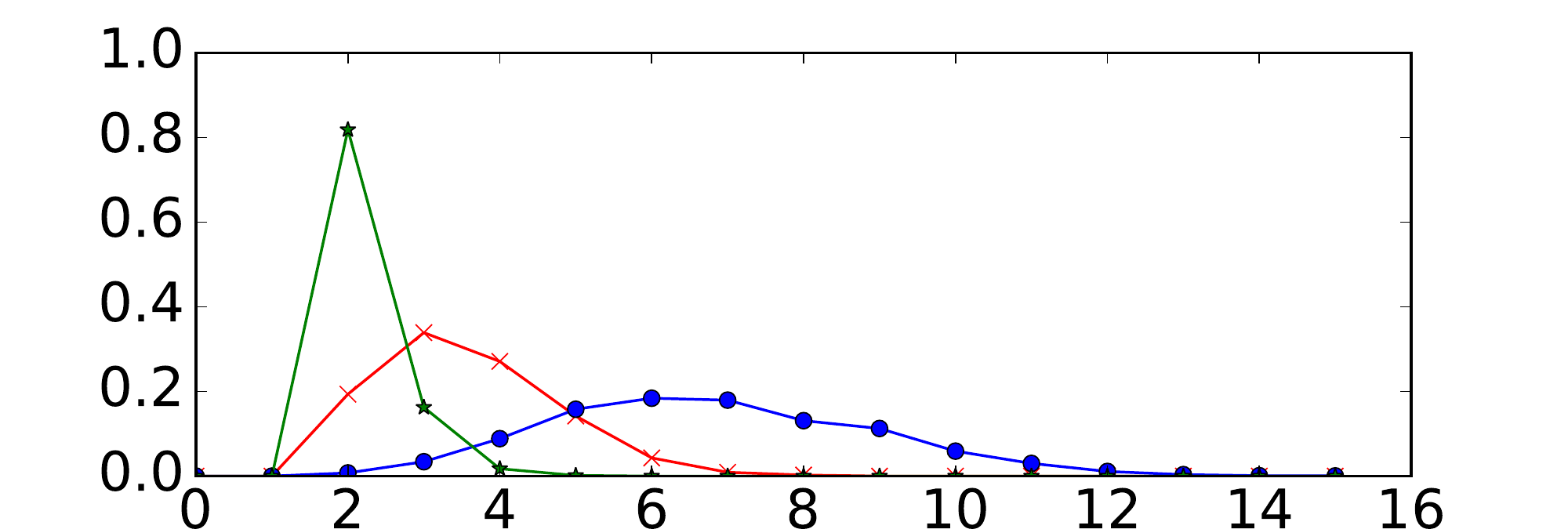} \label{fig:posterior_num_sim6_n2000}}
	\caption{Estimated posterior of the number of clusters in observed data for CRP-Oracle (red x), CRP (blue circle) and pCRP (green star). }
	\label{fig:posterior_num_clusters}
\end{figure}

Table~\ref{table:pcrp_simulation_posterior_summary} provides numerical summaries of this simulation. We can see all three methods lead to similar NMI, but pCRP consistently gives the highest value. Furthermore, pCRP leads to the lowest value of VI in most tests. The parsimonious effect of pCRP discussed above is further confirmed by the average and maximum number of clusters; see the columns $K$ and $K_{\max}$ in the table.

The posterior density plots in Figure~\ref{fig:pcrp_3methods_traceplot_sim1_posterior_densities} show that there is one small unnecessary cluster in CRP-Oracle and two small unnecessary clusters in CRP, while all three methods capture the general shape of the true density thus provide good fitting performance. The over-clustering effect of CRP is much reduced by pCRP as seen in Figure~\ref{fig:adhoc_3methods_traceplot_sim1_posterior_densities_pcrp1}.

\begin{table}[!h]
\centering
\begin{tabular}{|ccccc|}
\hline
\multicolumn{5}{|c|}{$N=300$}                                                                     \\ \hline
Method               & NMI (SE $\times 10^{-3}$) & VI (SE $\times 10^{-3}$) & $K$ (SE $\times 10^{-2}$) & $K_{\max}$ \\ \hline
Ground truth (Sim 1) & 1.0                & 0.0               & 3                  & -            \\
CRP-Oracle (Sim 1)   & 0.800 (1.1)        & 0.669 (4.5)       & 4.2 (2.3)          & 8            \\
CRP (Sim 1)          & 0.773 (1.2)        & 0.795 (5.4)       & 5.3 (3.3)          & 12           \\
pCRP (Sim 1)         & 0.827 (0.7)       & 0.580 (4.4)       & 3.6 (1.7)          & 7            \\ \hline
Ground truth (Sim 2) & 1.0                & 0.0               & 2                  & -            \\
CRP-Oracle (Sim 2)   & 0.211 (1.0)        & 1.803 (6.4)       & 3.5 (2.7)          & 8            \\
CRP (Sim 2)          & 0.189 (0.9)        & 2.164 (7.8)       & 6.2 (4.2)          & 13           \\
pCRP (Sim 2)         & 0.228 (1.0)        & 1.518 (2.1)       & 2.4 (1.3)          & 6            \\ \hline \hline
\multicolumn{5}{|c|}{$N=2000$}                                                                    \\ \hline
Method               & NMI (SE $\times 10^{-4}$) & VI (SE $\times 10^{-3}$) & $K$ (SE $\times 10^{-2}$) & $K_{\max}$ \\ \hline
Ground truth (Sim 1) & 1.0                & 0.0               & 3                  & -            \\
CRP-Oracle (Sim 1)   & 0.812 (5.3)        & 0.610 (2.6)       & 4.0 (2.3)          & 10           \\
CRP (Sim 1)          & 0.782 (8.5)        & 0.732 (4.0)       & 5.8 (3.6)          & 12           \\
pCRP (Sim 1)         & 0.823 (6.6)        & 0.869 (7.3)       & 3.5 (1.6)          & 7            \\ \hline
Ground truth (Sim 2) & 1.0                & 0.0               & 2                  & -            \\
CRP-Oracle (Sim 2)   & 0.258 (7.0)        & 1.537 (5.3)       & 3.5 (2.6)          & 8            \\
CRP (Sim 2)          & 0.238 (7.1)        & 1.755 (6.8)       & 6.8 (4.7)          & 15           \\
pCRP (Sim 2)         & 0.258 (4.2)        & 1.368 (0.7)      & 2.2 (1.0)          & 5            \\ \hline
\end{tabular}
\caption{Comparison of CRP and pCRP on Sim 1 and Sim 2. $K$ is the average number of found clusters. $K_{\max}$ is the maximum number of clusters during sampling. SE is the standard error of mean. Ground truth is calculated using the true assignments.}
\label{table:pcrp_simulation_posterior_summary}
\end{table}

\begin{figure}[h!]
	\center
	\subfigure[True clustering when $N$=3000]{\includegraphics[width=0.45\textwidth]{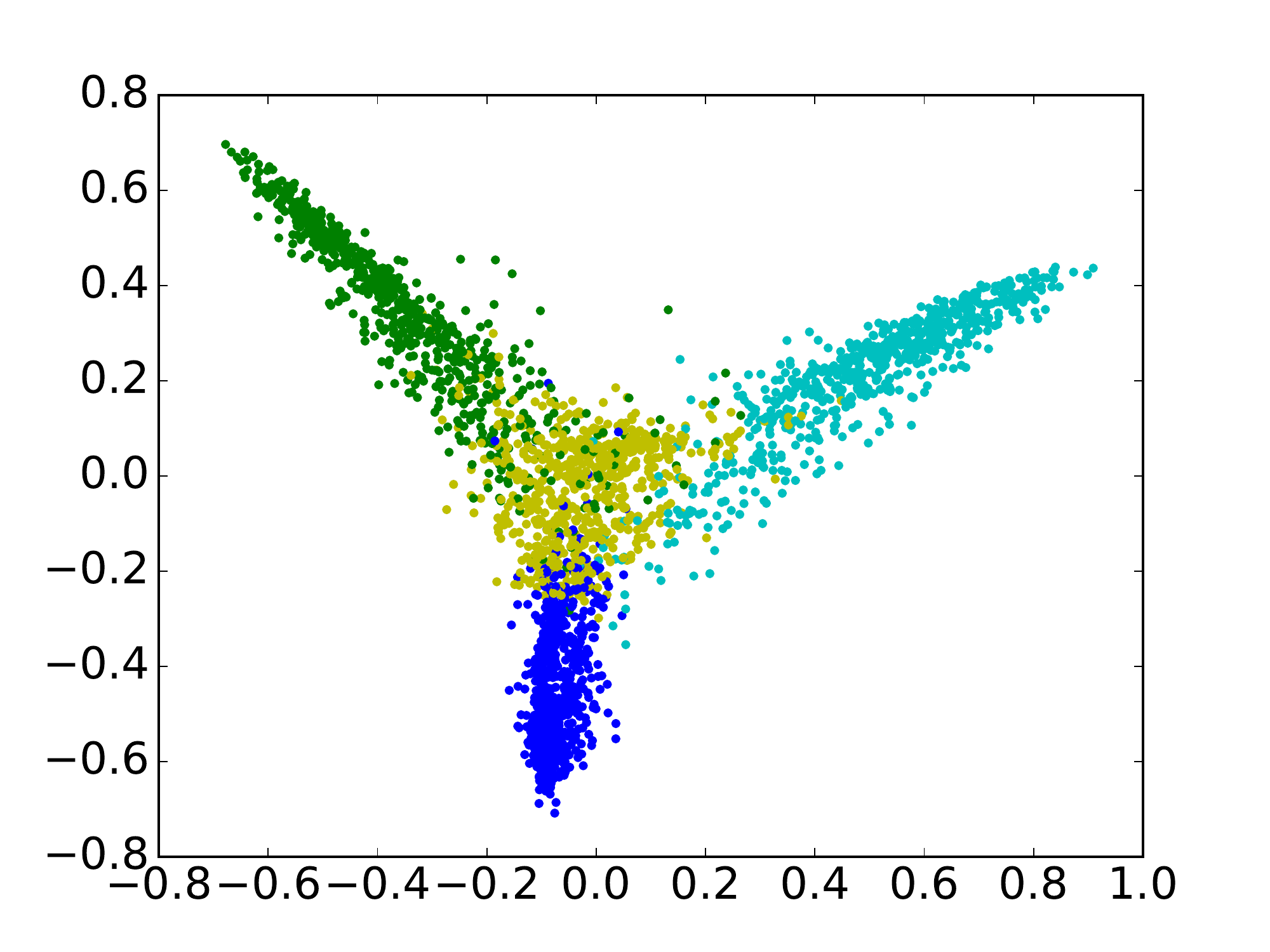} \label{fig:pcrp_digits14_n3000_true_clustering}}
	\subfigure[CRP-Oracle when $N$=3000]{\includegraphics[width=0.45\textwidth]{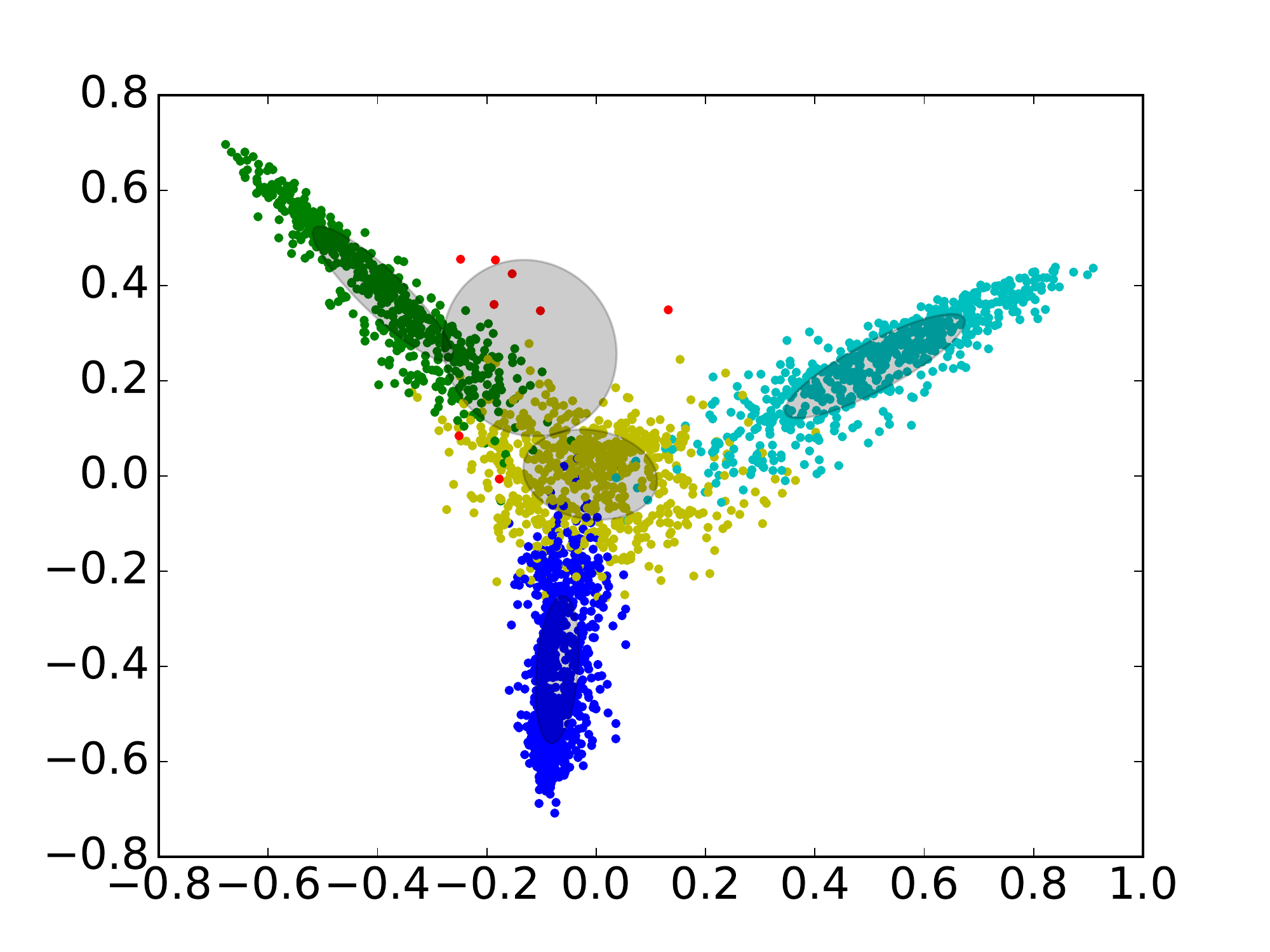} \label{fig:crp_digits14_n3000_crp_match}}
	\subfigure[CRP when $N=3000$]{\includegraphics[width=0.45\textwidth]{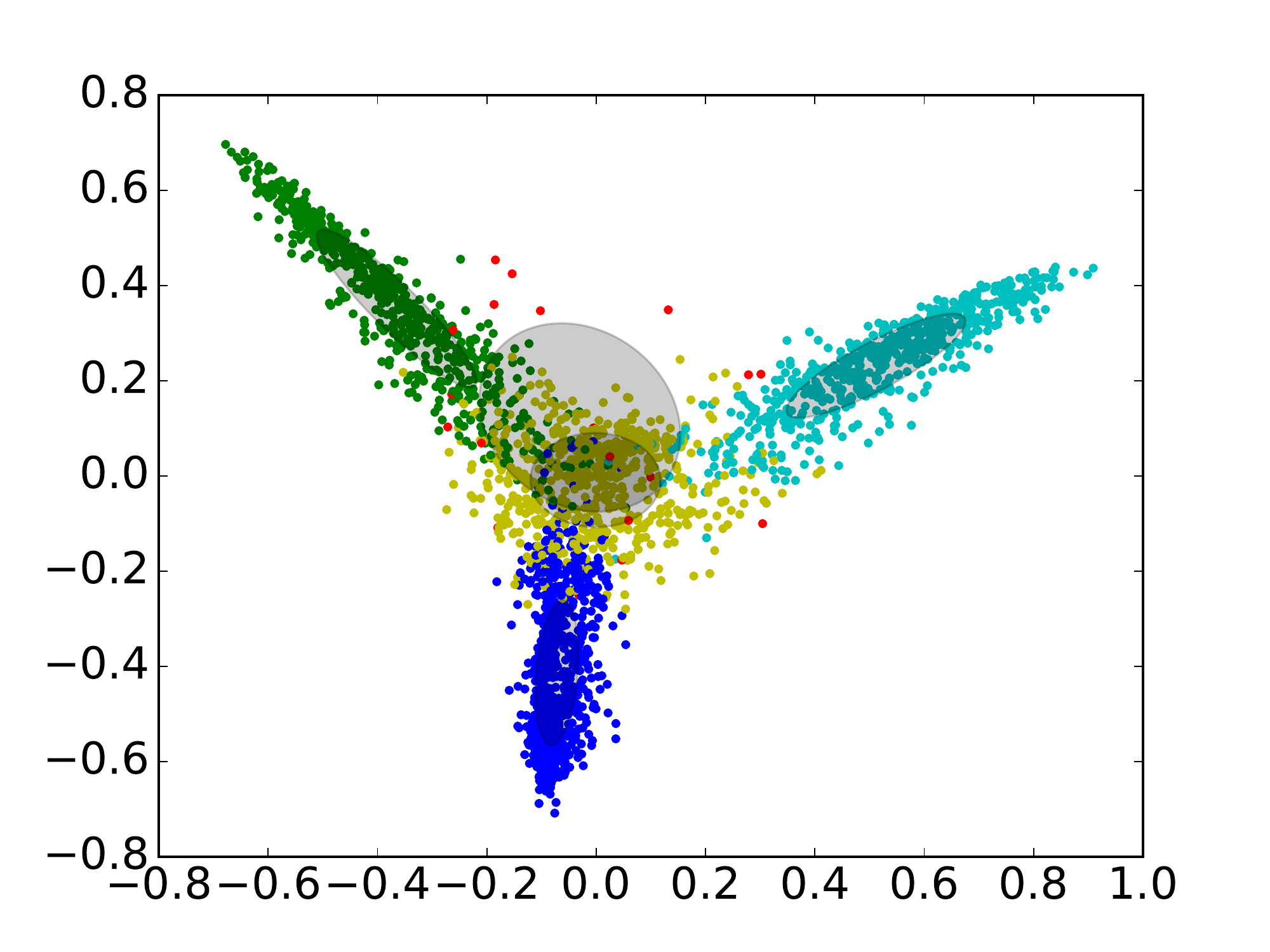} \label{fig:crp_digits14_n3000_crp1}}
	\subfigure[pCRP when $N$=3000]{\includegraphics[width=0.45\textwidth]{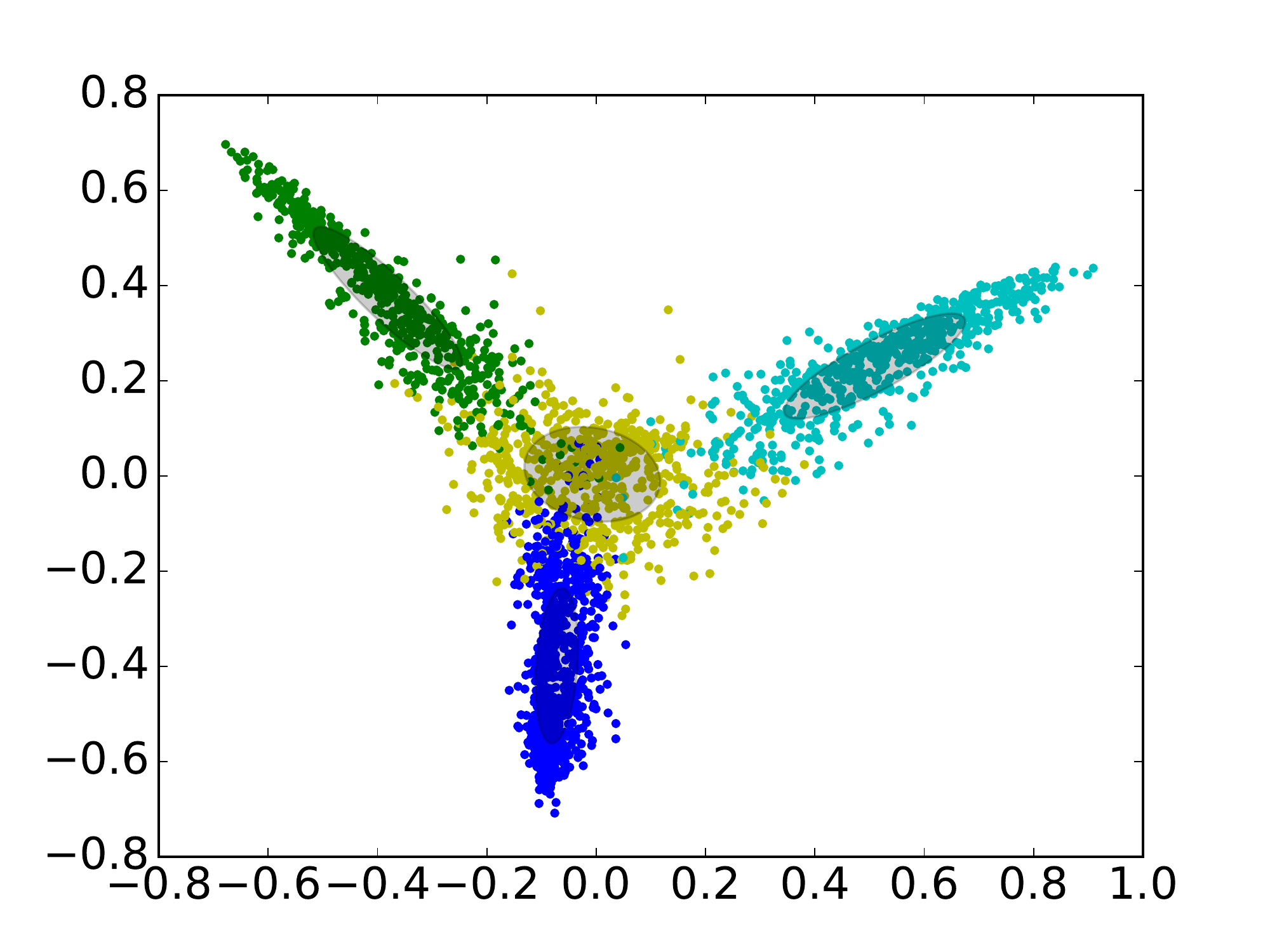} \label{fig:crp_digits14_n3000_pcrp1}}
	
	\caption{Results of clustering 3000 randomly sampled digits from 1 to 4 in spectral space. Observations in the same color represent the same digit. CRP-Oracle and CRP seem to over-fit the noise (the red cluster).}
	\label{fig:pcrp_digits14}
\end{figure}

\begin{table}[!h]
	\centering
	\begin{tabular}{|ccccc|}
		\hline
		\multicolumn{5}{|c|}{$N=1000$}                                                                    \\ \hline
		Method               & NMI (SE $\times 10^{-4}$) & VI (SE $\times 10^{-3}$) & $K$ (SE $\times 10^{-2}$) & $K_{\max}$   \\ \hline 
		Ground truth         & 1.0                & 0                 & 4                  & -            \\
		CRP-Oracle           & 0.651 (3.3)        & 1.382 (1.4)       & 4.37 (1.3)         & 7            \\
		CRP                  & 0.651 (3.3)        & 1.386 (1.4)       & 4.58 (1.6)         & 8            \\
		pCRP                 & 0.651 (3.3)        & 1.382 (1.4)       & 4.08 (0.6)        & 6            \\ \hline \hline
		\multicolumn{5}{|c|}{$N=3000$}                                                                    \\ \hline
		Method               & NMI (SE $\times 10^{-4}$) & VI (SE $\times 10^{-3}$) & $K$ (SE $\times 10^{-2}$) & $K_{\max}$   \\ \hline
		Ground truth         & 1.0                & 0.0               & 4                  & -            \\
		CRP-Oracle           & 0.651 (2.0)        & 1.400 (1.1)       & 5.17 (1.2)         & 8            \\
		CRP                  & 0.651 (2.0)        & 1.402 (1.1)       & 5.44 (1.6)         & 9            \\
		pCRP                 & 0.652 (1.9)        & 1.389 (1.1)       & 4.57 (1.2)         & 7            \\ \hline
	\end{tabular}
	\caption{Comparison of CRP-Oracle, CRP and pCRP on a 4 digits subset of MNIST. $K$ is the average number of found clusters. $K_{\max}$ is the maximum number of clusters during sampling.}
	\label{table:pcrp_digits_result}
\end{table}

\subsubsection{Digits 1-4}
In this experiment, we cluster 1000 and 3000 digits of the classes 1 to 4 in MNIST data set \citep{lecun2010mnist}, where the four clusters are approximate equally distributed. From cross validation on a different set of 1000 samples, we obtain the power value $r=1.05$. The concentration parameter $\alpha$ in CRP-Oracle is calculated as 0.58 ($N = 1000$) and 0.5 $(N = 3000)$. 

Figure \ref{fig:pcrp_digits14} shows the clustering result of all the three methods for $N=3000$. Both CRP and CRP-Oracle seem to over-fit the data by introducing a small cluster (in red), while pCRP gives a cleaner clustering result with four clusters. This comparison is further confirmed by Table~\ref{table:pcrp_digits_result}, where the average posterior cluster number in CRP apparently increases when $N$ grows to 3000. In contrast, pCRP is closer to the true situation by reducing the over-clustering effect, even compared to CRP-Oracle; see the columns of $K$ and $K_{\max}$. All methods lead to similar NMI but pCRP gives lower VI.

\begin{figure}
	\center
	\subfigure[The distribution of Old Faithful Geyser after standardization]{\includegraphics[width=0.4\textwidth]{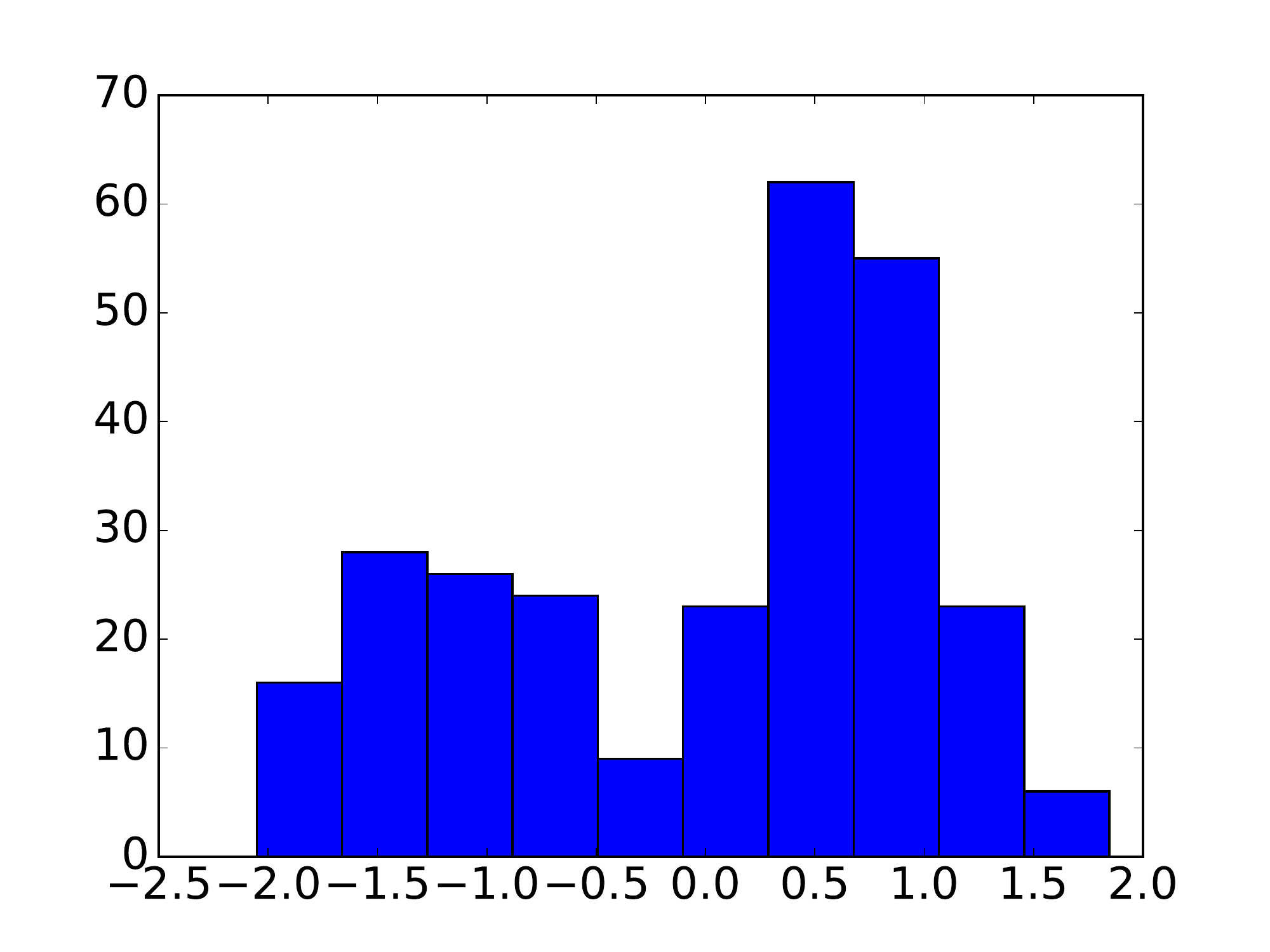} \label{fig:adhoc_oldfaithful_data_distribution}}
	\subfigure[Clustering results using CRP-Oracle, CRP, pCRP and manual clustering.]{\includegraphics[width=0.4\textwidth]{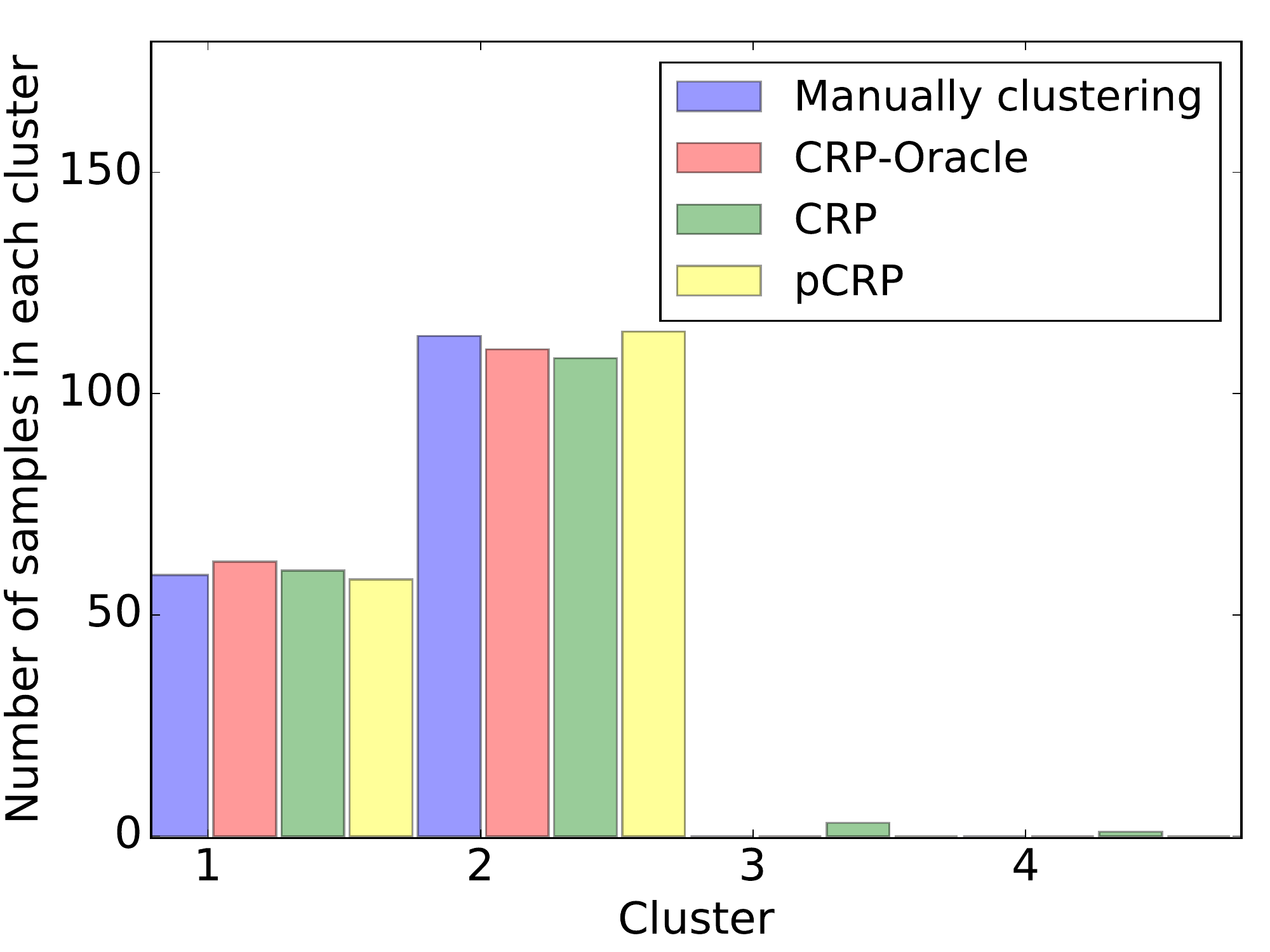} \label{fig:adhoc_oldfaithful_bar_comparison}}
	\caption{Clustering result for Old Faithful Geyser}
	\label{fig:adhoc_oldfaithful}
\end{figure}

\subsubsection{Old Faithful Geyser}
The Old Faithful Geyser data ($N=272$) are widely used to illustrate the performance of clustering algorithms. 
We use a test sample of 100 in CV leading to the power value $r=1.11$. We compare all methods on the other 172 data points. A manual clustering that consists of two Gaussian components is viewed as the ground truth. The concentration parameter is 0.39 in CRP-Oracle.  Figure~\ref{fig:adhoc_oldfaithful_bar_comparison} shows the size of each component obtained from all methods and the manual clustering. We can see that there are two mixture components in CRP-Oracle and pCRP, and four mixture components in the CRP method. In this case where the sample size is relatively small, we again see that pCRP successfully suppresses small components and generate results closer to the ground truth than CRP.


\subsection{Some issues in finite and infinite Gaussian mixture model}
\subsubsection{Non-identifiability due to overfitting finite mixture model or infinite mixture model}
In finite Gaussian mixture model, when the number of components is unknown, the analyst can intentionally or unintentionally propose an overfitting model, i.e., one with more components than the true component the data come from. And in infinite Gaussian mixture model, the model does not assume the upper bound of the number of components. Thus, the problem of non-identifiability in estimation of overfitted mixture model and infinite mixture model is well known. For example, \citep{fruhwirth2006finite} observed that identifiability will be violated as either one of the component weights is 0 or two of the component parameters are equal. 

More precisely, as in Section \ref{sec:fmm_background}, assume we have $N$ observations $\mathcalX = \{\bx_1, \bx_2, \ldots, \bx_N\}$ sampled i.i.d., from a finite mixture distribution with density 
\begin{equation}
	f_0(\bx | \bpi, \bgamma) = \sum_{k=1}^{K_0} \pi_k^0 \Phi(\bx | \bgamma_{k}^0) ,
\end{equation}
with $\bgamma_k \in \Gamma$, $\Gamma$ is the metric space of the parameter for some kernel function, $\Phi$ is the kernel function of each component and $K_0$ is the true component number both in finite and finite Gaussian mixture model. We wish to make Bayesian inference for the model parameters $\btheta = \{\bpi^0, \bgamma^0\}$. 
In such cases the model is non-identifiable since all values of the parameter in the form 
\begin{equation}
	\btheta = \{\pi_1^0, \ldots, \pi_{K_0}^0, 0; \bgamma_1^0, \ldots, \bgamma_{K_0}^0, \bgamma\} ,
\end{equation}
for all $\bgamma \in \Gamma$, and all values of the parameter in the form $\btheta = \{\pi_1^0, \ldots, \pi_j, \ldots,  \pi_{K_0}^0, \pi_{K_0 + 1};$ $\bgamma_1^0, \ldots, \bgamma_{K_0}^0, \bgamma_j^0\} $ with $\pi_j + \pi_{K_0 + 1} = \pi_j^0$ satisfy $f_0 = f_{\btheta}$. As stated in \citep{rousseau2011asymptotic}, this non-identifiability is much more tough to deal with than the non-identifiability corresponding to permutation of the labels in the mixture representation. Interesting readers can refer to \citep{rousseau2011asymptotic} for more details and the references there in.

\subsubsection{Read on}
If you are interested in a more through understanding of the hyperprior in mixture model, \citep{rasmussen1999infinite} gives some ideas how we can put hyperprior on NIW such that release the burden to select hype-parameter $\bbeta$ for NIW.

\subsection{Pruning methods for Dirichlet process mixture model}\label{chapter:pruning_dpmm}
Further to power Chinese restaurant process, we introduce some pruning methods for Dirichlet proces mixture model.

As a recap, some problem of DP mixture models have been brought about when apply them to practical problems. For example, they always produce more components than that the real data should have. The small mixture components are mainly caused by noise. Some approaches have been proposed for solving this problem. In \citep{mccullagh2008many}, an upper bound of the number of components is fixed in advance to limit the number in modeling. In \citep{escobar1995bayesian}, components with little data points are simply discarded, and these data points are reassigned to other existing components. However, this two methods based on simple upper bound or thresholds can not be directly used for real world data, because when you choose a larger bound, DP mixture models can still results in small clusters, and choosing the best thresholds is usually difficult. In this thread, we focus on how to shrink small clusters during sampling.

\subsubsection{Constrained sampling (cSampling)}
During Gibbs sampling, remove small number of clusters every $s$ (e.g., 20) iterations, reassign these data to big clusters by the probability of each cluster. Further extension on cSampling is that when we reassign these small clusters, we can use the assignment method similar to K-means. However, the basic idea is the same.

Important points: 

\begin{itemize}
\item Different to the method in \citep{escobar1995bayesian}, our cSampling method does not need to choose best threshold to discard unuseful clusters. A small threshold is good. It aims to remove very small clusters that can be easily recognized as ``noise".
\item When we remove small cluster during sampling, it will have influence on later sampling iterations. That is where the name ``constrained" come from.
\end{itemize}

\subsubsection{loss based sampling (lSampling)}
In \citep{kulis2011revisiting}, the authors introduce an algorithm called DP-means. We here briefly review DP-means. The authors considered asymptotic behavior of DP mixture models, obtaining a hard clustering algorithm that behaves similarly to K-means with the exception that a new cluster is formed whenever a point is farther than $distance$ away from existing cluster centroid. The $distance$ is very hard to decide, the authors used cross-validation to decide it.  However, this violates the setting of unsupervised learning. 

Inspired by DP-means, when we doing sampling during DP mixture models, we can shrink out small cluster by some metric, for example: marginal of data and component assignment $p(\mathcal{X}, \bz | \alpha, \bbeta)$, where $\alpha$ is the concentration parameter on Dirichlet Process, $\bbeta$ is the prior parameter on kernel (see Section~\ref{section:marginal-data-component}). And also we can use inertia, or so called the within-cluster sum of squares criterion (see Section~\ref{section:inertia}). Again, during Gibbs sampling, remove small number of clusters every $s$ (e.g., 20) iterations by applying to the chosen metric, i.e., if removing the small clusters will get smaller loss, we remove, otherwise, we keep them as they are. In our proposal, we recommend to use the following loss function
\begin{equation}
 \sum_{k=1}^K  \sqrt{ \sum_{j: j \in C_k}^{N_k} ||\bx_j - \overline{\bx}_k||^2 },
\end{equation}
where $C_k$ is the data samples in the $k^{th}$ cluster. The reason we use a square root over each cluster is that it can overcome identifiablity issue. One can imagine that if two clusters have same center value, the square root operation will force the two clusters into one cluster. This idea on the loss function comes from \citep{petralia2012repulsive}, in which case they put a ``repulsive" prior on the mixture components, thus overcomes identifiability issue in some sense.

\subsubsection{Posterior inference}
Although the proposed cSampling and lSampling is generic, we focus on its application in Gaussian mixture models for concreteness. Again, we develop a collapsed Gibbs sampling algorithm~\citep{neal2000markov} for posterior computation.

Let $\mathcalX$ be the observations, assumed to follow a mixture of multivariate Gaussian distributions. We use a conjugate normal-inverse-Wishart (NIW) prior $p(\bmu, \bSigma | \bbeta)$ for the mean vector $\bmu$ and covariance matrix $\bSigma$ in each multivariate Gaussian component, where $\bbeta$ consists of all the hyperparameters in NIW. A key quantity in a collapsed Gibbs sampler is the probability of each customer $i$ sitting with table $k$: $p(z_i = k | \bznoi, \mathcalX, \alpha, \bbeta)$, where $\bznoi$ are the seating assignments of all the other customers and $\alpha$ is the concentration parameter in CRP. This probability is calculated as follows:
\begin{equation} 
\begin{aligned}
p(z_i = k| \bznoi, \mathcal{X} , \alpha, \bbeta)  & \varpropto p(z_i = k | \bznoi, \alpha, \cancel{\bbeta})  p(\mathcal{X} |z_i = k, \bznoi, \cancel{\alpha}, \bbeta) \\
& = p(z_i = k| \bznoi, \alpha) p(\bxi |\mathcal{X}_{-i}, z_i = k, \bznoi, \bbeta) p(\mathcal{X}_{-i} |\cancel{z_i = k}, \bznoi, \bbeta)\\
& \varpropto p(z_i = k| \bznoi, \alpha) p(\bxi|\mathcal{X}_{-i}, z_i = k, \bznoi, \bbeta) \\
& \varpropto p(z_i = k| \bznoi, \alpha) p(\bxi | \xknoi, \bbeta), 										
\end{aligned}
\label{equation:pcrp_ifmm_collabsed_gibbs}
\end{equation}  
where $\xknoi$ are the observations in table $k$ excluding the $i^{th}$ observation. Algorithm~\ref{algo:pruning_csampling} and \ref{algo:pruning_lsampling} give the pseudo code of the collapsed Gibbs sampler to implement cSampling and lSampling in Gaussian mixture models.

\IncMargin{1em}
\begin{algorithm}
\SetKwData{Left}{left}\SetKwData{This}{this}\SetKwData{Up}{up}
\SetKwFunction{Union}{Union}\SetKwFunction{FindCompress}{FindCompress}
\SetKwInOut{Input}{input}\SetKwInOut{Output}{output}
\Input{Choose an initial $\bz$, set constrained step $s$, threshold=$thres$;}
\BlankLine
\For{$t \leftarrow 1$ \KwTo $T$ iterations}{
\For{$i \leftarrow 1$ \KwTo $N$}{
	Remove $\bxi$'s statistics from component $z_i$ \;
	\For{$k\leftarrow 1$ \KwTo $K$}{
		Calculate $p(z_i=k| \bznoi, \alpha) = \frac{\nknoi}{N + \alpha -1}$\;
		Calculate $p(\bxi | \xknoi, \bbeta)$\;
		Calculate $p(z_i = k | \bznoi, \mathcal{X}, \alpha, \bbeta) \propto p(z_i=k| \bznoi, \alpha) p(\bxi | \xknoi, \bbeta)$\;
	}
	Calculate $p(z_i = k^\star | \bznoi, \alpha)=\frac{\alpha}{N + \alpha -1}$\;
	Calculate $p(\bxi | \bbeta)$\;
	Calculate $p(z_i = k^\star | \bznoi, \mathcalX, \alpha, \bbeta) \propto p(z_i = k^\star | \bznoi, \alpha) p(\bxi | \bbeta)$\;
	Sample $k_{new}$ from $p(z_i | \bznoi, \mathcalX, \alpha, \bbeta)$ after normalizing\;
	Add $\bxi$'s statistics to the component $z_i=k_{new}$ \;
	If any component is empty, remove it and decrease $K$.
}
\If(\tcp*[f]{constrain step}){t == s}{\label{ut}
	Get cluster index $k$ where cluster number $n_k> thres$, put these indexes into set $\mathcal{C}$\;

	\For{$i \leftarrow 1$ \KwTo $N$}{
		Remove $\bxi$'s statistics from component $z_i$ \;
		\For{$k$ in $\mathcal{C}$}{
			Calculate $p(z_i=k| \bznoi, \alpha) = \frac{\nknoi}{N + \alpha -1}$\;
			Calculate $p(\bxi | \xknoi, \bbeta)$\;
			Calculate $p(z_i = k | \bznoi, \mathcal{X}, \alpha, \bbeta) \propto p(z_i=k| \bznoi, \alpha) p(\bxi | \xknoi, \bbeta)$\;
		}
		Sample $k_{new}$ from $p(z_i | \bznoi, \mathcalX, \alpha, \bbeta)$ after normalizing\;
		Add $\bxi$'s statistics to the component $z_i=k_{new}$ \;
		If any component is empty, remove it and decrease $K$.
	}
}
}
\caption{Collapsed Gibbs sampler for constrained sampling: reassign these data to big clusters by the probability of each cluster.}\label{algo:pruning_csampling}
\end{algorithm}\DecMargin{1em}

\IncMargin{1em}
\begin{algorithm}
\SetKwData{Left}{left}\SetKwData{This}{this}\SetKwData{Up}{up}
\SetKwFunction{Union}{Union}\SetKwFunction{FindCompress}{FindCompress}
\SetKwInOut{Input}{input}\SetKwInOut{Output}{output}
\Input{Choose an initial $\bz$, set loss-based step $s$;}
\BlankLine
\For{$t \leftarrow 1$ \KwTo $T$ iterations}{
\For{$i \leftarrow 1$ \KwTo $N$}{
	Remove $\bxi$'s statistics from component $z_i$ \;
	\For{$k\leftarrow 1$ \KwTo $K$}{
		Calculate $p(z_i=k| \bznoi, \alpha) = \frac{\nknoi}{N + \alpha -1}$\;
		Calculate $p(\bxi | \xknoi, \bbeta)$\;
		Calculate $p(z_i = k | \bznoi, \mathcal{X}, \alpha, \bbeta) \propto p(z_i=k| \bznoi, \alpha) p(\bxi | \xknoi, \bbeta)$\;
	}
	Calculate $p(z_i = k^\star | \bznoi, \alpha)=\frac{\alpha}{N + \alpha -1}$\;
	Calculate $p(\bxi | \bbeta)$\;
	Calculate $p(z_i = k^\star | \bznoi, \mathcalX, \alpha, \bbeta) \propto p(z_i = k^\star | \bznoi, \alpha) p(\bxi | \bbeta)$\;
	Sample $k_{new}$ from $p(z_i | \bznoi, \mathcalX, \alpha, \bbeta)$ after normalizing\;
	Add $\bxi$'s statistics to the component $z_i=k_{new}$ \;
	If any component is empty, remove it and decrease $K$.
}
\If(\tcp*[f]{loss-based  step}){t == s}{\label{ut}
	\While{$K \geq 2$}{
	Get cluster index $k$ where cluster number $n_k$ is minimal, put these indexes which are not equal to $k$ into set $\mathcal{C}$\;
		
	\For{$i \leftarrow 1$ \KwTo $N$}{
		Remove $\bxi$'s statistics from component $z_i$ \;
		\For{$k$ in $\mathcal{C}$}{
			Calculate $p(z_i=k| \bznoi, \alpha) = \frac{\nknoi}{N + \alpha -1}$\;
			Calculate $p(\bxi | \xknoi, \bbeta)$\;
			Calculate $p(z_i = k | \bznoi, \mathcal{X}, \alpha, \bbeta) \propto p(z_i=k| \bznoi, \alpha) p(\bxi | \xknoi, \bbeta)$\;
		}
		Sample $k_{new}$ from $p(z_i | \bznoi, \mathcalX, \alpha, \bbeta)$ after normalizing\;
		Add $\bxi$'s statistics to the component $z_i=k_{new}$ \;
		If any component is empty, remove it and decrease $K$.
	}
	Calculate current loss $l$ \;
 	}
 	Roll back to the status with minimal loss\;
}
}
\caption{Collapsed Gibbs sampler for loss-based sampling: remove small clusters if it results in smaller loss.}\label{algo:pruning_lsampling}
\end{algorithm}\DecMargin{1em}

\subsubsection{Examples}
We conduct experiments to demonstrate the main advantages of the proposed pruning sampling methods using both synthetic and real data. 

In all experiments, we run the Gibbs sampler 20,000 iterations with a burn-in of 10,000. The sampler is thinned by keeping every 5$^{th}$ draw. We use the same concentration parameter $\alpha = 1$ for both CRP and pruning sampling methods in all scenarios. In addition, we equip CRP with an unfair advantage to match the magnitude of its prior mean $\alpha \log(N)$ to the true number of clusters, termed {\it CRP-Oracle}. In order to measure overall clustering performance, we use normalized mutual information (NMI) \citep{mcdaid2011normalized} and variation of information (VI) \citep{meilua2003comparing}, which measures the similarity between the true and estimated cluster assignments. Higher NMI and lower VI indicate better performance.  If applicable, metrics using the true clustering are calculated to provide an upper bound for all methods, coded as `Ground Truth'. The metrics are discussed in Section~\ref{section:some-metrics}. Feel free to skip this section for a first reading.

\subsubsection{Simulation experiments}\label{sec:cdpmm_simulation}
The parameters of the simulations are as follows:

\begin{itemize}
\item Sim 1: $K_0=3$, $\bm{\pi}$=\{0.35, 0.4, 0.25\}, $\bm{\mu}$=\{0, 2, 5\} and $\bm{\Sigma}$=\{0.5, 0.5, 1\};
\item Sim 2: $K_0=2$, $\bm{\pi}$=\{0.65, 0.35\}, $\bm{\mu}$=\{1, 1\} and $\bm{\Sigma}$=\{10, 1\};
\end{itemize}

In practice, we think that a trivial threshold for cSampling is $4\%$, which means that we consider cluster with samples smaller than $4\%$ of total samples can be regarded as noise. 
For lSampling, we shrink by the proposed loss function every 20 steps.
Table~\ref{table:pruning_posterior_summary} shows the posterior summary for these tests. We notice that cSampling gives better results than CRP-Oracle but a little bit worse than pCRP. lSampling gives worse results than CRP-Oracle.

\begin{table}[]
\centering
\begin{tabular}{|cccccc|}
\hline
\multicolumn{6}{|c|}{$N=300$}                                                                                \\ \hline
Method               & NMI (SE $10^{-3}$) & VI (SE $10^{-3}$) & $K$ (SE $10^{-2}$) & $K_{\max}$ & $K_{mode}$ \\ \hline
Ground truth (Sim 1) & 1.0                & 0.0               & 3                  & -          & -          \\
CRP-Oracle (Sim 1)   & 0.800 (1.1)        & 0.669 (4.5)       & 4.2 (2.3)          & 8          & 4          \\
CRP (Sim 1)          & 0.773 (1.2)        & 0.795 (5.4)       & 5.3 (3.3)          & 12         & 5          \\
pCRP (Sim 1)         & 0.827 (0.74)       & 0.580 (4.4)       & 3.6 (1.7)          & 7          & 3          \\
cSampling (Sim 1)    & 0.829 (0.75)       & 0.695 (6.9)       & 3.3 (1.3)          & 7          & 3          \\
lSampling (Sim 1)    & 0.791 (1.6)        & 0.682 (5.2)       & 4.3 (3.0)          & 11         & 3          \\ \hline
Ground truth (Sim 2) & 1.0                & 0.0               & 2                  & -          & -          \\
CRP-Oracle (Sim 2)   & 0.211 (1.0)        & 1.803 (6.4)       & 3.5 (2.7)          & 8          & 3          \\
CRP (Sim 2)          & 0.189 (0.9)        & 2.164 (7.8)       & 6.2 (4.2)          & 13         & 6          \\
pCRP (Sim 2)         & 0.228 (1.0)        & 1.518 (2.1)       & 2.4 (1.3)          & 6          & 2          \\
cSampling (Sim 2)    & 0.231 (1.0)        & 1.526 (2.2)       & 2.5 (2.0)          & 7          & 2          \\
lSampling (Sim 2)    & 0.218 (1.1)        & 1.707 (5.7)       & 4.1 (4.3)          & 13         & 2          \\ \hline \hline
\multicolumn{6}{|c|}{$N=2000$}                                                                               \\ \hline
Method               & NMI (SE $10^{-4}$) & VI (SE $10^{-3}$) & $K$ (SE $10^{-2}$) & $K_{\max}$ & $K_{mode}$ \\ \hline
Ground truth (Sim 1) & 1.0                & 0.0               & 3                  & -          & -          \\
CRP-Oracle (Sim 1)   & 0.812 (5.3)        & 0.610 (2.6)       & 4.0 (2.3)          & 10         & 4          \\
CRP (Sim 1)          & 0.782 (8.5)        & 0.732 (4.0)       & 5.8 (3.6)          & 12         & 5          \\
pCRP (Sim 1)         & 0.823 (6.6)        & 0.869 (7.3)       & 3.5 (1.6)          & 7          & 3          \\
cSampling (Sim 1)    & 0.825 (2.6)        & 0.552 (3.1)       & 3.3 (1.4)          & 7          & 3          \\
lSampling (Sim 1)    & 0.815 (5.5)        & 0.580 (1.7)       & 4.4 (3.0)          & 10         & 3          \\ \hline
Ground truth (Sim 2) & 1.0                & 0.0               & 2                  & -          & -          \\
CRP-Oracle (Sim 2)   & 0.258 (7.0)        & 1.537 (5.3)       & 3.5 (2.6)          & 8          & 3          \\
CRP (Sim 2)          & 0.238 (7.1)        & 1.755 (6.8)       & 6.8 (4.7)          & 15         & 6          \\
pCRP (Sim 2)         & 0.258 (4.2)        & 1.368 (0.74)      & 2.2 (1.0)          & 5          & 2          \\
cSampling (Sim 2)    & 0.286 (4.0)        & 1.351 (0.81)      & 2.5 (2.0)          & 7          & 2          \\
lSampling (Sim 2)    & 0.278 (4.4)        & 1.396 (1.6)       & 4.1 (4.2)          & 11         & 2          \\ \hline
\end{tabular}
\caption{Posterior summary for pruning methods. $K_{\max}$ is the maximum number of clusters during sampling. $K_{mode}$ is the most probability of cluster number during sampling.}
\label{table:pruning_posterior_summary}
\end{table}

\subsubsection{Digits 1-4}
In this experiment, we cluster 1000 digits of the classes 1 to 4 in MNIST data set \citep{lecun2010mnist}, where the four clusters are approximate equally distributed. The concentration parameter $\alpha$ in CRP-Oracle is calculated as 0.5 ($N = 3000$).

Both CRP and CRP-Oracle seem to over-fit the data by introducing a small cluster, while cSampling gives a cleaner clustering result with four clusters. This comparison is further confirmed by Table~\ref{table:constrain_dp_digits_result}. In contrast, cSampling and lSampling are closer to the true situation by reducing the over-clustering effect, even compared to CRP-Oracle; see the columns of $K$ and $K_{\max}$. All methods lead to similar NMI. We also observe similar result as in simulation test that lSampling gives worse results than cSampling.


\begin{figure}[h!]
\center
\subfigure[True clustering when $N$=3000]{\includegraphics[width=0.31\textwidth]{img_visual/digits_posterior_clustering/true_clustering.pdf} \label{fig:cdpmm_digits14_n3000_true_clustering}}
~
\subfigure[cSampling when $N$=3000]{\includegraphics[width=0.31\textwidth]{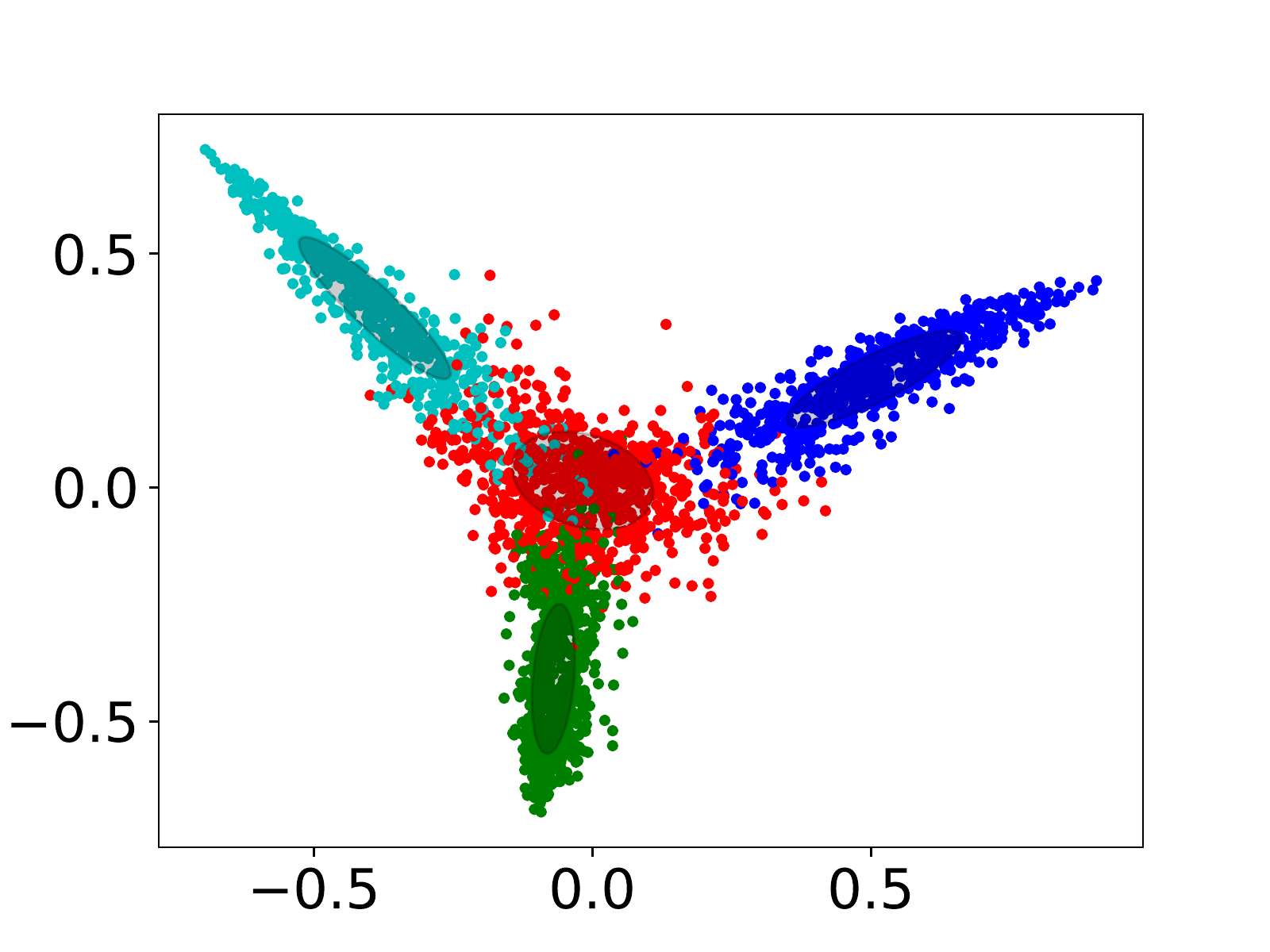} \label{fig:cdpmm_digits14_n3000_csampling}}
~
\subfigure[lSampling $N$=3000]{\includegraphics[width=0.31\textwidth]{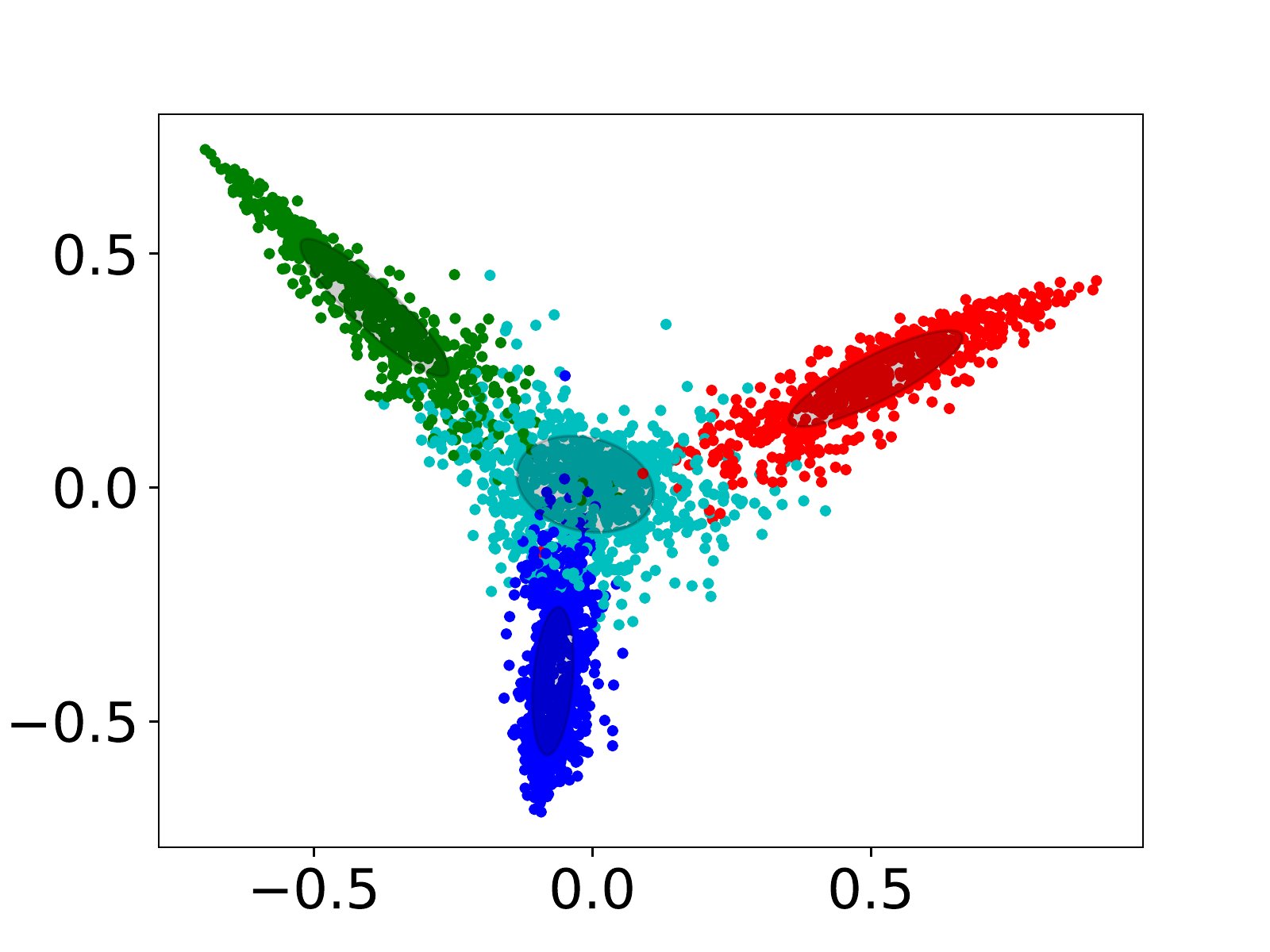} \label{fig:cdpmm_digits14_n3000_lsampling}}
\caption{Results of clustering 3000 randomly sampled digits from 1 to 4 in spectral space. Observations in the same color represent the same digit.}
\label{fig:cdpmm_digits14}
\end{figure}

\begin{table}[!h]
\centering
\begin{tabular}{|ccccc|}
\hline
\multicolumn{5}{|c|}{$N=3000$}                                                                    \\ \hline
Method               & NMI (SE $\times 10^{-4}$) & VI (SE $\times 10^{-3}$) & $K$ (SE $\times 10^{-2}$) & $K_{\max}$   \\ \hline
Ground truth         & 1.0                & 0.0               & 4                  & -            \\
CRP-Oracle           & 0.651 (2.0)        & 1.400 (1.1)       & 5.17 (1.2)         & 8            \\
CRP                  & 0.651 (2.0)        & 1.402 (1.1)       & 5.44 (1.6)         & 9            \\
pCRP                 & 0.652 (1.9)        & 1.389 (1.1)       & 4.57 (1.2)         & 7            \\ \hline
cSampling           & 0.659 (1.7)        & 1.353 (0.7)       & 4.25 (1.2)        & 7            \\ 
lSampling           & 0.653 (3.6)        & 1.379 (1.2)       & 4.81 (1.9)        & 9            \\ \hline
\end{tabular}
\caption{Comparison of CRP-Oracle, CRP and pCRP on a 4 digits subset of MNIST. $K$ is the average number of found clusters. $K_{\max}$ is the maximum number of clusters during sampling.}
\label{table:constrain_dp_digits_result}
\end{table}

\newpage
\section{Some metrics}\label{section:some-metrics}
In order to evaluate the Gibbs sampling procedure and to ensure that mixing is taking place, it is useful to have some metrics to calculate over the sampling iterations. We consider two kinds of metrics, one is label-related, in which case we use the true label of clustering to evaluate the process; the other one is non-label-related, in which case we do not use the true label to evaluate the process.

\subsection{Marginal of data and component assignment}\label{section:marginal-data-component}

\subsubsection{In Bayesian finite Gaussian mixture model}
Marginal of the data and component assignments $p(\mathcal{X}, \bz | \balpha, \bbeta)$ is useful for evaluating the Gibbs sampling process since it captures both changes in the likelihood of the data under the current assignments through $p(\mathcal{X} | \bz, \bbeta)$, as well as the probability of the current component assignments $p(\bz | \balpha)$. This marginal of data and component assignments can be calculated as follows
\begin{equation}
\begin{aligned}
p(\mathcal{X}, \bz | \balpha, \bbeta)  &= p(\mathcal{X} | \bz, \bbeta) p(\bz | \balpha) \\
						&= \left(\prod_{k=1}^K p(\mathcal{X}_k | \bbeta)\right) p(\bz| \balpha) ,
\end{aligned}
\label{equation:fmm_marginal_data_and_component}
\end{equation}
where $\mathcal{X}_k$ is the set of data observations assigned to component/cluster $k$. The terms in the product in Equation~\eqref{equation:fmm_marginal_data_and_component} can each be calculated using Equation~\eqref{equation:niw_marginal_data} and Equation~\eqref{equation:fmm_z_alpha}.

%

\subsubsection{In Bayesian infinite Gaussian mixture model}
Similar to the finite case, the marginal  of the data and component assignments $p(\mathcal{X}, \bz | \alpha, \bbeta)$ can be used as evaluation of the sampling process
\begin{equation}
\begin{aligned}
p(\mathcal{X}, \bz | \alpha, \bbeta)  &= p(\mathcal{X} | \bz, \bbeta) p(\bz | \alpha) \\
						&= \left(\prod_{k=1}^K p(\mathcal{X}_k | \bbeta)\right) p(\bz| \alpha) ,
\end{aligned}
\label{equation:ifmm_marginal_data_and_component}
\end{equation}
The terms in the product in Equation~\eqref{equation:ifmm_marginal_data_and_component} can each be calculated using Equation~\eqref{equation:niw_marginal_data} and Equation~\eqref{equation:ifmm_z_alpha}. The only difference between Equation~\eqref{equation:ifmm_marginal_data_and_component} and Equation~\eqref{equation:fmm_marginal_data_and_component} is that in Equation~\eqref{equation:fmm_marginal_data_and_component} the marginal probability of assignments depends on a vector $\balpha$, while in Equation~\eqref{equation:ifmm_marginal_data_and_component}, it depends on a scalar $\alpha$. Also in infinite case, the cluster number $K$ can increase or decrease at each iteration.

\subsection{Mixture likelihood}
The marginal likelihood can be used as a metric to evaluate the sampling iterations since it captures both the likelihood of data under current assignment through $p(\mathcalX | \bz, \bbeta)$ , and the probability of the current component assignment $p(\bz|\balpha)$. In this sense, an alternative metric can be utilized from Equation~\eqref{equation:mixture-assign-mle} after we sample out the distribution parameters $\bmu_k$'s and $\bSigma_k$'s for multivariate Gaussian distributions:
\begin{equation*}
	\begin{aligned}
		p(\bpi, \bgamma | \mathcalX, \bz) &= \prod_{i=1}^N \pi_{z_i} \normal(\bx_i | \bgamma_{z_i})
		&= \prod_{k=1}^K \pi_k^{N_k}\left[  \prod_{i:i\in C_k} \normal(\bx_i | \bgamma_k)   \right], 
	\end{aligned}
\end{equation*}
where $\bgamma_k=\{\bmu_k, \bSigma_k\}$, and $C_k$ is the data samples in the $k^{th}$ cluster. By evaluating with this metric, the mixture model tends to select maximum likelihood estimates.

\subsection{Inertia}\label{section:inertia}
Inertia \footnote{The name is following from Scikit-learn document, see also \url{http://scikit-learn.org/stable/modules/clustering.html}}, or within-cluster sum-of-squares is mostly used in K-means, in which it aims to choose centroids that minimize the inertia function
\begin{equation}
\sum_{k=1}^K  \sum_{j: j \in C_k}^{N_k} || \bx_j - \overline{\bx}_k||^2.
\end{equation}
where $C_k$ is the data samples in the $k^{th}$ cluster and $ \overline{\bx}_k$ is the mean vector for cluster $k$. This is exactly the same metric used in K-means. Inertia is not a normalized metric, so we just know that lower values are better and zero is optimal. But in very high-dimensional spaces, Euclidean distances tend to become inflated (this is an instance of the so-called ``curse of dimensionality”).

\subsection{Squared inertia}\label{section:squared-inertia}
We propose the following loss function
\begin{equation}
\sum_{k=1}^K \sqrt{\sum_{j: j \in C_k}^{N_k} || \bx_j - \overline{\bx}_k||^2 }.
\end{equation}
where $C_k$ is the data samples in the $k^{th}$ cluster and $ \overline{\bx}_k$ is the mean vector for cluster $k$. The square root has an important impact in favoring a smaller nunber of clusters; for example, inducing a price to be paid for introducing two clusters with the same mean. One can imagine that if two clusters have same center value, the square root operation will force the two clusters into one cluster. For example, if we have the loss value for two clusters 100 and 30 respectively. If we use this square root operation, we will get $\sqrt{100} + \sqrt{30} > \sqrt{100+30}$, thus favoring small cluster number. However, if we do not use square root operation, we will get $(100) + (30)=(100 + 30)$.

\subsection{Label-related metrics}
With the increasing popularity of algorithms for clustering, given a set of true cluster assignments, and the set of clusters found by an algorithm, these sets of cluster assignment can be compared to see how similar or different the sets are. We call this as label-related metrics. A normalized measure is desirable in many contexts, for example assigning a value of 0 where the two sets are totally dissimilar, and 1 where they are identical \citep{mcdaid2011normalized}. We first introduce two un-normalized measures, and a normalized measure is described, all of which come from \textit{information theory}, a field has deep links to statistics and machine learning.
A Python implementation is available online. \footnote{\url{https://github.com/junlulocky/infopy}}

\begin{figure}[h!]
	\centering
	\includegraphics[width=0.6\textwidth]{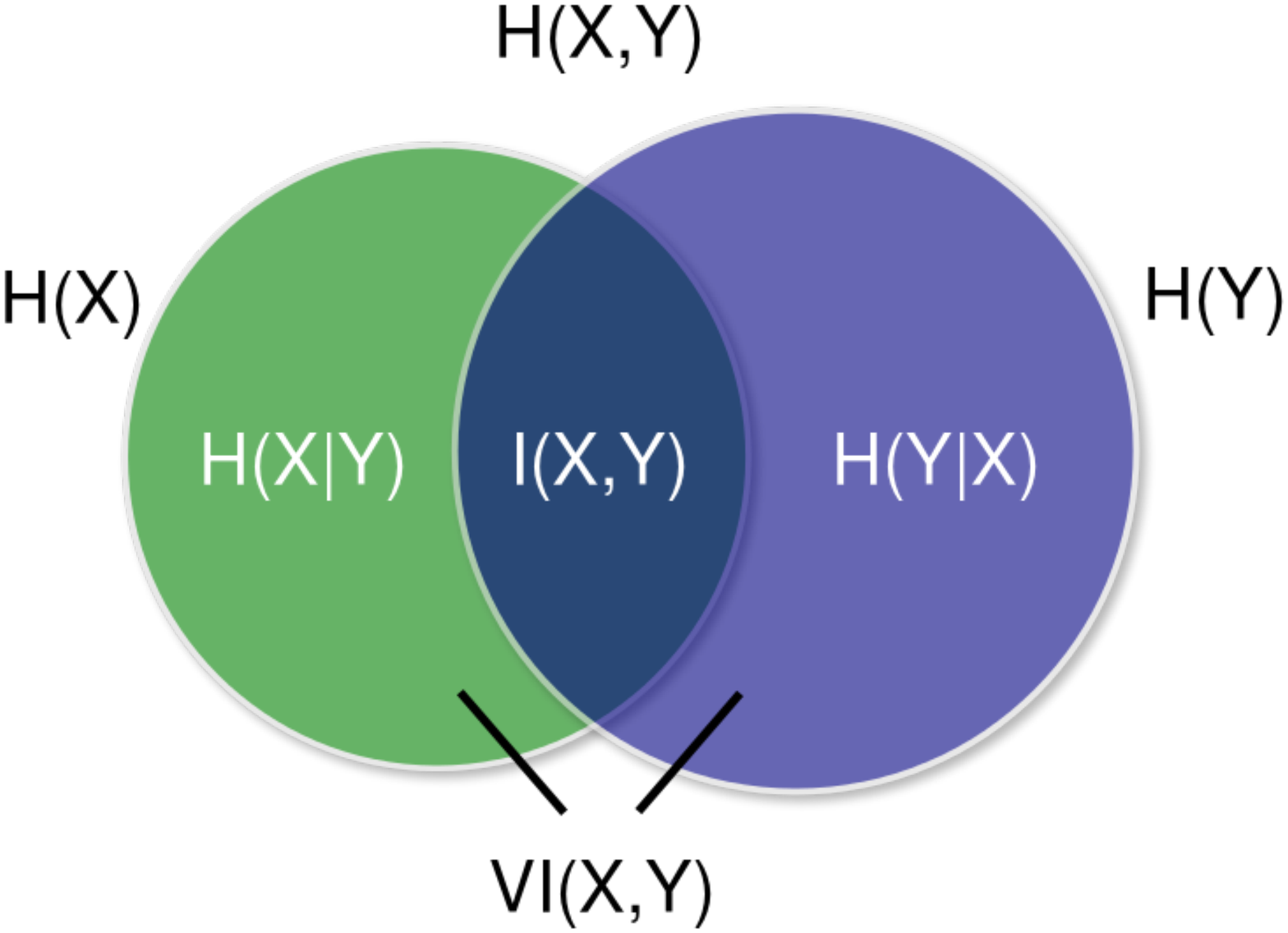}
	\caption{Venn diagram illustrating the relation between information entropies, mutual information and variation of information. The area containing both circles is the joint entropy $H(X,Y)$. The circle on the left (green and grey) is the individual entropy $H(X)$, with the green being the conditional entropy $H(X|Y)$. The circle on the right (purple and grey) is $H(Y)$, with the purple being $H(Y|X)$. The grey is the mutual information $I(X;Y)$. Figure is due to wikipedia.}
	\label{fig:venn_diagram_4_information_theory}
\end{figure}
\subsubsection{Mutual information and variation of information}

Formally, the mutual information of two discrete random variables $X$ and $Y$ can be defined as:
\begin{equation}
MI(X; Y) = \sum_{y \in Y} \sum_{x \in X} p(x,y) \log \left(\frac{p(x,y)}{p(x)p(y)} \right), 
\end{equation}
where $p(x,y)$ is the joint probability distribution function of $X$ and $Y$, and $p(x)$ and $p(y)$ are the marginal probability distribution functions of $X$ and $Y$ respectively \footnote{See also \url{https://nlp.stanford.edu/IR-book/html/htmledition/evaluation-of-clustering-1.html}}. 

Then the variation of information between the two discrete random variables is defined as 
\begin{equation}
VI(X;Y) = - \sum_{y \in Y} \sum_{x \in X}  p(x,y)\left[\log \frac{p(x,y)}{p(x)} + \log \frac{p(x,y)}{p(y)}\right]. 
\end{equation}
Noted that unlike the mutual information, however, the variation of information is a true metric, in that it obeys the triangle inequality.

The relationship between mutual information and variation of information can be shown in Figure~\ref{fig:venn_diagram_4_information_theory}, where the definition of the entropy of a random variable $X$ is $H(X)=\sum_{x\in X}p(x) \log p(x)$. From which we have the relationship between mutual information and variation of information: $VI(X;Y) = H(X) + H(Y) - 2MI(X;Y)$. 

In clustering, each clustering algorithm $C=\{C_1, C_2, \ldots, C_K\}$ defines the probability distribution $P_C$
\begin{equation}
P_C(k) = \frac{n_k}{N},
\end{equation}
where $n_k$ is the number of points in the $k^{th}$ cluster $C_k$ and $N$ is the total number of points in the data set. Different clustering algorithms can determine different number of clusters.

For any clustering distributions $P_{C_1}=(p_1, p_2, \ldots, p_n)$ and $P_{C_2}=(q_1,q_2, \ldots,q_m)$, define the probability distribution $P_{C_1}$ and $P_{C_2}$ and joint probability distribution $R(i,j)$
\begin{equation}
\begin{aligned}
P_{C_1}(i) &= \frac{n_i}{N}, \\
P_{C_2}(j) &= \frac{m_j}{N}, \\
R(i,j) &= \frac{N_{i,j}}{N} \triangleq \frac{|n_i \cap m_j|}{N}.
\end{aligned}
\end{equation}
where $|n_i \cap m_j|$ is the number of observations that is both in cluster $i$ of $C_1$ and cluster $j$ of $C_2$.
Thus, in clustering algorithms, we set $p(x,y) = R(i,j)$, $p(x) = P_{C_1}(i)$ and $p(y) = P_{C_2}(j)$. Then we arrive at the definition of mutual information (MI) and variation of information (VI) \citep{meilua2003comparing} in clustering
\begin{equation}
MI(C_1; C_2) = \sum_{i,j} \frac{N_{i,j}}{N} \log \frac{\frac{N_{i,j}}{N}}{\frac{n_i}{N} \cdot \frac{m_j}{N}} =  \sum_{i,j} \frac{N_{i,j}}{N} \log \frac{N \cdot N_{i,j}}{n_i \cdot m_j},
\end{equation}
and
\begin{equation}
VI(C_1; C_2) = - \sum_{i,j}    \frac{N_{i,j}}{N}[\log \frac{ \frac{N_{i,j}}{N}}{\frac{n_i}{N}} + \log \frac{ \frac{N_{i,j}}{N}}{\frac{m_j}{N}}] = - \sum_{i,j}    \frac{N_{i,j}}{N}[\log \frac{ N_{i,j}}{n_i} + \log \frac{ N_{i,j} }{m_j}]. 
\end{equation}

\subsubsection{Normalized mutual information}
A normalized measure is desirable in many contexts, for example assigning a value of 0 where the two sets are totally dissimilar, and 1 where they are identical.
From Figure~\ref{fig:venn_diagram_4_information_theory}, we find that the mutual information $MI(X;Y) < [H(X) + H(Y)]/2$. Thus we normalized the mutual information to get the normalized mutual information
\begin{equation}
NMI(X;Y) = \frac{MI(X;Y) }{ [H(X) + H(Y)]/2},
\end{equation}
and in clustering, we have
\begin{equation}
NMI(C_1;C_2) = \frac{MI(C_1;C_2) }{ [H(C_1) + H(C_2)]/2}.
\end{equation}
In most situations, we need to compare the clustering algorithm to a true clustering situation, in which case we just set $C_1$ to be the true clustering label. And thus, we expect the higher the mutual information (or normalized mutual information) the better; and the lower variation of information the better.



%
%
%
%
%


\newpage

\appendix
\section{Deriving the Dirichlet distribution}\label{appendix:drive-dirichlet}
\subsection{Derivation}
Let $X_1, X_2, \ldots, X_K$ be i.i.d., random variables drawn from the Gamma distribution such that $X_k \sim \gammadist(\alpha_k, 1)$ for $k\in \{1, 2, \ldots, K\}$. The joint p.d.f., of $X_1, X_2, \ldots, X_K$ is given by 
$$ f_{X_1,X_2,\ldots,X_K}(x_1, x_2, \ldots, x_K) =\left\{
\begin{aligned}
& \prod_{k=1}^{K}\frac{1}{\Gamma(\alpha_k)} x_k^{\alpha_k-1} \exp(- x_k) ,& \mathrm{\,\,if\,\,} x_k \geq 0.  \\
&0 , &\mathrm{\,\,if\,\,} \text{otherwise}.
\end{aligned}
\right.
$$
Define variables $Y_k$'s as follows
\begin{equation}\label{equation:dirichlet-define1}
	\begin{aligned}
		Y_k &= \frac{X_k}{\sum_{k=1}^{K}X_k}, \qquad \forall\,\, k \in \{1, 2, \ldots, K-1\}\\
		Y_K &= \frac{X_K}{\sum_{k=1}^{K}X_k} =1- \sum_{k=1}^{K-1}Y_k,
	\end{aligned}
\end{equation} 
and 
\begin{equation}\label{equation:dirichlet-define2}
Z_K = \sum_{k=1}^{K}X_k.
\end{equation}
Let $\bX = [X_1, X_2, \ldots, X_K]$, $\bY = [Y_1, Y_2, \ldots, Y_{K-1}, Z_K]$, $\bx = [x_1, x_2, \ldots, x_K]$, and $\by = [y_1, y_2, \ldots, y_{K-1}, z_K]$. 
By multidimensional transformation of variables, we have 
$$
f_{\bY}(\by) =  f_{\bX}(g^{-1}(\by))
\left|\det \left[J_{g^{-1}}(\by)  \right] \right|,
$$
where
$$
\begin{bmatrix}
	x_1\\
	x_2\\
	\vdots\\
	x_{K-1}\\
	x_K
\end{bmatrix}
=g^{-1}(\by) = g^{-1}\left(\begin{bmatrix}
y_1\\
y_2\\
\vdots\\
y_{K-1}\\
z_K
\end{bmatrix}\right)=
\begin{bmatrix}
	y_1\cdot z_K\\
	y_2\cdot z_K\\
	\vdots\\
	y_{K-1}\cdot z_K\\
	y_K \cdot z_K
\end{bmatrix},
$$
and the Jacobian matrix is given by 
$$
\begin{aligned}
J_{g^{-1}}(\by)&=
\begin{bmatrix}
	\frac{\partial}{\partial y_1}g_1^{-1}(\by) & \cdots & \frac{\partial}{\partial y_{K-1}}g_1^{-1}(\by)  & \frac{\partial}{\partial z_K}g_1^{-1}(\by) \\
	\vdots & \ddots & \cdots & \vdots \\
	\frac{\partial}{\partial y_1}g_K^{-1}(\by) & \cdots & \frac{\partial}{\partial y_{K-1}}g_K^{-1}(\by)  & \frac{\partial}{\partial z_K}g_K^{-1}(\by) \\
\end{bmatrix}\\
&=
\begin{bmatrix}
z_K & 0& \cdots & 0  &y_1 \\
0 & z_K& \cdots & 0  &y_2 \\
\vdots & \vdots& \ddots & \vdots & \vdots \\
0 & 0& \cdots & z_K  &y_{K-1} \\
-z_K & -z_k &  \cdots & -z_K  &(1-\sum_{k=1}^{K-1}y_k   ) 
\end{bmatrix}=z_K^{K-1}.
\end{aligned}
$$
This implies the joint p.d.f, of $\bY$ is 
$$
f_{\bY}(\by)  = f_{\bX}(g^{-1}(\by)) z_K^{K-1}
=\frac{y_1^{\alpha_1-1}  y_2^{\alpha_2-1} \ldots y_{K-1}^{\alpha_{K-1}-1}  y_K^{\alpha_K-1}}{\prod_{k=1}^{K} \Gamma(\alpha_k)} \exp(-z_K) z_K^{\alpha_1+\alpha_2+\ldots+\alpha_K-1}.
$$
We realize that the righthand size of above equation is proportional to a p.d.f. of Gamma distribution and 
$$
\int \exp(-z_K) z_K^{\alpha_1+\alpha_2+\ldots+\alpha_K-1} dz_K = \Gamma(\alpha_1 + \alpha_2+\ldots+\alpha_K).
$$
Let $\alpha_+=\alpha_1 + \alpha_2+\ldots+\alpha_K$,
this implies 
$$
f(y_1, y_2, \ldots, y_{K-1}) = \frac{\Gamma(\alpha_+)}{\prod_{k=1}^{K} \Gamma(\alpha_k)} \prod_{k=1}^{K}y_k^{\alpha_k-1}.
$$
We notice that $Y_k$'s are defined that $0<Y_k<1$ for all $k\in \{1, 2, \ldots, K\}$, and $\sum_{k=1}^{K}Y_k=1$. This implies the above equation is the p.d.f., of the Dirichlet distribution. The construction shown above can be utilized to generate random variables from the Dirichlet distribution.

\subsection{Properties of Dirichlet distribution}
Suppose $\bY=[Y_1, Y_2, \ldots, Y_K]\sim \dirichlet(\balpha)$ with $\balpha=[\alpha_1, \alpha_2, \ldots, \alpha_K]$, we here show the moments and properties of the Dirichlet distribution. 
\paragraph{Mean of Dirichlet distribution} 
Write out the expectation:
$$
\begin{aligned}
\Exp[Y_1] 
&=\int \cdots \int  
y_1 \cdot \dirichlet(\by | \balpha )
d y_1 dy_2\cdots dy_K\\
&=\int \cdots \int  
y_1\frac{\Gamma(\alpha_+)}{\prod_{k=1}^{K} \Gamma(\alpha_k)} \prod_{k=1}^{K}y_k^{\alpha_k-1}
d y_1 dy_2\cdots dy_K\\
&=\frac{\Gamma(\alpha_+)}{\prod_{k=1}^{K} \Gamma(\alpha_k)} 
\int \cdots \int  
y_1^{\alpha_1+1-1} \prod_{k=2}^{K}y_k^{\alpha_k-1}
d y_1 dy_2\cdots dy_K\\
&=\frac{\Gamma(\alpha_+)}{\prod_{k=1}^{K} \Gamma(\alpha_k)}   
\frac{\Gamma(\alpha_1+1) \prod_{k=2}^{K} \Gamma(\alpha_k)}{\Gamma(\alpha_++1)}\\
&=\frac{\Gamma(\alpha_+)}{\Gamma(\alpha_1)}   
\frac{\Gamma(\alpha_1+1) }{\Gamma(\alpha_++1)}\\
&=\frac{\alpha_1}{\alpha_+},
\end{aligned}
$$
where the last equality comes from the fact that $\Gamma(x+1)=x\Gamma(x)$.

\paragraph{Variance of Dirichlet distribution} 
Write out the variance $\Var[Y_i] = \Exp[Y_i^2] - \Exp[Y_i]^2$. Similarly from the proof of the mean, we have
$$
\Exp[Y_i^2] = \frac{\Gamma(\alpha_+)}{\Gamma(\alpha_++2)} \frac{\Gamma(\alpha_i+2)}{\Gamma(\alpha_i)} =\frac{(\alpha_i+1)\alpha_i}{(\alpha_++1)\alpha_+}.
$$
This implies 
$$
\Var[Y_i] = \Exp[Y_i^2] - \Exp[Y_i]^2 = \frac{(\alpha_i+1)\alpha_i}{(\alpha_++1)\alpha_+} - (\frac{\alpha_i}{\alpha_+})^2 = \frac{\alpha_i (\alpha_+-\alpha_i)}{\alpha_+^2(\alpha_++1)}.
$$

\paragraph{Covariance of Dirichlet distribution}
Write out the covariance $\Cov[Y_i Y_j] = \Exp[Y_i Y_j] - \Exp[Y_i]\Exp[Y_j]$.
Again, similarly from the proof of the mean, for $i\neq j$, we have 
$$
\Exp[Y_i Y_j ] = \frac{\Gamma(\alpha_+)}{\Gamma(\alpha_++2)} \frac{\Gamma(\alpha_i+1)}{\Gamma(\alpha_i)}
\frac{\Gamma(\alpha_j+1)}{\Gamma(\alpha_j)} = \frac{\alpha_i \alpha_j}{\alpha_+ (\alpha_++1)}.
$$
This implies
$$
\Cov[Y_i Y_j ]=\Exp[Y_i Y_j] - \Exp[Y_i]\Exp[Y_j]=\frac{\alpha_i \alpha_j}{\alpha_+ (\alpha_++1)} - \frac{\alpha_i \alpha_j}{\alpha_+^2} = \frac{-\alpha_i \alpha_j}{\alpha_+^2 (\alpha_++1) }.
$$

\paragraph{Marginal distribution of $Y_i$}
By definition in Equation~\eqref{equation:dirichlet-define1} and Equation~\eqref{equation:dirichlet-define2}, we have
$Z_K - X_i \sim \gammadist(\alpha_+ - \alpha_i, 1)$.
This implies 
$$
Y_i = \frac{X_i}{Z_K} =\frac{X_i}{X_i +(Z_K-X_i)} \sim \betadist(\alpha_i, \alpha_+-\alpha_i).
$$
which is from the fact about the p.d.f., of two independent Gamma random variables. \footnote{Suppose $X\sim \gammadist(a, \lambda)$ and $Y\sim \gammadist(b, \lambda)$, then $\frac{X}{X+Y} \sim \betadist(a, b)$.}

\paragraph{Aggregation property}
Suppose $[Y_1, Y_2, \ldots, Y_K]\sim \dirichlet([\alpha_1, \alpha_2, \ldots, \alpha_K])$, Then, Let $M=Y_i+Y_j$, it follows that 
$$
\begin{aligned}
&\gap [Y_1, \ldots Y_{i-1}, Y_{i+1}, \ldots, Y_{j-1}, Y_{j+1}, \ldots, Y_K, M] \\
&\sim \dirichlet([\alpha_1, \ldots, \alpha_{i-1}, \alpha_{i+1}, \ldots, \alpha_{j-1}, \alpha_{j+1}, \ldots, \alpha_K, \alpha_i+\alpha_j]).
\end{aligned}
$$
\begin{proof}
We realize that $M\sim \gammadist(\alpha_i+\alpha_j, 1)$. Again by the multidimensional transformation of variables as shown in the beginning of this section, we conclude the result.
\end{proof}
The results can be extended to a more general case. If $\{A_1, A_2, \ldots, A_r\}$ is a partition of $\{1, 2, \ldots, K\}$, then 
$$
\left[\sum_{i\in A_1} Y_i, \sum_{i\in A_2} Y_i, \ldots, \sum_{i\in A_r} Y_i\right] \sim \dirichlet\left(\left[\sum_{i\in A_1} \alpha_i, \sum_{i\in A_2} \alpha_i, \ldots, \sum_{i\in A_r} \alpha_i\right]\right).
$$

\paragraph{Condition distribution}
Let $Y_0 = \sum_{k=3}^{K}Y_i$ and $\alpha_0=\alpha_+-\alpha_1-\alpha_2$, then $[Y_1, Y_2, Y_0] \sim \dirichlet([\alpha_1, \alpha_2, \alpha_0])$. Therefore
$$
f_{Y_1,Y_2}(y_1,y_2)=
\frac{\Gamma(\alpha_1+\alpha_2+\alpha_0)}{\Gamma(\alpha_1)\Gamma(\alpha_2)\Gamma(\alpha_0)}
y_1^{\alpha_1-1}y_2^{\alpha_2-1} (1-y_1-y_2)^{\alpha_0-1}.
$$
Similarly, we have 
$$
f_{Y_2}(y_2) 
=  \frac{\Gamma(\alpha_1+\alpha_2+\alpha_0)}{\Gamma(\alpha_2)\Gamma(\alpha_1+\alpha_0)}
y_2^{\alpha_2-1} (1-y_2)^{\alpha_1+\alpha_0-1}= \betadist(y_1| \alpha_2, \alpha_1+\alpha_0),
$$
which is a p.d.f., of a Beta distribution. Therefore, the conditional p.d.f., of $Y_1|Y_2 = y_2$ is given by 
$$
f_{Y_1|Y_2 = y_2}(y_1|y_2) = \frac{f_{Y_1,Y_2}(y_1,y_2)}{f_{Y_2}(y_2)} 
=\frac{\Gamma(\alpha_1+\alpha_0)}{\Gamma(\alpha_1)\Gamma(\alpha_0)}
\left(\frac{y_1}{1-y_2}\right)^{\alpha_1-1} \left(1-\frac{y_1}{1-y_2}\right)^{\alpha_0-1} \frac{1}{1-y_2},
$$
which implies 
$$
\frac{1}{1-y_2} Y_1|Y_2=y_2 \sim \betadist(\alpha_1, \alpha_0).
$$
Apply this procedure, we will have 
$$
\bY_{-i} | Y_i\sim (1-y_i)\dirichlet(\alpha_{-i}),
$$
where $Y_{-i}$ is all the $K-1$ variables except $Y_i$, and similarly for $\alpha_{-i}$.

\section{Cholesky decomposition}
\begin{svgraybox}
	\begin{theorem}[Cholesky Decomposition]\label{theorem:cholesky-factor-exist}
		Every positive definite matrix $\bA\in \real^{n\times n}$ can be factored as 
		$$
		\bA = \bR^\top\bR,
		$$
		where $\bR$ is an upper triangular matrix with positive diagonal elements. This decomposition is known as \textbf{Cholesky decomposition}  of $\bA$. $\bR$ is known as the \textbf{Cholesky factor} or \textbf{Cholesky triangle} of $\bA$.
	\end{theorem}
\end{svgraybox}

%
%

\subsection{Existence of the Cholesky decomposition}
Before showing the existence of Cholesky decomposition, we need the following definitions and lemmas.
\begin{svgraybox}
\begin{definition}[Positive Definite and Positive Semidefinite]
	A matrix $\bA\in \real^{n\times n}$ is positive definite if $\bx^\top\bA\bx>0$ for all nonzero $\bx\in \real^n$.
	And a matrix $\bA\in \real^{n\times n}$ is positive semidefinite if $\bx^\top\bA\bx \geq 0$ for all $\bx\in \real^n$.
\end{definition}
\end{svgraybox}

\begin{svgraybox}
\begin{lemma}[Positive Diagonals of Positive Definite Matrices]\label{lemma:positive-in-pd}
	The diagonal elements of a positive definite matrix $\bA$ are all positive.
\end{lemma}	
\end{svgraybox}
\begin{proof}[of Lemma~\ref{lemma:positive-in-pd}]
	From the definition of positive definite matrix, we have $\bx^\top\bA \bx >0$ for all nonzero $\bx$. In particular, let $\bx=\be_i$ where $\be_i$ is the $i$-th unit vector with the $i$-th entry equal to 1 and other entries equal to 0. Then, 
	$$
	\be_i^\top\bA \be_i = \bA_{ii}>0, \qquad \forall i \in \{1, 2, \cdots, n\}.
	$$	
	This completes the proof.
\end{proof}

\begin{svgraybox}
\begin{lemma}[Schur Complement of Positive Definite Matrices]\label{lemma:pd-of-schur}
	For any positive definite matrix $\bA\in \real^{n\times n}$, its Schur complement of $\bA_{11}$ is $\bS_{n-1}=\bA_{2:n,2:n}-\frac{1}{\bA_{11}} \bA_{2:n,1}\bA_{2:n,1}^\top$ and it is also positive definite. 
	
	Note that the subscript $n-1$ of $\bS_{n-1}$ means it is of size $(n-1)\times (n-1)$ and it is a Schur complement of a $n\times n$ positive definite matrix. We will use this notation in the following section.
\end{lemma}	
\end{svgraybox}
\begin{proof}[of Lemma~\ref{lemma:pd-of-schur}]
	For any nonzero vector $\bv\in \real^{n-1}$, we can construct a vector $\bx\in \real^n$
	$$
	\bx = 
	\begin{bmatrix}
		-\frac{1}{\bA_{11}} \bA_{2:n,1}^\top  \bv \\
		\bv
	\end{bmatrix},
	$$
	which is nonzero. Then
	$$
	\begin{aligned}
		\bx^\top\bA\bx 
		&= [-\frac{1}{\bA_{11}} \bv^\top \bA_{2:n,1}\qquad \bv^\top]
		\begin{bmatrix}
			\bA_{11} & \bA_{2:n,1}^\top \\
			\bA_{2:n,1} & \bA_{2:n,2:n}
		\end{bmatrix}
		\begin{bmatrix}
			-\frac{1}{\bA_{11}} \bA_{2:n,1}^\top  \bv \\
			\bv
		\end{bmatrix} \\
		&= [-\frac{1}{\bA_{11}} \bv^\top \bA_{2:n,1}\qquad \bv^\top]
		\begin{bmatrix}
			0 \\
			\bS_{n-1}\bv
		\end{bmatrix} \\
		&= \bv^\top\bS_{n-1}\bv.
	\end{aligned}
	$$
	Since $\bA$ is positive definite, we have $\bx^\top\bA\bx = \bv^\top\bS_{n-1}\bv >0$ for all nonzero $\bv$. Thus, $\bS$ is positive definite.
\end{proof}

\begin{mdframed}[hidealllines=\mdframehidelineNote,backgroundcolor=\mdframecolorNote]
\textbf{A word on the Schur complement}:
It can be easily proved that this Schur complement $\bS_{n-1}=\bA_{2:n,2:n}-\frac{1}{\bA_{11}} \bA_{2:n,1}\bA_{2:n,1}^\top$ is also nonsingular if $\bA$ is nonsngular and $\bA_{11}\neq 0$. Similarly, the Schur complement of $\bA_{nn}$ in $\bA$ is $\bar{\bS}_{n-1} =\bA_{1:n-1,1:n-1} - \frac{1}{\bA_{nn}}\bA_{1:n-1,n} \bA_{1:n-1,n}^\top$ which is also positive definite if $\bA$ is positive definite. This property can help prove the leading principle minors of positive definite matrices are all positive. See Appendix~\ref{appendix:leading-minors-pd} for more details. 
\end{mdframed}

We then prove the existence of Cholesky decomposition using these lemmas.
\begin{proof}[\textbf{of Theorem~\ref{theorem:cholesky-factor-exist}: Existence of Cholesky Decomposition}]
	For any positive definite matrix $\bA$, we can write out (since $\bA_{11}$ is positive)
	$$
	\begin{aligned}
		\bA &= 
		\begin{bmatrix}
			\bA_{11} & \bA_{2:n,1}^\top \\
			\bA_{2:n,1} & \bA_{2:n,2:n}
		\end{bmatrix} \\
		&=\begin{bmatrix}
			\sqrt{\bA_{11}} &\bzero\\
			\frac{1}{\sqrt{\bA_{11}}} \bA_{2:n,1} &\bI 
		\end{bmatrix}
		\begin{bmatrix}
			\sqrt{\bA_{11}} & \frac{1}{\sqrt{\bA_{11}}}\bA_{2:n,1}^\top \\
			\bzero & \bA_{2:n,2:n}-\frac{1}{\bA_{11}} \bA_{2:n,1}\bA_{2:n,1}^\top
		\end{bmatrix}\\
		&=\begin{bmatrix}
			\sqrt{\bA_{11}} &\bzero\\
			\frac{1}{\sqrt{\bA_{11}}} \bA_{2:n,1} &\bI 
		\end{bmatrix}
		\begin{bmatrix}
			1 & \bzero \\
			\bzero & \bA_{2:n,2:n}-\frac{1}{\bA_{11}} \bA_{2:n,1}\bA_{2:n,1}^\top
		\end{bmatrix}
		\begin{bmatrix}
			\sqrt{\bA_{11}} & \frac{1}{\sqrt{\bA_{11}}}\bA_{2:n,1}^\top \\
			\bzero & \bI
		\end{bmatrix}\\
		&=\bR_1^\top
		\begin{bmatrix}
			1 & \bzero \\
			\bzero & \bS_{n-1}
		\end{bmatrix}
		\bR_1.
	\end{aligned}
	$$
	where  
	$$\bR_1 = 
	\begin{bmatrix}
		\sqrt{\bA_{11}} & \frac{1}{\sqrt{\bA_{11}}}\bA_{2:n,1}^\top \\
		\bzero & \bI
	\end{bmatrix}.
	$$
	Since we proved the Schur complement $\bS_{n-1}$ is positive definite. We can factor it in the same way
	$$
	\bS_{n-1}=
	\hat{\bR}_2^\top
	\begin{bmatrix}
		1 & \bzero \\
		\bzero & \bS_{n-2}
	\end{bmatrix}
	\hat{\bR}_2.
	$$
	We then have
	$$
	\begin{aligned}
		\bA &= \bR_1^\top
		\begin{bmatrix}
			1 & \bzero \\
			\bzero & \hat{\bR}_2^\top
			\begin{bmatrix}
				1 & \bzero \\
				\bzero & \bS_{n-2}
			\end{bmatrix}
			\hat{\bR}_2.
		\end{bmatrix}
		\bR_1\\
		&=
		\bR_1^\top
		\begin{bmatrix}
			1 &\bzero \\
			\bzero &\hat{\bR}_2^\top
		\end{bmatrix}
		\begin{bmatrix}
			1 &\bzero \\
			\bzero &\begin{bmatrix}
				1 & \bzero \\
				\bzero & \bS_{n-2}
			\end{bmatrix}
		\end{bmatrix}
		\begin{bmatrix}
			1 &\bzero \\
			\bzero &\hat{\bR}_2
		\end{bmatrix}
		\bR_1\\
		&=
		\bR_1^\top \bR_2^\top
		\begin{bmatrix}
			1 &\bzero \\
			\bzero &\begin{bmatrix}
				1 & \bzero \\
				\bzero & \bS_{n-2}
			\end{bmatrix}
		\end{bmatrix}
		\bR_2 \bR_1.
	\end{aligned}
	$$
	The same formula can be recursively applied. This process gradually continues down to the bottom-right corner giving us the decomposition
	$$
	\begin{aligned}
		\bA &= \bR_1^\top\bR_2^\top\cdots \bR_n^\top \bR_n\cdots \bR_2\bR_1\\
		&= \bR^\top \bR,
	\end{aligned}
	$$
	where $\bR_1, \bR_2, \cdots \bR_n,$ are upper triangular matrices with positive diagonal elements and $\bR=\bR_1\bR_2\cdots\bR_n$ is also an upper triangular matrix with positive diagonal elements.
\end{proof}
The process in the proof can also be used to compute the Cholesky decomposition. In next section, we use another point of view to do the computation.  
\begin{mdframed}[hidealllines=\mdframehideline,backgroundcolor=\mdframecolor]
	\begin{corollary}[$\bR^\top\bR$ is PD]\label{lemma:r-to-pd}
		For any upper triangular matrix with positive diagonal elements, then 
		$
		\bA = \bR^\top\bR
		$
		is positive definite.
	\end{corollary}
\end{mdframed}
\begin{proof}[of Corollary~\ref{lemma:r-to-pd}]
	If an upper triangular matrix $\bR$ has positive diagonal, it is full column rank, and the null space of $\bR$ is 0. As a result, $\bR\bx \neq \bzero$ for any nonzero vector $\bx$. Thus $\bx^\top\bA\bx = ||\bR\bx||^2 >0$ for any nonzero vector $\bx$.
\end{proof}
This corollary can be extended to any $\bR$ with independent columns.

\subsection{Computing the Cholesky decomposition}

To compute Cholesky decomposition, we write out the equality $\bA=\bR^\top\bR$:
$$
\begin{aligned}
	\bA=\left[
	\begin{matrix}
		\bA_{11} & \bA_{1,2:n} \\
		\bA_{2:n,1} & \bA_{2:n,2:n}
	\end{matrix}
	\right] 
	&=\left[
	\begin{matrix}
		\bR_{11} & 0 \\
		\bR_{1,2:n}^\top & \bR_{2:n,2:n}^\top
	\end{matrix}
	\right] 
	\left[
	\begin{matrix}
		\bR_{11} & \bR_{1,2:n} \\
		0 & \bR_{2:n,2:n}
	\end{matrix}
	\right]\\
	&=
	\left[
	\begin{matrix}
		\bR_{11}^2 & \bR_{11}\bR_{1,2:n} \\
		\bR_{11}\bR_{1,2:n}^\top & \bR_{1,2:n}^\top\bR_{1,2:n} + \bR_{2:n,2:n}^\top\bR_{2:n,2:n}
	\end{matrix}
	\right],  
\end{aligned}
$$
which allows to determine the first row of $\bR$
$$
\bR_{11} = \sqrt{\bA_{11}}, \qquad \bR_{1,2:n} = \frac{1}{\bR_{11}}\bA_{1,2:n}.
$$
Let $\bA_2=\bR_{2:n,2:n}^\top\bR_{2:n,2:n}$. The equality $\bA_{2:n,2:n} = \bR_{1,2:n}^\top\bR_{1,2:n} + \bR_{2:n,2:n}^\top\bR_{2:n,2:n}$ gives out
$$
\begin{aligned}
	\bA_2=\bR_{2:n,2:n}^\top\bR_{2:n,2:n} &= \bA_{2:n,2:n} - \bR_{1,2:n}^\top\bR_{1,2:n} \\
	&= \bA_{2:n,2:n} - \frac{1}{\bA_{11}} \bA_{1,2:n}^\top\bA_{1,2:n} \\
	&= \bA_{2:n,2:n} - \frac{1}{\bA_{11}} \bA_{2:n,1}\bA_{1,2:n} \qquad (\bA \mbox{ is positive definite}).
\end{aligned}
$$
$\bA_2$ is the Schur complement of $\bA_{11}$ in $\bA$ of size $(n-1)\times (n-1)$. And to get $\bR_{2:n,2:n}$ we must compute the Cholesky decomposition of matrix $\bA_2$ of $(n-1)\times (n-1)$. Again, this is a recursive algorithm and formulated in Algorithm~\ref{alg:compute-choklesky}.

\begin{algorithm}[H] 
	\caption{Cholesky Decomposition} 
	\label{alg:compute-choklesky} 
	\begin{algorithmic}[1] 
		\Require 
		positive definite matrix $\bA$ with size $n\times n$; 
		\State Calculate first of $\bR$ by $\bR_{11} = \sqrt{\bA_{11}}, \bR_{1,2:n} = \frac{1}{\bR_{11}}\bA_{1,2:n}$, ($n$ flops);
		\State Compute the Cholesky decomposition of the $(n-1)\times (n-1)$ matrix
		$$
		\bA_2=\bR_{2:n,2:n}^\top\bR_{2:n,2:n}=\bA_{2:n,2:n} - \frac{1}{\bA_{11}} \bA_{2:n,1}\bA_{1,2:n}, \qquad\text{($n^2-n$ flops).}
		$$
		
	\end{algorithmic} 
\end{algorithm}

Further, this process can be used to determine if a matrix is positive definite or not. If we try to factor a non positive definite matrix, at some point, we will encounter a nonpositive element in entry (1,1) of $\bA, \bA_2, \bA_3, \cdots$.

\begin{theorem}\label{theorem:cholesky-complexity}
	Algorithm~\ref{alg:compute-choklesky} requires $\sim(1/3)n^3$ flops to compute a Cholesky decomposition of an $n\times n$ positive definite matrix.
\end{theorem}

\begin{proof}[of Theorem~\ref{theorem:cholesky-complexity}]
	Step 1 takes 1 square root and $(n-1)$ division which takes $n$ flops totally.  
	
	For step 2, Note that $\frac{1}{\bA_{11}} \bA_{2:n,1}\bA_{1,2:n} = (\frac{1}{\sqrt{\bA_{11}}} \bA_{2:n,1})(\frac{1}{\sqrt{\bA_{11}}}\bA_{1,2:n}) = \bR_{1,2:n}^\top \bR_{1,2:n}$. If we calculate the complexity directly from the equation in step 2, we will get the same complexity as LU decomposition. But since $\bR_{1,2:n}^\top \bR_{1,2:n}$ is symmetric, the complexity of $\bR_{1,2:n}^\top \bR_{1,2:n}$ reduces from $(n-1)\times(n-1)$ multiplications to $1+2+\cdots+(n-1)=\frac{n^2-n}{2}$ multiplications. The cost of matrix division reduces from $(n-1)\times(n-1)$ to $1+2+\cdots+(n-1)=\frac{n^2-n}{2}$ as well. So it costs $n^2-n$ flops for step 2.
	
	Simple calculation will show the total complexity is $\frac{2n^3+3n^2+n}{6}$ flops which is $(1/3)n^3$ flops if we keep only the leading term.
\end{proof}

\section{Leading Principle Minors of PD Matrices}\label{appendix:leading-minors-pd}

In Lemma~\ref{lemma:pd-of-schur}, we proved for any positive definite matrix $\bA\in \real^{n\times n}$, it's Schur complement of $\bA_{11}$ is $\bS_{n-1}=\bA_{2:n,2:n}-\frac{1}{\bA_{11}} \bA_{2:n,1}\bA_{2:n,1}^\top$ and it is also positive definite. 
This is also true for its Schur complement of $\bA_{nn}$, i.e., $\bS_{n-1}^\prime = \bA_{1:n-1,1:n-1} -\frac{1}{\bA_{nn}} \bA_{1:n-1,n}\bA_{1:n-1,n}^\top$ is also positive definite.

We then claim all the leading principle minors of a positive definite matrix $\bA \in \real^{n\times n}$ are positive.

\begin{proof}
	We will prove by induction.
	Since all the components on the diagonal of positive definite matrices are all positive (see Lemma~\ref{lemma:positive-in-pd}). The case for $n=1$ is trivial that $\det(\bA_{11})> 0$.
	
	Suppose all the leading principle minors for $k\times k$ matrices are all positive. If we could prove this is also true for $(k+1)\times (k+1)$ matrices, then we complete the proof.
	
	For a $(k+1)\times (k+1)$ matrix $\bM=\begin{bmatrix}
		\bA & \bb\\
		\bb^\top & d
	\end{bmatrix}$, where $\bA$ is a $k\times k$ submatrix. Then its Schur complement of $d$, $\bS_{k} = \bA - \frac{1}{d} \bb\bb^\top$ is also positive definite and its determinant is positive from the assumption. And $\det(\bM) = \det(d)\det( \bA - \frac{1}{d} \bb\bb^\top) $=
	\footnote{By the fact that if matrix $\bM$ has a block formulation: $\bM=\begin{bmatrix}
			\bA & \bB \\
			\bC & \bD 
		\end{bmatrix}$, then $\det(\bM) = \det(\bD)\det(\bA-\bB\bD^{-1}\bC)$.}  
	$d\cdot \det( \bA - \frac{1}{d} \bb\bb^\top)>0$, which completes the proof.
\end{proof}

\section{Convexity results}\label{appendix:revisit}
We prove that $x \mapsto \Gamma(Kx)/(x^{K-1} [\Gamma(x)]^K)$ is strictly log-convex, $x \mapsto \Gamma(Kx)/[\Gamma(x)]^K$ is strictly log-concave and the function $x \mapsto \Gamma(x+t)/\Gamma(x)$ is also strictly log-concave where $\Gamma(x)$ is the Gamma function.
\begin{svgraybox}
\begin{theorem}
Define 
\begin{equation}
F(x) = \frac{\Gamma(Kx)}{x^{K-1} [\Gamma(x)]^K}, \quad G(x) = \frac{\Gamma(K x)}{[\Gamma(x)]^K}. 
\label{equation:gx_log_concave_appendix}
\end{equation}
For $x > 0$ and $K$ is an arbitrary positive integer, the function F is strictly log-convex and the function G is strictly log-concave. 
\end{theorem}\label{theorem:gx_log_concave_appendix}
\end{svgraybox}

\begin{proof}
Follow from \citep{abramowitz1966handbook} we get $\Gamma(Kx) = (2\pi)^{\frac{1}{2}(1-K)} K^{K x-\frac{1}{2} } \prod_{i=0}^{K-1}\Gamma(x + \frac{i}{K})$.  Then
\begin{equation}
\log F(x) = \frac{1}{2}(1-K) \log(2\pi) + (Kx - \frac{1}{2})\log K + \sum_{i=0}^{K-1} \log \Gamma(x + \frac{i}{K}) - K \log \Gamma(x) - (K-1)\log x
\end{equation}
and
\begin{equation}
[\log F(x)]^\prime =K \log K + \sum_{i=0}^{K-1} \Psi(x + \frac{i}{K}) - K \Psi(x) - (K-1)x^{-1},  
\end{equation}
where $\Psi(x)$ is the Digamma function, and
\begin{equation}
\Psi^\prime(x)  = \sum_{h=0}^\infty \frac{1}{(x+h)^2}. 
\label{equation:fbgmm_hyperprior_digamma_derivative_appendix}
\end{equation}
Thus 
\begin{equation}
[\log F(x)]^{\prime \prime} = \left[\sum_{i=0}^{K-1} \Psi^\prime(x + \frac{i}{K}) \right]- K \Psi^\prime(x) + \frac{K-1}{x^2}> 0, \quad (x>0). 
\end{equation}
The last inequality comes from (\ref{equation:fbgmm_hyperprior_digamma_derivative_appendix}) (also, we can find the derivative of Digamma function is monotone decreasing). 
Easily, we can get 
\begin{equation}
\log G(x) = \frac{1}{2}(1-K) \log(2\pi) + (Kx - \frac{1}{2})\log K + \sum_{i=0}^{K-1} \log \Gamma(x + \frac{i}{K}) - K \log \Gamma(x)
\end{equation}
and
\begin{equation}
[\log G(x)]^\prime =K \log K + \sum_{i=0}^{K-1} \Psi(x + \frac{i}{K}) - K \Psi(x),
\end{equation}
Thus
\begin{equation}
[\log G(x)]^{\prime \prime} = \left[\sum_{i=0}^{K-1} \Psi^\prime(x + \frac{i}{K}) \right]- K \Psi^\prime(x) < 0, \quad (x>0). 
\end{equation}
This concludes the theorem. 
\end{proof}
This theorem is a general case of Theorem 1 in \citep{merkle1997log}.

\begin{svgraybox}
\begin{theorem}
Define 
\begin{equation}
H(x) = \frac{\Gamma(x + t)}{\Gamma(x)}. 
\label{equation:gx_log_concave_2}
\end{equation}
For $x > 0$ and $t$ is a constant that $x+t >0$ (or for simplicity we can let $t>0$), the function H is strictly log-concave. 
\end{theorem}\label{theorem:gx_log_concave_2}
\end{svgraybox}

\begin{proof}
We can easily get
\begin{equation}
	\log H(x) =\log \Gamma(x+t) - \log \Gamma(x)
\end{equation}
and
\begin{equation}
	[\log H(x)]^\prime =  \Psi(x + t) - \Psi(x),
\end{equation}
where $\Psi(x)$ is the Digamma function, and
\begin{equation}
	\Psi^\prime(x)  = \sum_{h=0}^\infty \frac{1}{(x+h)^2}. 
	\label{equation:fbgmm_hyperprior_digamma_derivative_2}
\end{equation}
Thus
\begin{equation}
	[\log H(x)]^{\prime \prime} =  \Psi^\prime(x + t)-  \Psi^\prime(x) < 0, \quad (x>0, t>0). 
\end{equation}
The last inequality comes from (\ref{equation:fbgmm_hyperprior_digamma_derivative_2}) which is monotone decreasing and concludes the theorem. 
\end{proof}

\newpage
\vskip 0.2in
\bibliography{bib}

\end{document}